\documentclass[12pt]{article}
\usepackage{etoc}
\usepackage{titletoc}

\usepackage{xcolor}
\definecolor{deepyellow}{RGB}{190,150,0}
\usepackage{graphicx,epstopdf,epsfig}
\usepackage{amsmath,amsthm,verbatim,amssymb}
\usepackage[dvipsnames, svgnames, x11names]{xcolor}
\usepackage{lscape}
\usepackage{threeparttable}
\usepackage{algorithm, algorithmic}
\usepackage{float}
\usepackage{mathrsfs}
\usepackage{booktabs}
\usepackage{multirow}
\usepackage{hyperref}
\hypersetup{colorlinks=true,
	linkcolor=blue,
	anchorcolor=blue,
	citecolor=blue}
\usepackage{enumerate}
\usepackage[round]{natbib}

\makeatletter 
\usepackage{dsfont}
\date{}
\usepackage{subcaption}

\newtheorem{theorem}{Theorem}[section]
\newtheorem{lemma}{Lemma}[section]
\newtheorem{lemma*}{Lemma}
\newtheorem{remark}{Remark}[section]
\newtheorem{remark*}{Remark}
\newtheorem{corollary}{Corollary}[section]

\newtheorem{definition}{Definition}[section]
\newtheorem{assumption}{Assumption}[section]
\usepackage{bbm}

\def\Ebb{\mathbb{E}}

\def \cA {{\cal A}}
\def \cB {{\cal B}}

\def \cD {{\cal D}}

\def \cF {{\cal F}}
\def \cG {{\cal G}}

\def \cL {{\cal L}}

\def \cN {{\cal N}}
\def \cO {{\cal O}}
\def \cP {{\cal P}}
\def \cQ {{\cal Q}}
\def \cR {{\cal R}}
\def \cS {{\cal S}}
\def \cT {{\cal T}}

\def \cW {{\cal W}}
\def \cX {{\cal X}}

\def \bD {\mathbb{D}}

\def \bS {\mathbb{S}}
\def \bE {\mathbb{E}}
\def \bT {\mathbb{T}}
\def \bZ {\mathbb{Z}}
\newcommand{\norm}[1]{\left\| #1 \right\|}  
\def \TV {\mathrm{TV}}

\usepackage{booktabs}
\usepackage{makecell}

\usepackage{xcolor}
\usepackage{tcolorbox}

\usepackage[normalem]{ulem}

\begin{document}

\title{Offline Deep $Q^*$ Estimation with Diffusion Models}
\author{
Xiaohong Chen
\thanks{
Department of Economics, Yale University, New Haven, 06511, USA.
Email: xiaohong.chen@yale.edu
}
\and
Yuling Jiao
\thanks{
School of Artificial Intelligence,
 and Hubei Key Laboratory of Computational Science, Wuhan University, Wuhan, 430072,  China.
Email: yulingjiaomath@whu.edu.cn
}     
\and
Lican Kang
\thanks{Institute for Math and AI,  Hubei Key Laboratory of Computational Science, and
School of Artificial Intelligence,
Wuhan University, Wuhan, 430072, China. Email: kanglican@whu.edu.cn }
\and
Jerry Zhijian Yang
\thanks{
School of Mathematics and Statistics, and Hubei Key Laboratory of Computational Science, Wuhan University, Wuhan, 430072,  China.
Email: zjyang.math@whu.edu.cn
}
\and
Chen Zhong
\thanks{
School of Mathematics and Statistics, Wuhan University, Wuhan, 430072, China. Email: zcmath@whu.edu.cn
}
}

\maketitle

\begin{abstract}

In offline reinforcement learning (RL), estimating the optimal action-value function $Q^\ast$ can be formulated as solving the optimal Bellman equation based solely on offline observations. A fundamental challenge is that the reward function and transition kernel are unknown, so the optimal Bellman operator is not directly observable from data. To address this issue, we propose a novel framework that decouples operator estimation from value function learning. In this approach, we first formulate  conditional diffusion models to estimate the reward law and transition kernel, which induces a data-driven approximation of the optimal Bellman operator. We then plug these estimators into the Bellman equation and obtain a deep estimator of  $Q^\ast$ by minimizing the empirical Bellman residual over a neural network function class. 
This formulation separates high-dimensional conditional distribution estimation from the fixed point estimation problem and yields a coherent end-to-end statistical framework. In particular, conditional diffusion models provide a flexible and expressive approach to modeling unknown environment dynamics under limited data, addressing a central challenge in offline RL. 
Theoretically, we first establish sharp nonasymptotic convergence rates for learning the optimal Bellman operator through an end-to-end analysis of conditional diffusion estimation in total variation distance. 
We establish the smoothness of $Q^*$ through the Bellman equation using the the H\"older assumptions on the conditional reward and transition densities. With suitable sample sizes, we obtain the oracle value-stage rate $\widetilde{\mathcal O}\bigl(n^{-\frac{2\alpha}{d_x+d_a+2\alpha}}\bigr)$ for the excess Bellman residual risk. We then translate this risk bound into an expected $L^2$ estimation error of order $\widetilde{\mathcal O}\bigl(n^{-\frac{\alpha}{d_x+d_a+2\alpha}}\bigr)$ under a concentrability condition, where $d_x$ and $d_a$ denote the state and action dimensions, respectively. For the total offline budget $\widetilde n=n+n_r+n_x$, we obtain a value estimation rate of $\widetilde{\mathcal O}\bigl(\widetilde n^{-\frac{\alpha}{2d_x+d_a+2\alpha}}\bigr)$ and establish an information-theoretic lower bound of order $\widetilde n^{-\frac{\alpha}{d_x+d_a+2\alpha}}$ for estimating $Q^*$.
Importantly, our theoretical analysis does not rely on Bellman completeness assumptions commonly used in deep RL theory.
Extensive numerical experiments demonstrate the effectiveness of the proposed method and its strong empirical performance.

\vspace{0.5cm} 
\end{abstract}
\section{Introduction}
In deep reinforcement learning (DRL), value functions or policies are represented by deep neural networks (DNNs), which makes RL applicable to high-dimensional control and planning problems  \citep{mousavi2016deep,li2023deep}. 
The empirical success of DRL in Atari games \citep{mnih2015human} and Go \citep{silver2016mastering} demonstrates that deep function approximation can support effective decision-making in complex environments. 
At the same time, many practical problems do not permit repeated online interaction with the environment, because data collection may be costly, unsafe, or otherwise infeasible. 
In such settings, the learner must rely on previously collected data, often generated by an unknown behavior policy 
\citep{levine2020offline,fujimoto2019off,prudencio2023survey}. 
This offline regime fundamentally changes the nature of the problem: the primary challenge is no longer exploration, but the reliable estimation of value functions or policies from limited and potentially biased data. 

From the value function modeling perspective, a central object in offline RL is to estimate the optimal action-value function $Q^\ast$, defined as the fixed point of the optimal Bellman equation. 
Classical fitted-value methods, including fitted $Q$-iteration (FQI) \citep{ernst2005tree} and its deep extensions, such as deep $Q$-networks (DQN) \citep{mnih2015human} and deep fitted $Q$-iteration (DFQI) \citep{fan2020theoretical}, approximate $Q^\ast$ through repeated Bellman regression with DNNs. 
These methods directly connect sampled transitions with value estimation, but recursive bootstrapping can propagate approximation and optimization errors, and max-based targets in deep $Q$-learning may induce overestimation bias \citep{van2016deep,fan2020theoretical}.
Deep approximate policy iteration (DAPI) \citep{jiao2025deep} follows a related
route based on the Bellman equation, where each policy evaluation step solves a
Bellman equation over a DNN class. These iterative methods provide a natural
approach to value learning, but repeated fitted regression or policy evaluation
steps can be computationally demanding when many trajectories or repeated
Bellman solves are required.
Another line focuses on direct value estimation. Methods such as Minimax Squared Bellman Optimality Error Minimization (MSBO) and Minimax Average Bellman Optimality Error Minimization (MABO) estimate $Q^\ast$ through adversarial Bellman-error formulations \citep{xie2020q}.
These minimax formulations provide a direct statistical criterion for value estimation and can be applied to batch data, but they usually require solving coupled optimization problems over multiple function classes, which can be computationally demanding.
Theoretical analyses of these methods typically rely on concentrability coefficients and Bellman completeness assumptions, which respectively control distribution mismatch and the representability of Bellman updates within a prescribed function class 
\citep{chen2019information,fan2020theoretical,feng2023over,jiao2025deep}. 
While these assumptions have enabled important non-asymptotic guarantees, they may become restrictive when flexible DNN approximators are used in offline RL, particularly in 
high-dimensional settings where Bellman updates may no longer be accurately represented within a fixed function class.

A complementary direction is model-based RL and its offline variants, including Model-Based Policy Optimization (MBPO) \citep{janner2019trust}, Model-based Offline Policy Optimization (MOPO) \citep{yu2020mopo}, MOReL: Model-Based Offline RL \citep{kidambi2020morel}, and Conservative Offline Model-Based Policy Optimization (COMBO) \citep{yu2021combo}.
Model-based methods make the environment mechanism explicit by estimating rewards or transitions, and therefore provide a natural alternative to directly fitting value functions. However, existing model-based offline RL methods often use the learned model for rollout generation, uncertainty penalization, or conservative policy optimization. Their performance therefore depends heavily on the quality and stability of the learned environment model, especially in 
high-dimensional and continuous settings.
Taken together, these approaches highlight a fundamental tradeoff: 
{\it 
direct value-learning methods remain close to the Bellman equation but rely on recursive or coupled optimization, whereas model-based methods make the environment mechanism explicit but depend on the quality of reward and transition estimation. 
}

This tradeoff suggests a statistical perspective that separates environment modeling from value estimation. Since the Bellman equation is fully determined by the reward law and transition kernel, it is natural to first estimate these conditional objects and then use them to construct an estimator of $Q^\ast$. Related ideas appear in inverse RL (IRL), where demonstrations are used to recover reward information \citep{ziebart2008maximum,wulfmeier2015maximum,arora2021survey}, and in conditional moment and instrumental variable models, where nuisance components are estimated and then incorporated into a target equation \citep{dikkala2020minimax,xu2020learning,kim2025optimality}. From this viewpoint, offline RL can be formulated as a statistical problem in which the propagation of estimation error from environment modeling to value estimation becomes an explicit object of analysis.  
The main difficulty in this formulation lies in estimating high-dimensional conditional distributions for rewards and transitions. Diffusion models \citep{ho2020denoising,song2020score} provide a flexible generative framework 
for learning complex conditional distributions, 
and recent theoretical work establishes strong statistical guarantees for diffusion-based estimators \citep{oko2023diffusion,chen2023score,jiao2024latent,chang2026deep}. 
Among these results, \citet{tang2025conditional} established minimax optimal rates in total variation distance for conditional distribution estimation and proved manifold adaptivity under the Wasserstein metric. Our analysis further quantifies how estimation errors in the conditional reward and transition laws propagate through the optimal Bellman operator to the estimator of $Q^\ast$.
In particular, diffusion models are well suited for complex and multimodal distributions, which makes them attractive for modeling conditional reward laws and transition kernels in continuous offline RL.
Recently, diffusion models have also become an active direction in RL, with applications to conditional decision making and trajectory generation \citep{ajay2023conditional}, offline policy learning and behavior regularization \citep{wang2022diffusion,chen2023score,mao2024diffusion}, value estimation \citep{mazoure2023value}, and policy optimization in online or continuous-control settings \citep{ren2024diffusion,ma2025efficient,frans2025diffusion}.

Motivated by these developments, we propose a framework for offline $Q^\ast$ estimation that integrates conditional diffusion modeling with value function learning. We use conditional diffusion models to estimate the reward law and transition kernel, which define a data-driven approximation of the optimal Bellman operator. These estimates are then incorporated into the Bellman equation, and $\widehat{Q}$ is obtained via empirical Bellman residual minimization over a neural network function class. This construction yields an end-to-end statistical framework in which the interaction between generative estimation and value learning can be rigorously analyzed.
Our theoretical results quantify the contributions of conditional distribution estimation and value function learning to the overall estimation error. The analysis begins with the conditional diffusion estimators and traces their errors through the estimated Bellman operator to the Bellman residual risk. The H\"older regularity of the conditional reward and transition densities also yields the smoothness of $Q^*$, which determines the approximation accuracy at the value-learning stage. With suitable sample sizes, we obtain an excess Bellman residual risk bound of order $\widetilde{\mathcal O}(n^{-\frac{2\alpha}{d_x+d_a+2\alpha}})$ and, under a concentrability condition, an expected $L^2$ value estimation error of order $\widetilde{\mathcal O}(n^{-\frac{\alpha}{d_x+d_a+2\alpha}})$. Here, $n$ denotes the sample size for value function learning, $d_x$ and $d_a$ are the state and action dimensions, and $\alpha$ is the H\"older smoothness index of the conditional densities. We further analyze the dependence on the total offline sample budget $\widetilde n=n+n_r+n_x$, obtaining the value estimation rate $\widetilde{\mathcal O}(\widetilde n^{-\frac{\alpha}{2d_x+d_a+2\alpha}})$. An information-theoretic lower bound of order $\widetilde n^{-\frac{\alpha}{d_x+d_a+2\alpha}}$ complements these upper bounds. Throughout the analysis, we control value function approximation directly through the smoothness of $Q^*$ without imposing Bellman completeness.
Notably, our analysis avoids the Bellman completeness assumption commonly imposed
in theoretical analyses of DRL
\citep{chen2019information,fan2020theoretical,feng2023over,jiao2025deep}; see
Section  \ref{sec:theory} for details.
The key point is that our framework for Bellman operator learning and loss decomposition reduces the problem to approximating $Q^*$ of the optimal Bellman operator,
rather than requiring the Bellman updates of an
entire function class to be representable  within the same class.
Empirically, we conduct numerical experiments on OpenAI Gym tasks, including Pendulum and Swimmer  (\url{https://gymnasium.farama.org/}) and on the Medical Information Mart for Intensive Care III (MIMIC-III) clinical dataset (\url{https://physionet.org/content/mimiciii/1.4/}). We compare our method with MOPO, DFQI, and MABO in terms of both training rewards and evaluation rewards. The results support the benefit of separating environment-law estimation from value estimation in offline RL.

\subsection{Contribution}

Our main contributions are summarized as follows:
\begin{itemize}
    \item We propose a framework for offline RL that integrates conditional diffusion modeling with value function estimation. 
    Conditional diffusion models are used to estimate the reward law and transition kernel, which induce a data-driven approximation of the optimal Bellman operator. 
    These estimates are incorporated into a $Q^\ast$ estimation procedure via empirical Bellman residual minimization over a DNN class. Extensive numerical experiments demonstrate the effectiveness of the proposed method.

\item We establish theoretical guarantees linking conditional diffusion estimation to value function learning. We derive convergence rates for the conditional diffusion estimators in total variation distance and characterize how their errors propagate through the Bellman operator. By deriving the smoothness of $Q^*$ from the H\"older regularity of the conditional densities, we obtain oracle value-stage rates without imposing a separate smoothness assumption on $Q^*$ or requiring Bellman completeness. We further establish convergence rates under a fixed total offline sample budget and an information-theoretic lower bound for $Q^*$ estimation. The main guarantees are summarized in Theorem~\ref{th:inf}.
\end{itemize}

\begin{theorem}[Informal]\label{th:inf}
Let $\widehat Q$ be the  proposed estimator  of $Q^*$.
Under some 
regular conditions,
the Bellman residual  satisfies 
\[
\mathbb{E} \Big[ \cL_{\mathcal T^\ast}(\widehat Q)-\cL_{\mathcal T^\ast}(Q^\ast) \Big] \lesssim
\underbrace{\varepsilon_{\mathrm{opt}}}_{\text{Optimal Bellman operator learning error}}
+
\underbrace{\varepsilon_{\mathrm{value}}}_{\text{Value function learning  error}}.
\]
In particular, the operator learning error $\varepsilon_{\mathrm{opt}}$ consists of the reward law estimation error $\widetilde{\mathcal O}\big(n_r^{-\frac{\alpha}{1+d_x+d_a+2\alpha}}\big)$,
the transition kernel estimation error $\widetilde{\mathcal O}(n_x^{-\frac{\alpha}{2d_x+d_a+2\alpha}})$, and the Monte Carlo error
$\widetilde{\mathcal O}\bigl(  \frac{\widetilde n^{\frac{d_x+d_a}{d_x+d_a+2\alpha}}}{m}\bigr)$ induced by approximating the learned Bellman operator.
Meanwhile, the value function estimation error satisfies $\varepsilon_{\mathrm{value}}=\widetilde{\mathcal O}\big(n^{-\frac{2\alpha}{d_x+d_a+2\alpha}}\big)$. 
With suitable choices of $n_r$, $n_x$, and $m$, the excess Bellman residual risk attains the rate $\widetilde{\mathcal O}\bigl(n^{-\frac{2\alpha}{d_x+d_a+2\alpha}}\bigr)$. For the total offline budget $\widetilde n=n+n_r+n_x$, we obtain the corresponding residual risk rate $\widetilde{\mathcal O}\bigl(\widetilde n^{-\frac{2\alpha}{2d_x+d_a+2\alpha}}\bigr)$ with suitable sample allocation. Furthermore, under the stated concentrability conditions, we have
\[
\mathbb E\|\widehat Q-Q^*\|_{L^2(\rho)}=\widetilde{\mathcal O}\bigl(n^{-\frac{\alpha}{d_x+d_a+2\alpha}}\bigr),
\]
and the corresponding value estimation rate under the total offline budget is $\widetilde{\mathcal O}\bigl(\widetilde n^{-\frac{\alpha}{2d_x+d_a+2\alpha}}\bigr)$. We also establish an information-theoretic lower bound of order $\widetilde n^{-\frac{\alpha}{d_x+d_a+2\alpha}}$ for estimating $Q^*$.
\end{theorem}

\subsection{Related Work}
We review the literature on DRL and diffusion-based methods in RL.

\paragraph{Deep RL.}
DRL uses DNNs to represent value functions or policies, and has been widely studied for high dimensional control and planning problems \citep{mousavi2016deep,li2023deep}. A basic route to value function modeling is to approximate the optimal action-value function $Q^\ast$ through Bellman regression. FQI casts Bellman updates as supervised regression problems \citep{ernst2005tree}, and DFQI extends this procedure to a class of DNNs \citep{fan2020theoretical,kang2023error}. 
DQN combines the 
fitted-value idea with deep representation learning, experience replay, and target networks, achieving strong performance on Atari games \citep{mnih2015human,bellemare2013arcade,hester2018deep,fan2020theoretical}. 
DAPI follows a related iterative framework in which each iteration solves a Bellman equation over a DNN class to obtain an updated value function estimator \citep{jiao2025deep}.
However, these methods rely on recursive fitted targets, so approximation and optimization errors may accumulate across iterations. In addition, max-based updates may cause overestimation bias, which Double DQN reduces by separating action selection from value evaluation \citep{van2016deep}. Despite these improvements, DQN-type methods are mainly suited to discrete action spaces.
Existing non-asymptotic theory for deep value learning has established convergence guarantees under concentrability coefficients and Bellman completeness assumptions 
\citep{chen2019information,fan2020theoretical,feng2023over,jiao2025deep}. 
These assumptions control distribution mismatch and the propagation of recursive Bellman approximation errors. 
In contrast, our framework separates estimation of the reward and
transition laws from value function learning.

The limitation of DQN-type methods in continuous action spaces motivates policy-gradient and actor-critic (AC) methods. In AC methods, the actor updates the policy, while the critic estimates value functions through Bellman-type equations \citep{konda1999actor,grondman2012survey}. Proximal Policy Optimization (PPO) stabilizes policy-gradient updates using a clipped surrogate objective \citep{schulman2017proximal}. Deep Deterministic Policy Gradient (DDPG) and Soft Actor-Critic (SAC) extend AC learning to continuous-control tasks \citep{haarnoja2018soft}. Although these methods improve policy-learning flexibility, their critics still face the challenge of controlling fitted value function errors. 
The error propagation of fitted Bellman targets motivates another line of work that seeks more direct criteria for estimating $Q^\ast$. MSBO and MABO formulate this task through minimax Bellman optimality error objectives, replacing simple recursive regression with adversarial Bellman error minimization \citep{xie2020q}. 
Minimax Bellman error ideas have also been extended to heterogeneous offline RL;
\citet{miao2025reinforcement} estimate individual \(Q\) functions using minimax
Bellman error criteria and combine them with pessimistic personalized policy
learning.
These formulations provide a clearer statistical criterion for value estimation. However, they typically require solving coupled optimization problems over multiple function classes, which can be computationally demanding and sensitive to optimization design \citep{uehara2020minimax,xie2020q}.

A complementary response is to model the environment mechanism before computing values or policies. MBPO learns an ensemble dynamics model and uses short rollouts generated by the model to augment real data for policy optimization, reducing the accumulation of model error over long horizons \citep{janner2019trust}. In offline RL, MOPO introduces rewards penalized by uncertainty to discourage rollouts through unreliable model regions \citep{yu2020mopo}. MOReL constructs a pessimistic MDP by detecting uncertain state action pairs and assigning conservative values outside the reliable region of the model \citep{kidambi2020morel}. COMBO combines data generated by the model with conservative value learning to use synthetic rollouts while controlling overestimation on poorly modeled samples \citep{yu2021combo}. These methods show that estimating rewards and transitions is a useful alternative to direct value regression. Our work follows this environment modeling perspective, but focuses on estimating the Bellman operator from reward laws and transition kernels and analyzing the resulting excess Bellman loss, rather than on policy optimization with learned rollouts.

\paragraph{Diffusion RL.}
Diffusion models have recently been used in RL as flexible generative tools across several settings. They have been applied to conditional decision making and trajectory generation \citep{ajay2023conditional,zhu2023diffusion}, policy representation and behavior regularization in offline RL \citep{wang2022diffusion,lu2023contrastive,chen2022offline,chen2024srpo,mao2024diffusion}, value estimation \citep{mazoure2023value}, and policy improvement or online optimization for continuous control \citep{ren2024diffusion,fang2024diffusion,wang2024diffusion,ma2025efficient,frans2025diffusion}.

For offline decision making, Decision Diffuser treats decision making as conditional generation of state trajectories. It conditions the generated trajectories on returns, constraints, or skills and recovers actions using an inverse dynamics model \citep{ajay2023conditional}. Related diffusion methods use conditional generation or guidance to model complex behavior, select actions with high value, or improve offline policy learning \citep{zhu2023diffusion,lu2023contrastive,chen2022offline,chen2024srpo,mao2024diffusion}. Another group of methods uses diffusion models to represent policies or behavior distributions. Diffusion $Q$-learning represents the policy by a conditional diffusion model and combines behavior cloning with maximization of the $Q$ function \citep{wang2022diffusion}. Contrastive energy prediction learns a guidance function for diffusion sampling and applies it to offline RL \citep{lu2023contrastive}. Selecting from behavior candidates uses a generative behavior model together with action evaluation to avoid actions outside the support of the offline data. Score regularized policy optimization uses the score function of a pretrained diffusion behavior model to regularize policy optimization \citep{chen2022offline,chen2024srpo}. Diffusion DICE uses diffusion guidance within the support of the behavior distribution to move this distribution toward the optimal policy distribution \citep{mao2024diffusion}. These methods use diffusion primarily for action or trajectory generation within policy learning.

Another line combines diffusion policies with RL optimization and policy improvement. Diffusion Policy Policy Optimization fine tunes diffusion policies through policy gradient updates for continuous control and robot learning \citep{ren2024diffusion}. Diffusion AC methods use diffusion policies in actor critic learning. They formulate constrained policy iteration as a diffusion noise regression problem for offline RL \citep{fang2024diffusion}, or introduce an entropy regulator for online RL with maximum entropy \citep{wang2024diffusion}. Soft Diffusion AC uses the state action value function as an energy function to train diffusion policies \citep{ma2025efficient}. Diffusion guidance can also define a controllable policy improvement operator without requiring an explicit value function \citep{frans2025diffusion}. These methods focus on policy learning rather than on estimating the environment laws that define the Bellman operator.

Other adjacent studies address different statistical targets or decision
problems. Diffuser uses diffusion as a trajectory planner rather than to
estimate the environment laws for a single transition
\citep{janner2022planning}.
Diffusion models have also been used to quantify uncertainty in trajectory
planning, construct visual world models, and generate trajectories adapted
to a target policy
\citep{sun2023conformal,alonso2024diffusionworld,fang2024learning}.
They have further been used to estimate action values from sampled future
state sequences and to represent successor state measures through the
Bellman flow equation
\citep{mazoure2023value,schramm2024bellman}.
Beyond diffusion models, related studies estimate transition distributions
with contrastive energy models and analyze offline RL based on learned
models under partial coverage or in tabular MDPs
\citep{chen2024transition,uehara2022pessimistic,li2024settling}.

The work most closely related to ours is the diffusion world model developed by \citet{ding2024diffusionworld}. Both this model and ours use conditional diffusion models to generate future states and rewards for value function learning. The diffusion world model jointly generates multistep future states and rewards and uses them for value estimation and offline policy learning. 
We instead estimate the one-step conditional reward law and transition kernel and plug these estimates into the optimal Bellman operator for $Q^\ast$ estimation.

\subsection{Outlines}

The rest of this paper is organized as follows. Section \ref{sec:preliminaries} introduces preliminaries. 
Section \ref{sec:method} presents the proposed method. Section \ref{sec:theory} establishes the theoretical results.
Section \ref{sec:experiment} reports numerical experiments. 
Section \ref{sec:conclusion} concludes the paper. Proofs of all theoretical results are deferred to the Appendix.

\section{Preliminaries}\label{sec:preliminaries}
\subsection{Markov Decision Process}
A discounted Markov decision process (MDP) \citep{sutton1998reinforcement,agarwal2019reinforcement} consists of a quintuple $(\cX,\cA,\cP,\cR,\gamma)$, where $\cX \subset \mathbb{R}^{d_x} $ denotes the state space, $\cA \subset \mathbb{R}^{d_a}$ denotes the action space, $\cP:\cX\times\cA\to\mathcal{M}(\cX)$ denotes the transition kernel, $\cR(\cdot| {x,a})$ denotes the distribution of the immediate reward $R({x,a}) \in \mathbb{R}$, and $\gamma\in[0,1)$ denotes the discount factor. Here, $\mathcal{M}(\cX)$ denotes the set of probability measures on $(\cX,\cB(\cX))$. For each state--action pair $({x,a})\in\cX\times\cA$, 
the measure $\cP(\cdot| x,a)$ gives the distribution of the next state after taking action $a$ at state $x$. In addition, for every measurable set $D\in\cB(\cX)$, the map $(x,a)\mapsto \cP(D| x,a)$ is measurable. The setting in this paper involves continuous state-action spaces, and, without loss of generality, assumes that $\cX\times\cA=[0,1]^{d_x}\times [0,1]^{d_a}$. Let $\pi(\cdot|x)$ denote a stochastic policy, namely, a probability distribution over actions given the current state $x$. Given an initial distribution $\nu\in\mathcal{M}(\cX)$ with $X_1\sim \nu$, the offline dataset $\{(X_i,A_i,R_i,X_i^\prime)\}_{i=1}^n,$ with $X_i^\prime=X_{i+1},$
satisfies
\[
X_1\sim \nu,~
A_i\sim \pi(\cdot| X_i),~
R_i\sim \cR(\cdot| X_i,A_i),~
X_i^\prime\sim \cP(\cdot| X_i,A_i),~ i=1,\ldots,n.
\]
For theoretical analysis, the sample data are assumed to be independent and identically distributed (i.i.d.), and we let $\mu^\pi\in\mathcal{M}(\cX\times\cA)$ denote the joint distribution of the state-action pair $(X,A)$ under policy $\pi$. For a given policy $\pi$, the action-value function is defined by
\[
Q^\pi(x,a)
:=
\Ebb \left[\sum_{t=0}^{\infty}\gamma^t R_{t+1}|  X_1=x,A_1=a,\pi\right].
\]
In the general case, we assume that there exists a positive constant $R_{\max}$ such that $R(x,a) \in [0,R_{\max}]$ for each pair $(x,a) \in \cX \times \cA$, then $Q^{\pi}$ take values in $\left[0, \frac{R_{\max}}{1-\gamma}\right]$. 
Suppose there exists a policy $\pi^*$ that maximizes $Q^\pi$ such that $Q^* := Q^{\pi^*}.$ The optimal action-value function $Q^*$ satisfies the optimal Bellman equation $Q^* = \mathcal{T}^* Q^*$,  
where the optimal Bellman operator $\mathcal{T}^*$ is given by
\begin{equation*}
\mathcal T^*Q(x,a)=\mathbb E_{R\sim\mathcal R(\cdot\mid x,a)}[R]+\gamma \mathbb E_{X'\sim P(\cdot\mid x,a)}\left[\max_{a'\in\mathcal A}Q(X',a')\right].
\end{equation*}
In this work, we aim to estimate $Q^*$ from offline data without further interaction with the environment. Let $\pi_b$ denote the behavior policy and $\pi^*$ an optimal policy, and let $\mu^{\pi_b}$ and $\mu^{\pi^*}$ denote their respective joint distributions of the state--action pair $(X,A)$. 
We assume that $\mu^{\pi_b}$ and $\mu^{\pi^*}$ have the same support and restrict our analysis to this common support throughout the paper. Without loss of generality, we take this support to be $\mathcal X\times\mathcal A=[0,1]^{d_x}\times [0,1]^{d_a}$.
The available data consist of three independent batches collected under $\pi_b$: $\mathbb S_r:=\{(X_i,A_i,R_i,X_i^\prime)\}_{i=1}^{n_r}$, $\mathbb S_x:=\{(X_i,A_i,R_i,X_i^\prime)\}_{i=1}^{n_x}$, and $\bD:=\{(X_i,A_i,R_i,X_i^\prime)\}_{i=1}^{n}$, where $n_r$, $n_x$, and $n$ are the respective sample sizes. Throughout the paper, all observations are assumed to be independent and identically distributed. We use $\mathbb S_r$ and $\mathbb S_x$ to estimate the optimal Bellman operator and $\bD$ to learn the value function, as detailed in Section~\ref{sec:method}.

\subsection{Diffusion Models}\label{subsec:cdm}
We briefly review diffusion models 
\citep{song2020score,yang2023diffusion,chen2024overview}
and their role in conditional generative modeling. Mathematically, diffusion models are described by a pair of stochastic differential equations (SDEs): a forward diffusion process that gradually perturbs data into noise, and a reverse-time generative process that transforms noise back into samples from the target distribution. The key ingredient in this construction is the score function, which is typically learned through score matching, while sample generation is carried out by numerically discretizing the reverse SDE.

\paragraph{Forward Process.}
The forward process starts from an initial data sample $\bar{Y}_0\sim \zeta$ and gradually corrupts it by injecting Gaussian noise. Its dynamics are governed by the SDE
\begin{equation}\label{eq:Forward-process}
d\bar{Y}_t=f(\bar{Y}_t,t) dt+g(t) dB_t,\qquad \bar{Y}_0\sim \zeta,
\end{equation}
where $\bar{Y}_t\in\mathbb{R}^d$ denotes the data state at time $t\in[0,T]$, $B_t$ is a standard Brownian motion, and $\zeta$ is the target data distribution. 
The drift term $f:[0,T]\times\mathbb{R}^d\to\mathbb{R}^d$ characterizes the deterministic evolution, while the diffusion coefficient $g:[0,T]\to\mathbb{R}$ controls the stochastic noise injection. As time evolves, the forward process progressively smooths the original data distribution through Gaussian perturbations, eventually mapping $\zeta$ to a tractable prior distribution, typically a standard Gaussian.

\paragraph{Backward Process.}
By the time-reversal theory of diffusion processes \citep{anderson1982reverse,haussmann1986time}, the forward SDE in \eqref{eq:Forward-process} admits a reverse-time representation involving the score function $\nabla\log \zeta_t$, where $\zeta_t$ denotes the marginal distribution of $\bar{Y}_t$. The corresponding reverse-time SDE is given by
\begin{equation*}
d\bar{Y}_t^R =
\Big[f(\bar{Y}_t^R,t)-g^2(t)\nabla \log \zeta_t(\bar{Y}_t^R)\Big]dt
+ g(t) d\widetilde{B}_t,
\end{equation*}
where $\widetilde{B}_t$ denotes a standard Brownian motion in reverse time. For numerical implementation, it is convenient to rewrite the reverse dynamics as a forward-time process. Applying the time transformation $t\mapsto T-t$ with a sufficiently large terminal time $T$, we obtain the  SDE on $[0,T]$:
\begin{equation}\label{eq:re-backward-process}
dY_t
=
\Big[-f(Y_t,T-t)+g^2(T-t)\nabla \log \zeta_{T-t}(Y_t)\Big]dt
+g(T-t) dB_t,~~ Y_0\sim \zeta_T.
\end{equation}
Here, $Y_t=\bar{Y}_{T-t}$. Starting from the prior distribution $\zeta_T$ at time $t=0$, the process evolves toward $Y_T$, which approximates a sample from the target distribution $\zeta$. Therefore, the main task is to learn the time-dependent score function $\nabla \log \zeta_{T-t}(Y_t)$ from data.

Conditional diffusion models generalize this framework by replacing the target distribution $\zeta$ with a conditional distribution $\zeta(\cdot| X)$ depending on a covariate $X$. 
In this case, both the forward and reverse diffusion processes become conditional on $X$, and the corresponding score function changes to
$
\nabla \log \zeta_t(\cdot| X).
$
Consequently, the drift term in the reverse SDE also depends on the conditioning variable $X$. 
Since the conditional score function is unknown in practice, it must be estimated from data. 
A standard approach is score matching \citep{hyvarinen2005estimation,vincent2011connection}, where the score function is parameterized by a DNN and learned by minimizing a denoising objective.
Once the score estimator is obtained, samples from the target conditional distribution can be generated by numerically solving the reverse SDE \eqref{eq:re-backward-process}. 
In practice, this is typically implemented through stochastic discretization schemes such as the Euler--Maruyama (EM)   method, which yields the final sampling procedure for the conditional diffusion model.


\subsection{Deep Neural Networks} \label{section:ReLU-DNNs}
We now introduce DNNs equipped with Rectified Linear Unit (ReLU) activation functions. A ReLU DNN is given by
\[
h_{\boldsymbol{\theta}}(x)=\mathcal{A}_{\mathcal{D}}\circ\sigma\circ\mathcal{A}_{\mathcal{D}-1}\circ\sigma\circ\cdots\circ\sigma\circ\mathcal{A}_1\circ\sigma\circ\mathcal{A}_0(x)\in\mathbb{R}^q,\qquad x\in\mathbb{R}^d,
\]
where $\sigma(x)=\max\{0,x\}$ denotes the ReLU activation function applied element-wise, $\boldsymbol{\theta}$ denotes the entire set of network parameters, and $\mathcal{D}$ denotes the depth of the network. For each $i=0,1,\ldots,\mathcal{D}$, the affine transformation $\mathcal{A}_i(\cdot)$ is defined as
\[
\mathcal{A}_i(x)=\mathbf{M}_i x+\boldsymbol{b}_i,\qquad \mathbf{M}_i\in\mathbb{R}^{k_{i+1}\times k_i},\qquad \boldsymbol{b}_i\in\mathbb{R}^{k_{i+1}},
\]
where $(k_0,k_1,\ldots,k_{\mathcal{D}},k_{\mathcal{D}+1})$ denotes the sequence of layer dimensions. Specifically, $k_0=d$ corresponds to the input dimension, and $k_{\mathcal{D}+1}=q$ corresponds to the output dimension. We write $h_{\boldsymbol{\theta}}=(h_{\boldsymbol{\theta},1},\ldots,h_{\boldsymbol{\theta},q})^\top$, where $h_{\boldsymbol{\theta},j}$ denotes the $j$-th component of $h_{\boldsymbol{\theta}}$. Consequently, the network consists of $\mathcal{D}$ hidden layers and a total of $\mathcal{D}+1$ layers. Accordingly, the parameter set of the DNN $h_{\boldsymbol{\theta}}(\cdot)$can be explicitly written as
\[
\boldsymbol{\theta}:=\big((\mathbf{M}_0,\boldsymbol{b}_0),(\mathbf{M}_1,\boldsymbol{b}_1),\ldots,(\mathbf{M}_{\mathcal{D}},\boldsymbol{b}_{\mathcal{D}})\big).
\]
We denote the number of non-zero elements in $\boldsymbol{\theta}$ as
\[
\|\boldsymbol{\theta}\|_0=\sum_{i=0}^{\mathcal{D}}\Big(\|\operatorname{vec}(\mathbf{M}_i)\|_0+\|\boldsymbol{b}_i\|_0\Big),
\]
and the largest absolute value among all parameters in $\boldsymbol{\theta}$ as
\[
\|\boldsymbol{\theta}\|_\infty=\max\left\{\max_{i\in\{0,\ldots,\mathcal{D}\}}\|\operatorname{vec}(\mathbf{M}_i)\|_\infty,\max_{i\in\{0,\ldots,\mathcal{D}\}}\|\boldsymbol{b}_i\|_\infty\right\},
\]
where $\operatorname{vec}(\mathbf{M})$ denotes the column-wise vectorization of the matrix $\mathbf{M}$. For convenience, we define the width $\mathcal{W}$, the size $\mathcal{S}$, the weight bound $\mathcal{B}$, and the uniform function bound $C_{\text{NN}}$ of the network as
\[
\mathcal{W}=\max\{k_1,\ldots,k_{\mathcal{D}}\},~~\mathcal{S}=\|\boldsymbol{\theta}\|_0,~~\|\boldsymbol{\theta}\|_\infty\leq\mathcal{B},~~
\max_{j\in[q]}\|h_{\boldsymbol{\theta},j}\|_\infty\leq C_{\text{NN}},
\]
respectively, where $[q]:=\{1,\ldots,q\}$ and $\|h_{\boldsymbol{\theta},j}\|_\infty:=\sup\limits_{x\in[0,1]^d}|h_{\boldsymbol{\theta},j}(x)|$. 

\begin{definition}[H{\"o}lder Class]\label{def:holder}
For $s>0$ with $s=r+\tau$, where $r\in\mathbb{N}_0$, $\tau\in(0,1]$, and $d\in\mathbb{N}$, define the H{\"o}lder class $\mathcal{H}^{s}(\mathbb{R}^d,\Sigma)$ by
\[
\mathcal{H}^{s}(\mathbb{R}^d,\Sigma):=\left\{g:\mathbb{R}^d\to\mathbb{R}:\max_{\|\xi\|_1\le r}\|\partial^\xi g\|_\infty\le\Sigma,\max_{\|\xi\|_1=r}\sup_{x\ne y}\frac{|\partial^\xi g(x)-\partial^\xi g(y)|}{\|x-y\|^\tau}\le\Sigma\right\}.
\]
For the domain $[0,1]^d\subseteq\mathbb{R}^d$, write
\[
\mathcal{H}^{s}:=\left\{g:[0,1]^d\to\mathbb{R}:g\in\mathcal{H}^{s}(\mathbb{R}^d,\Sigma)\right\}.
\]
\end{definition}

\begin{definition}[Total Variation Distance]
	Let $\nu_1$ and $\nu_2$ be two probability distributions over a sample space $\Omega \subseteq \mathbb{R}^d$. The total variation  distance between $\nu_1$ and $\nu_2$ is defined by
\[
\TV(\nu_1,\nu_2)=\sup_{D\in\mathcal{B}(\Omega)}\left|\nu_1(D)-\nu_2(D)\right|.
\]
	For a function $g: \Omega \to \mathbb{R}$, let $\| g \|_{\infty, \Omega} := \sup\limits_{x \in \Omega} |g(x)|$. Then, an equivalent characterization of the $\TV$ distance is
	\[
	\TV(\nu_1, \nu_2) = \frac{1}{2}\sup_{\substack{\| g \|_{\infty, \Omega} \leq 1}} \left| \int_{\Omega} g(x) \nu_1(dx) - \int_{\Omega} g(x) \nu_2(dx) \right|.
	\]
\end{definition}

\begin{definition}[Covering Number]
Let $\mathcal{G}$ be a set endowed with a semi-metric $\tilde{\rho}$. For any $\varepsilon > 0$, the $\varepsilon$-covering number $\mathcal{N}(\mathcal{G}, \varepsilon, \tilde{\rho})$ of $\mathcal{G}$ with respect to $\tilde{\rho}$ is defined as
\[
\mathcal{N}(\mathcal{G}, \varepsilon, \tilde{\rho})
= \inf\biggl\{k\in\mathbb{N}  \bigg| 
\exists g_1,g_2,\dots,g_k\in\mathcal{G},\text{ such that }
\forall \tilde{g}\in\mathcal{G},\min_{1\le j\le k}\tilde{\rho}(\tilde{g},g_j)\le \varepsilon\biggr\}.
\]
\end{definition}

\begin{definition}[Admissible distributions \citep{chen2019information}]
A distribution \(\rho\in\mathcal M(\cX\times\cA)\) is called admissible if there exist \(h\ge 1\) and a possibly nonstationary randomized policy $\pi$
such that \(\rho(B)=\mathbb P \big((X_h,A_h)\in B| X_1\sim\nu,\pi \big)\) for every \(B\in\mathcal B(\cX\times\cA)\), where \(A_h\sim\pi(\cdot| X_h)\) and \(X_{h+1}\sim\cP(\cdot| X_h,A_h)\). 
Fix in advance a nonempty restricted class $\Pi_{\mathrm{tar}}$ of possibly nonstationary randomized policies, rather than the class of all policies with Lebesgue-density action kernels. For $\pi\in\Pi_{\mathrm{tar}}$ and $h\geq1$, let $\rho_h^\pi$ denote the law of $(X_h,A_h)$ under $X_1\sim\nu$ and $\pi$, and set  $\mathfrak M:=\{\rho_h^\pi:h\geq1,\pi\in\Pi_{\mathrm{tar}}\}$, 
and assume that $\mu^{\pi^*}\in\mathfrak M$.
\end{definition}

\subsection{Notations}
We introduce some additional notation used throughout the paper. For any $u,v\in\mathbb{R}$, write $u=\mathcal{O}(v)$ or $u\lesssim v$ if $u\le Cv$ for some constant $C>0$, and write $u=\Omega(v)$ or $u\gtrsim v$ if $u\ge C^{\prime}v$ for some constant $C^{\prime}>0$. Moreover, write $u\asymp v$ if both $u\lesssim v$ and $u\gtrsim v$ hold. 
We also use $\widetilde{\mathcal O}$ to suppress logarithmic factors.
For any $s,t\in\mathbb{R}$, let $\lceil s\rceil$ and $\lfloor s\rfloor$ denote the ceiling and floor of $s$, and define $s\vee t:= \max\{s,t\}$ and $s\wedge t:= \min\{s,t\}$.
Let $\mathbb{N}_0$ and $\mathbb{N}$ denote the sets of non-negative integers and positive integers, respectively. For a vector $\boldsymbol{z}=(z_1,\ldots,z_d)^\top\in\mathbb{R}^d$ and $q\in[1,\infty)$, define   $\|\boldsymbol{z}\|_q:=(\sum_{j=1}^d |z_j|^q)^{1/q}$, and define $\|\boldsymbol{z}\|_\infty:=\max_{1\le j\le d}|z_j|$ for $q=\infty$. In particular, $\|\boldsymbol{z}\|_2:=(\sum_{j=1}^d z_j^2)^{1/2}$ denotes the Euclidean norm. For a probability measure $\nu$ and a measurable function $h:\mathbb{R}^d\to\mathbb{R}$,  define the $L^q(\nu)$ norm by  $\|h\|_{L^q(\nu)}^q:=\Ebb_{z\sim\nu}|h(z)|^q$.


\section{Method}\label{sec:method}
In this section, we introduce the proposed framework for offline deep $Q^\ast$ estimation based on conditional diffusion models. 
Recall that the optimal action-value function $Q^\ast$ is characterized as the fixed point of the optimal Bellman equation. 
This naturally motivates estimating $Q^\ast$ through Bellman residual minimization. 
We first define the residual loss under the optimal policy,
\[
\mathcal L_{\mathcal T^*}^{*}(Q):=\mathbb E_{(X,A)\sim\mu^{\pi^*}}\bigl[|\mathcal T^*Q(X,A)-Q(X,A)|^2\bigr].
\]
Since the offline observations are collected under the behavior policy $\pi_b$, we define the corresponding population Bellman residual loss as
\begin{equation}\label{eq:population-bellman-loss}
\mathcal L_{\mathcal T^*}(Q):=\mathbb E_{(X,A)\sim\mu}\bigl[|\mathcal T^*Q(X,A)-Q(X,A)|^2\bigr],
\end{equation}
where $\mu=\mu^{\pi_b}$ is the sampling distribution of the state--action pair $(X,A)$.
If $\mu^{\pi^*}\ll\mu^{\pi_b}$ and the density ratio $\frac{\mathrm d\mu^{\pi^*}}{\mathrm d\mu^{\pi_b}}$ has a finite second moment under $\mu^{\pi_b}$, then by Lemma~\ref{lem:residual-loss-transfer}, we have $(\mathcal L_{\mathcal T^*}^{*}(Q))^2\lesssim \frac{R^2_{\max}}{(1-\gamma)^2}\mathcal L_{\mathcal T^*}(Q)$  for uniformly bounded $Q$ (these conditions are included in Assumption~\ref{apt:concentrability}).
Under the H\"older continuity condition on the conditional densities stated in Assumption~\ref{apt:diffusion-density-holder}, we use Lemmas~\ref{lem:Q-smoothness} and~\ref{lem:residual-loss-transfer} to reformulate our objective as minimizing the Bellman residual loss under the behavior distribution
\[
Q^*\in\arg\min_Q\mathcal L_{\mathcal T^*}(Q).
\]
A fundamental difficulty is that the optimal Bellman operator $\mathcal T^\ast$ is unknown, since both the reward law and the transition kernel are unavailable in offline RL. 
This observation motivates learning these conditional distributions directly from offline data and then constructing a data-driven approximation of the optimal Bellman operator. 
To this end, we use conditional diffusion models to estimate the reward law and transition kernel, as described in Section \ref{sec:rl-diffusion-learning}. 
The resulting estimators induce an estimated optimal Bellman operator, which is subsequently incorporated into the 
value-learning procedure in Section \ref{sec:rl-vf-learning}. 
Accordingly, the proposed method requires three independent batch datasets generated under a behavior policy $\pi_b$:
$ 
\mathbb S_r
=
\{(X_i,A_i,R_i,X_i^{\prime})\}_{i=1}^{n_r},
~
\mathbb S_x
=
\{(X_i,A_i,R_i,X_i^\prime)\}_{i=1}^{n_x},
~
\mathbb D
=
\{(X_i,A_i,R_i,X_i^\prime)\}_{i=1}^n.
$ 
The datasets $\mathbb S_r$ and $\mathbb S_x$ are used for reward law and transition kernel learning, respectively, while $\mathbb D$ is used for value function learning. 
Next, we present the detailed construction of the optimal Bellman operator learning and the corresponding value-learning procedure in the following two sections.

\subsection{Optimal Bellman Operator Learning}\label{sec:rl-diffusion-learning}
To learn the optimal Bellman operator $\mathcal T^\ast$, it is sufficient to estimate the underlying reward law and transition kernel that define the Bellman equation. 
Consequently, we use conditional diffusion models to learn these conditional distributions from offline datasets.  Specifically, the reward distribution ${\mathcal{R}(\cdot|x,a)}$ is estimated from the batch dataset 
{$\mathbb{S}_r$}, while the transition distribution $\mathcal{P}(\cdot|x,a)$ is learned using $\mathbb{S}_x 
$. The resulting approximated counterparts are denoted as $\widehat{\mathcal{R}}(\cdot|x,a)$ and $\widehat{\mathcal{P}}(\cdot|x,a)$, respectively. 
For notational convenience, let $\mathbb S:=\mathbb S_r\cup\mathbb S_x$ denote the combined dataset. 
We write $\boldsymbol w:=(x,a) \in \mathbb{R}^{d_w}$ for a fixed state--action observation with dimension $d_w:=d_x+d_a$.
Once these two conditional models have been learned, we can generate rewards and next states for any given state-action pair $\boldsymbol{w}$ 
by sampling from $\widehat{\mathcal{R}}(\cdot|\boldsymbol{w})$ and $\widehat{\mathcal{P}}(\cdot|\boldsymbol{w})$. Since the constructions of the reward and transition models are parallel, 
we focus on the transition model below; the reward model  can be developed analogously.  

To estimate the transition kernel, we let 
$p(\cdot|\boldsymbol w)$ denote the  
distribution of the next state   
given current state-action pair $\boldsymbol w$, that is, 
{$p(\cdot|\boldsymbol w)
\equiv \mathcal P(\cdot|\boldsymbol{w} )
$,
}
and introduce the forward diffusion process
\begin{equation*}
d \bar{\boldsymbol{x}}_t =-2 \bar{\boldsymbol{x}}_t dt+2d B_t,~~
\bar{\boldsymbol{x}}_0 \sim p(\cdot|\boldsymbol{w})
,~~ t \in [0,T].
\end{equation*}
Then,  the corresponding  reverse process ${\boldsymbol{x}}_t:=\bar{\boldsymbol{x}}_{T-t}$ is given by
\begin{equation}\label{eq:r-Back-process}
	d {\boldsymbol{x}}_t=[2 {\boldsymbol{x}}_t+4\nabla \log p_{T-t}({\boldsymbol{x}}_t|\boldsymbol{w})]dt+2d
    {B_t}.
\end{equation}
Here, $p_t(\cdot | \boldsymbol w)$ denotes the conditional density function of $\bar{\boldsymbol x}_t$, and $\nabla\log p_t(\cdot | \boldsymbol w)$ denotes the corresponding conditional score function to be estimated.

\paragraph{Score Estimation.}
The reverse process described in \eqref{eq:r-Back-process} theoretically transforms noise into samples from the target conditional distribution $p(\boldsymbol{x}|\boldsymbol{w})$. However, direct simulation is intractable because the true conditional score function $\nabla_{\boldsymbol{x}} \log p_t(\boldsymbol{x}|\boldsymbol{w})$ is unknown, rendering the drift term $\boldsymbol{s}^*(t, \boldsymbol{x}, \boldsymbol{w}) := \nabla_{\boldsymbol{x}} \log p_t(\boldsymbol{x}|\boldsymbol{w})$ uncomputable. To address this, we employ score matching techniques \citep{hyvarinen2005estimation,vincent2011connection} to learn an approximate score estimator $\widehat{\boldsymbol{s}}(t, \boldsymbol{x}, \boldsymbol{w})$. 
We define the following loss function
$$\mathcal{L}(\boldsymbol{s}) := \frac{1}{T - 2T_1} \int_{T_1}^{T- T_1} \mathbb{E}_{\bar{\boldsymbol{x}}_t, \boldsymbol{W}} \| \boldsymbol{s}(t, \bar{\boldsymbol{x}}_t, \boldsymbol{W}) - \nabla\log p_t(\bar{\boldsymbol{x}}_t|\boldsymbol{W}) \|^2 dt,$$
where $T_1 \in (0,T)$ is introduced as a time truncation parameter to avoid instability near the endpoints of the diffusion process, 
and  $\boldsymbol W:=(X,A)$ denotes the  random state--action pair. 
Using denoising score matching,
the objective can be equivalently expressed, up to a constant, as
\begin{align*}
\mathcal L(\boldsymbol s)
&=\frac{1}{T-2T_1}
\int_{T_1}^{T-T_1}
\mathbb E_{\bar{\boldsymbol x}_0, \boldsymbol W}
\mathbb E_{\bar{\boldsymbol x}_t|\bar{\boldsymbol x}_0,\boldsymbol W}
\left\|\boldsymbol s(t,\bar{\boldsymbol x}_t,\boldsymbol W) +\frac{\bar{\boldsymbol x}_t-\bar{\boldsymbol x}_0 \mu_t}{\sigma_t^2}
\right\|^2 dt \\
&=\frac{1}{T-2T_1} \int_{T_1}^{T-T_1}
\mathbb E_{\bar{\boldsymbol x}_0, \boldsymbol W}
\mathbb E_{Z}
\left\|\boldsymbol s(t,\bar{\boldsymbol{x}}_0 \mu_t + Z\sigma_t ,\boldsymbol W) +\frac{Z}{\sigma_t}\right\|^2 dt,
\end{align*}
where $\mu_t=e^{-2t}$ and $\sigma_t^2=1-e^{-4t}$. 
For the empirical objective, write $\boldsymbol W_i:=(X_i,A_i),~\bar{\boldsymbol x}_{0,i}=X_i^{\prime},$ of the $i$-th transition in $\mathbb S_x$. Let $\{({t_j},Z_j)\}_{j=1}^{m_x}$ be i.i.d. auxiliary samples with 
$t_j\sim U[T_1,T-T_1]$ and $Z_j\sim\mathcal N(0,\mathbf{I}_{d_x})$. 
The estimator $\widehat{\boldsymbol s}$ is given by
\begin{equation}\label{eq:score-hat}
\widehat{\boldsymbol{s}} \in \underset{\boldsymbol{s} \in {\cG }}{\arg\min}   ~ \widehat{\mathcal{L}}(\boldsymbol{s}) 
:= \frac{1}{m_x n_x} \sum_{j=1}^{m_x} \sum_{i=1}^{n_x} \left\| \boldsymbol{s}(t_j, \bar{\boldsymbol x}_{0,i} \mu_{t_j} + Z_j\sigma_{t_j}, \boldsymbol{W}_i) +\frac{Z_j}{\sigma_{t_j}} \right\|^2,
\end{equation}
where $\cG$ denotes ReLU DNNs in Section \ref{section:ReLU-DNNs}. The resulting drift estimator is then given by
$\widehat{b}(t,\boldsymbol{x},\boldsymbol{w})=4\widehat{\boldsymbol{s}}(t,\boldsymbol{x},\boldsymbol{w})+2\boldsymbol{x}$.

\paragraph{EM Discretization.}
With the   estimated drift term $\widehat{b}$, 
we can define an SDE initializing from the prior distribution $p_{T_1}(\boldsymbol{x}|\boldsymbol{w})$, i.e., 
\begin{align}
	\label{eq:r-hat-model}
	d\widehat{\boldsymbol{x}}_t &=
	\widehat{b}(T-t,\widehat{\boldsymbol{x}}_t,\boldsymbol{w})d t+2 d B_t,\notag \\
	\widehat{\boldsymbol{x}}_{T_1} & \sim p_{T-T_1}(\boldsymbol{x} |\boldsymbol{w}),~~ T_1\leq t \leq T-T_1.
\end{align}
Subsequently, we can employ a discrete-approximation method for this dynamics \eqref{eq:r-hat-model}. Let
\[
T_1 = {t_0 }< t_1 < \cdots < t_K = T-T_1, \quad K \in \mathbb{N}^+,
\]
be the discretization points on $[T_1, T-T_1]$. We consider the explicit EM scheme:
\begin{align}
	\label{eq:r-tilde-model}
	d\widetilde{\boldsymbol{x}}_t &= \widehat{b}(T - t_k, \widetilde{\boldsymbol{x}}_{t_k}, \boldsymbol{w})dt + 2dB_t, \notag \\
	\widetilde{\boldsymbol{x}}_{t_0} &= \boldsymbol{x}_{T_1} \sim p_{T-T_1}(\boldsymbol{x}|\boldsymbol{w}), \quad t \in [t_k, t_{k+1}],
\end{align}
for $k = 0, 1, \cdots, K - 1$. However, the conditional distribution $p_{T-T_1}(\boldsymbol{x}|\boldsymbol{w})$ is unknown, rendering it impossible to utilize dynamics \eqref{eq:r-tilde-model} for generating new samples. Thus, we substitute $p_{T-T_1}(\boldsymbol{x}|\boldsymbol{w})$ with $\mathcal{N}(\mathbf{0}, \mathbf{I}_{d_x})$ since $p_{T-T_1}(\boldsymbol{x}|\boldsymbol{w})$ converges to $\mathcal{N}(\mathbf{0}, \mathbf{I}_{d_x})$ as $T$ tends to infinity. This adaptation presents a practically viable strategy. Therefore, the sampling dynamics take the following form:
\begin{align}
	\label{eq:r-check-model}
	d\check{\boldsymbol{x}}_t &= \widehat{b}(T - t_k, \check{\boldsymbol{x}}_{t_k}, \boldsymbol{w})dt + 2dB_t, \notag \\
	\check{\boldsymbol{x}}_{t_0} &\sim \mathcal{N}(\mathbf{0}, \mathbf{I}_{d_x}), \quad t \in [t_k, t_{k+1}]. 
\end{align}
Consequently, we can employ the dynamics in \eqref{eq:r-check-model} to generate new samples that are approximately distributed according to the target conditional distribution. 
We work with the state space \(\mathcal X=[0,1]^{d_x}\) and reward space \([0,R_{\max}]\).  We therefore truncate the generated samples to their respective domains. For \(l<u\), define $\operatorname{CLIP}(z,l,u):=\min\{u,\max\{l,z\}\}$, 
where the operation is applied componentwise when \(z\) is a vector. 
We present the conditional diffusion model for transition kernel
learning in Algorithm \ref{algo:CDM}.
\begin{algorithm}[H] 
	\caption{Transition Kernel Learning via Conditional Diffusion Models}
	\label{algo:CDM}
	\begin{algorithmic}[1]
		\REQUIRE $N$, 
		$\mathbb S_x$,  $\{(t_j, Z_j)\}_{j=1}^{m_x} \sim \text{U}[T_1,T-T_1] ~\mbox{and}~\mathcal{N}(Z;\mathbf{0},\mathbf{I}_{d_x})$.
		\STATE 
        Obtain the score estimation $\widehat{\boldsymbol{s}}$ by \eqref{eq:score-hat}
		\STATE Initialize $\check{\boldsymbol{x}}_{t_0} \sim \mathcal{N}(\mathbf{0},\mathbf{I}_{d_x})$.
		\FOR{$k = 0$ \TO $K-1$}
		\STATE Sample $\boldsymbol{\epsilon}_{k}\sim{\mathcal{N}}(\mathbf{0},\mathbf{I}_{d_x})$.
		\STATE 
        Update $\check{\boldsymbol{x}}_{t_{k+1}} \leftarrow \check{\boldsymbol{x}}_{t_k}+
        \widehat{b}(T - t_k, \check{\boldsymbol{x}}_{t_k}, \boldsymbol{w})
        (t_{k+1}-t_k) 
        +2 \sqrt{t_{k+1}-t_k}\boldsymbol{\epsilon}_{k}$;
		\ENDFOR
		\STATE
		Output $\operatorname{CLIP}(\check{\boldsymbol{x}}_{t_K},0,1)$.
	\end{algorithmic}
\end{algorithm}

\subsection{Value Function Learning}\label{sec:rl-vf-learning}
With the learned reward law and transition kernel 
$\widehat{\mathcal R}(\cdot |  x,a)$ 
and $\widehat{\mathcal P}(\cdot |  x,a)$,  
constructed in Section \ref{sec:rl-diffusion-learning}, we can draw $m$ samples from these conditional distributions for any given state--action pair $(x,a)$, denoted by 
$ \{\widehat R_j \}_{j=1}^m\sim \widehat{\mathcal R}(\cdot |  x,a),~ \{\widehat X_j^\prime\}_{j=1}^m \sim \widehat{\mathcal P}(\cdot | x,a).$
For each fixed $(x,a)$, the empirical operator uses the same reward and next-state samples for all candidate $Q$-functions.
Therefore, we can define the empirical optimal Bellman operator:
\begin{equation}\label{eq:T-hat}
	\widehat{\mathcal{T}}^*Q(x,a):= \frac{1}{m}\sum_{j=1}^m\widehat{R}_j +\frac{\gamma}{m} \sum_{j=1}^m\max_{a^{\prime}\in\mathcal A} Q(\widehat{X}^{\prime}_j,a^{\prime}).
\end{equation}
We can verify that the empirical optimal Bellman operator $\widehat{\mathcal T}^\ast$ is a $\gamma$-contraction (See Lemma \ref{lem:T-hat-fixed-point}). 
Accordingly, 
we define the corresponding   Bellman residual loss:
\begin{equation}\label{eq:loss-modified}
	\mathcal{L}_{{\widehat{\mathcal T}}^*}(Q):=
	\mathbb{E}_{(X,A)\sim \mu }\left|\widehat{\mathcal T}^*Q(X,A)-Q(X,A)\right|^2.
\end{equation}
Given the batch dataset $\mathbb D$, we define the final estimator $\widehat Q$ as the empirical risk minimizer associated with the 
empirical optimal Bellman operator $\widehat{\mathcal T}^\ast$. 
Specifically,
\begin{equation}\label{eq:empirical-learned}
\widehat{Q}\in \arg\min_{Q\in { \cQ}}\widehat{\mathcal{L}}_{\widehat{\mathcal T}^*}(Q):=\frac{1}{n}\sum_{i=1}^n \left|\widehat{\mathcal T}^*Q(X_i,A_i)-Q(X_i,A_i)\right|^2.
\end{equation}
Here, $\cQ$ denotes the class of ReLU DNNs used to estimate the optimal action-value function $Q^\ast$. By the definition of the MDP, the action-value function is bounded above by $\frac{R_{\max}}{1-\gamma}$. Consequently, without loss of generality, we restrict $\cQ$ to functions uniformly bounded by $\frac{R_{\max}}{1-\gamma}$.

Therefore, in the proposed framework, once the reward model and transition model are learned, the optimal Bellman operator can be estimated accordingly. The remaining step then reduces to estimating $Q^\ast$ through Bellman residual minimization under the learned Bellman operator.  
We summarize the overall structure of the proposed method in the following Algorithm \ref{algo:ODQDM}.
\begin{algorithm}[H] 
	\caption{Offline Deep $Q^*$ Estimation with Diffusion Models}
	\label{algo:ODQDM}
    \begin{algorithmic}[1]
    \REQUIRE   Batch datasets $\bS_r$, $\bS_x$, $\bD$, and sample size $m$.
\ENSURE Estimator $\widehat Q$ of the optimal action-value function $Q^\ast$.
        \STATE  
Learn the reward law and transition kernel via Algorithm \ref{algo:CDM}  using the datasets $\mathbb S_r$ and $\mathbb S_x$, respectively.
\STATE 
Construct the estimator of the optimal Bellman operator, $\widehat{\mathcal T}^\ast$, via \eqref{eq:T-hat}.
\STATE Obtain the estimator of the optimal action-value function, $\widehat Q$, via \eqref{eq:empirical-learned}.
 \end{algorithmic}
\end{algorithm}

Algorithm~\ref{algo:ODQDM} provides a concise description of the proposed estimation procedure. The reward law and transition kernel are first estimated from offline data and then incorporated into an empirical approximation of the optimal Bellman operator. Specifically, for each state and action pair \((x,a)\), the learned transition model generates \(m\) next-state samples \(\{\widehat X'_j(x,a)\}_{j=1}^{m}\). For each sampled next state \(\widehat X'_j(x,a)\), the maximization step solves $a^{\max}_j \in
\arg\max\limits_{a'\in\mathcal A}
Q\bigl(\widehat X'_j(x,a),a'\bigr)$ and substitutes \(Q(\widehat X'_j(x,a),a^{\max}_j)\) into the empirical Bellman backup. The resulting maximum values are averaged over the \(m\) next-state samples and combined with the Monte Carlo average of the generated rewards to form \(\widehat{\mathcal T}^{\ast}Q(x,a)\). 
In this work, we consider a continuous action space. Hence, the maximization involved in the estimated Bellman operator is generally approximated numerically, for example, through action discretization or gradient-based optimization. 
Our theoretical analysis assumes access to the exact maximization operator.  
An approximate action optimization procedure introduces an additional optimization error, which is not considered in the present analysis. 
Finally, the  estimator \(\widehat Q\) is obtained by empirical Bellman residual minimization with respect to the learned operator \(\widehat{\mathcal T}^{\ast}\).

Existing offline RL methods usually use data or learned models in different ways. Fitted value learning methods, such as FQI, DQN, DFQI, and DAPI, estimate \(Q^\ast\) through recursive Bellman regression \citep{ernst2005tree,mnih2015human,fan2020theoretical,jiao2025deep}.
Minimax Bellman error methods, such as MSBO and MABO, replace recursive regression with optimization over Bellman error criteria and auxiliary function classes \citep{xie2020q}. Model based offline RL methods learn reward or transition models, but these models are typically used for rollout generation, uncertainty penalization, or conservative policy optimization \citep{janner2019trust,yu2020mopo,kidambi2020morel,yu2021combo}. In contrast, the proposed method uses the learned reward and transition laws directly to construct the optimal Bellman operator. This leads to a simple final stage: after the conditional diffusion models are trained, \(Q^\ast\) is estimated by a standard empirical risk minimization problem. The use of conditional diffusion models further allows flexible modeling of complex reward and transition laws, including continuous and potentially multimodal conditional distributions, without adding rollout trajectories generated by the model to the replay buffer.

\section{Theoretical Analysis}\label{sec:theory}
In this section, we develop the theoretical analysis of the proposed method. Our main goal is to derive a value function estimation error bound for the estimator \(\widehat Q\). 
The analysis uses a randomized comparison kernel supported on a $\delta$-greedy comparison set of positive Lebesgue measure. Under Assumption~\ref{apt:concentrability}, we establish that
\[
\|\widehat Q-Q^*\|_{L^2(\rho)}\leq\frac{\sqrt{ C_\delta}}{1-\gamma}\|\widehat Q-\cT^*\widehat Q\|_{L^2(\mu)}+\frac{\gamma\delta}{1-\gamma}.
\]
Here, $\mu$ is the data generating distribution of $(X,A)$, $\rho$ is the target distribution used to evaluate the value function error, and ${C_\delta}$ quantifies the mismatch between the data distribution and the target and successor comparison distributions.
This result shows 
that bounding the 
value function estimation error reduces to controlling the Bellman residual, i.e., 
$
\mathcal L_{\mathcal T^\ast}(\widehat Q)
-
\mathcal L_{\mathcal T^\ast}(Q^\ast).
$
As described in Section 
\ref{sec:method}, our method consists of two key components: learning the reward law and transition kernel through conditional diffusion models, and then estimating the value function through Bellman residual minimization with the learned Bellman operator. 
Therefore, the final error depends jointly on the accuracy of the learned reward and transition models and the  error from value function learning. 
To make this interaction explicit, 
we decompose the Bellman residual using the following lemma (see
Appendix~\ref{sec:main-results} for the proof).
\begin{lemma}[Error decomposition]
\label{lem:error-dec}
Conditional on $\bS,\bT,\bZ$, the estimator $\widehat Q$ satisfies
\begin{align*}
\bE_{\bD}\big[\mathcal L_{\mathcal T^*}(\widehat Q)-\mathcal L_{\mathcal T^*}(Q^*)\big]
&\leq
\underbrace{2
\bE_{\bD}\left[\|\mathcal T^*\widehat{Q}-\widehat{\mathcal T}^*\widehat{Q}\|_{L^2(\mu)}^2\right]}_{\mbox{Part}~\textbf{I}}
+\underbrace{8\|\widehat{\mathcal T}^*Q^*-\mathcal T^*Q^*\|_{L^2(\mu)}^2}_{\mbox{Part}~\textbf{II}}\\
&~ +\underbrace{2\bE_{\bD}\sup_{Q\in\mathcal Q}\big(\mathcal L_{\widehat{\mathcal T}^*}(Q)-2\widehat{\mathcal L}_{\widehat{\mathcal T}^*}(Q)\big)+8(1+\gamma)^2\inf_{Q\in\mathcal Q}\|Q-Q^*\|_\infty^2}_{\mbox{Part}~\textbf{III}}. 
\end{align*}
\end{lemma}
In Lemma \ref{lem:error-dec}, the Parts $\textbf{I}$ and $\textbf{II}$ measure the discrepancy between the true and learned Bellman operators and therefore capture the effect of conditional diffusion estimation.  
The Part $\textbf{III}$ is the sum of a statistical error and an approximation error arising from the value-learning step under the learned Bellman operator.
This decomposition forms the basis of our end-to-end error analysis.
Consequently, the analysis reduces to establishing upper bounds for the optimal Bellman operator learning error and the value function learning error separately. 
These two components are analyzed in Sections \ref{sec:eb1}--\ref{sec:eb2}, respectively. 
Combining these two bounds yields the final convergence rate for the proposed estimator $\widehat Q$, presented in Section \ref{sec:mr}.
 
\subsection{Error Bound for Optimal Bellman Operator Learning}\label{sec:eb1}
In this section, we aim to bound  the error of   optimal Bellman operator learning, that is, bound Parts $\textbf{I}$ and $\textbf{II}$.
By direct calculation (See Appendix \ref{sec:B.5})  and recalling that
$\boldsymbol{W}=(X,A)$, we obtain
\begin{equation}\label{eq:operator-hat-Q}
\begin{aligned}
\bE_{\bS} 
\left[\|\mathcal T^*\widehat{Q}-\widehat{\mathcal T}^*\widehat{Q}\|_{L^2(\mu)}^2
\right] 
&\lesssim R_{\max}^2\Ebb_{\bS}\Ebb_{\boldsymbol{W}}\left[\TV^2\big(\cR(\cdot|\boldsymbol{W}),\widehat{\cR}(\cdot|\boldsymbol{W})\big)\right]\\
&~~+\frac{R_{\max}^2}{(1-\gamma)^2}n^{\frac{d_x+d_a}{d_x+d_a+2\alpha}}\frac{\log^6 n(\log m+\log n)}{m}\\
&~~ +\frac{R_{\max}^2}{(1-\gamma)^2}\Ebb_{\bS}\Ebb_{\boldsymbol{W}}\left[\TV^2\big(\cP(\cdot|\boldsymbol{W}),\widehat{\cP}(\cdot|\boldsymbol{W})\big)\right],
\end{aligned}
\end{equation}
\begin{equation}\label{eq:operator-Q*}
\begin{aligned}
\bE_{\bS} \left[\|\widehat{\mathcal T}^*Q^*-\mathcal T^*Q^*\|_{L^2(\mu)}^2\right]&\lesssim  R_{\max}^2 \bE_{\bS}\bE_{\boldsymbol{W}} \Big[\TV^2\big(\mathcal R(\cdot|\boldsymbol{W}),\widehat{\mathcal R}(\cdot|\boldsymbol{W})\big)\Big]+ \frac{R_{\max}^2}{(1-\gamma)^2m}\\
&~~+ \frac{R_{\max}^2}{(1-\gamma)^2} \bE_{\bS} \bE_{\boldsymbol{W}}\left[\TV^2\bigl(\mathcal P(\cdot|\boldsymbol{W}),\widehat{\mathcal P}(\cdot|\boldsymbol{W})\bigr)\right].
\end{aligned}
\end{equation}
Therefore, it suffices to bound the errors of  two conditional diffusion models. 
Recall the processes $\boldsymbol x_t$ in \eqref{eq:r-Back-process}, $\widehat{\boldsymbol x}_t$ in \eqref{eq:r-hat-model}, $\widetilde{\boldsymbol x}_t$ in \eqref{eq:r-tilde-model}, and $\check{\boldsymbol x}_t$ in \eqref{eq:r-check-model}; for any $\boldsymbol w\in[0,1]^{d_w}$, 
let $p_{T-t}(\cdot|\boldsymbol w)$, $\widehat p_{T-t}(\cdot|\boldsymbol w)$, $\widetilde p_{T-t}(\cdot|\boldsymbol w)$, and $\check p_{T-t}(\cdot|\boldsymbol w)$ denote their corresponding  distributions, respectively. 
By the data processing property of total variation, coordinatewise clipping of the generated outputs to their prescribed domains does not enlarge the bounds below.
We then decompose the TV error into three parts:
\begin{equation}\label{eq:TV-decomposition}
\begin{aligned}
\mathbb{E}_{\bS_x,\bT,\bZ}
\bE_{\boldsymbol{W}}
\Big[\TV^2 \big(\check{p}_{T_1}(\cdot|\boldsymbol{W}),p(\cdot|\boldsymbol{W})\big)\Big]
&\lesssim  \mathbb{E}_{\bS_x,\bT,\bZ}
\bE_{\boldsymbol{W}}\Big[\TV^2 \big(\widetilde{p}_{T_1}(\cdot|\boldsymbol{W}),p_{T_1}(\cdot|\boldsymbol{W})\big)\Big]\\
&\quad +\mathbb{E}_{\bS_x,\bT,\bZ}
\bE_{\boldsymbol{W}}
\Big[\TV^2 \big(\check{p}_{T_1}(\cdot|\boldsymbol{W}),\widetilde{p}_{T_1}(\cdot|\boldsymbol{W})\big)\Big]\\
&\quad + \mathbb{E}_{\boldsymbol{W}}\Big[\TV^2 \big(p_{T_1}(\cdot|\boldsymbol{W}),p(\cdot|\boldsymbol{W})\big)\Big].
\end{aligned}
\end{equation}
The three terms on the right-hand side of \eqref{eq:TV-decomposition} correspond respectively to the score estimation error and EM discretizations, the error from replacing the initialization distribution, and the time-truncation error. We then proceed to establish individual bounds for 
these three terms
in Subsections 
\ref{sec:er1}--\ref{sec:TV-time-truncation-error}.  To this end, we introduce several assumptions below.

\begin{assumption}[Bounded Conditional Density]\label{apt:diffusion-density-bound}
The conditional density $p(\boldsymbol x|\boldsymbol w)$ is supported on $[0,1]^{d_x}\times[0,1]^{d_w}$ and is uniformly bounded above and below by positive constants $B_u$ and $B_l$, respectively.
\end{assumption}

\begin{assumption}[H\"older Continuity]\label{apt:diffusion-density-holder}
The conditional density $p(\boldsymbol{x}|\boldsymbol{w})$ belongs to the H\"older class $\mathcal{H}^{\alpha}([0,1]^{d_x+d_w},\Sigma)$ for some $\alpha>0$ and $\Sigma>0$, where $\alpha=p+q$ with $p\in\mathbb{N}_0$ and $q\in(0,1]$.
\end{assumption}
\begin{assumption}[Boundary Smoothness]\label{apt:diffusion-density-smooth}
For some constants $a_0\in(0,\frac{1}{2})$ and $\Sigma_0>0$,
the restriction of $p(\boldsymbol{x}|\boldsymbol{w})$ to $\big([0,1]^{d_x}\backslash[a_0,1-a_0]^{d_x}\big)\times[0,1]^{d_w}$ admits a Whitney extension in $\mathcal{H}^{3\alpha+2}([-1,1]^{d_x}\times[0,1]^{d_w},\Sigma_0)$.
\end{assumption}

\begin{assumption}[Bounded Partial Derivatives]\label{apt:diffusion-density-derivative}
For every multi-index $\boldsymbol{\gamma}\in\mathbb{N}_0^{d_w}$ satisfying $\|\boldsymbol{\gamma}\|_1\le \lfloor6\alpha+3\rfloor+1$, the partial derivative $\partial_{\boldsymbol{w}}^{\boldsymbol{\gamma}}p(\boldsymbol{x}|\boldsymbol{w})$ exists and is uniformly bounded on $[0,1]^{d_x+d_w}$ by a constant $B_{\boldsymbol{\gamma}}$.
\end{assumption}

\begin{remark}
The regularity conditions on the target density in Assumptions~\ref{apt:diffusion-density-bound}--\ref{apt:diffusion-density-derivative} are imposed to facilitate the theoretical analysis of diffusion models and are consistent with commonly used assumptions in the literature. The compact support condition in Assumption~\ref{apt:diffusion-density-bound} is standard in related works such as \cite{lee2023convergence, li2023towards, oko2023diffusion}. Assumption~\ref{apt:diffusion-density-holder} requires H\"older continuity, which is a classical assumption in nonparametric estimation and is closely related to Assumption 3.1 in \cite{fu2024unveil}. In addition, Assumptions~\ref{apt:diffusion-density-smooth} and~\ref{apt:diffusion-density-derivative} impose boundary smoothness and bounded derivatives, following the technical framework of \citet{oko2023diffusion}, and thereby permit comparable theoretical analysis.
\end{remark}

\subsubsection{Bound \texorpdfstring{$\mathbb{E}_{\bS_x,\bT,\bZ, \boldsymbol{W}}\Big[\TV^2 \Big(\widetilde{p}_{T_1}(\cdot|\boldsymbol{W}),p_{T_1}(\cdot|\boldsymbol{W})\Big)\Big]$}{TV Time Discrete Error}
}\label{sec:er1}

We begin by establishing an error bound for score function estimation. To state the error decomposition, we introduce two auxiliary losses. For any score function \(\boldsymbol{s}\), define
\[
{\mathcal{L}}_{\bS_x}(\boldsymbol{s})
:=
\frac{1}{n_x}\sum_{i=1}^{n_x}
\ell(\boldsymbol{s},\bar{\boldsymbol{x}}_{0,i},\boldsymbol{W}_i),
~~ 
\widehat{\mathcal{L}}_{\bS_x,\bT,\bZ}(\boldsymbol{s})
:=
\frac{1}{n_x}\sum_{i=1}^{n_x}
\widehat{\ell}(\boldsymbol{s},\bar{\boldsymbol{x}}_{0,i},\boldsymbol{W}_i),
\]
where
\[
\ell(\boldsymbol{s},\bar{\boldsymbol{x}}_{0},\boldsymbol{W})
:=
\frac{1}{T-2T_1}
\int_{T_1}^{T-T_1}
\mathbb{E}_{Z}
\left\|
\boldsymbol{s}\big(t,\bar{\boldsymbol{x}}_{0}\mu_t+Z\sigma_t,\boldsymbol{W}\big)
+\frac{Z}{\sigma_t}
\right\|^2dt,
\]
and
\[
\widehat{\ell}(\boldsymbol{s},\bar{\boldsymbol{x}}_{0},\boldsymbol{W})
:=
\frac{1}{m_x}
\sum_{j=1}^{m_x}
\left\|
\boldsymbol{s}\big(t_j,\bar{\boldsymbol{x}}_{0}\mu_{t_j}+Z_j\sigma_{t_j},\boldsymbol{W}\big)
+\frac{Z_j}{\sigma_{t_j}}
\right\|^2.
\]
With these definitions, the excess risk of the estimated score function \(\widehat{\boldsymbol{s}}\) satisfies
$$\mathcal{L}(\widehat{\boldsymbol{s}}) - \mathcal{L}(\boldsymbol{s}^*)
\leq 
\underbrace{\mathcal{L}(\widehat{\boldsymbol{s}}) - 2{\mathcal{L}}_{\bS_x}(\widehat{\boldsymbol{s}}) + \mathcal{L}(\boldsymbol{s}^*)}_{\mbox{Statistical~error}} + 2\underbrace{\left( {\mathcal{L}}_{\bS_x}(\widehat{\boldsymbol{s}}) - \widehat{\mathcal{L}}_{\bS_x,\bT,\bZ}(\widehat{\boldsymbol{s}}) \right)}_{\mbox {Statistical~error}}  + 2\inf_{\boldsymbol{s}\in \cG} \underbrace{\Big(  \mathcal{\widehat{L}}_{\bS_x,\bT,\bZ}({\boldsymbol{s}}) - \mathcal{L}(\boldsymbol{s}^*) \Big)}_{\mbox{Approximation~error}  }.
$$
The derivation of this decomposition is provided in Appendix~\ref{sec:score-decompose}. 
We then bound the approximation error in Lemma~\ref{lem:score-approximation} and  two statistical errors in Lemma~\ref{lem:score-statistical-error}.

\begin{lemma}[Approximation Error]
\label{lem:score-approximation}
Suppose that Assumptions \ref{apt:diffusion-density-bound}-\ref{apt:diffusion-density-derivative} hold. 
Let $N \gg 1$,
$T_1=N^{-C_\mu}$ and $T=C_{\lambda}\log N$, where 
$C_{\mu}=\frac{2\alpha}{\alpha \wedge 1}$ and $C_{\lambda}=\alpha$ are constants. 
Then there exists a ReLU neural network 
$\boldsymbol{s}\in \cG(\mathcal D, \mathcal W, \mathcal S,\mathcal B)$
satisfying
$$
\boldsymbol{s}(t,{\boldsymbol{x}},\boldsymbol{w}) = \boldsymbol{s}(t, \boldsymbol{x}_{\mathrm{clip}}({\boldsymbol{x}}, -C_0, C_0),\boldsymbol{w}) ~\text{for}~ \Vert{\boldsymbol{x}}\Vert_\infty > C_0, ~\text{where}~ C_0 = \mathcal{O}(\sqrt{\log N}),
$$
with network size
$$
\mathcal D = \mathcal{O}(\log^4 N),~ \mathcal W = \mathcal{O}(N^{d_x+d_w}\log^7 N),~ \mathcal S = \mathcal{O}(N^{d_x+d_w}\log^9 N),~ \mathcal B = \exp\left(\mathcal{O}(\log^4 N)\right).
$$
For any $\boldsymbol{w}\in [0,1]^{d_w}$ and $t\in[T_1,T-T_1]$, we have
$$
\int_{{\boldsymbol{x}}\sim p_t(\cdot|\boldsymbol{w})}\Vert \boldsymbol{s}(t,{\boldsymbol{x}},\boldsymbol{w}) - \nabla\log p_t({\boldsymbol{x}}|\boldsymbol{w}) \Vert^2  d{\boldsymbol{x}} \lesssim \frac{N^{-2\alpha}\log N}{\sigma_t^2}.
$$
Moreover, we can take $\boldsymbol{s}$ satisfying $\Vert\boldsymbol{s}(t,\cdot,\cdot)\Vert_{\infty} \lesssim \frac{\sqrt{\log N}}{\sigma_t}$.
\end{lemma}
We now restrict the following subclass of ReLU networks within $\cG(\mathcal{D},\mathcal{W},\mathcal{S},\mathcal{B})$:
\begin{align*}
\mathcal{C}&:= \Big\{\boldsymbol{s}\in \cG(\mathcal{D},\mathcal{W},\mathcal{S},\mathcal{B}) \Big| \Vert\boldsymbol{s}(t,\cdot,\cdot)\Vert_\infty \lesssim \frac{\sqrt{\log N}}{\sigma_t}, \\
& \boldsymbol{s}(t,\boldsymbol{x},\boldsymbol{w}) = \boldsymbol{s}\big(t, \boldsymbol{s}_{\mathrm{clip}}(\boldsymbol{x}, -C_0, C_0)\big) ~\text{for}~ \Vert\boldsymbol{x}\Vert_\infty > C_0, ~\text{where}~ C_0 = \mathcal{O}(\sqrt{\log N}), \\
& \mathcal D = \mathcal{O}(\log^4 N),~ \mathcal W = \mathcal{O}(N^{d_x+d_w}\log^7 N),~ \mathcal S = \mathcal{O}(N^{d_x+d_w}\log^9 N),~ \mathcal B = \exp\left(\mathcal{O}(\log^4 N)\right)
\Big\}. 
\end{align*}
Throughout the paper, we set $T_1=N^{-C_\mu}$ and $T=C_\lambda\log N$, where $C_\mu=\frac{2\alpha}{\alpha\wedge1}$ and $C_\lambda=\alpha$.

\begin{lemma}[Statistical Error]\label{lem:score-statistical-error}
Suppose that Assumptions \ref{apt:diffusion-density-bound}--\ref{apt:diffusion-density-smooth} hold. Let $N \gg 1,$ then the score estimator $\widehat{\boldsymbol{s}}$ defined in \eqref{eq:score-hat}, with neural network architectures restricted to the class $\mathcal{C}$, satisfies
\begin{equation*}
\mathbb{E}_{\bS_x,\bT,\bZ} \Big[ \mathcal{L}(\widehat{\boldsymbol{s}}) - 2{\mathcal{L}}_{\bS_x}(\widehat{\boldsymbol{s}}) + \mathcal{L}(\boldsymbol{s}^*) \Big]
\lesssim
\frac{{N}^{d_x+d_w}\log^{14} N\left(\log^4 N + \log n_x\right)}{n_x}
\end{equation*}
and
\begin{equation*}
\mathbb{E}_{\bS_x,\bT,\bZ} \Big[ {\mathcal{L}}_{\bS_x}(\widehat{\boldsymbol{s}}) - \widehat{\mathcal{L}}_{\bS_x,\bT,\bZ}(\widehat{\boldsymbol{s}}) \Big]
\lesssim 
\frac{{N}^{C_\mu+4C_\lambda+\frac{d_x+d_w}{2}} \log^{\frac{13}{2}}N\bigl(\log^3 N + \log^2 N \log m_x + \log^{\frac{3}{2}} m_x \bigr)}{\sqrt{m_x}}.
\end{equation*}
\end{lemma}

Combining the approximation error bound in Lemma \ref{lem:score-approximation} with the statistical error estimates in Lemma \ref{lem:score-statistical-error}, we obtain the following upper bound for the score estimator.

\begin{theorem}[Error Bound for Score Estimation]
\label{thm:score-error}
Suppose that Assumptions \ref{apt:diffusion-density-bound}-\ref{apt:diffusion-density-derivative}  hold. 
By choosing $N = \lfloor n_x^{\frac{1}{d_x+d_w + 2\alpha}} \rfloor + 1 \lesssim n_x^{\frac{1}{d_x+d_w + 2\alpha}}$, $m_x = n_x^{\frac{(d_x+d_w + 12\alpha)(\alpha \wedge 1)+4\alpha}{(d_x+d_w + 2\alpha)(\alpha \wedge 1)}}$, the score estimator ${\widehat{\boldsymbol{s}}}$ defined in  \eqref{eq:score-hat},
with the neural network structure belonging to class $\mathcal{C}$, satisfies 
$$
\mathbb{E}_{\bS_x,\bT,\bZ} \bigg[ \frac{1}{T - 2T_1} \int_{T_1}^{T-T_1} \mathbb{E}_{\boldsymbol{x}_t, \boldsymbol{W}} \| \widehat{\boldsymbol{s}}(T-t, \boldsymbol{x}_t, \boldsymbol{W}) - \nabla \log p_{T-t}(\boldsymbol{x}_t|\boldsymbol{W}) \|^2 dt \bigg] \lesssim n_x^{-\frac{2\alpha}{d_x+d_w + 2\alpha}} \log^{18}n_x.
$$
\end{theorem}

\begin{remark}
Existing diffusion model theory places wide emphasis on the $L_2$-accuracy of score estimators. This is evident in the assumptions adopted by references 
\citep{chen2023improved,conforti2023score,
	lee2022convergence,lee2023convergence,
benton2023linear,li2023towards,gao2023wasserstein,jiao2024latent}.
Rigorous theoretical investigations into score estimation itself, however, remain limited.
Among the few studies that address this gap, 
\citep[Theorem 5]{wang2021deep} establishes the consistency of score estimation. 
\citep[Theorem 3.1]{oko2023diffusion} and \citep[Theorem 2]{chen2023score} further quantify its convergence rates. These works, though, are confined entirely to a continuous-time setting. They ignore time discretization effects and assume integrals over $[T_1, T-T_1]$ are tractable.
In contrast, our analysis explicitly accounts for time discretization and provides a convergence rate tailored to the discrete-time regime. A key strength of our result is that the bound achieves the minimax optimal rate in nonparametric statistics \citep{stone1980optimal,AB2009} up to logarithmic terms.
\end{remark}
We now derive an upper bound for 
$\mathbb{E}_{\bS_x,\bT,\bZ}
\bE_{\boldsymbol{W}}
\Big[\TV^2 \Big(\widetilde{p}_{T_1}(\cdot|\boldsymbol{W}),p_{T_1}(\cdot|\boldsymbol{W})\Big)\Big]$. 
By Theorem \ref{thm:score-error}, for any $\epsilon > 0$, there exists a constant $\Delta_\epsilon > 0$ such that, if the maximum time step satisfies 
$\max\limits_{0\leq k \leq K - 1}(t_{k+1}-t_k) \leq \Delta_\epsilon$, then
\begin{align} 
& \mathbb{E}_{\bS_x,\bT,\bZ} \left[\sum_{k=0}^{K-1}\int_{t_k}^{t_{k + 1}} \mathbb{E}_{{\boldsymbol{x}}_{t_k},\boldsymbol{W}} \Vert \widehat{\boldsymbol{s}}(T-t_k, {\boldsymbol{x}}_{t_k},{\boldsymbol{W}}) - \nabla\log p_{T-t_k}({\boldsymbol{x}}_{t_k}|\boldsymbol{W})\Vert^2  dt  \right] \notag\\
& \leq (T-2T_1) \cdot
\mathbb{E}_{\bS_x,\bT,\bZ}\left[\frac{1}{T-2T_1}\int_{T_1}^{T-T_1}\mathbb{E}_{{\boldsymbol{x}}_t,\boldsymbol{W}}\Vert\widehat{\boldsymbol{s}}(T-t,{\boldsymbol{x}}_t,\boldsymbol{W}) - \nabla\log p_{T-t}({\boldsymbol{x}}_t|\boldsymbol{W})\Vert^2  dt\right] + \epsilon \notag \\
&\lesssim n_x^{-\frac{2\alpha}{d_x+d_w + 2\alpha}}\log^{19}n_x + \epsilon. \label{eq:discrete-generalization}
\end{align}

Taking 
$\epsilon = n_x^{-\frac{2\alpha}{d_x+d_w + 2\alpha}}$
and choosing $\Delta_\epsilon=\Delta_{n_x}$, we obtain the following lemma.

\begin{lemma}\label{lem:TV-time-discrete-error}
Suppose that the conditions of Theorem \ref{thm:score-error} hold. By taking $\max\limits_{0\leq k \leq K - 1}(t_{k+1}-t_k)  = \mathcal{O}\left(\min\left\{\Delta_{n_x}, n_x^{-\frac{10\alpha(\alpha \wedge 1)+4\alpha}{(d_x+d_w + 2\alpha)(\alpha \wedge 1)}}\right\}\right)$,  we then have
$$
\mathbb{E}_{\bS_x,\bT,\bZ}
\bE_{\boldsymbol{W}}
\Big[\TV^2 \Big(\widetilde{p}_{T_1}(\cdot|\boldsymbol{W}),p_{T_1}(\cdot|\boldsymbol{W})\Big)\Big] \lesssim {n_x}^{-\frac{2\alpha}{d_x+d_w + 2\alpha}}\log^{19}{n_x}.
$$
\end{lemma}
\begin{remark}
The quantity $\Delta_{n_x}$ introduced above controls the error caused by approximating the score loss integrated over time in \eqref{eq:discrete-generalization} using the values at the left endpoints. The architecture of $\mathcal C$ gives the following uniform Lipschitz bound with respect to the time and state variables: $\mathcal L_N:=\sqrt{d_x(1+d_x)} \mathcal B^{\mathcal D+1}\mathcal W^{\mathcal D}.$
For $\epsilon=n_x^{-\frac{2\alpha}{d_x+d_w+2\alpha}}$, we can  choose
\[
\Delta_{n_x}:=\frac{N^{-2\alpha}}{(\log N)^{21}\left(\mathcal L_N^2+N^{8C_\lambda+2C_\mu}\right)}.
\]
The corresponding number of steps for a uniform Euler discretization is
\[
K:=\left\lceil\frac{T-2T_1}{\Delta_{n_x}}\right\rceil=\left\lceil (T-2T_1)(\log N)^{21}N^{2\alpha}\left(\mathcal L_N^2+N^{8C_\lambda+2C_\mu}\right)\right\rceil.
\]
This choice bounds the approximation error in \eqref{eq:discrete-generalization} by $\epsilon$.
\end{remark}

\subsubsection{Bound \texorpdfstring{$\mathbb{E}_{\bS_x,\bT,\bZ}
\bE_{\boldsymbol{W}}
\Big[\TV^2 \Big(\check{p}_{T_1}(\cdot|\boldsymbol{W}),\widetilde{p}_{T_1}(\cdot|\boldsymbol{W})\Big)\Big]$}{TV Initial Distribution Error}
}\label{sec:TV-initial-distribution-error}

In this section, we bound the error introduced by the initialization replacement   in the sampling dynamics. Given that \eqref{eq:r-tilde-model} and \eqref{eq:r-check-model} differ only with respect to their initial distributions, the data processing inequality implies
\begin{equation} \label{eq:data-processing-inequality}
\mathbb{E}_{\bS_x,\bT,\bZ}
\bE_{\boldsymbol{W}}
\Big[\TV^2 \Big(\check{p}_{T_1}(\cdot|\boldsymbol{W}),\widetilde{p}_{T_1}(\cdot|\boldsymbol{W})\Big)\Big]
\lesssim \mathbb{E}_{\boldsymbol{W}} \Big[ \mathrm{TV}^2\Big(p_{T-T_1}(\cdot|\boldsymbol{W}), \mathcal{N}(\boldsymbol{0}, \boldsymbol{I}_{d_x}) \Big) \Big].
\end{equation}
To bound this TV distance, we first analyze the KL divergence between the two multivariate distributions. By Lemma \ref{lem:kl-guass} and Pinsker's inequality, we derive the upper bound stated in the following lemma.
\begin{lemma}\label{lem:TV-initial-distribution-error}
Under Assumption \ref{apt:diffusion-density-bound}, taking $N = \lfloor {n_x}^{\frac{1}{d_x+d_w + 2\alpha}} \rfloor + 1 \lesssim {n_x}^{\frac{1}{d_x+d_w + 2\alpha}},$  we have 
$$
\mathbb{E}_{\bS_x,\bT,\bZ}
\bE_{\boldsymbol{W}}
\Big[\TV^2 \Big(\check{p}_{T_1}(\cdot|\boldsymbol{W}),\widetilde{p}_{T_1}(\cdot|\boldsymbol{W})\Big)\Big] \lesssim {n_x}^{-\frac{2\alpha}{d_x+d_w+2\alpha}}.
$$
\end{lemma}

\vspace{1em}

\subsubsection{Bound \texorpdfstring{$\bE_{\boldsymbol{W}}\Big[\TV^2 \Big(p_{T_1}(\cdot|\boldsymbol{W}),p(\cdot|\boldsymbol{W})\Big)\Big]$}{TV Time Truncation Error}
}\label{sec:TV-time-truncation-error}

In this section, we bound the time-truncation error, which arises from the early stopping 
at time $T_1$. 
This term captures the discrepancy between the intermediate distribution $p_{T_1}(\cdot|\boldsymbol{w})$ and the target distribution $p(\cdot|\boldsymbol{w})$.
For small $T_1$, the forward diffusion process only mildly perturbs the target distribution. 
Assumption~\ref{apt:diffusion-density-bound} ensures that the conditional density $p(\boldsymbol x | \boldsymbol w)$ is supported on $[0,1]^{d_x}\times[0,1]^{d_w}$, while Assumption~\ref{apt:diffusion-density-holder} gives its H\"{o}lder continuity in $\boldsymbol x$ on this domain. 
These properties allow us to control the perturbation induced by the short diffusion time, leading to the following bound.
\begin{lemma}\label{lem:TV-time-truncation-error}
Under Assumption \ref{apt:diffusion-density-bound} and \ref{apt:diffusion-density-holder}, taking $N = \lfloor n_x^{\frac{1}{d_x+d_w + 2\alpha}} \rfloor + 1 \lesssim {n_x}^{\frac{1}{d_x+d_w + 2\alpha}},$  we have 
$$
\bE_{\boldsymbol{W}}\Big[\TV^2 \Big(p_{T_1}(\cdot|\boldsymbol{W}),p(\cdot|\boldsymbol{W})\Big)\Big]\lesssim n_x^{-\frac{2\alpha}{d_x+d_w+2\alpha}}.
$$
\end{lemma}

The bounds in Lemmas \ref{lem:TV-time-discrete-error}--\ref{lem:TV-time-truncation-error} together imply the following error bound for transition kernel  estimation.
 
\begin{theorem} \label{thm:TV-diffusion-density-error}
Suppose that the conditions of 
Lemmas \ref{lem:TV-time-discrete-error}--\ref{lem:TV-time-truncation-error} 
 are satisfied. Then,
\begin{equation*}
	\mathbb{E}_{\bS_x,\bT,\bZ}\bE_{  \boldsymbol{W} }
	\Big[\TV^2 \Big(\widehat{\mathcal{P}}(\cdot|\boldsymbol{W}),\mathcal{P}(\cdot|\boldsymbol{W})\Big)\Big]
	\lesssim n_x^{-\frac{2\alpha}{2d_x+d_a + 2\alpha}} \log^{19}n_x.
\end{equation*}
\end{theorem}

Similarly, we can also derive an error bound for reward law estimation in the following corollary.
\begin{corollary}\label{cor:rl-diffusion-bound}
Under the    
conditions of Theorem \ref{thm:TV-diffusion-density-error} with the same H\"older index $\alpha$, 
we have
\begin{equation*}
\mathbb{E}_{\bS_r,\bT,\bZ} \bE_{ \boldsymbol{W} }
\Big[\TV^2 \Big(\widehat{\mathcal{R}}(\cdot|\boldsymbol{W}),\mathcal{R}(\cdot|\boldsymbol{W})\Big)\Big]
\lesssim n_r^{-\frac{2\alpha}{1+d_x+d_a+ 2\alpha}} \log^{19}n_r.
\end{equation*}
\end{corollary}

\begin{remark}
Theorem \ref{thm:TV-diffusion-density-error} and Corollary \ref{cor:rl-diffusion-bound} establish an end-to-end theoretical guarantee for diffusion-based learning of the transition kernel and reward law. Our result strengthens earlier studies that do not provide end-to-end guarantees, including \cite{chen2023improved, conforti2023score, lee2022convergence, lee2023convergence, benton2023linear, li2023towards, gao2023wasserstein}, and improves upon end-to-end analyses based on Lipschitz-continuous score networks, such as \cite{chen2023score, jiao2025model}, which typically yield suboptimal rates. It is also comparable to the guarantees established in \cite{oko2023diffusion, jiao2024latent, fu2024unveil, chang2026deep}; in particular, unlike \cite{fu2024unveil}, our analysis explicitly accounts for the numerical discretization error of the reverse process.
\end{remark}

By \eqref{eq:operator-hat-Q} and \eqref{eq:operator-Q*}, Theorem \ref{thm:TV-diffusion-density-error}, and Corollary \ref{cor:rl-diffusion-bound}, 
we obtain the following error bound for optimal Bellman operator learning.
\begin{theorem}\label{lem:rl-diffusion-bound}
Suppose that the conditions of Theorem~\ref{thm:TV-diffusion-density-error} and Corollary~\ref{cor:rl-diffusion-bound} hold and that the ReLU DNN class $\mathcal Q$ has $\mathcal S=\mathcal O\big(n^{\frac{d_x+d_a}{d_x+d_a+2\alpha}}\log^5 n\big)$, $\mathcal D=\mathcal O(\log n)$,  $\mathcal W=\mathcal O\big(n^{\frac{d_x+d_a}{d_x+d_a+2\alpha}} \log^3 n\big)$, and $\mathcal B=\mathcal O\big(n^{\frac{d_x+d_a}{d_x+d_a+2\alpha}}\big)$. Then,
\begin{align*}
\Ebb_{\bD,\bS,\bT,\bZ} 
\left[\|\mathcal T^*\widehat{Q}-\widehat{\mathcal T}^*\widehat{Q}\|_{L^2(\mu)}^2 
\right]
&\lesssim R^2_{\max} n_r^{-\frac{2\alpha}{1+d_x+d_a + 2\alpha}}\log^{19}n_r+\frac{ R^2_{\max}}{(1-\gamma)^2} n_x^{-\frac{2\alpha}{2d_x+d_a + 2\alpha}} \log^{19}n_x\\
&~~ +\frac{R_{\max}^2}{(1-\gamma)^2}n^{\frac{d_x+d_a}{d_x+d_a+2\alpha}}\frac{\log^6 n(\log m+\log n)}{m},\\
\Ebb_{\bD,\bS,\bT,\bZ}
\left[
\|\widehat{\mathcal T}^*Q^*-\mathcal T^*Q^*\|_{L^2(\mu)}^2
\right]
&\lesssim R^2_{\max} n_r^{-\frac{2\alpha}{1+d_x+d_a + 2\alpha}}\log^{19}n_r+\frac{ R^2_{\max}}{(1-\gamma)^2} n_x^{-\frac{2\alpha}{2d_x+d_a + 2\alpha}} \log^{19}n_x\\
&~~  +\frac{R_{\max}^2}{(1-\gamma)^2m}.
\end{align*}
\end{theorem}

\subsection{Error Bound for Value Function Learning}\label{sec:eb2}
In this section, we derive an upper bound on the value function
estimation error by controlling Part III in the preceding decomposition. Part III contains a statistical error term $\sup\limits_{Q\in\mathcal Q}\left[\mathcal{L}_{\widehat{\mathcal T}^*}(Q)-2\widehat{\mathcal{L}}_{\widehat{\mathcal T}^*}(Q)\right]$ and an approximation error term $\inf\limits_{Q\in\mathcal Q}\|Q-Q^*\|_\infty^2$. We bound these two terms separately below.

\paragraph{Bound statistical error 
$\sup\limits_{Q\in \cQ}\Big(\cL_{\widehat{\mathcal{T}}^*}(Q)
	-2\widehat{\cL}_{\widehat{\mathcal{T}}^*}(Q)\Big).$
} We bound this term using the offset Rademacher complexity technique \citep{liang2015learning}. 
Conditionally on the learned Bellman operator $\widehat{\cT}^*$, for each $Q\in\cQ$, define 
\begin{equation}\label{eq:lwhat-Q-Z}
\ell_{Q}(X, A):=\left|\widehat{\mathcal T}^* Q(X,A)- Q(X,A)\right|^2.
\end{equation}
The empirical offset Rademacher complexity of $ \cQ$ is defined as
\begin{equation*}
\mathfrak{R}_n^{\text{off}}(\mathcal{Q},\eta  |  \mathbb{D}): =\mathbb{E}_{\tau  }\left[\sup_{ Q \in  \cQ}\Big(\frac{1}{n}\sum_{i = 1}^n\tau_i \ell_{ Q}(X_i, A_i) -\frac{\eta}{n}\sum_{i = 1}^n\ell_{ Q} ^2(X_i, A_i)\Big) \Big| \bD\right],
\end{equation*}
where $\eta>0$ and $\tau:=\{\tau_i\}_{i=1}^n$ are i.i.d. Rademacher random variables satisfying
$\mathbb{P}(\tau_i=1)=\mathbb{P}(\tau_i=-1)=\frac{1}{2}$. 
The corresponding offset Rademacher complexity is
\begin{align*}
\mathfrak{R}_n^{\text{off}}(\mathcal{Q},\eta):=\mathbb{E}_{\bD}\mathfrak{R}_n^{\text{off}}(\mathcal{Q},\eta  |  \mathbb{D}) =\mathbb{E}_{\bD,\tau}
\left[\sup_{ Q \in  \cQ}\Big(\frac{1}{n}\sum_{i = 1}^n\tau_i\ell_{ Q} (X_i, A_i)-\frac{\eta}{n}\sum_{i = 1}^n\ell^2_{ Q}(X_i,A_i)\Big)\right].   
\end{align*}
By the empirical covering number and Hoeffding's inequality (see Lemma \ref{lem:bernstein-hoeffding-inequality}), we can obtain 
\begin{equation} \label{eq:offset-bound}
\mathfrak{R}_n^{\text{off}}(\mathcal{Q},\eta)\leq \frac{1+\log \mathcal{N}(\cQ, \delta, \|\cdot\|_\infty)}{2\eta n}+\frac{4R_{\max}(1+\gamma)}{1-\gamma}(1 + 2\eta {{C}})\delta,
\end{equation}
where $C:=\frac{4R^2_{\max}}{(1-\gamma)^2}$. In our analysis, we transform the statistical error into the offset Rademacher complexity for estimation purposes, yielding
\begin{equation*}
\mathbb{E}_{\bD}\sup_{Q\in \cQ}\Big(\cL_{\widehat{\mathcal{T}}^*}( Q)-2\widehat{\cL}_{\widehat{\mathcal{T}}^*}( Q) \Big)\leq \mathbb{E}_{\bD,\tau}\sup_{ Q \in  \cQ}\left(\frac{3}{n}\sum_{i = 1}^n\tau_i\ell_{ Q} (X_i, A_i)-\frac{1}{{{C}} n}\sum_{i = 1}^n\ell^2_{ Q}(X_i,A_i)\right).
\end{equation*}
From these and choosing $\eta=\frac{1}{3C}$ in \eqref{eq:offset-bound}, 
we can obtain the statistical error bound, as shown in the following lemma.
\begin{lemma}
\label{lem:LQ-sta}

Assume that the ReLU DNNs class $\mathcal{Q}$ is set with the size $\mathcal{S} = \mathcal{O}\big(n^{\frac{d_x+d_a}{d_x+d_a + 2\alpha}} \log^5 n\big)$,  depth $\mathcal{D} = \mathcal{O}(\log n)$, $\mathcal{W}=\mathcal{O}(n^{\frac{d_x+d_a}{d_x+d_a + 2\alpha}} \log^3 n)$, and  weights bound $\mathcal{B} = \mathcal{O}\big(n^{\frac{d_x+d_a}{d_x+d_a + 2\alpha}}\big)$, then the statistical error satisfies
\begin{equation}\label{eq:RL-statistical-error}
	\mathbb{E}_{\bD}\sup_{Q\in \cQ}\Big(\cL_{\widehat{\mathcal{T}}^*}( Q)-2\widehat{\cL}_{\widehat{\mathcal{T}}^*}(Q) \Big)\lesssim \frac{R^2_{\max}}{(1-\gamma)^2}n^{-\frac{2\alpha}{d_x+d_a+2\alpha}}\log^4 n.
\end{equation}
\end{lemma}
\begin{remark}
To derive the statistical error bound, we assume that the batch data are i.i.d., following the standard setting in the error analysis of DRL documented in \cite{jiang2016doubly,fan2020theoretical,uehara2022review}.
In our analysis, we use the offset Rademacher complexity technique, which yields a fast statistical rate of order
$
\frac{
\log \mathcal N(\cQ,\frac{1}{n},\|\cdot\|_\infty)
}{n}.
$
Moreover, this i.i.d. assumption can be extended to dependent data by taking into account the sequential structure of MDPs. In that case, a standard approach imposes suitable mixing conditions to establish the corresponding statistical error bound, as in \cite{feng2023over,jiao2025deep}. 
\end{remark}

\paragraph{Bound approximation error
$\inf\limits_{Q\in\mathcal Q}\|Q-Q^*\|_\infty^2$.}

This error can be controlled using ReLU approximation theory for H\"older classes under the following Lemma \ref{lem:Q-smoothness}. 
Then, the desired bound  for this approximation error follows from the ReLU approximation result in \citet{jiao2025deep}.  
See Appendix \ref{sec:rl-approximation}.

The error decomposition in Lemma~\ref{lem:error-dec} and the form of this approximation term together explain
why our analysis does not require Bellman completeness. In DFQI, DAPI, and related recursive value learning methods, the approximation error typically involves the Bellman updates of a function class. Typical examples include \(\inf\limits_{Q\in\mathcal Q}\|\mathcal T^*Q-Q^*\|_\infty\), \(\inf\limits_{Q\in\mathcal Q}\|\mathcal T^\pi Q-Q^\pi\|_\infty\), and the inherent Bellman error \(\sup\limits_{f\in\mathcal F}\inf\limits_{g\in\mathcal F}\|\mathcal T^*f-g\|_\infty\), where \(\mathcal Q\) denotes the DNN class and \(\mathcal F\) denotes a candidate value function class \citep{munos2008finite}. 
Controlling these terms usually requires that the Bellman updates remain in the same function class or can be well approximated by it. This is the role of Bellman completeness.
In contrast, our approximation error reduces to $\inf\limits_{Q\in\mathcal Q}\|Q- Q^*\|_\infty^2$, so it is controlled directly by the approximation capacity of the DNN class for $Q^*$, as formalized in Lemma~\ref{lem:Q-smoothness}.

\begin{lemma}\label{lem:Q-smoothness}
Suppose that the conditional reward and transition densities satisfy Assumption~\ref{apt:diffusion-density-holder}. Then the optimal action-value function $Q^*$ belongs to the H\"older class $\mathcal H^\alpha$.
\end{lemma}

Consequently, the above statistical and approximation error analyses yield the following error bound for value function learning.
\begin{theorem}\label{lem:RL-error}
Suppose that the conditions of Lemma \ref{lem:Q-smoothness} hold and the ReLU DNNs
class $\mathcal{Q}$ is configured
with $\mathcal{S} = \mathcal{O}\big(n^{\frac{d_x+d_a}{d_x+d_a + 2\alpha}} \log^5 n\big)$, 
$\mathcal{W}=\mathcal{O}(n^{\frac{d_x+d_a}{d_x+d_a + 2\alpha}}\log ^3 n)$, $\mathcal{D} = \mathcal{O}(\log n)$, 
and $\mathcal{B} = \mathcal{O}\big(n^{\frac{d_x+d_a}{d_x+d_a + 2\alpha}}\big)$. 
Then,
\begin{equation*}
\mathbb E_\bD\sup_{Q\in\mathcal Q}\left[\mathcal L_{\widehat{\mathcal T}^{*}}(Q)-2\widehat{\mathcal L}_{\widehat{\mathcal T}^{*}}(Q)\right]+8(1+\gamma)^2\inf_{Q\in\mathcal Q}\|Q-Q^*\|_\infty^2\lesssim \frac{R^2_{\max}}{(1-\gamma)^2} n^{-\frac{2\alpha}{d_x+d_a+2\alpha}} \log^7 n.
\end{equation*}
\end{theorem}

\begin{remark}
Under the H\"older smoothness assumption \(Q^*\in\mathcal H^\alpha\), Theorem~\ref{lem:RL-error} yields an error bound of order \(\widetilde{\mathcal O}\bigl(n^{-\frac{2\alpha}{d_x+d_a+2\alpha}}\bigr)\). The H\"older condition may be replaced by other regularity assumptions. For instance, suppose that \(Q^*\) belongs to a Besov class with mixed smoothness \(\alpha\). Let the ReLU DNN class \(\mathcal Q\) have \(\mathcal S=\widetilde{\mathcal O}\bigl(n^{\frac{1}{2\alpha+1}}\bigr)\), \(\mathcal D=\mathcal O(\log n)\), \(\mathcal W=\widetilde{\mathcal O}\bigl(n^{\frac{1}{2\alpha+1}}\bigr)\), and \(\mathcal B=n^{\mathcal O(1)}\). The same analysis yields the convergence rate \(\widetilde{\mathcal O}\bigl(n^{-\frac{2\alpha}{2\alpha+1}}\bigr)\) \citep{montanelli2019new,suzuki2019adaptivity}. Alternatively, suppose that \(Q^*\) satisfies a spectral Barron condition. Let \(\mathcal Q\) have \(\mathcal S=\widetilde{\mathcal O}\bigl((d_x+d_a)n^{\frac{d_x+d_a}{2d_x+2d_a+2}}\bigr)\), \(\mathcal D=\mathcal O(1)\), \(\mathcal W=\widetilde{\mathcal O}\bigl(n^{\frac{d_x+d_a}{2d_x+2d_a+2}}\bigr)\), and \(\mathcal B=\mathcal O(1)\). The resulting convergence rate is \(\widetilde{\mathcal O}\bigl(n^{-\frac{d_x+d_a+2}{2d_x+2d_a+2}}\bigr)\) \citep{klusowski2018approximation,ma2022uniform}.
\end{remark}

\subsection{Main Results}\label{sec:mr}

This section presents our main theoretical results, with emphasis on the convergence rate of 
the estimator $\widehat Q$. 
By the error decomposition of Lemma \ref{lem:error-dec},  Theorem~\ref{lem:rl-diffusion-bound}, and Theorem~\ref{lem:RL-error}, we first establish the excess risk bound in the following theorem.

\begin{theorem}\label{thm:main-result}
Suppose that the conditions of Theorems~\ref{lem:rl-diffusion-bound} and~\ref{lem:RL-error} hold. Then,
\begin{align*}
\mathbb E_{\bD,\bS,\bT,\bZ}\Big[\cL_{\mathcal T^*}(\widehat Q)-\cL_{\mathcal T^*}(Q^*)\Big]
&\lesssim R_{\max}^2n_r^{-\frac{2\alpha}{1+d_x+d_a+2\alpha}}\log^{19}n_r+\frac{R_{\max}^2}{(1-\gamma)^2}n_x^{-\frac{2\alpha}{2d_x+d_a+2\alpha}}\log^{19}n_x\\
&~~+\frac{R_{\max}^2}{(1-\gamma)^2}n^{\frac{d_x+d_a}{d_x+d_a+2\alpha}}\frac{\log^6n(\log m+\log n)}{m}\\
&~~+\frac{R_{\max}^2}{(1-\gamma)^2}n^{-\frac{2\alpha}{d_x+d_a+2\alpha}}\log^7n.
\end{align*}
The following rates hold under the corresponding conditions.
\begin{enumerate}
\item\label{item:main-rate-n}
If $n_r\gtrsim(1-\gamma)^{\frac{1+d_x+d_a+2\alpha}{\alpha}}n^{\frac{1+d_x+d_a+2\alpha}{d_x+d_a+2\alpha}}$, $n_x\gtrsim n^{\frac{2d_x+d_a+2\alpha}{d_x+d_a+2\alpha}}$, and $m\gtrsim n$, then
$$
\mathbb E_{\bD,\bS,\bT,\bZ}\Big[\cL_{\mathcal T^*}(\widehat Q)-\cL_{\mathcal T^*}(Q^*)\Big]\lesssim\frac{R_{\max}^2}{(1-\gamma)^2}n^{-\frac{2\alpha}{d_x+d_a+2\alpha}}\log^{19}n.
$$

\item\label{item:main-rate-total}
Let $\widetilde n$ be the total offline sample size and set $n=\omega_n\widetilde n$, $n_r=\omega_r\widetilde n$, and $n_x=\omega_x\widetilde n$, where $\omega_n,\omega_r,\omega_x>0$ and $\omega_n+\omega_r+\omega_x=1$. If $m\gtrsim n$, then
$$
\mathbb E_{\bD,\bS,\bT,\bZ}\Big[\cL_{\mathcal T^*}(\widehat Q)-\cL_{\mathcal T^*}(Q^*)\Big]\lesssim\frac{R_{\max}^2}{(1-\gamma)^2}\widetilde n^{-\frac{2\alpha}{2d_x+d_a+2\alpha}}\log^{19}\widetilde n.
$$

\end{enumerate}
\end{theorem}

\begin{remark}
Theorem \ref{thm:main-result} provides convergence rates for the excess risk of $\widehat Q$, and it consists of four terms. 
In particular, the terms $R^2_{\max} n_r^{-\frac{2\alpha}{1+d_x+d_a+2\alpha}} \log^{19} n_r$, $\frac{ R^2_{\max}}{(1-\gamma)^2} n_x^{-\frac{2\alpha}{2d_x+d_a+2\alpha}} \log^{19} n_x$,
and $\frac{R_{\max}^2}{(1-\gamma)^2}n^{\frac{d_x+d_a}{d_x+d_a+2\alpha}}\frac{\log^6 n(\log m+\log n)}{m}$
are all attributable to the conditional diffusion estimation error characterized in Theorem \ref{lem:rl-diffusion-bound}. The remaining term $\frac{R^2_{\max}}{(1-\gamma)^2} n^{-\frac{2\alpha}{d_x+d_a+2\alpha}} \log^7 n$ is attributable to the approximation and statistical errors arising in the estimation of $Q^*$ as established in Lemma \ref{lem:LQ-sta}.

Under the sample size conditions in part~(\ref{item:main-rate-n}), we obtain the polynomial rate $n^{-\frac{2\alpha}{d_x+d_a+2\alpha}}$. Under the fixed-budget allocation in part~(\ref{item:main-rate-total}), the corresponding rate is $\widetilde n^{-\frac{2\alpha}{2d_x+d_a+2\alpha}}$. Since the proposed procedure first learns the conditional reward distribution and transition kernel and then constructs the Bellman operator, its convergence rate reflects the statistical complexity of estimating the underlying environment dynamics. In particular, $Q^*$ is defined on the state--action space $(x,a)$ of dimension $d_x+d_a$, whereas transition distribution estimation involves the joint space $(x,a,x')$ of dimension $2d_x+d_a$. In part~(\ref{item:main-rate-n}), sufficiently many reward and transition samples control the environment estimation errors at the rate determined by value function learning. In part~(\ref{item:main-rate-total}), all three sample sizes are proportional to the total budget, and the transition distribution estimation term governs the bound. This explains the different dimension dependence of the two rates.
\end{remark}

Theorem \ref{thm:main-result} controls the  Bellman residual  under
the sampling distribution \(\mu\). To translate this residual risk into a direct
estimation error of $\widehat Q$, we need to compare the sampling
distribution \(\mu\) with the target distribution \(\rho\) under which the
value function error is evaluated. This motivates the following 
concentrability condition
\citep{chen2019information,xie2020q,fan2020theoretical}.
\begin{assumption}[Concentrability coefficient]
\label{apt:concentrability}
For $\delta>0$ and $Q\in\cQ$, define
\[
G_{Q,\delta}(x):=\left\{a\in\cA:\max\{Q(x,a),Q^*(x,a)\}\geq\max_{\widetilde a\in\cA}\max\{Q(x,\widetilde a),Q^*(x,\widetilde a)\}-\delta\right\}.
\]
Assume that $G_{Q,\delta}(x)$ has positive Lebesgue measure for every $x\in\cX$, and set
\[
\pi_{Q,\delta}(B| x):=\frac{\int_B\mathbf 1\{a\in G_{Q,\delta}(x)\} \mathrm da}{\int_{\cA}\mathbf 1\{a\in G_{Q,\delta}(x)\} \mathrm da}.
\]
Let $\mathfrak M_{\delta}$ be the smallest class containing $\mathfrak M$ and closed under $\nu\mapsto\cP(\nu)\otimes\pi_{Q,\delta}$ for $Q\in\cQ$, where $\cP(\nu)\otimes\pi_{Q,\delta}$ is the law of $(X',A')$ under $(X,A)\sim\nu$, $X'\sim\cP(\cdot| X,A)$, and $A'\sim\pi_{Q,\delta}(\cdot| X')$. For every $\delta>0$, there exists $ C_\delta<\infty$ such that every $\nu\in\mathfrak M_\delta$ satisfies
$\nu\ll\mu,~ \left\|\frac{\mathrm d\nu}{\mathrm d\mu}\right\|_\infty\leq C_\delta.$
\end{assumption}
Assumption \ref{apt:concentrability} can be viewed as the continuous version of the concentrability condition in \citet{chen2019information}. Under this assumption, 
for every $\delta>0$ and $\rho\in\mathfrak M$, we have 
$
\|\widehat Q-Q^*\|_{L^2(\rho)}\leq\frac{\sqrt{ C_\delta}}{1-\gamma}\|\widehat Q-\cT^*\widehat Q\|_{L^2(\mu)}+\frac{\gamma\delta}{1-\gamma}.
$
Consequently, 
we can establish the following value function estimation error bound.

\begin{theorem}\label{thm:value-est}
Suppose that Assumption~\ref{apt:concentrability} and the conditions of Theorem~\ref{thm:main-result} hold. Then, for every $\delta>0$ and $\rho\in\mathfrak M$,
\[
\mathbb E_{\bD,\bS,\bT,\bZ}\Big[\|\widehat Q-Q^*\|_{L^2(\rho)}\Big]\leq\frac{\sqrt{C_\delta}}{1-\gamma}\mathbb E_{\bD,\bS,\bT,\bZ}\Big[\|\widehat Q-\mathcal T^*\widehat Q\|_{L^2(\mu)}\Big]+\frac{\gamma\delta}{1-\gamma}.
\]
Assume $C_\delta\leq\widetilde C<\infty$ uniformly in sample size for the choices of $\delta$ below.
\begin{enumerate}
\item\label{item:value-rate-n}
Under the conditions of Theorem~\ref{thm:main-result}(\ref{item:main-rate-n}), taking $\delta=n^{-\frac{\alpha}{d_x+d_a+2\alpha}}$ gives
$$
\mathbb E_{\bD,\bS,\bT,\bZ}\Big[\|\widehat Q-Q^*\|_{L^2(\rho)}\Big]\lesssim\frac{R_{\max}\sqrt{\widetilde C}}{(1-\gamma)^2}n^{-\frac{\alpha}{d_x+d_a+2\alpha}}\log^{\frac{19}{2}}n.
$$

\item\label{item:value-rate-total}
Under the conditions of Theorem~\ref{thm:main-result}(\ref{item:main-rate-total}), taking $\delta=\widetilde n^{-\frac{\alpha}{2d_x+d_a+2\alpha}}$ gives
$$
\mathbb E_{\bD,\bS,\bT,\bZ}\Big[\|\widehat Q-Q^*\|_{L^2(\rho)}\Big]\lesssim\frac{R_{\max}\sqrt{\widetilde C}}{(1-\gamma)^2}\widetilde n^{-\frac{\alpha}{2d_x+d_a+2\alpha}}\log^{\frac{19}{2}}\widetilde n.
$$
\end{enumerate}
\end{theorem}

\begin{remark} 
While Theorem \ref{thm:main-result} characterizes the excess risk convergence of $\widehat Q$, Theorem \ref{thm:value-est} directly establishes an   error bound for the value function, building on the concentrability assumption and Theorem \ref{thm:main-result}. The detailed proof is provided in Appendix \ref{sec:main-results}.
After ignoring logarithmic orders and other constants, our method attains the  oracle value-stage rate $\cO\big(n^{-\frac{\alpha}{d_x+d_a+2\alpha}}\big)$ in Theorem \ref{thm:value-est},
which outperforms the convergence rate of MABO \citep{xie2020q}, given by  $\cO (n^{-\frac{\alpha}{2d_x+2d_a+4\alpha}})$. Moreover, compared with existing DRL convergence results, our convergence rate is also superior. For instance,
\cite{feng2023over} obtained a convergence rate $\cO\big(n^{-\frac{\alpha}{2d_x+2d_a+4\alpha+2}}\big)$, while \cite{jiao2025deep} achieved
$\cO\big(n^{-\frac{\alpha}{d_x+d_a+4\alpha}}\big)$. 

We highlight that our analysis avoids the Bellman completeness assumption commonly used in existing DRL theory
\citep{chen2019information,fan2020theoretical,feng2023over,jiao2025deep}. In fitted value iteration, approximate policy iteration, and minimax Bellman error methods, the approximation error is often tied to the Bellman updates of the working function class. This is commonly captured by the inherent Bellman error $\sup\limits_{g\in\cF}\inf\limits_{f\in\cF}\|f-\mathcal T^\ast g\|_{L^2(\mu)}$, where $\cF$ denotes a candidate value function class \citep{munos2008finite}. Controlling this error requires the Bellman image $\mathcal T^\ast\cF$ to be contained in the same function class or to be well approximated by it. This requirement is precisely the role of Bellman completeness. 
In contrast, our approximation error is controlled directly by
approximating \(Q^\ast\), the fixed point of the optimal Bellman
operator \(\mathcal T^\ast\). It suffices to assume that $Q^*$ is H\"older continuous. Bellman completeness of the whole function class is not required.
\end{remark}

Having established upper bounds for the proposed estimator, we now derive an information-theoretic lower bound for estimating $Q^*$ under the same total offline sample budget. This bound applies to all estimators based on the available data, without requiring intermediate estimation of the conditional reward and transition distributions.

\begin{theorem}[Information-theoretic lower bound]
\label{thm:Q-minimax-lower}
Let $M$ denote an MDP and $\mathfrak E_\alpha$ the class of MDPs satisfying Assumptions~\ref{apt:diffusion-density-bound} and~\ref{apt:diffusion-density-holder} for both the conditional transition and reward densities. Suppose that $\mu=\rho=\operatorname{Unif}(\cX\times\cA)$, then, 
\[
\inf_{\widehat Q}\sup_{M\in\mathfrak E_\alpha}\mathbb E_M\|\widehat Q-Q_M^*\|_{L^2(\rho)}\gtrsim\widetilde n^{-\frac{\alpha}{d_x+d_a+2\alpha}}.
\]
Here, $Q_M^*$ denotes the optimal action-value function under model $M$.
The infimum is taken over all possibly randomized measurable estimators
$\widehat Q$ with respect to
$\sigma(\mathbb S_r,\mathbb S_x,\bD)$, and $\mathbb E_M$ is taken with
respect to the joint distribution of the observations and any estimator
randomization induced by $M$.
\end{theorem}

\begin{remark}
Theorem~\ref{thm:Q-minimax-lower} provides an information-theoretic lower bound for estimating the optimal action-value function over $\mathfrak E_\alpha$. The corresponding minimax problem allows arbitrary estimators of $Q^*$ and does not require intermediate recovery of the conditional reward and transition distributions. Thus, the lower bound applies to value function estimation itself, including procedures that first estimate these distributions.

Under the same total offline sample budget, the lower bound has rate $\widetilde n^{-\frac{\alpha}{d_x+d_a+2\alpha}}$, whereas the upper bound in Theorem~\ref{thm:value-est}(\ref{item:value-rate-total}) has rate $\widetilde n^{-\frac{\alpha}{2d_x+d_a+2\alpha}}$ up to logarithmic factors. The upper bound reflects the proposed procedure's estimation of the conditional environment distributions before Bellman optimization. In particular, its transition estimation term involves the joint space $(x,a,x')$ of dimension $2d_x+d_a$, while the lower bound is constructed through perturbations on the state--action space $(x,a)$ of dimension $d_x+d_a$. These results distinguish the guarantee for the proposed procedure from the lower bound over all estimators.
\end{remark}

\section{Numerical Experiments}\label{sec:experiment}

In this section, we evaluate the empirical performance of the proposed method on both synthetic benchmark environments and real-world offline RL tasks.   In Section~\ref{sec:exp-openai}, we first consider OpenAI Gymnasium environments, including Pendulum and Swimmer. In Section~\ref{sec:exp-mimic-iii}, we then examine the method on the real-world MIMIC-III clinical database. We compare the proposed method with three representative offline RL baselines: MOPO \citep{yu2020mopo}, DFQI \citep{fan2020theoretical,feng2023over}, and MABO \citep{xie2020q}.

\subsection{OpenAI Gymnasium Environments}\label{sec:exp-openai}

We  consider two continuous control benchmarks from Gymnasium, Pendulum and Swimmer (\url{https://gymnasium.farama.org/}). Pendulum is a classic low-dimensional control task in which the agent applies continuous torque to swing a pendulum upward and keep it balanced near the upright position. It mainly tests whether an algorithm can learn smooth and stable control from continuous actions. Swimmer is a MuJoCo locomotion task in which a multi-link body must coordinate joint torques to move forward in a fluid-like environment. Compared with Pendulum, Swimmer has higher-dimensional dynamics and requires coordinated control across multiple joints. These two environments therefore provide complementary tests of offline continuous-control performance.

For both tasks, we construct the offline dataset by rolling out a random policy for 200 episodes. The baseline methods are trained on the full offline dataset. Our method uses 20,000 samples for diffusion model training and another 20,000 samples to estimate \(Q^\ast\) through
empirical Bellman residual minimization. We train all methods for 300 epochs and evaluate each learned policy over 50 episodes. All reported results are averaged over three seeds. In the training curves, the solid line denotes the mean training reward over the three seeds, and the shaded region represents the minimum-to-maximum range across these seeds.

\begin{figure}[H]
    \centering
    \begin{minipage}{0.5\linewidth}
        \centering
        \includegraphics[width=0.7\linewidth]{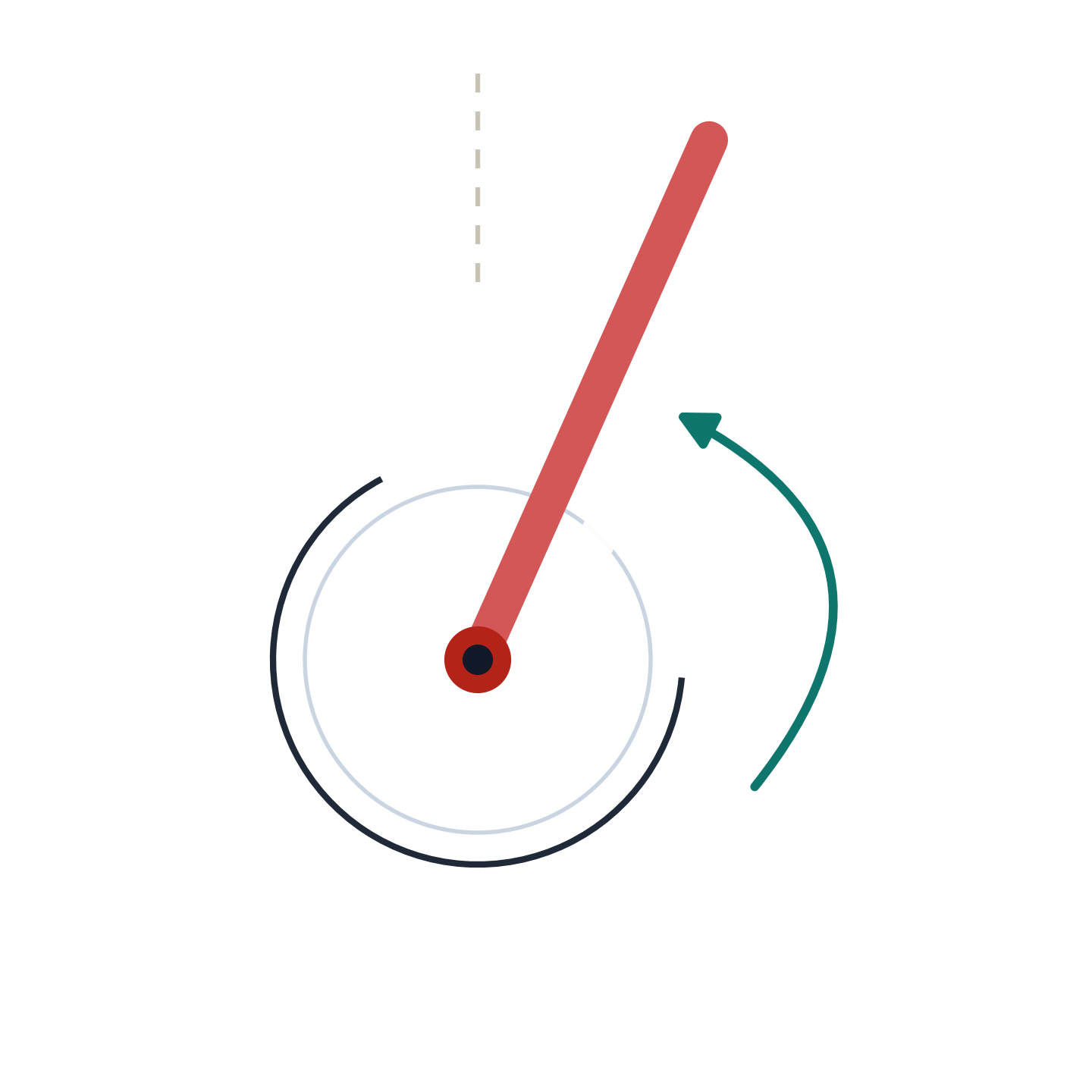}
    \end{minipage}\hfill
    \begin{minipage}{0.5\linewidth}
        \centering
        \includegraphics[width=0.7\linewidth]{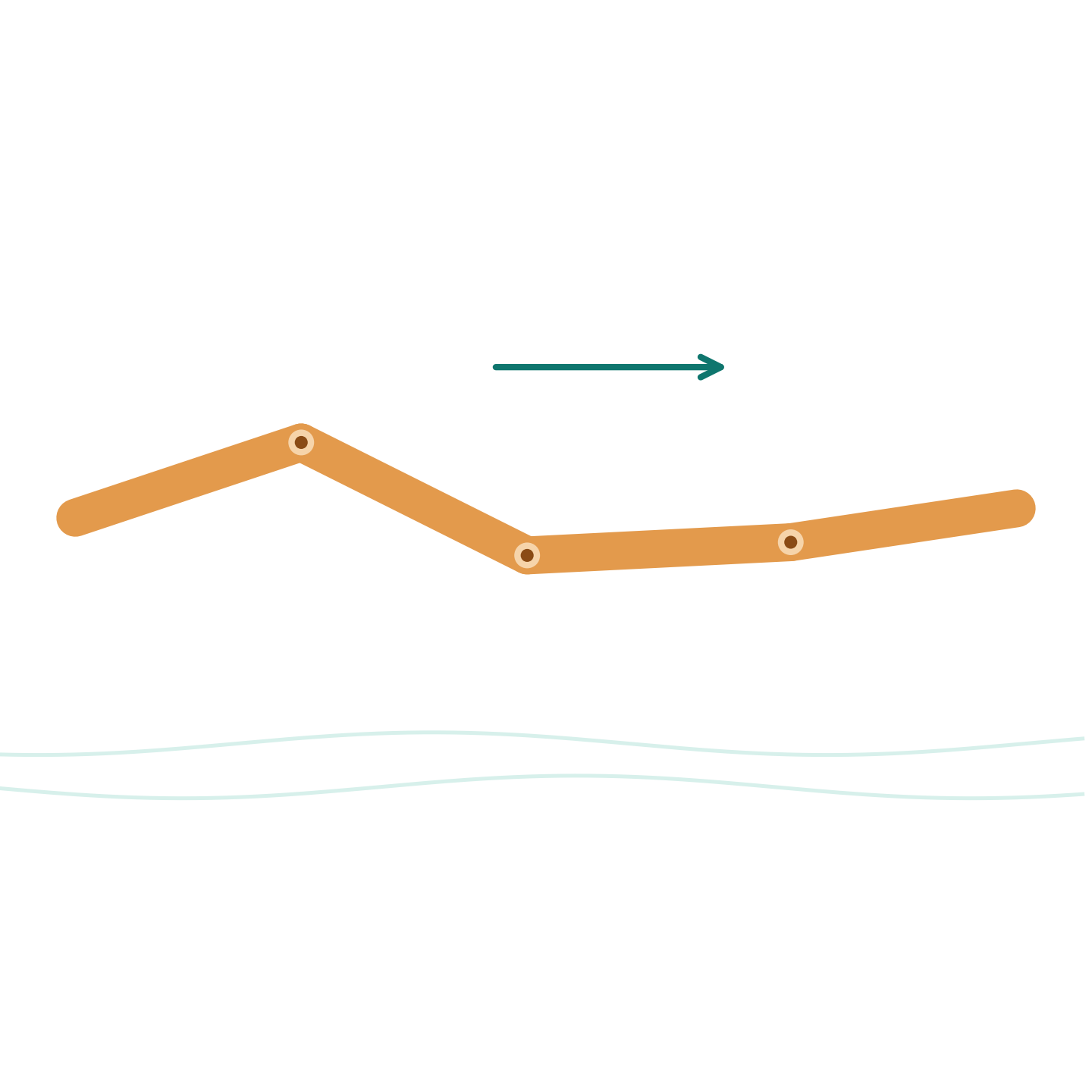}
    \end{minipage}
    \caption{
    Illustrations of the benchmark environments:
    Pendulum (Left) and 
    Swimmer (Right).}
    \label{fig:benchmark_envs}
\end{figure}

\paragraph{Pendulum.}
In the Pendulum task, the agent generates continuous torque actions to swing the pendulum up and maintain it near the upright position. Figure~\ref{fig:pendulum_results} shows the training rewards and evaluation rewards. The training curves show that all methods improve from low initial rewards and reach a stable reward region during training. Although DFQI shows competitive training behavior, its evaluation performance is substantially lower than that of our method. MABO and MOPO also improve during training, but they obtain lower final evaluation rewards. As reported in Table~\ref{tab:gym_results}, our method achieves the best evaluation reward of $-177.06$, improving over the strongest baseline result of $-206.39$ from MABO.

\begin{figure}[H]
    \centering
    \includegraphics[width=\linewidth]{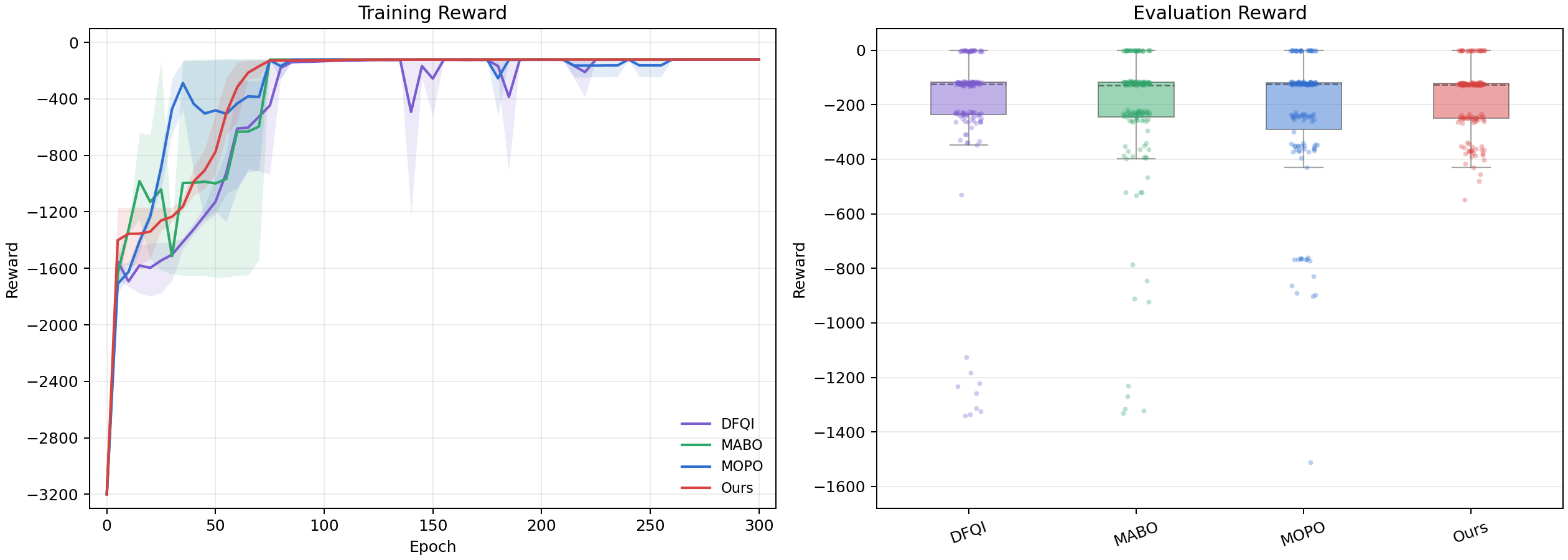}
    \caption{
    Performance comparison on the Gymnasium Pendulum task. Left: training reward over 300 epochs, where the solid line is the mean over three seeds and the shaded region is the min-max range. Right: 
    box plot of evaluation rewards over 150 returns.}
    \label{fig:pendulum_results}
\end{figure}

\paragraph{Swimmer.}
In the Swimmer task, the agent controls a multi-link body by applying torques at its joints, with the objective of producing coordinated forward movement. Figure~\ref{fig:swimmer_results} shows the corresponding results. The training curves indicate that the methods exhibit different reward patterns on this more challenging task. Our method maintains a relatively high reward level during training and achieves the best final evaluation performance. As shown in Table~\ref{tab:gym_results}, our method obtains an evaluation reward of $71.90$, while the best baseline result is $39.33$ from MABO. These results suggest that the proposed framework can remain effective in a task that requires coordinated multi-joint control.

\begin{figure}[H]
    \centering
    \includegraphics[width=\linewidth]{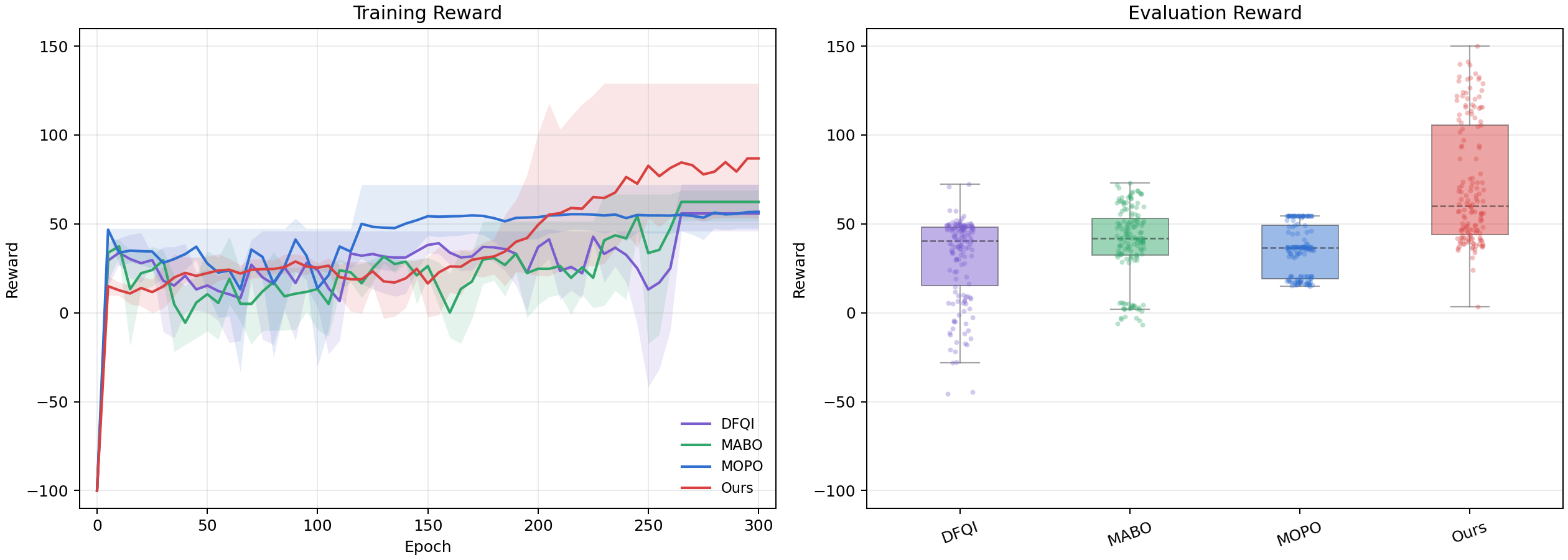}
    \caption{Performance comparison on the Gymnasium Swimmer task. Left: training reward over 300 epochs, where the solid line is the mean over three seeds and the shaded region is the min-max range. Right: box plot of evaluation rewards over 150 returns.}
    \label{fig:swimmer_results}
\end{figure}

\begin{table}[H]
    \centering
    \caption{Evaluation results on the Gymnasium environment tasks.}
    \label{tab:gym_results}
    \renewcommand{\arraystretch}{1.2}
    \begin{tabular*}{0.85\linewidth}{@{\extracolsep{\fill}}ccccc}
        \toprule
        Task & DFQI & MABO & MOPO & Ours \\
        \midrule
        Pendulum & -212.92 & -206.39 & -232.40 & \textbf{-177.06} \\
        Swimmer  & 30.93   & 39.33   & 34.98   & \textbf{71.90} \\
        \bottomrule
    \end{tabular*}
\end{table}

\subsection{Real-World Data}\label{sec:exp-mimic-iii}

We examine the proposed method on a real-world sepsis treatment task constructed from the MIMIC-III database (\url{https://physionet.org/content/mimiciii/1.4/}). MIMIC-III is a publicly available critical care clinical database constructed from anonymized electronic health records of ICU patients at Beth Israel Deaconess Medical Center \citep{johnson2016mimic}. It contains longitudinal ICU records, including demographics, vital signs, laboratory tests, medication records, fluid input and output, diagnosis codes, nursing records, and disease severity scores. These temporal records allow patient disease trajectories to be reconstructed for sequential treatment decision making.

We construct the study cohort following standard preprocessing protocols used in RL studies of sepsis treatment \citep{raghu2017deep,nanayakkara2022unifying}. Suspected infection events are identified from antibiotic administration and microbiology culture records, and organ dysfunction is assessed by the Sequential Organ Failure Assessment (SOFA) score. We retain adult ICU patients and include only the first eligible sepsis-related ICU stay for each patient. Patients are required to have a maximum SOFA score of at least 2 within 24 hours after sepsis onset, and ICU stays with severe missingness or invalid treatment records are excluded. The final processed cohort contains approximately 12,600 patients.

Each ICU stay is divided into consecutive four-hour decision windows. Clinical measurements within a window are aggregated as the current state, and treatments administered in the same window define the action. Two adjacent windows form a transition sample \((x_t, a_t, r_t, x_{t+1})\). The original state contains 41 clinical variables. We apply principal component analysis to the non-SOFA variables, retain the first 10 principal components, and concatenate the raw SOFA score, resulting in an 11-dimensional state representation. The action is defined by intravenous fluids and vasopressors, each discretized into three dosage levels, yielding 9 discrete treatment actions. The reward is the negative next-step SOFA score, \(r_t=-\mathrm{SOFA}_{t+1}\), so a larger reward indicates lower subsequent organ dysfunction. This SOFA-based reward has clear clinical interpretability and has been used in related studies \citep{wang2023bless,miao2025reinforcement}. After processing, the dataset contains approximately 215,000 transition samples.

We select 5,000 patients as the offline training set, corresponding to approximately 85,000 transitions. All compared algorithms are trained on this fixed dataset, and simulator-generated transitions are not added to the replay buffer. Since the observed data only contain outcomes under clinician actions, we train a personalized simulator following the idea of PerSim \citep{agarwal2021persim}. Given a current state and a candidate action, the simulator predicts the next state and reward. It is used only for evaluation and does not provide additional training samples.
All algorithms are trained for 300 epochs. Since early intervention is important for sepsis progression, we set \(\gamma=0.7\) to emphasize immediate and short-term treatment effects. During training, we periodically evaluate each policy by simulated rollouts from fixed held-out initial states. In the training curves, the solid line denotes the mean reward over three seeds, and the shaded region represents the minimum-to-maximum range. Final evaluation is performed on held-out patient initial states that are not used in offline training, where we perform 50 independent rollouts.

Figure~\ref{fig:mimic_results} and Table~\ref{tab:mimic_results} summarize the experimental results. The proposed method achieves the highest final reward of \(-5.33\), improving over the clinician policy with reward \(-5.59\). Since the reward is defined as the negative future SOFA score, this improvement suggests lower predicted organ dysfunction in subsequent decision windows. Our method also outperforms MOPO, MABO, and DFQI. To further interpret the relative quality of different policies, we additionally evaluate a random policy that uniformly selects from the 9 treatment actions. The random policy obtains an average reward of \(-8.27\), which is substantially lower than the rewards of the clinician policy and all learned policies. This gap indicates that the compared algorithms 
learn useful treatment structure from the offline data, while the proposed method achieves the best simulated return among them.

\begin{figure}[H]
    \centering
    \includegraphics[width=0.9\linewidth]{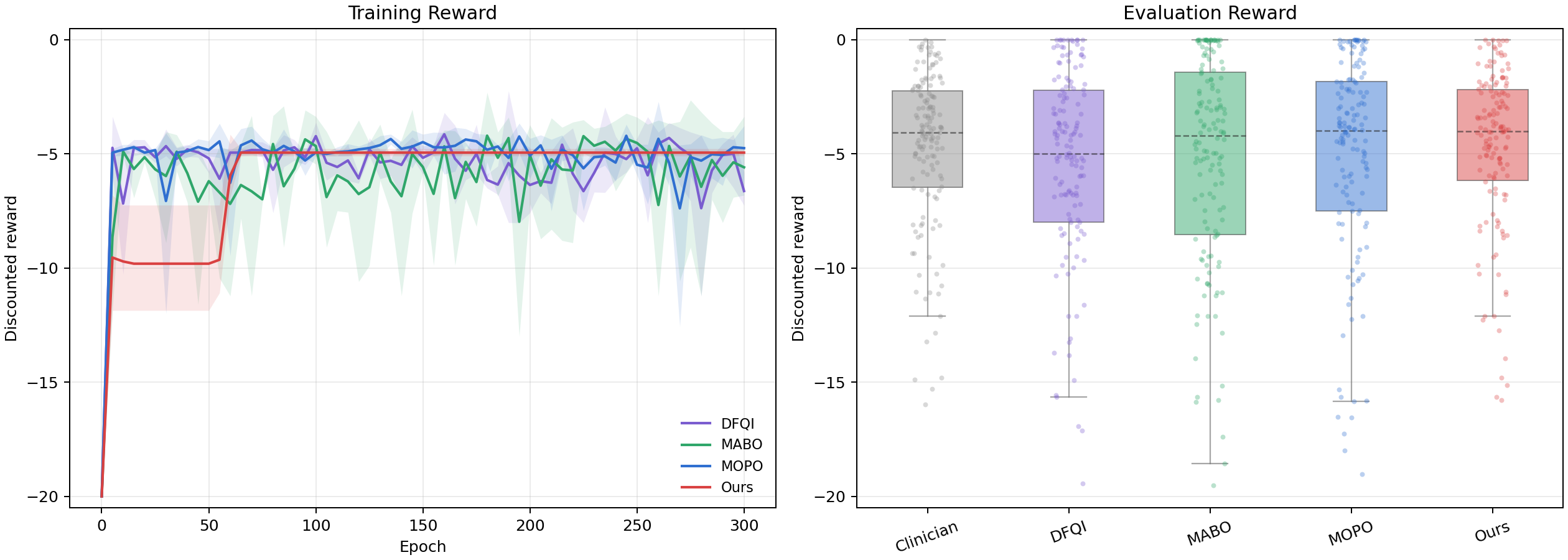}
    \caption{Performance comparison on the MIMIC-III sepsis treatment task. Left: training reward over 300 epochs, where the solid line is the mean over three seeds and the shaded region is the min-max range. Right: box plot of evaluation rewards over 150 returns.
    }
    \label{fig:mimic_results}
\end{figure}

\begin{table}[H]
    \centering
    \caption{Evaluation results on the MIMIC-III sepsis treatment task.}
    \label{tab:mimic_results}
    \renewcommand{\arraystretch}{1.2}
    \begin{tabular*}{0.85\linewidth}{@{\extracolsep{\fill}}cccccc}
        \toprule
        Method & Clinician & MABO & MOPO & DFQI & Ours \\
        \midrule
        Reward
        & -5.59
        & -5.92
        & -5.88
        & -6.69
        & \textbf{-5.33} \\
        \bottomrule
    \end{tabular*}
\end{table} 
\section{Conclusion}\label{sec:conclusion}
We develop an offline RL framework for deep $Q^\ast$ estimation based on conditional diffusion models.  
The proposed method learns the reward law and transition kernel from offline data through conditional diffusion modeling, and then incorporates the resulting Bellman operator estimator into a deep $Q^\ast$ learning procedure via Bellman residual minimization. This design cleanly separates conditional generative modeling from value learning and enables a precise analysis of how optimal Bellman operator estimation errors propagate to value estimation. 
Notably, this separation also allows our analysis to avoid the Bellman completeness assumption commonly used in existing DRL theory. 

Ignoring logarithmic factors, the theoretical analysis yields convergence rates of order $\mathcal O(n_r^{-\frac{2\alpha}{1+d_x+d_a+2\alpha}})$, $\mathcal O(n_x^{-\frac{2\alpha}{2d_x+d_a+2\alpha}})$, and $\mathcal O\bigl(\frac{n^{\frac{d_x+d_a}{d_x+d_a+2\alpha}}}{m}\bigr)$ for reward law estimation, transition kernel estimation, and the Monte Carlo error induced by approximating the learned Bellman operator, respectively. Under the stated conditions, it further gives an excess Bellman residual risk rate of order $\mathcal O(n^{-\frac{2\alpha}{d_x+d_a+2\alpha}})$, and the value function estimation error of $\widehat Q$ attains the rate $\mathcal O(n^{-\frac{\alpha}{d_x+d_a+2\alpha}})$, which matches the oracle value-stage convergence rate of the nonparametric least-squares estimator. Under the total offline budget $\widetilde n=n+n_r+n_x$, we also obtain a value function estimation rate of $\mathcal O(\widetilde n^{-\frac{\alpha}{2d_x+d_a+2\alpha}})$. Furthermore, we establish an information-theoretic lower bound of order $\widetilde n^{-\frac{\alpha}{d_x+d_a+2\alpha}}$ for estimating $Q^*$.
Numerical experiments support the theoretical findings and demonstrate the practical effectiveness of the proposed method.  
Several directions remain for future investigation. It would be of interest to extend the present framework to broader offline policy optimization problems and to more challenging settings, such as high dimensional locomotion, partially observable systems, and large-scale world-modeling problems.

\begin{appendix}
\section*{Appendix} 
In this appendix, Appendix \ref{sec:apd-bounds-on-the-score} provides the technical analysis for score function estimation. 
Appendix \ref{sec:apd-bounds-T*} establishes the optimal Bellman operator learning bounds,
Appendix \ref{sec:apd-bounds-on-rl} proves the value function learning results, Appendix \ref{sec:main-results} contains the proof of the main theorem, and Appendix \ref{sec:apd-auxiliary-results} collects auxiliary lemmas and technical approximation results used throughout the proofs.

\startcontents[appendix]
\printcontents[appendix]{l}{1}{\subsection*{Table of Contents}\setcounter{tocdepth}{2}}


\section{Bounds on the Score Function}\label{sec:apd-bounds-on-the-score}

We first outline the main idea before presenting the formal lemmas. The conditional score can be written as
\[
\nabla \log p_t(\boldsymbol x|\boldsymbol w)
=
\frac{\nabla p_t(\boldsymbol x|\boldsymbol w)}
     {p_t(\boldsymbol x|\boldsymbol w)}.
\]
Thus, it suffices to approximate both the diffused density $p_t(\boldsymbol x|\boldsymbol w)$ and its gradient. 
For the density term, the starting point is the decomposition
\[
p_t(\boldsymbol x|\boldsymbol w)
=
\underbrace{
\int_{\mathbb R^{d_x}}
}_{\text{Domain}}
~~
\underbrace{
p(\boldsymbol z|\boldsymbol w)
}_{\text{Density}}
~~
\underbrace{
\frac{1}{\sigma_t^{d_x}(2\pi)^{\frac{d_x}{2}}}
\exp\left(
-\frac{\|\boldsymbol x-\mu_t\boldsymbol z\|^2}{2\sigma_t^2}
\right)
}_{\text{Kernel}}
d\boldsymbol z.
\]
The approximation proceeds in four steps. 
\begin{itemize}
    \item {First (Domain localization).} By Lemma~\ref{lem:integral-clipping}, the integration domain can be localized from $\mathbb R^{d_x}$ to a neighborhood \( A_{{\boldsymbol{x}}} = \prod_{i=1}^{d_x} a_{i,{\boldsymbol{x}}},~ a_{i,{\boldsymbol{x}}} = \left[ \frac{{x}_i - C\sigma_t \sqrt{\log \epsilon^{-1}}}{\mu_t}, \frac{{x}_i + C\sigma_t \sqrt{\log \epsilon^{-1}}}{\mu_t} \right] \), with the clipping error controlled by the Gaussian tail.

    \item Second (Density approximation). On the localized region, we use a Taylor expansion to approximate the conditional density $p(\boldsymbol z|\boldsymbol w)$ (See Appendix~\ref{sec:approxmate-p}).

    \item Third (Kernel approximation). On the localized region, we use another Taylor expansion to approximate the Gaussian transition kernel (See Appendix~\ref{subsec:c2}).

    \item Finally (Neural network approximation). After the previous approximations, the localized integral reduces to a finite combination of polynomial-type functions, which can be approximated by ReLU DNNs (See Appendix~\ref{subsec:c3}). 
\end{itemize}

Following the work of \cite{oko2023diffusion,jiao2024latent}, we establish several high-probability bounds for the conditional density $p_t({\boldsymbol{x}}|\boldsymbol{w})$ and the corresponding score function in the following section.

\subsection{Bounds on $p_t(\boldsymbol{x}|\boldsymbol{w})$}
\begin{lemma}
\label{lem:p-bound}
Under Assumption~\ref{apt:diffusion-density-bound}, let $D_t(\boldsymbol{x}):=\max\limits_{1\le i\le d_x}\{(-x_i)_+,(x_i-\mu_t)_+\}.$
Then, for any $\boldsymbol{x}\in\mathbb{R}^{d_x}$ and
$\boldsymbol{w}\in[0,1]^{d_w}$,
\[
\exp\left(-\frac{d_xD_t(\boldsymbol{x})^2}{\sigma_t^2}\right)
\lesssim p_t(\boldsymbol{x}|\boldsymbol{w})
\lesssim \exp\left(-\frac{D_t(\boldsymbol{x})^2}{2\sigma_t^2}\right)
\le \exp\left(-\frac{(\|\boldsymbol{x}\|_\infty-\mu_t)_+^2}{2\sigma_t^2}\right).
\]
\end{lemma}

\begin{proof}
For each $i$, define
\[
q_i(x_i):=\int_0^1\frac{1}{\sigma_t\sqrt{2\pi}}
\exp\left(-\frac{(x_i-\mu_t z_i)^2}{2\sigma_t^2}\right)dz_i, ~~\Delta_i:=(-x_i)_+\vee (x_i-\mu_t)_+,
\]
By Assumption~\ref{apt:diffusion-density-bound},
\[
B_l\prod_{i=1}^{d_x}q_i(x_i)\le p_t(\boldsymbol{x}|\boldsymbol{w})
\le B_u\prod_{i=1}^{d_x}q_i(x_i).
\]
For each $i$, choose
\[
I_i=
\begin{cases}
\left[0,1\wedge\frac{\sigma_t}{\mu_t}\right], & x_i<0,\\[4pt]
\left[\frac{x_i}{\mu_t},\frac{x_i}{\mu_t}+\frac12\left(1\wedge\frac{\sigma_t}{\mu_t}\right)\right], & 0\le x_i\le\frac{\mu_t}{2},\\[4pt]
\left[\frac{x_i}{\mu_t}-\frac12\left(1\wedge\frac{\sigma_t}{\mu_t}\right),\frac{x_i}{\mu_t}\right], & \frac{\mu_t}{2}<x_i\le\mu_t,\\[4pt]
\left[1-\left(1\wedge\frac{\sigma_t}{\mu_t}\right),1\right], & x_i>\mu_t.
\end{cases}
\]
Then $I_i\subset[0,1]$, $|I_i|\gtrsim1\wedge\frac{\sigma_t}{\mu_t}$, and for any $z_i\in I_i$, we conclude
\[
|x_i-\mu_tz_i|\le \Delta_i+\sigma_t.
\]
Indeed, if $x_i<0$, then
\[
|x_i-\mu_tz_i|\le -x_i+\mu_t\left(1\wedge\frac{\sigma_t}{\mu_t}\right)\le \Delta_i+\sigma_t.
\]
If $0\le x_i\le\mu_t$, then $\Delta_i=0$ and, by the definition of $I_i$,
\[
|x_i-\mu_tz_i|=\mu_t\left|z_i-\frac{x_i}{\mu_t}\right|
\le \frac{\mu_t}{2}\left(1\wedge\frac{\sigma_t}{\mu_t}\right)\le\sigma_t=\Delta_i+\sigma_t.
\]
If $x_i>\mu_t$, then
\[
|x_i-\mu_tz_i|\le x_i-\mu_t+\mu_t\left(1\wedge\frac{\sigma_t}{\mu_t}\right)\le \Delta_i+\sigma_t.
\]
Thus,
\[
\begin{aligned}
q_i(x_i)
&\ge \int_{I_i}\frac{1}{\sigma_t\sqrt{2\pi}}
\exp\left(-\frac{(x_i-\mu_t z_i)^2}{2\sigma_t^2}\right)dz_i \\
&\ge \frac{|I_i|}{\sigma_t\sqrt{2\pi}}
\exp\left(-\frac{(\Delta_i+\sigma_t)^2}{2\sigma_t^2}\right) \\
&\gtrsim \exp\left(-\frac{\Delta_i^2}{\sigma_t^2}\right),
\end{aligned}
\]
where the last inequality follows from $\frac{|I_i|}{\sigma_t}\gtrsim
\frac{1\wedge\frac{\sigma_t}{\mu_t}}{\sigma_t}
=\min\left\{\frac{1}{\sigma_t},\frac{1}{\mu_t}\right\}\gtrsim1,
\frac{(\Delta_i+\sigma_t)^2}{2\sigma_t^2}
\le \frac{\Delta_i^2}{\sigma_t^2}+1.$
Therefore,
\[
\begin{aligned}
p_t(\boldsymbol{x}|\boldsymbol{w})\ge B_l\prod_{i=1}^{d_x}q_i(x_i)\gtrsim \prod_{i=1}^{d_x}\exp\left(-\frac{\Delta_i^2}{\sigma_t^2}\right) \ge \exp\left(-\frac{d_xD_t(\boldsymbol{x})^2}{\sigma_t^2}\right).
\end{aligned}
\]
For the upper bound, since $\mu_tz_i\in[0,\mu_t]$ and
$\Delta_i=(-x_i)_+\vee (x_i-\mu_t)_+$, we have
$|x_i-\mu_tz_i|\ge\Delta_i$. Thus
\[
q_i(x_i)\le \frac{1}{\sigma_t\sqrt{2\pi}}\exp\left(-\frac{\Delta_i^2}{2\sigma_t^2}\right).
\]
Moreover, by $u=\frac{x_i-\mu_tz_i}{\sigma_t}$,
\[
q_i(x_i)=\frac{1}{\mu_t}\int_{\frac{x_i-\mu_t}{\sigma_t}}^{\frac{x_i}{\sigma_t}}\frac{1}{\sqrt{2\pi}}\exp\left(-\frac{u^2}{2}\right)du
\lesssim \frac{1}{\mu_t}\exp\left(-\frac{\Delta_i^2}{2\sigma_t^2}\right),
\]
where the last inequality follows from $|u|=\frac{|x_i-\mu_tz_i|}{\sigma_t}\ge \frac{\Delta_i}{\sigma_t}.$  Therefore,
\[
q_i(x_i)\lesssim
\left(\frac{1}{\sigma_t}\wedge\frac{1}{\mu_t}\right)
\exp\left(-\frac{\Delta_i^2}{2\sigma_t^2}\right)
\lesssim \exp\left(-\frac{\Delta_i^2}{2\sigma_t^2}\right),
\]
where the last inequality follows from $\mu_t^2+\sigma_t^2=1$. Consequently,
\[
\begin{aligned}
p_t(\boldsymbol{x}|\boldsymbol{w})\le B_u\prod_{i=1}^{d_x}q_i(x_i) \lesssim \prod_{i=1}^{d_x}\exp\left(-\frac{\Delta_i^2}{2\sigma_t^2}\right)\le \exp\left(-\frac{D_t(\boldsymbol{x})^2}{2\sigma_t^2}\right).
\end{aligned}
\]
Since $D_t(\boldsymbol{x})=\max\limits_{1\le i\le d_x}\{(-x_i)_+,(x_i-\mu_t)_+\}
\ge \max\limits_{1\le i\le d_x}(|x_i|-\mu_t)_+
=(\|\boldsymbol{x}\|_\infty-\mu_t)_+,$
we further have
\[
p_t(\boldsymbol{x}|\boldsymbol{w})\lesssim
\exp\left(-\frac{(\|\boldsymbol{x}\|_\infty-\mu_t)_+^2}{2\sigma_t^2}\right).
\]
This completes the proof.
\end{proof}

\subsection{Bounds on the Derivatives of $p_t(\boldsymbol{x}|\boldsymbol{w}), \nabla \log p_t({\boldsymbol{x}}|\boldsymbol{w})$ and Clipping}
In this section, we give the upper bounds on the derivatives of \( p_t({\boldsymbol{x}}|\boldsymbol{w}) \) and \( \nabla \log p_t({\boldsymbol{x}}|\boldsymbol{w}) \). We first state the following lemma.

\begin{lemma}[Integral Clipping]\label{lem:integral-clipping}
Under Assumption \ref{apt:diffusion-density-bound}, for any \( {\boldsymbol{x}} \in \mathbb{R}^{d_x},~ \boldsymbol{w}\in [0,1]^{d_w},~\beta \in \mathbb{N}^{d_x} \) and \( 0 < \epsilon < e^{-1} \), there exists a constant \( C > 0 \) such that
\begin{align*}
\Bigg| \int_{\mathbb{R}^{d_x}} \prod_{i=1}^{d_x} &\left( \frac{{x}_i - \mu_t z_i}{\sigma_t} \right)^{\beta_i} \frac{1}{\sigma_t^{d_x} (2\pi)^{\frac{d_x}{2}}} p(\boldsymbol{z}|\boldsymbol{w}) \exp\left( -\frac{\|{\boldsymbol{x}} - \mu_t \boldsymbol{z}\|^2}{2\sigma_t^2} \right) d\boldsymbol{z} \\
& - \int_{A_{{\boldsymbol{x}}}} \prod_{i=1}^{d_x} \left( \frac{{x}_i - \mu_t z_i}{\sigma_t} \right)^{\beta_i} \frac{1}{\sigma_t^{d_x} (2\pi)^{\frac{d_x}{2}}} p(\boldsymbol{z}|\boldsymbol{w}) \exp\left( -\frac{\|{\boldsymbol{x}} - \mu_t \boldsymbol{z}\|^2}{2\sigma_t^2} \right) d\boldsymbol{z} \Bigg| \lesssim \epsilon,
\end{align*}
where \( A_{{\boldsymbol{x}}} = \prod_{i=1}^{d_x} a_{i,{\boldsymbol{x}}} \) with \( a_{i,{\boldsymbol{x}}} = \left[ \frac{{x}_i - C\sigma_t \sqrt{\log \epsilon^{-1}}}{\mu_t}, \frac{{x}_i + C\sigma_t \sqrt{\log \epsilon^{-1}}}{\mu_t} \right] \).
\end{lemma}

\begin{proof}
It follows that
\begin{align*}
&\left| \int_{\mathbb{R}^{d_x} \setminus A_{{\boldsymbol{x}}}} \prod_{i=1}^{d_x} \left| \frac{{x}_i - \mu_t z_i}{\sigma_t} \right|^{\beta_i} \frac{1}{\sigma_t^{d_x} (2\pi)^{\frac{d_x}{2}}} p(\boldsymbol{z}|\boldsymbol{w}) \exp\left( -\frac{\|{\boldsymbol{x}} - \mu_t \boldsymbol{z}\|^2}{2\sigma_t^2} \right) d\boldsymbol{z} \right|\\
&\lesssim \frac{B_u}{\sigma_t^{d_x} (2\pi)^{\frac{d_x}{2}}} \int_{\mathbb{R}^{d_x} \setminus A_{{\boldsymbol{x}}}} \prod_{i=1}^{d_x} \left| \frac{{x}_i - \mu_t z_i}{\sigma_t} \right|^{\beta_i} \mathbbm{1}_{\{\boldsymbol{z}\in [0,1]^{d_x}\}} \exp\left( -\frac{\|{\boldsymbol{x}} - \mu_t \boldsymbol{z}\|^2}{2\sigma_t^2} \right) d\boldsymbol{z}\\
&\lesssim \frac{B_u}{\sigma_t^{d_x} (2\pi)^{\frac{d_x}{2}}} \sum_{j=1}^{d_x} \int_{\mathbb{R} \times \dots \times \mathbb{R} \times (\mathbb{R} \setminus a_{j,{\boldsymbol{x}}}) \times \mathbb{R} \times \dots \times \mathbb{R}} \prod_{i=1}^{d_x} \left| \frac{{x}_i - \mu_t z_i}{\sigma_t} \right|^{\beta_i} \mathbbm{1}_{\{z_i\in [0,1]\}} \exp\left( -\frac{\|{\boldsymbol{x}} - \mu_t \boldsymbol{z}\|^2}{2\sigma_t^2} \right) d\boldsymbol{z}\\
&\lesssim B_u \sum_{j=1}^{d_x} \left[ \left( \prod_{i=1, i \neq j}^{d_x} \frac{1}{\sigma_t \sqrt{2\pi}} \int_{\mathbb{R}} \left| \frac{{x}_i - \mu_t z_i}{\sigma_t} \right|^{\beta_i} \mathbbm{1}_{\{z_i\in [0,1]\}} \exp\left( -\frac{({x}_i - \mu_t z_i)^2}{2\sigma_t^2} \right) dz_i \right) \right.\\
&~~~~~~~~~~~~~~~~~~~~~ \left. \cdot \frac{1}{\sigma_t \sqrt{2\pi}} \int_{\mathbb{R} \setminus a_{j,{\boldsymbol{x}}}} \left| \frac{{x}_j - \mu_t z_j}{\sigma_t} \right|^{\beta_j} \mathbbm{1}_{\{z_j\in [0,1]\}} \exp\left( -\frac{({x}_j - \mu_t z_j)^2}{2\sigma_t^2} \right) dz_j \right].
\end{align*}
For \( i \neq j \), let \( y_i = \frac{{x}_i - \mu_t z_i}{\sigma_t} \). Then, it holds that
\begin{equation}\label{eq:apd-a.1}
\begin{aligned}
&\frac{1}{\sigma_t \sqrt{2\pi}} \int_{\mathbb{R}} \left| \frac{{x}_i - \mu_t z_i}{\sigma_t} \right|^{\beta_i} \mathbbm{1}_{\{z_i \in [0,1]\}} \exp\left( -\frac{({x}_i - \mu_t z_i)^2}{2\sigma_t^2} \right) dz_i\\
&\leq \frac{1}{\sqrt{2\pi}\mu_t} \int_{\mathbb{R}} |y_i|^{\beta_i} \exp\left( -\frac{y_i^2}{2} \right) dy_i\lesssim  \frac{1}{\mu_t}.
\end{aligned}
\end{equation}
Notice that
\begin{equation*}
\left| \frac{{x}_i - \mu_t z_i}{\sigma_t}\right|^{\beta_i}   \exp\left( -\frac{({x}_i - \mu_t z_i)^2}{2\sigma_t^2} \right)\leq (\beta_i)^{\frac{\beta_i}{2}}\exp \big(-\frac{\beta_i}{2}\big)\lesssim 1.
\end{equation*}
Then
\begin{gather*}
\frac{1}{\sigma_t \sqrt{2\pi}} \int_{\mathbb{R}} \left| \frac{{x}_i - \mu_t z_i}{\sigma_t} \right|^{\beta_i} \mathbbm{1}_{\{z_i\in [0,1]\}} \exp\left( -\frac{({x}_i - \mu_t z_i)^2}{2\sigma_t^2} \right) dz_i\lesssim \frac{1}{\sigma_t} \int_{\mathbb{R}} \mathbbm{1}_{\{z_i\in [0,1]\}}dz_i\lesssim  \frac{1}{\sigma_t}.
\end{gather*}
Thus, by \eqref{eq:apd-a.1}, we can derive
$$
\frac{1}{\sigma_t \sqrt{2\pi}} \int_{\mathbb{R}} \left| \frac{{x}_i - \mu_t z_i}{\sigma_t} \right|^{\beta_i} \mathbbm{1}_{\{z_i\in [0,1]\}} \exp\left( -\frac{({x}_i - \mu_t z_i)^2}{2\sigma_t^2} \right) dz_i \lesssim \mathcal{O}(1),
$$
since $\sigma_t^2+\mu_t^2=1$. For \( i = j \), we have
\begin{align*}
&\frac{1}{\sigma_t \sqrt{2\pi}} \int_{\mathbb{R} \setminus a_{j,{\boldsymbol{x}}}} \left| \frac{{x}_j - \mu_t z_j}{\sigma_t} \right|^{\beta_j} \mathbbm{1}_{\{z_j \in [0,1]\}} \exp\left( -\frac{({x}_j - \mu_t z_j)^2}{2\sigma_t^2} \right) dz_j\\
&\lesssim \frac{1}{\mu_t} \int_{|y_j| \geq C \sqrt{\log \epsilon^{-1}}} |y_j|^{\beta_j} \exp\left( -\frac{y_j^2}{2} \right) dy_j\\
&\lesssim \frac{1}{\mu_t} \int_{y_j \geq C \sqrt{\log \epsilon^{-1}}} y_j^{\beta_j} \exp\left( -\frac{y_j^2}{2} \right) dy_j\\
&\lesssim \frac{1}{\mu_t}\epsilon^{\frac{C^2}{2}} \left( \log \epsilon^{-1} \right)^{\frac{\beta_j}{2}}.
\end{align*}
Choosing \( C \geq 2 \), we obtain
\begin{equation}\label{eq:apd-a-2}
\frac{1}{\sigma_t \sqrt{2\pi}} \int_{\mathbb{R} \setminus a_{j,{\boldsymbol{x}}}} \left| \frac{{x}_j - \mu_t z_j}{\sigma_t} \right|^{\beta_j} \mathbbm{1}_{\{z_j \in [0,1]\}} \exp\left( -\frac{({x}_j - \mu_t z_j)^2}{2\sigma_t^2} \right) dz_j \lesssim \frac{\epsilon}{\mu_t}.
\end{equation}
By setting $C\geq \max\{2,\norm{\beta}_1\} $ and $\epsilon<e^{-1}$, we obtain that when $\Big| \frac{x_j-\mu_t z_j}{\sigma_t} \Big|\geq \beta_j$,
$$
\left| \frac{{x}_j - \mu_t z_j}{\sigma_t}\right|^{\beta_j}   \exp\left( -\frac{({x}_j - \mu_t z_j)^2}{2\sigma_t^2} \right)
$$
decreases as $\Big| \frac{x_j-\mu_t z_j}{\sigma_t} \Big|$ increases. Therefore, we obtain that
\begin{align*}
&\frac{1}{\sigma_t \sqrt{2\pi}} \int_{\mathbb{R} \setminus a_{j,{\boldsymbol{x}}}} \left| \frac{{x}_j - \mu_t z_j}{\sigma_t} \right|^{\beta_j} \mathbbm{1}_{\{z_j \in [0,1]\}} \exp\left( -\frac{({x}_j - \mu_t z_j)^2}{2\sigma_t^2} \right) dz_j\\
&\lesssim \frac{\epsilon^{\frac{C^2}{2}} \left( \log \epsilon^{-1} \right)^{\frac{\beta_j}{2}}}{\sigma_t}\int_{\mathbb{R} \setminus a_{j,{\boldsymbol{x}}}} \mathbbm{1}_{\{z_j \in [0,1]\}}  dz_j\lesssim \frac{\epsilon}{\sigma_t}.
\end{align*}
Then by \eqref{eq:apd-a-2}, we have
$$
\frac{1}{\sigma_t \sqrt{2\pi}} \int_{\mathbb{R} \setminus a_{j,{\boldsymbol{x}}}} \left| \frac{{x}_j - \mu_t z_j}{\sigma_t} \right|^{\beta_j} \mathbbm{1}_{\{z_j \in [0,1]\}} \exp\left( -\frac{({x}_j - \mu_t z_j)^2}{2\sigma_t^2} \right) dz_j\lesssim \epsilon,
$$
which implies that
\begin{align*}
&\left| \int_{\mathbb{R}^{d_x} \setminus A_{{\boldsymbol{x}}}} \prod_{i=1}^{d_x} \left( \frac{{x}_i - \mu_t z_i}{\sigma_t} \right)^{\beta_i} \frac{1}{\sigma_t^{d_x} (2\pi)^{\frac{d_x}{2}}} p(\boldsymbol{x}|\boldsymbol{w}) \exp\left( -\frac{\|{\boldsymbol{x}} - \mu_t \boldsymbol{z}\|^2}{2\sigma_t^2} \right) d\boldsymbol{z} \right|\\
&\lesssim \sum_{j=1}^{d_x} (\mathcal{O}(1))^{d_x-1} \cdot \epsilon\lesssim \epsilon.
\end{align*}
This completes the proof.
\end{proof}

\begin{lemma}[Boundedness of Derivatives]
\label{lem:derivatives-bound}
Under Assumptions~\ref{apt:diffusion-density-bound} and~\ref{apt:diffusion-density-derivative}, for any $k\in\mathbb{N}_0$, $0\le l\le \lfloor6\alpha+3\rfloor+1$, $1\le i_1,\ldots,i_k\le d_x$, and $1\le j_1,\ldots,j_l\le d_w$, the following bound holds:
\begin{equation}\label{eq:Dp-bound}
\left|\partial_{x_{i_1}}\cdots\partial_{x_{i_k}}\partial_{w_{j_1}}\cdots\partial_{w_{j_l}}p_t(\boldsymbol{x}|\boldsymbol{w})\right|
\lesssim \frac{1}{\sigma_t^k}.
\end{equation}
Moreover, we have
\begin{gather}
\Vert \nabla\log p_t({\boldsymbol{x}}|\boldsymbol{w}) \Vert \lesssim \frac{1}{\sigma_t}\cdot\left(\frac{D_t(\boldsymbol{x})}{\sigma_t}\vee 1\right), \label{eq:Tlogp-bound}\\
\Vert\partial_{{x}_i}\nabla\log p_t({\boldsymbol{x}}|\boldsymbol{w})\Vert \lesssim \frac{1}{\sigma_t^2}\cdot\left(\frac{D_t(\boldsymbol{x})^2}{\sigma^2_t}\vee 1\right), \label{eq:DTlogp-r-bound}
\end{gather}
and
\begin{equation}\label{eq:DTlogp-t-bound}
\|\partial_t\nabla \log p_t(\boldsymbol{x}|\boldsymbol{w})\|
\lesssim \frac{\big|\partial_t\mu_t\big|+\big|\partial_t\sigma_t\big|}{\sigma_t^2}\cdot \left(\frac{D_t(\boldsymbol{x})^2}{\sigma^2_t}\vee 1\right)^{\frac{3}{2}}.
\end{equation}
\end{lemma}

\begin{proof}
For notational convenience, we define
\begin{gather*}
\mathcal{N}(\boldsymbol{z}|\boldsymbol{w}) := \frac{p(\boldsymbol{z}|\boldsymbol{w})}{\sigma_t^{d_x}(2\pi)^{\frac{d_x}{2}}}\exp\left(-\frac{\Vert{\boldsymbol{x}}-\mu_t \boldsymbol{z}\Vert ^2}{2\sigma_t^2}\right),\\
\mathcal{N}^{\boldsymbol{\gamma}}(\boldsymbol{z}|\boldsymbol{w}) := \frac{\partial_{\boldsymbol{w}}^{\boldsymbol{\gamma}} p(\boldsymbol{z}|\boldsymbol{w})}{\sigma_t^{d_x}(2\pi)^{\frac{d_x}{2}}}\exp\left(-\frac{\Vert{\boldsymbol{x}}-\mu_t \boldsymbol{z}\Vert ^2}{2\sigma_t^2}\right).
\end{gather*}
Turning first to the proof of \eqref{eq:Dp-bound}, we let
$$
f_1({\boldsymbol{x}}|\boldsymbol{w}) = p_t({\boldsymbol{x}}|\boldsymbol{w}) =\int \frac{p(\boldsymbol{z}|\boldsymbol{w})}{\sigma_t^{d_x}(2\pi)^{\frac{d_x}{2}}}\exp\left(-\frac{\Vert{\boldsymbol{x}}-\mu_t \boldsymbol{z}\Vert^2}{2\sigma_t^2}\right) d\boldsymbol{z}=\int \cN (\boldsymbol{z}|\boldsymbol{w})d\boldsymbol{z}.
$$
For multi-indices $\boldsymbol{\beta}\in\mathbb{N}^{d_x}$ and $\boldsymbol{\gamma}\in \mathbb{N}^{d_w}$,  let  $f_1^{(\boldsymbol{\beta},\boldsymbol{\gamma})}({\boldsymbol{x}}|\boldsymbol{w}):= \partial_{{x}_1}^{\beta_1}\cdots\partial_{{x}_{d_x}}^{\beta_{d_x}}\partial_{{w}_1}^{\gamma_1}\cdots\partial_{{w}_{d_w}}^{\gamma_{d_w}}f_1({\boldsymbol{x}}|\boldsymbol{w})$, we define $B_{\boldsymbol{\beta}}:=\{\boldsymbol{u}\in\mathbb{N}^{d_x}|u_i\leq \beta_i,~1\leq i \leq d_x\}$. Then, it holds that 
$$
\partial_{{x}_1}^{\beta_1}\partial_{{x}_2}^{\beta_2}\cdots\partial_{{x}_{d_x}}^{\beta_{d_x}} e^{-\Vert{\boldsymbol{x}}\Vert^2/2} = \sum_{\boldsymbol{u}\in B_{\boldsymbol{\beta}}}B_{\boldsymbol{u}}\partial_{{x}_1}^{u_1}\partial_{{x}_2}^{u_2}\cdots\partial_{{x}_{d_x}}^{u_{d_x}} e^{-\Vert{\boldsymbol{x}}\Vert^2/2},
$$
where $B_{\boldsymbol{u}}$ are specific constants. Consequently, $f_1^{(\boldsymbol{\beta},\boldsymbol{\gamma})}({\boldsymbol{x}}|\boldsymbol{w})$ can be expressed as
$$
f_1^{(\boldsymbol{\beta},\boldsymbol{\gamma})}({\boldsymbol{x}}|\boldsymbol{w}) = \frac{1}{\sigma_t^{\sum_{i=1}^{d_x}\beta_i}}\cdot \int \sum_{\boldsymbol{u}\in B_{\boldsymbol{\beta}}}B_{\boldsymbol{u}}\prod_{i=1}^{d_x}\left(\frac{{x}_i-\mu_t z_i}{\sigma_t}\right)^{u_i}\mathcal{N}^{\boldsymbol{\gamma}}(\boldsymbol{z}|\boldsymbol{w}) d\boldsymbol{z}.
$$
For multi-indices $\boldsymbol{\beta}\in\mathbb{N}_0^{d_x}$ and $\boldsymbol{\gamma}\in\mathbb{N}_0^{d_w}$ satisfying $\|\boldsymbol{\beta}\|_1=k$ and $\|\boldsymbol{\gamma}\|_1=l\le \lfloor6\alpha+3\rfloor+1$, Assumption~\ref{apt:diffusion-density-derivative} gives
$$
\partial_{\boldsymbol{w}}^{\boldsymbol{\gamma}} p(\boldsymbol{z}|\boldsymbol{w}) \leq B_{\boldsymbol{\gamma}},~~ \int \sum_{\boldsymbol{u}\in B_{\boldsymbol{\beta}}}B_{\boldsymbol{u}}\prod_{i=1}^{d_x}\left(\frac{{x}_i-\mu_t z_i}{\sigma_t}\right)^{u_i}\mathcal{N}^{\boldsymbol{\gamma}}(\boldsymbol{z}|\boldsymbol{w}) d\boldsymbol{z} \lesssim \frac{1}{\sigma_t ^k}\wedge \frac{1}{\mu_t^k}\lesssim 1.
$$
Hence \eqref{eq:Dp-bound} follows from the identity $\sigma_t^2+\mu_t^2=1$.

We now proceed to establish \eqref{eq:Tlogp-bound} and \eqref{eq:DTlogp-r-bound}. Without loss of generality, we analyze the first coordinate of $\nabla\log p_t({\boldsymbol{x}}|\boldsymbol{w})$,
the bounds for the remaining coordinates follow identically. Let $f_2({\boldsymbol{x}}|\boldsymbol{w}):= [\sigma_t\nabla p_t({\boldsymbol{x}}|\boldsymbol{w})]_1$. 
Then, we have
\begin{align*}
&[\nabla\log p_t({\boldsymbol{x}}|\boldsymbol{w})]_1 = \frac{1}{\sigma_t}\cdot\frac{f_2({\boldsymbol{x}}|\boldsymbol{w})}{f_1({\boldsymbol{x}}|\boldsymbol{w})},\\
& [\partial_{{x}_i}\nabla\log p_t({\boldsymbol{x}}|\boldsymbol{w})]_1 = \frac{1}{\sigma_t}\cdot\left(\frac{\partial_{{x}_i}f_2({\boldsymbol{x}}|\boldsymbol{w})}{f_1({\boldsymbol{x}}|\boldsymbol{w})} - \frac{f_2({\boldsymbol{x}}|\boldsymbol{w})(\partial_{{x}_i}f_1({\boldsymbol{x}}|\boldsymbol{w}))}{f_1^2({\boldsymbol{x}}|\boldsymbol{w})}\right).
\end{align*}
Explicitly expanding the terms yields
\begin{align*}
&\frac{f_2({\boldsymbol{x}}|\boldsymbol{w})}{f_1({\boldsymbol{x}}|\boldsymbol{w})} = -\frac{
\int\left(\frac{{x}_1-\mu_t z_1}{\sigma_t}\right) \mathcal{N}(\boldsymbol{z}|\boldsymbol{w})  d\boldsymbol{z}}
{\int \mathcal{N}(\boldsymbol{z}|\boldsymbol{w})  d\boldsymbol{z}},\\
&\frac{\partial_{{x}_i}f_1({\boldsymbol{x}}|\boldsymbol{w})}{f_1({\boldsymbol{x}}|\boldsymbol{w})} =- \frac{1}{\sigma_t}\frac{
\int\left(\frac{{x}_i-\mu_t z_i}{\sigma_t}\right) \mathcal{N}(\boldsymbol{z}|\boldsymbol{w})  d\boldsymbol{z}}
{\int \mathcal{N}(\boldsymbol{z}|\boldsymbol{w})  d\boldsymbol{z}},\\
&\frac{\partial_{{x}_i}f_2({\boldsymbol{x}}|\boldsymbol{w})}{f_1({\boldsymbol{x}}|\boldsymbol{w})} = \frac{1}{\sigma_t} \frac{
\int \left(-\mathbbm{1}_{\{i=1\}} + \frac{{x}_1-\mu_t z_1}{\sigma_t}\frac{{x}_i-\mu_t z_i}{\sigma_t}\right) \mathcal{N}(\boldsymbol{z}|\boldsymbol{w})  d\boldsymbol{z}}
{\int \mathcal{N}(\boldsymbol{z}|\boldsymbol{w})  d\boldsymbol{z}}.
\end{align*}
The integral terms in the three equations above can be expressed in the following unified form:
\begin{equation*}
\mathcal{F}(\boldsymbol{x}|\boldsymbol{w})=\frac{
\int\prod_{i=1}^{d_x}\left(\frac{{x}_i-\mu_t z_i}{\sigma_t}\right)^{\beta_i}\mathcal{N}(\boldsymbol{z}|\boldsymbol{w}) d\boldsymbol{z}}
{\int\mathcal{N}(\boldsymbol{z}|\boldsymbol{w}) d\boldsymbol{z}},
\end{equation*}
with $\sum_{i=1}^{d_x}\beta_i\leq 2$.
According to Lemma \ref{lem:integral-clipping}, for any $0 < \epsilon < e^{-1}$, there exists a constant $C > 0$ such that
\begin{align*}
\Bigg| \int_{\mathbb{R}^{d_x}} \prod_{i=1}^{d_x} \left( \frac{{x}_i - \mu_t z_i}{\sigma_t} \right)^{\beta_i} \mathcal{N}(\boldsymbol{z}|\boldsymbol{w}) d\boldsymbol{z} - \int_{A_{{\boldsymbol{x}}}} \prod_{i=1}^{d_x} \left( \frac{{x}_i - \mu_t z_i}{\sigma_t} \right)^{\beta_i} \mathcal{N}(\boldsymbol{z}|\boldsymbol{w}) d\boldsymbol{z} \Bigg| \lesssim \epsilon,
\end{align*}
where $A_{{\boldsymbol{x}}} = \prod_{i=1}^{d_x}a_{i,{\boldsymbol{x}}}$ with $a_{i,{\boldsymbol{x}}} = \big[\frac{{x}_i - C\sigma_t\sqrt{\log\epsilon^{-1}}}{\mu_t}, \frac{{x}_i + C\sigma_t\sqrt{\log\epsilon^{-1}}}{\mu_t}\big].$

Consequently, conditioned on $p_t({\boldsymbol{x}}|\boldsymbol{w})\geq \epsilon$, we have
\begin{equation}\label{eq:derivative-bound-F-bound}
\begin{aligned}
|\mathcal{F}(\boldsymbol{x}|\boldsymbol{w})| &\leq \frac{
\int_{A_{{\boldsymbol{x}}}}\prod_{i=1}^{d_x}\left|\frac{{x}_i-\mu_t z_i}{\sigma_t}\right|^{\beta_i}\mathcal{N}(\boldsymbol{z}|\boldsymbol{w}) d\boldsymbol{z}}
{\int_{A_{{\boldsymbol{x}}}}\mathcal{N}(\boldsymbol{z}|\boldsymbol{w}) d\boldsymbol{z}} + \mathcal{O}(1)\\
&\leq\max_{\boldsymbol{z}\in A_{{\boldsymbol{x}}}}\left[\prod_{i=1}^{d_x}\left|\frac{{x}_i-\mu_t z_i}{\sigma_t}\right|^{\beta_i}\right] + \mathcal{O}(1)  \\
&\leq\left(C^2\log\epsilon^{-1}\right)^{\frac{\sum_{i=1}^{d_x}\beta_i}{2}} + \mathcal{O}(1) \\
&\lesssim \left(\log\epsilon^{-1}\right)^{\frac{\sum_{i=1}^{d_x}\beta_i}{2}},
\end{aligned}
\end{equation}
which implies that
$$
\left| \frac{f_2({\boldsymbol{x}}|\boldsymbol{w})}{f_1({\boldsymbol{x}}|\boldsymbol{w})} \right| \lesssim \sqrt{\log\epsilon^{-1}},~
\left| \frac{\partial_{{x}_i}f_1({\boldsymbol{x}}|\boldsymbol{w})}{f_1({\boldsymbol{x}}|\boldsymbol{w})} \right| \lesssim \frac{\sqrt{\log\epsilon^{-1}}}{\sigma_t},~
\left| \frac{\partial_{{x}_i}f_2({\boldsymbol{x}}|\boldsymbol{w})}{f_1({\boldsymbol{x}}|\boldsymbol{w})} \right| \lesssim \frac{\log\epsilon^{-1}}{\sigma_t}.
$$
Then, we obtain
$$
\Vert\nabla\log p_t({\boldsymbol{x}}|\boldsymbol{w})\Vert \lesssim \frac{\sqrt{\log\epsilon^{-1}}}{\sigma_t}, ~ \Vert\partial_{{x}_i}\nabla\log p_t({\boldsymbol{x}}|\boldsymbol{w})\Vert \lesssim \frac{\log\epsilon^{-1}}{\sigma_t^2}.
$$
For given a constant $0<B_{l,1}<1$, we replace $\epsilon$ with $B_{l,1}\exp\left(-\frac{d_xD_t(\boldsymbol{x})^2}{\sigma_t^2}\right)$, then

$$
p_t({\boldsymbol{x}}|\boldsymbol{w}) \geq B_{l,1}\exp\left(-\frac{d_xD_t(\boldsymbol{x})^2}{\sigma_t^2}\right).
$$
Then we have
$$
\Vert \nabla\log p_t({\boldsymbol{x}}|\boldsymbol{w}) \Vert \lesssim \frac{1}{\sigma_t}\cdot\left(\frac{D_t(\boldsymbol{x})}{\sigma_t}\vee 1\right).
$$
and for $ 1\leq i\leq d_x$,
$$
\Vert\partial_{{x}_i}\nabla\log p_t({\boldsymbol{x}}|\boldsymbol{w})\Vert \lesssim \frac{1}{\sigma_t^2}\cdot\left(\frac{D_t(\boldsymbol{x})^2}{\sigma^2_t}\vee 1\right).
$$
Finally, we prove \eqref{eq:DTlogp-t-bound}. By a simple calculation, we have
\begin{align*}
&\partial_t[\nabla\log p_t({\boldsymbol{x}}|\boldsymbol{w})]_1= \partial_t\left(\frac{1}{\sigma_t}\cdot\frac{f_2({\boldsymbol{x}}|\boldsymbol{w})}{f_1({\boldsymbol{x}}|\boldsymbol{w})}\right) \\
&= \left(\partial_t\frac{1}{\sigma_t}\right)\frac{f_2({\boldsymbol{x}}|\boldsymbol{w})}{f_1({\boldsymbol{x}}|\boldsymbol{w})} - \frac{1}{\sigma_t} \cdot \frac{\partial_t f_1({\boldsymbol{x}}|\boldsymbol{w})}{f_1({\boldsymbol{x}}|\boldsymbol{w})} \cdot \frac{f_2({\boldsymbol{x}}|\boldsymbol{w})}{f_1({\boldsymbol{x}}|\boldsymbol{w})} + \frac{1}{\sigma_t} \cdot \frac{\partial_t f_2({\boldsymbol{x}}|\boldsymbol{w})}{f_1({\boldsymbol{x}}|\boldsymbol{w})}\\
&= -\frac{\partial_t\sigma_t}{\sigma_t}[\nabla\log p_t({\boldsymbol{x}}|\boldsymbol{w})]_1 \\
&~~ - \frac{\int \Big[\partial_t \sigma_t \left(\Vert{\boldsymbol{x}} - \mu_t\boldsymbol{z}\Vert^2\sigma_t^{-3}-d_x\sigma_t^{-1}\right)+(\partial_t \mu_t)\boldsymbol{z}^T (\boldsymbol{x}-\mu_t \boldsymbol{z})\sigma_t^{-2}  \Big] \mathcal{N}(\boldsymbol{z}|\boldsymbol{w})  d\boldsymbol{z}}{\sigma_t \int \mathcal{N}(\boldsymbol{z}|\boldsymbol{w})  d\boldsymbol{z}} \cdot \frac{f_2(\boldsymbol{x}|\boldsymbol{w})}{f_1(\boldsymbol{x}|\boldsymbol{w})}\\
&~~ + \frac{\int \left((\partial_t \mu_t) z_1 + (x_1 - \mu_t z_1) \Bigg[
\begin{aligned}
&(d_x+1) (\partial_t \sigma_t) \sigma_t^{-1} - \|\boldsymbol{x} - \mu_t \boldsymbol{z}\|^2 (\partial_t \sigma_t) \sigma_t^{-3} \\
&- (\partial_t \mu_t) \boldsymbol{z}^\top (\mu_t \boldsymbol{z} - \boldsymbol{x}) \sigma_t^{-2}  
\end{aligned}
\Bigg] \right)\mathcal{N}(\boldsymbol{z}|\boldsymbol{w})  d\boldsymbol{z}}{\sigma^2_t \int \mathcal{N}(\boldsymbol{z}|\boldsymbol{w})  d\boldsymbol{z}}.
\end{align*}
According to \eqref{eq:derivative-bound-F-bound} and Lemma \ref{lem:p-bound}, we obtain
$$
\Vert\partial_t\nabla\log p_t({\boldsymbol{x}}|\boldsymbol{w})\Vert \lesssim \frac{\big|\partial_t\mu_t\big|+\big|\partial_t\sigma_t\big|}{\sigma_t^2}\cdot \left(\frac{D_t(\boldsymbol{x})^2}{\sigma^2_t}\vee 1\right)^{\frac{3}{2}}.
$$
The proof is complete.
\end{proof}

\begin{lemma}[Error Bounds for Clipping]\label{lem:clipping-bound}
Under Assumptions~\ref{apt:diffusion-density-bound} and~\ref{apt:diffusion-density-derivative}, let
$0<\epsilon<e^{-1}$ and $0<\epsilon_p<1$. For all $0<t<\infty$, there exists a constant $C>0$ such that
\begin{gather}
\int_{D_t(\boldsymbol{x})\ge C\sigma_t\sqrt{\log\epsilon^{-1}}}
p_t(\boldsymbol{x}|\boldsymbol{w})\|\nabla\log p_t(\boldsymbol{x}|\boldsymbol{w})\|^2d\boldsymbol{x}
\lesssim \frac{\epsilon}{\sigma_t}, \label{eq:int-pTp2-bound}\\
\int_{D_t(\boldsymbol{x})\ge C\sigma_t\sqrt{\log\epsilon^{-1}}}
p_t(\boldsymbol{x}|\boldsymbol{w})d\boldsymbol{x}
\lesssim \sigma_t\epsilon. \notag 
\end{gather}
Moreover, for $D_t(\boldsymbol{x})\le C\sigma_t\sqrt{\log\epsilon^{-1}}$, 
it follows that 
\begin{gather*}
\int_{D_t(\boldsymbol{x})\le C\sigma_t\sqrt{\log\epsilon^{-1}}}
p_t(\boldsymbol{x}|\boldsymbol{w})\mathbbm{1}_{\{p_t(\boldsymbol{x}|\boldsymbol{w})\le\epsilon_p\}}
\|\nabla\log p_t(\boldsymbol{x}|\boldsymbol{w})\|^2d\boldsymbol{x}
\lesssim \frac{\epsilon_p}{\sigma_t^2}(\log\epsilon^{-1})^{\frac{d_x+2}{2}}, 
\\
\int_{D_t(\boldsymbol{x})\le C\sigma_t\sqrt{\log\epsilon^{-1}}}
p_t(\boldsymbol{x}|\boldsymbol{w})\mathbbm{1}_{\{p_t(\boldsymbol{x}|\boldsymbol{w})\le\epsilon_p\}}d\boldsymbol{x}
\lesssim \epsilon_p(\log\epsilon^{-1})^{\frac{d_x}{2}}. 
\end{gather*}
\end{lemma}

\begin{proof}
We begin by establishing \eqref{eq:int-pTp2-bound}. For the region $D_t(\boldsymbol{x})\geq C\sigma_t \sqrt{\log \epsilon^{-1}}$, applying Lemma \ref{lem:p-bound} and Lemma \ref{lem:derivatives-bound} yields
\[
p_t(\boldsymbol{x}|\boldsymbol{w})\|\nabla\log p_t(\boldsymbol{x}|\boldsymbol{w})\|^2
\lesssim \frac{1}{\sigma_t^2}\exp\left(-\frac{D_t(\boldsymbol{x})^2}{2\sigma_t^2}\right)
\frac{D_t(\boldsymbol{x})^2}{\sigma_t^2}.
\]
Since $\{\boldsymbol{x}:D_t(\boldsymbol{x})\le r\}=[-r,\mu_t+r]^{d_x}$, its volume is $(\mu_t+2r)^{d_x}$, and hence the shell volume is bounded by $(\mu_t+2r)^{d_x-1}dr$. Therefore,
\[
\begin{aligned}
&\int_{D_t(\boldsymbol{x})\ge C\sigma_t\sqrt{\log\epsilon^{-1}}}p_t(\boldsymbol{x}|\boldsymbol{w})\|\nabla\log p_t(\boldsymbol{x}|\boldsymbol{w})\|^2d\boldsymbol{x}\\
&\lesssim \frac{1}{\sigma_t^2}\int_{C\sigma_t\sqrt{\log\epsilon^{-1}}}^{\infty}
\exp\left(-\frac{r^2}{2\sigma_t^2}\right)\frac{r^2}{\sigma_t^2}(\mu_t+2r)^{d_x-1}dr\\
&\lesssim \frac{1}{\sigma_t}\int_{C\sqrt{\log\epsilon^{-1}}}^{\infty}e^{-\frac{z^2}{2}}z^2(1+\sigma_t z)^{d_x-1}dz\\
&\lesssim \frac{1}{\sigma_t}\int_{C\sqrt{\log\epsilon^{-1}}}^{\infty}e^{-\frac{z^2}{2}}z^{d_x+1}dz\\
&\lesssim  \frac{1}{\sigma_t}\epsilon^{\frac{C^2}{2}}\cdot(\log\epsilon^{-1})^{\frac{d_x+1}{2}}\\
&\lesssim \frac{\epsilon}{\sigma_t},
\end{aligned}
\]
where we used $z=\frac{r}{\sigma_t}$ and chose $C$ sufficiently large. Similarly,
\[
\begin{aligned}
\int_{D_t(\boldsymbol{x})\ge C\sigma_t\sqrt{\log\epsilon^{-1}}}p_t(\boldsymbol{x}|\boldsymbol{w})d\boldsymbol{x} \lesssim \sigma_t\int_{C\sqrt{\log\epsilon^{-1}}}^{\infty}e^{-z^2/2}z^{d_x-1}dz
\lesssim \sigma_t\epsilon.
\end{aligned}
\]
On $D_t(\boldsymbol{x})\le C\sigma_t\sqrt{\log\epsilon^{-1}}$, Lemma~\ref{lem:derivatives-bound} gives
\[
\|\nabla\log p_t(\boldsymbol{x}|\boldsymbol{w})\|^2\lesssim \frac{D_t(\boldsymbol{x})^2}{\sigma_t^4}\lesssim\frac{\log\epsilon^{-1}}{\sigma_t^2}.
\]
Therefore,
\[
\begin{aligned}
&\int_{D_t(\boldsymbol{x})\le C\sigma_t\sqrt{\log\epsilon^{-1}}}p_t(\boldsymbol{x}|\boldsymbol{w})\mathbbm{1}_{\{p_t(\boldsymbol{x}|\boldsymbol{w})\le\epsilon_p\}}\|\nabla\log p_t(\boldsymbol{x}|\boldsymbol{w})\|^2d\boldsymbol{x}\\
&\lesssim \frac{\epsilon_p\log\epsilon^{-1}}{\sigma_t^2}\left(\mu_t+2C\sigma_t\sqrt{\log\epsilon^{-1}}\right)^{d_x}
\lesssim \frac{\epsilon_p}{\sigma_t^2}(\log\epsilon^{-1})^{\frac{d_x+2}{2}}.
\end{aligned}
\]
Similarly,
\[
\begin{aligned}
\int_{D_t(\boldsymbol{x})\le C\sigma_t\sqrt{\log\epsilon^{-1}}}p_t(\boldsymbol{x}|\boldsymbol{w})\mathbbm{1}_{\{p_t(\boldsymbol{x}|\boldsymbol{w})\le\epsilon_p\}}d\boldsymbol{x}\lesssim \epsilon_p\left(\mu_t+2C\sigma_t\sqrt{\log\epsilon^{-1}}\right)^{d_x}
\lesssim \epsilon_p(\log\epsilon^{-1})^{\frac{d_x}{2}}.
\end{aligned}
\]
This completes the proof.
\end{proof}

\subsection{Approximating of $p(\boldsymbol{x}|\boldsymbol{w})$ via Local Polynomials}\label{sec:approxmate-p}

\begin{lemma}[Approximating $p(\boldsymbol{x}|\boldsymbol{w})$ via local polynomials]\label{lem:density-holder-local-polynomials}
Suppose that Assumption \ref{apt:diffusion-density-holder} holds. Let $N \gg 1$, 
there exists $g(\boldsymbol{x}|\boldsymbol{w})$ such that
$$ \big|p(\boldsymbol{x}|\boldsymbol{w}) - g(\boldsymbol{x}|\boldsymbol{w})\big| \lesssim N^{-\alpha},~ \boldsymbol{x}\in [0,1]^{d_x},~\boldsymbol{w}\in [0,1]^{d_w}. $$
\end{lemma}
\begin{proof}
We write $f(\boldsymbol{x},\boldsymbol{w}) := p(\boldsymbol{x}|\boldsymbol{w})$.
By Assumption \ref{apt:diffusion-density-holder}, we know that $\Vert f \Vert_{\mathcal{H}^{\alpha}([0,1]^{d_x}\times [0,1]^{d_w})} \leq 2^p\Sigma$.
For any
\begin{gather*} \boldsymbol{m} = (m_1,m_2,\cdots,m_{d_x})^{\top} \in [N]^{d_x} := \{1,\cdots,N\}^{d_x}, \boldsymbol{x} = (x_1,x_2,\cdots,x_{d_x})^{\top}\in[0,1]^{d_x},\\ \boldsymbol{n} = (n_1,n_2,\cdots,n_{d_w})^{\top} \in [N]^{d_w} := \{1,\cdots,N\}^{d_w}, \boldsymbol{w} = (w_1,w_2,\cdots,w_{d_w})^{\top}\in[0,1]^{d_w}, \end{gather*}
for $n\in[N]$, define the local weight function $\eta_{n,N}:[0,1]\to[0,1]$ by
\[ \eta_{1,N}(w):= \begin{cases} 1,&0\le w\le \frac{1}{N},\\ 2-Nw,&\frac{1}{N}<w\le\frac{2}{N},\\ 0,&\frac{2}{N}<w\le1, \end{cases} \]
\[ \eta_{n,N}(w):= \begin{cases} Nw-n+1,&\frac{n-1}{N}<w\le \frac{n}{N},\\ n+1-Nw,&\frac{n}{N}<w\le\frac{n+1}{N},\\ 0,&\text{otherwise}, \end{cases} ~~ 2\le n\le N-1, \]
and
\[ \eta_{N,N}(w):= \begin{cases} 0,&0\le w\le\frac{N-1}{N},\\ Nw-N+1,&\frac{N-1}{N}<w\le1. \end{cases} \]
For $\boldsymbol{n}\in[N]^{d_w}$, define $\eta_{\boldsymbol{n},N}(\boldsymbol{w}):=\prod_{i=1}^{d_w}\eta_{n_i,N}(w_i).$
We define
\begin{gather*}  \psi_{\boldsymbol{m},\boldsymbol{n}}(\boldsymbol{x},\boldsymbol{w}) :=\mathbbm{1}_{\big\{\boldsymbol{x}\in\left(\frac{\boldsymbol{m}-\boldsymbol{1}}{N},\frac{\boldsymbol{m}}{N}\right]\big\}} \eta_{\boldsymbol{n},N}(\boldsymbol{w}) =\prod_{i=1}^{d_x}\mathbbm{1}_{\big\{x_i\in\left(\frac{m_i-1}{N},\frac{m_i}{N}\right]\big\}} \eta_{\boldsymbol{n},N}(\boldsymbol{w}).  \end{gather*}
For $1\le i\le d_x$, the interval corresponding to $m_i=1$ is understood as $\left[0,\frac{1}{N}\right]$.
The functions $\{\eta_{n,N}\}_{n=1}^{N}$ form a partition of unity on $[0,1]$. Consequently, the functions
$\{\psi_{\boldsymbol{m},\boldsymbol{n}}\}_{\boldsymbol{m}\in[N]^{d_x},\boldsymbol{n}\in[N]^{d_w}}$
form a partition of unity on $[0,1]^{d_x}\times[0,1]^{d_w}$, that is,
$$ \sum_{\boldsymbol{m}\in[N]^{d_x},\boldsymbol{n}\in[N]^{d_w}} \psi_{\boldsymbol{m},\boldsymbol{n}}(\boldsymbol{x},\boldsymbol{w}) \equiv1,~~ \boldsymbol{x}\in[0,1]^{d_x},~\boldsymbol{w}\in[0,1]^{d_w}. $$
Denote
\begin{gather*} C_{\boldsymbol{m},\boldsymbol{n},\boldsymbol{u},\boldsymbol{v}} =\frac{1}{\boldsymbol{u}!\boldsymbol{v}!} \partial^{\boldsymbol{u}}_{\boldsymbol{x}}\partial^{\boldsymbol{v}}_{\boldsymbol{w}} f\Big(\frac{\boldsymbol{m}}{N},\frac{\boldsymbol{n}}{N}\Big),\\ 
g_{\boldsymbol{m},\boldsymbol{n},\boldsymbol{u},\boldsymbol{v}}(\boldsymbol{x},\boldsymbol{w}) =\psi_{\boldsymbol{m},\boldsymbol{n}}(\boldsymbol{x},\boldsymbol{w}) \Big(\boldsymbol{x}-\frac{\boldsymbol{m}}{N}\Big)^{\boldsymbol{u}} \Big(\boldsymbol{w}-\frac{\boldsymbol{n}}{N}\Big)^{\boldsymbol{v}}, \\ 
G_{\boldsymbol{m},\boldsymbol{n}}(\boldsymbol{x},\boldsymbol{w}) =\sum_{\norm{\boldsymbol{u}}_1+\norm{\boldsymbol{v}}_1\le p} C_{\boldsymbol{m},\boldsymbol{n},\boldsymbol{u},\boldsymbol{v}} \Big(\boldsymbol{x}-\frac{\boldsymbol{m}}{N}\Big)^{\boldsymbol{u}} \Big(\boldsymbol{w}-\frac{\boldsymbol{n}}{N}\Big)^{\boldsymbol{v}}.  \end{gather*}
Then $g_{\boldsymbol{m},\boldsymbol{n},\boldsymbol{u},\boldsymbol{v}}(\boldsymbol{x},\boldsymbol{w})$ is supported on
$$  \left\{(\boldsymbol{x},\boldsymbol{w})\in[0,1]^{d_x}\times[0,1]^{d_w}: \boldsymbol{x}\in\left(\frac{\boldsymbol{m}-\boldsymbol{1}}{N},\frac{\boldsymbol{m}}{N}\right],  \left|w_i-\frac{n_i}{N}\right|\le \frac{1}{N},~1\le i\le d_w\right\}, $$
where the first $\boldsymbol{x}$-cell is understood coordinatewise to contain the left endpoint $0$.
For $\boldsymbol{x}\in[0,1]^{d_x}$ and $\boldsymbol{w}\in[0,1]^{d_w}$, we define
$$ g(\boldsymbol{x},\boldsymbol{w}) :=\sum_{\boldsymbol{m}\in[N]^{d_x},\boldsymbol{n}\in[N]^{d_w}} \psi_{\boldsymbol{m},\boldsymbol{n}}(\boldsymbol{x},\boldsymbol{w}) G_{\boldsymbol{m},\boldsymbol{n}}(\boldsymbol{x},\boldsymbol{w}). $$
Then,
\begin{align*} \big|f(\boldsymbol{x},\boldsymbol{w})-g(\boldsymbol{x},\boldsymbol{w})\big| &=\left|\sum_{\boldsymbol{m},\boldsymbol{n}} \psi_{\boldsymbol{m},\boldsymbol{n}}(\boldsymbol{x},\boldsymbol{w}) f(\boldsymbol{x},\boldsymbol{w}) -\sum_{\boldsymbol{m},\boldsymbol{n}} \psi_{\boldsymbol{m},\boldsymbol{n}}(\boldsymbol{x},\boldsymbol{w}) G_{\boldsymbol{m},\boldsymbol{n}}(\boldsymbol{x},\boldsymbol{w})\right|\\ &\le\sum_{\boldsymbol{m},\boldsymbol{n}} \psi_{\boldsymbol{m},\boldsymbol{n}}(\boldsymbol{x},\boldsymbol{w}) \big|f(\boldsymbol{x},\boldsymbol{w}) -G_{\boldsymbol{m},\boldsymbol{n}}(\boldsymbol{x},\boldsymbol{w})\big|. \end{align*}
Using the multivariate Taylor formula, there exists $\theta\in[0,1]$ such that
\begin{align*} f(\boldsymbol{x},\boldsymbol{w}) ={}&\sum_{\norm{\boldsymbol{u}}_1+\norm{\boldsymbol{v}}_1<p} C_{\boldsymbol{m},\boldsymbol{n},\boldsymbol{u},\boldsymbol{v}} \Big(\boldsymbol{x}-\frac{\boldsymbol{m}}{N}\Big)^{\boldsymbol{u}} \Big(\boldsymbol{w}-\frac{\boldsymbol{n}}{N}\Big)^{\boldsymbol{v}}\\ &+\sum_{p(\boldsymbol{u},\boldsymbol{v})} \frac{1}{\boldsymbol{u}!\boldsymbol{v}!} \partial^{\boldsymbol{u}}_{\boldsymbol{x}}\partial^{\boldsymbol{v}}_{\boldsymbol{w}} f(\boldsymbol{x}^{\prime},\boldsymbol{w}^{\prime}) \Big(\boldsymbol{x}-\frac{\boldsymbol{m}}{N}\Big)^{\boldsymbol{u}} \Big(\boldsymbol{w}-\frac{\boldsymbol{n}}{N}\Big)^{\boldsymbol{v}}, \end{align*}
where
\begin{gather*} p(\boldsymbol{u},\boldsymbol{v}) :=\big(\norm{\boldsymbol{u}}_1+\norm{\boldsymbol{v}}_1=p\big),\\  \boldsymbol{x}^{\prime} =(1-\theta)\frac{\boldsymbol{m}}{N}+\theta\boldsymbol{x}, ~~ \boldsymbol{w}^{\prime} =(1-\theta)\frac{\boldsymbol{n}}{N}+\theta\boldsymbol{w}. 
\end{gather*}
Subtracting $G_{\boldsymbol{m},\boldsymbol{n}}$ from this expansion gives
\begin{align*} &\psi_{\boldsymbol{m},\boldsymbol{n}}(\boldsymbol{x},\boldsymbol{w}) \big|f(\boldsymbol{x},\boldsymbol{w}) -G_{\boldsymbol{m},\boldsymbol{n}}(\boldsymbol{x},\boldsymbol{w})\big|\\ &\le \psi_{\boldsymbol{m},\boldsymbol{n}}(\boldsymbol{x},\boldsymbol{w}) \sum_{p(\boldsymbol{u},\boldsymbol{v})} \frac{1}{\boldsymbol{u}!\boldsymbol{v}!} \left| \partial^{\boldsymbol{u}}_{\boldsymbol{x}}\partial^{\boldsymbol{v}}_{\boldsymbol{w}} f(\boldsymbol{x}^{\prime},\boldsymbol{w}^{\prime}) - \partial^{\boldsymbol{u}}_{\boldsymbol{x}}\partial^{\boldsymbol{v}}_{\boldsymbol{w}} f\Big(\frac{\boldsymbol{m}}{N},\frac{\boldsymbol{n}}{N}\Big) \right|\\ &\hspace{3em}\times \Big|\boldsymbol{x}-\frac{\boldsymbol{m}}{N}\Big|^{\boldsymbol{u}} \Big|\boldsymbol{w}-\frac{\boldsymbol{n}}{N}\Big|^{\boldsymbol{v}}\\ &\le 2^p\Sigma  \psi_{\boldsymbol{m},\boldsymbol{n}}(\boldsymbol{x},\boldsymbol{w}) \sum_{p(\boldsymbol{u},\boldsymbol{v})} \frac{1}{\boldsymbol{u}!\boldsymbol{v}!} \Big|\boldsymbol{x}-\frac{\boldsymbol{m}}{N}\Big|^{\boldsymbol{u}} \Big|\boldsymbol{w}-\frac{\boldsymbol{n}}{N}\Big|^{\boldsymbol{v}} \norm{\theta(\boldsymbol{x},\boldsymbol{w}) -\frac{\theta(\boldsymbol{m},\boldsymbol{n})}{N}}_{\infty}^{q}. \end{align*}
Whenever $\psi_{\boldsymbol{m},\boldsymbol{n}}(\boldsymbol{x},\boldsymbol{w})\neq0$, we have
$\left|x_i-\frac{m_i}{N}\right|\le \frac{1}{N}$ for $1\le i\le d_x$ and
$\left|w_i-\frac{n_i}{N}\right|\le \frac{1}{N}$ for $1\le i\le d_w$. Therefore,
\begin{align*} &\psi_{\boldsymbol{m},\boldsymbol{n}}(\boldsymbol{x},\boldsymbol{w}) \big|f(\boldsymbol{x},\boldsymbol{w}) -G_{\boldsymbol{m},\boldsymbol{n}}(\boldsymbol{x},\boldsymbol{w})\big|\\ &\le 2^p\Sigma  \psi_{\boldsymbol{m},\boldsymbol{n}}(\boldsymbol{x},\boldsymbol{w}) \sum_{p(\boldsymbol{u},\boldsymbol{v})} \frac{1}{\boldsymbol{u}!\boldsymbol{v}!  N^{\norm{\boldsymbol{u}}_1+\norm{\boldsymbol{v}}_1+q}}\\ &= 2^p\Sigma  \psi_{\boldsymbol{m},\boldsymbol{n}}(\boldsymbol{x},\boldsymbol{w}) \frac{(d_x+d_w)^p}{p! N^{p+q}} \lesssim \psi_{\boldsymbol{m},\boldsymbol{n}}(\boldsymbol{x},\boldsymbol{w}) N^{-\alpha}. \end{align*}
It follows that
$$ \big|f(\boldsymbol{x},\boldsymbol{w})-g(\boldsymbol{x},\boldsymbol{w})\big| \lesssim \sum_{\boldsymbol{m},\boldsymbol{n}} \psi_{\boldsymbol{m},\boldsymbol{n}}(\boldsymbol{x},\boldsymbol{w}) N^{-\alpha} =N^{-\alpha}. $$
Finally, define
$$ g(\boldsymbol{x}|\boldsymbol{w}) =g(\boldsymbol{x},\boldsymbol{w}), ~~ \boldsymbol{x}\in[0,1]^{d_x},~ \boldsymbol{w}\in[0,1]^{d_w}. $$
Then
$$ \big|p(\boldsymbol{x}|\boldsymbol{w})-g(\boldsymbol{x}|\boldsymbol{w})\big| =\big|f(\boldsymbol{x},\boldsymbol{w})-g(\boldsymbol{x},\boldsymbol{w})\big| \lesssim N^{-\alpha}. $$
The proof is complete.
\end{proof}

By additionally incorporating   Assumption \ref{apt:diffusion-density-smooth}, we can derive the following results.
\begin{lemma}\label{lem:density-smooth-local-polynomials}
Suppose Assumptions \ref{apt:diffusion-density-bound}-\ref{apt:diffusion-density-smooth} hold. Let $N \gg 1$, 
there exists ${g(\boldsymbol{x}|\boldsymbol{w})}$ that satisfies
$$
\big|p(\boldsymbol{x}|\boldsymbol{w}) - g(\boldsymbol{x}|\boldsymbol{w})\big| \lesssim N^{-\alpha},~ \boldsymbol{x}\in [0,1]^{d_x},~\boldsymbol{w}\in [0,1]^{d_w}	,
$$
and
$$
\left|p(\boldsymbol{x}|\boldsymbol{w}) - g(\boldsymbol{x}|\boldsymbol{w})\right| \lesssim N^{-(3\alpha+ 2)}, ~ \boldsymbol{x} \in [0,1]^{d_x}\backslash [a_0, 1-a_0]^{d_x},~\boldsymbol{w}\in [0,1]^{d_w}.
$$
Moreover, $g(\boldsymbol{x}|\boldsymbol{w})$ has the following form
\begin{align*}
g(\boldsymbol{x}|\boldsymbol{w}) =&
\sum_{\boldsymbol{m},\boldsymbol{n}}~ \sum_{\norm{\boldsymbol{u}}_1+\norm{\boldsymbol{v}}_1< \alpha}C^{(0)}_{\boldsymbol{m},\boldsymbol{n},\boldsymbol{u},\boldsymbol{v}}g_{\boldsymbol{m},\boldsymbol{n},\boldsymbol{u},\boldsymbol{v}}(\boldsymbol{x},\boldsymbol{w}) \mathbbm{1}_{\{\boldsymbol{x}\in [a_0, 1-a_0]^{d_x}\}}\\
&+ \sum_{\boldsymbol{m},\boldsymbol{n}}~ \sum_{\norm{\boldsymbol{u}}_1+\norm{\boldsymbol{v}}_1< 3\alpha+2}C^{(1)}_{\boldsymbol{m},\boldsymbol{n},\boldsymbol{u},\boldsymbol{v}}g_{\boldsymbol{m},\boldsymbol{n},\boldsymbol{u},\boldsymbol{v}}(\boldsymbol{x},\boldsymbol{w})  \mathbbm{1}_{\{\boldsymbol{x}\in [0,1]^{d_x}\backslash [a_0, 1-a_0]^{d_x}\}},
\end{align*}
where $C^{(0)}_{\boldsymbol{m},\boldsymbol{n},\boldsymbol{u},\boldsymbol{v}}$ and $C^{(1)}_{\boldsymbol{m},\boldsymbol{n},\boldsymbol{u},\boldsymbol{v}}$ satisfy
$$
|C^{(0)}_{\boldsymbol{m},\boldsymbol{n},\boldsymbol{u},\boldsymbol{v}}| \lesssim \frac{1}{\boldsymbol{u}! \boldsymbol{v}!}, ~~ |C^{(1)}_{\boldsymbol{m},\boldsymbol{n},\boldsymbol{u},\boldsymbol{v}}| \lesssim \frac{1}{\boldsymbol{u}! \boldsymbol{v}!}.
$$
\end{lemma}

\begin{proof}
Since $p(\boldsymbol{x}|\boldsymbol{w})\in\mathcal{H}^{\alpha}([0,1]^{d_x}\times [0,1]^{d_w}, \Sigma)$, by Lemma \ref{lem:density-holder-local-polynomials}, we can construct
$$
g^{(0)}(\boldsymbol{x},\boldsymbol{w}) =  \sum_{\boldsymbol{m},\boldsymbol{n}}~ \sum_{\norm{\boldsymbol{u}}_1+\norm{\boldsymbol{v}}_1< \alpha}C^{(0)}_{\boldsymbol{m},\boldsymbol{n},\boldsymbol{u},\boldsymbol{v}}g_{\boldsymbol{m},\boldsymbol{n},\boldsymbol{u},\boldsymbol{v}}\big(\boldsymbol{x},\boldsymbol{w}\big) \mathbbm{1}_{\{\boldsymbol{x}\in [0,1]^{d_x}\}}
$$
such that
$$
\left|p(\boldsymbol{x}|\boldsymbol{w}) - g^{(0)}(\boldsymbol{x},\boldsymbol{w})\right| \lesssim N^{-\alpha}, ~ \boldsymbol{x} \in [0,1]^{d_x},~\boldsymbol{w}\in [0,1]^{d_w}, 
$$
where 
$$|C^{(0)}_{\boldsymbol{m},\boldsymbol{n},\boldsymbol{u},\boldsymbol{v}}| \leq \frac{2^{\alpha}\Sigma}{\boldsymbol{u}!\boldsymbol{v}!} \lesssim \frac{1}{\boldsymbol{u}!\boldsymbol{v}!}.
$$
By Assumption~\ref{apt:diffusion-density-smooth}, there exists a function $g^{\prime}\in\mathcal{H}^{3\alpha+2}([-1,1]^{d_x}\times[0,1]^{d_w},\Sigma_0)$ such that $g^{\prime}(\boldsymbol{x},\boldsymbol{w})=p(\boldsymbol{x}|\boldsymbol{w})$ on $\big([0,1]^{d_x}\backslash[a_0,1-a_0]^{d_x}\big)\times[0,1]^{d_w}$. Hence,
\[
\left|g^{\prime}(\boldsymbol{x},\boldsymbol{w})-p(\boldsymbol{x}|\boldsymbol{w})\right|=0\lesssim N^{-(3\alpha+2)}
\]
on $\big([0,1]^{d_x}\backslash[a_0,1-a_0]^{d_x}\big)\times[0,1]^{d_w}$.
By using Lemma \ref{lem:density-holder-local-polynomials} and replacing $p(\boldsymbol{x}|\boldsymbol{w}) $ with $g^{\prime}(\boldsymbol{x},\boldsymbol{w}) $, we obtain 
$$
g^{(1)}(\boldsymbol{x},\boldsymbol{w}) =  \sum_{\boldsymbol{m},\boldsymbol{n}}~ \sum_{\norm{\boldsymbol{u}}_1+\norm{\boldsymbol{v}}_1< 3\alpha+2}C^{(1)}_{\boldsymbol{m},\boldsymbol{n},\boldsymbol{u},\boldsymbol{v}}g_{\boldsymbol{m},\boldsymbol{n},\boldsymbol{u},\boldsymbol{v}}\big(\boldsymbol{x},\boldsymbol{w}\big) \mathbbm{1}_{\{\boldsymbol{x}\in [0,1]^{d_x}\}}
$$
such that
$$
\big|g^{\prime}(\boldsymbol{x},\boldsymbol{w})- g^{(1)}(\boldsymbol{x},\boldsymbol{w})\big| \lesssim N^{-(3\alpha + 2)}, ~ \boldsymbol{x} \in [0,1]^{d_x},~\boldsymbol{w}\in [0,1]^{d_w},
$$
where
$$|C^{(1)}_{\boldsymbol{m},\boldsymbol{n},\boldsymbol{u},\boldsymbol{v}}| \leq \frac{2^{3\alpha+2}\Sigma_0}{\boldsymbol{u}!\boldsymbol{v}!} \lesssim \frac{1}{\boldsymbol{u}!\boldsymbol{v}!}.
$$
Let $g(\boldsymbol{x}|\boldsymbol{w}) = g^{(0)}(\boldsymbol{x},\boldsymbol{w})\mathbbm{1}_{\{\boldsymbol{x}\in [a_0, 1-a_0]^{d_x}\}} + g^{(1)}(\boldsymbol{x},\boldsymbol{w})\mathbbm{1}_{\{\boldsymbol{x}\in [0,1]^{d_x}\backslash [a_0, 1-a_0]^{d_x}\}}$, that is,
\begin{align*}
g(\boldsymbol{x}|\boldsymbol{w}) =&
\sum_{\boldsymbol{m},\boldsymbol{n}}~ \sum_{\norm{\boldsymbol{u}}_1+\norm{\boldsymbol{v}}_1< \alpha}C^{(0)}_{\boldsymbol{m},\boldsymbol{n},\boldsymbol{u},\boldsymbol{v}}g_{\boldsymbol{m},\boldsymbol{n},\boldsymbol{u},\boldsymbol{v}}(\boldsymbol{x},\boldsymbol{w}) \mathbbm{1}_{\{\boldsymbol{x}\in [a_0, 1-a_0]^{d_x}\}}\\
&+ \sum_{\boldsymbol{m},\boldsymbol{n}}~ \sum_{\norm{\boldsymbol{u}}_1+\norm{\boldsymbol{v}}_1< 3\alpha+2}C^{(1)}_{\boldsymbol{m},\boldsymbol{n},\boldsymbol{u},\boldsymbol{v}}g_{\boldsymbol{m},\boldsymbol{n},\boldsymbol{u},\boldsymbol{v}}(\boldsymbol{x},\boldsymbol{w})  \mathbbm{1}_{\{\boldsymbol{x}\in [0,1]^{d_x}\backslash [a_0, 1-a_0]^{d_x}\}}.
\end{align*}
Then, $g(\boldsymbol{x}|\boldsymbol{w})$ satisfies that
\begin{align*}
\big| p(\boldsymbol{x}|\boldsymbol{w}) - g(\boldsymbol{x}|\boldsymbol{w}) \big|\lesssim& \big| p(\boldsymbol{x}|\boldsymbol{w}) -g^{(0)}(\boldsymbol{x},\boldsymbol{w}) \big| \mathbbm{1}_{\{\boldsymbol{x}\in [a_0, 1-a_0]^{d_x}\}}\\
+& \Big (\big| p(\boldsymbol{x}|\boldsymbol{w}) -g^{\prime}(\boldsymbol{x},\boldsymbol{w}) \big|+\big| g^{\prime}(\boldsymbol{x},\boldsymbol{w})-g^{(1)}(\boldsymbol{x},\boldsymbol{w})\big| \Big)\mathbbm{1}_{\{\boldsymbol{x}\in [0,1]^{d_x}\backslash [a_0, 1-a_0]^{d_x}\}}\\
\lesssim &  N^{-\alpha}\mathbbm{1}_{\{\boldsymbol{x}\in [a_0, 1-a_0]^{d_x}\}}+N^{-(3\alpha+2)}\mathbbm{1}_{\{\boldsymbol{x}\in [0,1]^{d_x}\backslash [a_0, 1-a_0]^{d_x}\}},
\end{align*}
which implies that
$$
\left|p(\boldsymbol{x}|\boldsymbol{w}) - g(\boldsymbol{x}|\boldsymbol{w})\right| \lesssim N^{-\alpha} + N^{-(3\alpha + 2)} \lesssim N^{-\alpha}, ~ \boldsymbol{x} \in [0,1]^{d_x},~\boldsymbol{w}\in [0,1]^{d_w},
$$
and 
$$
\left|p(\boldsymbol{x}|\boldsymbol{w}) - g(\boldsymbol{x}|\boldsymbol{w})\right| \lesssim N^{-(3\alpha + 2)}, ~ \boldsymbol{x} \in [0,1]^{d_x} \backslash [a_0, 1-a_0]^{d_x},~\boldsymbol{w}\in [0,1]^{d_w}.
$$
The proof is complete.
\end{proof}

\subsection{Approximating of $p_t(\boldsymbol{x}|\boldsymbol{w})$ and $\sigma_t\nabla p_t({\boldsymbol{x}}|\boldsymbol{w})$ via Local Polynomial Integrals}\label{subsec:c2}
In this section, we approximate $p_t(\boldsymbol{x}|\boldsymbol{w})$ and $\sigma_t \nabla p_t(\boldsymbol{x}|\boldsymbol{w})$ by means of local polynomial integrals. Under Assumptions~\ref{apt:diffusion-density-bound}--\ref{apt:diffusion-density-smooth}, the expression for $g(\boldsymbol{x}|\boldsymbol{w})$ given in Lemma~\ref{lem:density-smooth-local-polynomials} can be reformulated as follows
\begin{align*}
g(\boldsymbol{x}|\boldsymbol{w})=&
\sum_{\boldsymbol{m},\boldsymbol{n}}~ \sum_{\norm{\boldsymbol{u}}_1+\norm{\boldsymbol{v}}_1< \alpha}C^{(0)}_{\boldsymbol{m},\boldsymbol{n},\boldsymbol{u},\boldsymbol{v}}g_{\boldsymbol{m},\boldsymbol{n},\boldsymbol{u},\boldsymbol{v}}(\boldsymbol{x},\boldsymbol{w}) \mathbbm{1}_{\{\boldsymbol{x}\in [a_0, 1-a_0]^{d_x}\}}\\
&+ \sum_{\boldsymbol{m},\boldsymbol{n}}~ \sum_{\norm{\boldsymbol{u}}_1+\norm{\boldsymbol{v}}_1< 3\alpha+2}C^{(1)}_{\boldsymbol{m},\boldsymbol{n},\boldsymbol{u},\boldsymbol{v}}g_{\boldsymbol{m},\boldsymbol{n},\boldsymbol{u},\boldsymbol{v}}(\boldsymbol{x},\boldsymbol{w})  \mathbbm{1}_{\{\boldsymbol{x}\in [0,1]^{d_x}\}}\\
&- \sum_{\boldsymbol{m},\boldsymbol{n}}~ \sum_{\norm{\boldsymbol{u}}_1+\norm{\boldsymbol{v}}_1< 3\alpha+2}C^{(1)}_{\boldsymbol{m},\boldsymbol{n},\boldsymbol{u},\boldsymbol{v}}g_{\boldsymbol{m},\boldsymbol{n},\boldsymbol{u},\boldsymbol{v}}(\boldsymbol{x},\boldsymbol{w})  \mathbbm{1}_{\{\boldsymbol{x}\in [a_0, 1-a_0]^{d_x}\}}.
\end{align*}
It can be formulated as the combination of the following functional forms
\begin{equation}\label{eq:diffusion-g-unit-express}
\sum_{\boldsymbol{m},\boldsymbol{n}}~ \sum_{\alpha(\boldsymbol{u},\boldsymbol{v})< C_\alpha}C_{\boldsymbol{m},\boldsymbol{n},\boldsymbol{u},\boldsymbol{v}}g_{\boldsymbol{m},\boldsymbol{n},\boldsymbol{u},\boldsymbol{v}}(\boldsymbol{x},\boldsymbol{w})\mathbbm{1}_{\{\boldsymbol{x}\in [C_{\beta},C_{\gamma}]^{d_x}\}},
\end{equation}
where 
\begin{gather*}
C_{\alpha} \in \{\alpha, 3\alpha + 2\}, ~C_{\beta}\in \{0,a_0\},~C_{\gamma} \in \{1-a_0, 1\},\\
\alpha(\boldsymbol{u},\boldsymbol{v})=\norm{\boldsymbol{u}}_1 + \norm{\boldsymbol{v}}_1,~ C_{\boldsymbol{m},\boldsymbol{n},\boldsymbol{u},\boldsymbol{v}}\in \{ C^{(0)}_{\boldsymbol{m},\boldsymbol{n},\boldsymbol{u},\boldsymbol{v}},C^{(1)}_{\boldsymbol{m},\boldsymbol{n},\boldsymbol{u},\boldsymbol{v}}\}.
\end{gather*}
By replacing $p$ with $g$, we denote 
\begin{equation}\label{eq:diffusion-local-poly-g1}
g_1(t,\boldsymbol{x},\boldsymbol{w}) := \int_{\mathbb{R}^{d_x}}g(\boldsymbol{z}|\boldsymbol{w} )\mathbbm{1}_{\{\boldsymbol{z}\in [0,1]^{d_x}\}}\cdot \frac{1}{\sigma_t^{d_x}(2\pi)^{\frac{d_x}{2}}} \exp\left(-\frac{\Vert{\boldsymbol{x}}-\mu_t \boldsymbol{z}\Vert ^2}{2\sigma_t^2}\right) d\boldsymbol{z}.
\end{equation}
By virtue of Lemma \ref{lem:density-holder-local-polynomials} or Lemma \ref{lem:density-smooth-local-polynomials}, the discrepancy between  $p_t({\boldsymbol{x}}|\boldsymbol{w})$ and $g_1(t,\boldsymbol{x},\boldsymbol{w}) $ admits the following bound
\begin{align*}
\big|  p_t({\boldsymbol{x}}|\boldsymbol{w}) -g_1(t,\boldsymbol{x},\boldsymbol{w})  \big| &\leq \int_{\mathbb{R}^{d_x}}\big|  p(\boldsymbol{z}|\boldsymbol{w})-g(\boldsymbol{z}|\boldsymbol{w}) \big|\cdot \frac{\mathbbm{1}_{\{\boldsymbol{z}\in [0,1]^{d_x}\}}}{\sigma_t^{d_x}(2\pi)^{\frac{d_x}{2}}} \exp\left(-\frac{\Vert{\boldsymbol{x}}-\mu_t \boldsymbol{z}\Vert ^2}{2\sigma_t^2}\right)  d\boldsymbol{z} \\
& \lesssim N^{-\alpha}\cdot \int_{\mathbb{R}^{d_x}} \frac{\mathbbm{1}_{\{\boldsymbol{z}\in [0,1]^{d_x}\}}}{\sigma_t^{d_x}(2\pi)^{\frac{d_x}{2}}} \exp\left(-\frac{\Vert{\boldsymbol{x}}-\mu_t \boldsymbol{z}\Vert ^2}{2\sigma_t^2}\right)  d\boldsymbol{z}\lesssim N^{-\alpha}.
\end{align*}
Therefore, it remains to approximate $g_1(t,\boldsymbol{x},\boldsymbol{w})$. By Lemma~\ref{lem:integral-clipping}, for any $\epsilon>0$, there exists a constant $C>0$ such that, for all $\boldsymbol{x}\in\mathbb{R}^{d_x}$,
$$
\left|p_t({\boldsymbol{x}}|\boldsymbol{w}) - \int_{A_{{\boldsymbol{x}}}} p(\boldsymbol{z}|\boldsymbol{w})\cdot\frac{\mathbbm{1}_{\{\boldsymbol{z}\in [0,1]^{d_x}\}}}{\sigma_t^{d_x}(2\pi)^{\frac{d_x}{2}}} \exp\left(-\frac{\Vert{\boldsymbol{x}}-\mu_t \boldsymbol{z}\Vert ^2}{2\sigma_t^2}\right)  d\boldsymbol{z}\right| \lesssim \epsilon,
$$
where $A_{{\boldsymbol{x}}} = \prod_{i=1}^{d_x}a_{i,{\boldsymbol{x}}}$ with $a_{i,{\boldsymbol{x}}} = \Big[\frac{x_i - C\sigma_t\sqrt{\log\epsilon^{-1}}}{\mu_t}, \frac{x_i + C\sigma_t\sqrt{\log\epsilon^{-1}}}{\mu_t}\Big]$. We denote
$$
g_2(t,{\boldsymbol{x}},\boldsymbol{w}):= \int_{A_{{\boldsymbol{x}}}}g(\boldsymbol{z}|\boldsymbol{w})\cdot\frac{\mathbbm{1}_{\{\boldsymbol{z}\in [0,1]^{d_x}\}}}{\sigma_t^{d_x}(2\pi)^{\frac{d_x}{2}}} \exp\left(-\frac{\Vert{\boldsymbol{x}}-\mu_t \boldsymbol{z}\Vert ^2}{2\sigma_t^2}\right)  d\boldsymbol{z}.
$$
Since $p(\boldsymbol{x}|\boldsymbol{w})$ is bounded, $g(\boldsymbol{z}|\boldsymbol{w})$ inherits this boundedness. Substituting $p(\boldsymbol{z}|\boldsymbol{w})$ with $g(\boldsymbol{z}|\boldsymbol{w})$ in Lemma \ref{lem:integral-clipping},
we obtain a bound for the gap between $g_1(t,{\boldsymbol{x}},\boldsymbol{w})$ and $g_2(t,{\boldsymbol{x}},\boldsymbol{w})$
$$
\big| g_1(t,{\boldsymbol{x}},\boldsymbol{w})-g_2(t,{\boldsymbol{x}},\boldsymbol{w}) \big| \lesssim \epsilon.
$$
Note that $g_2(t,{\boldsymbol{x}},\boldsymbol{w})$ involves an integral of the exponential function, which poses non-trivial analytical challenges. To mitigate this difficulty, we leverage polynomial approximations for the exponential function.
For any $1\leq i \leq d_x$ and $\boldsymbol{z}\in A_{{\boldsymbol{x}}}$, the inequality  $\frac{|x_i - \mu_t z_i|}{\sigma_t} \leq C\sqrt{\log\epsilon^{-1}}$ holds. Exploiting this bound, we invoke the Taylor expansion to derive the following result
$$
\left|
\exp\left(-\frac{(x_i -\mu_t z_i)^2}{2\sigma_t^2}\right) - \sum_{j=0}^{k-1}\frac{1}{j!}\left(-\frac{(x_i-\mu_t z_i)^2}{2\sigma_t^2}\right)^j
\right| \leq \frac{C^{2k}\log^k\epsilon^{-1}}{k!2^k}, ~ z_i\in[C_{l}(x_i), C_{u}(x_i)],
$$
where
$$
C_l(x_i) = \max\left\{ \frac{x_i - C\sigma_t\sqrt{\log\epsilon^{-1}}}{\mu_t}, -1\right\},
$$
and
$$
C_u(x_i) = \min\left\{ \frac{x_i + C\sigma_t\sqrt{\log\epsilon^{-1}}}{\mu_t}  , 1\right\}.
$$
By choosing $k \geq \frac{3}{2}C^2 u\log\epsilon^{-1}$ and applying the inequality $k! \geq (\frac{k}{3})^k$ when $k \geq 3$, we obtain
$$
\left|
\exp\left(-\frac{(x_i -\mu_t z_i)^2}{2\sigma_t^2}\right) - \sum_{j=0}^{k-1}\frac{1}{j!}\left(-\frac{(x_i-\mu_t z_i)^2}{2\sigma_t^2}\right)^j
\right| \leq \epsilon^{\frac{3}{2}C^2 u\log u}.
$$
Accordingly, we may set
$$
u = \max\left\{e, \frac{2}{3C^2}\left(1 + \frac{\log d_x}{\log \epsilon^{-1}}\right)\right\}
$$
such that
$$
\epsilon^{\frac{3}{2}C^2u\log u} \leq \frac{\epsilon}{d_x},
$$
where $k = \mathcal{O}\left(\log \epsilon^{-1}\right)$. 
By multiplying over the $d_x$ dimensions indexed by $i$, we have
\begin{align*}
& \left|
\exp\left(-\frac{\Vert{\boldsymbol{x}} - \mu_t \boldsymbol{z}\Vert^2}{2\sigma_t^2}\right) - \prod_{i=1}^{d_x}\sum_{j=0}^{k-1}\frac{1}{j!}\left(-\frac{(x_i-\mu_t z_i)^2}{2\sigma_t^2}\right)^j
\right| \\
& = \left|
\prod_{i=1}^{d_x} \exp\left(-\frac{(x_i - \mu_t z_i)^2}{2\sigma_t^2}\right) - \prod_{i=1}^{d_x}\sum_{j=0}^{k-1}\frac{1}{j!}\left(-\frac{(x_i-\mu_t z_i)^2}{2\sigma_t^2}\right)^j
\right| \\
& \leq \sum_{i=1}^{d_x} \left|
\exp\left(-\frac{(x_i -\mu_t z_i)^2}{2\sigma_t^2}\right) - \sum_{j=0}^{k-1}\frac{1}{j!}\left(-\frac{(x_i-\mu_t z_i)^2}{2\sigma_t^2}\right)^j
\right| \\
&~~ \cdot  \prod_{j=0,j\ne i}^{d_x} \max \left\{ \Big|\exp\Big(-\frac{(x_i -\mu_t z_i)^2}{2\sigma_t^2}\Big) \Big|,\Big | \sum_{j=0}^{k-1}\frac{1}{j!}\Big(-\frac{(x_i-\mu_t z_i)^2}{2\sigma_t^2}\Big)^j \Big|  \right\} \\
& \leq d_x\cdot \frac{\epsilon}{d_x} \cdot \big(1+\frac{\epsilon}{d_x}\big)^{d_x+1}\\
& \lesssim \epsilon.
\end{align*}
Therefore, we only need to approximate
$$
g_3(t,{\boldsymbol{x}},\boldsymbol{w}):= 
\int_{A_{{\boldsymbol{x}}}}g(\boldsymbol{z}|\boldsymbol{w})\mathbbm{1}_{\{\boldsymbol{z}\in [0,1]^{d_x}\}}\cdot
\prod_{i=1}^{d_x} \frac{\mathcal{M}^k_i(\boldsymbol{z})}{\sigma_t(2\pi)^{1/2}}
d\boldsymbol{z},
$$
where
$$
\mathcal{M}^k_i(\boldsymbol{z}): =  \sum_{j=0}^{k-1}\frac{1}{j!}\left(-\frac{(x_i-\mu_t z_i)^2}{2\sigma_t^2}\right)^j.
$$
Observe that the difference between $g_2(t,{\boldsymbol{x}},\boldsymbol{w})$ and $g_3(t,{\boldsymbol{x}},\boldsymbol{w})$ can be bounded as follows
$$
\begin{aligned}
\big|g_2(t,{\boldsymbol{x}},\boldsymbol{w}) - g_3(t,{\boldsymbol{x}},\boldsymbol{w})\big| &\lesssim \epsilon \cdot \frac{1}{\sigma_t^{d_x}(2\pi)^{\frac{d_x}{2}}} \int_{A_{{\boldsymbol{x}}}} g(\boldsymbol{z}|\boldsymbol{w})\mathbbm{1}_{\{\boldsymbol{z}\in [0,1]^{d_x}\}}  d\boldsymbol{z} \\
& \lesssim \epsilon \cdot \frac{1}{\sigma_t^{d_x}(2\pi)^{\frac{d_x}{2}}} \int_{A_{{\boldsymbol{x}}}} \left(p(\boldsymbol{z}|\boldsymbol{w}) + {N}^{-\alpha}\right)\mathbbm{1}_{\{\boldsymbol{z}\in [0,1]^{d_x}\}}  d\boldsymbol{z} \\
& \lesssim \epsilon \cdot (B_u + N^{-\alpha}) \frac{\left(2C\sigma_t\sqrt{\log\epsilon^{-1}}\right)^{d_x}}{\mu_t ^{d_x}\sigma^{d_x}_t} \\
& \lesssim \epsilon \cdot  \frac{\log^{\frac{d_x}{2}} \epsilon^{-1}}{\mu_t^{d_x}}.
\end{aligned}
$$
Additionally, we derive a second bound
$$
\begin{aligned}
\big|g_2(t,{\boldsymbol{x}},\boldsymbol{w}) - g_3(t,{\boldsymbol{x}},\boldsymbol{w})\big| &\lesssim \epsilon \cdot \frac{1}{\sigma_t^{d_x}(2\pi)^{\frac{d_x}{2}}} \int_{A_{{\boldsymbol{x}}}} g(\boldsymbol{z}|\boldsymbol{w})\mathbbm{1}_{\{\boldsymbol{z}\in [0,1]^{d_x}\}}  d\boldsymbol{z} \\
& \lesssim \epsilon \cdot \frac{1}{\sigma_t^{d_x}(2\pi)^{\frac{d_x}{2}}} \int_{A_{{\boldsymbol{x}}}} \left(p(\boldsymbol{z}|\boldsymbol{w}) + {N}^{-\alpha}\right)\mathbbm{1}_{\{\boldsymbol{z}\in [0,1]^{d_x}\}}  d\boldsymbol{z} \\
& \lesssim \epsilon \cdot \frac{B_u + N^{-\alpha}}{\sigma_t^{d_x}(2\pi)^{\frac{d_x}{2}}} \\
& \lesssim \epsilon \cdot  \frac{\log^{\frac{d_x}{2}} \epsilon^{-1}}{\sigma_t^{d_x}}.
\end{aligned}
$$
Combining the two estimates and using the identity $\sigma_t^2+\mu_t^2=1$, we conclude that
$$
\big|g_2(t,{\boldsymbol{x}},\boldsymbol{w}) - g_3(t,{\boldsymbol{x}},\boldsymbol{w})\big|\lesssim \epsilon   \log^{\frac{d_x}{2}} \epsilon^{-1}.
$$
This shows that it suffices to approximate $g_3(t,{\boldsymbol{x}},\boldsymbol{w})$ to sufficiently high accuracy using a ReLU neural network. Note that $g(\boldsymbol{z}|\boldsymbol{w})$ can be expressed as the combination of \eqref{eq:diffusion-g-unit-express}. Thus, $g_3(t,{\boldsymbol{x}},\boldsymbol{w})$ can be expressed as the combination of the following functional forms	
\begin{equation} \label{eq:diffusion-local-poly-g3}
\begin{aligned}
&g_3(t,{\boldsymbol{x}},\boldsymbol{w})= 
\int_{A_{{\boldsymbol{x}}}}g(\boldsymbol{z}|\boldsymbol{w})\mathbbm{1}_{\{\boldsymbol{z}\in [0,1]^{d_x}\}}\cdot \prod_{i=1}^{d_x}
\frac{\mathcal{M}^k_i(\boldsymbol{z})}{\sigma_t(2\pi)^{1/2}}
d\boldsymbol{z}\\
&=\sum_{\boldsymbol{m},\boldsymbol{n}}~ \sum_{\alpha(\boldsymbol{u},\boldsymbol{v})< C_\alpha}C_{\boldsymbol{m},\boldsymbol{n},\boldsymbol{u},\boldsymbol{v}}
\int_{A_{{\boldsymbol{x}}}}
g_{\boldsymbol{m},\boldsymbol{n},\boldsymbol{u},\boldsymbol{v}}(\boldsymbol{z},\boldsymbol{w})\mathbbm{1}_{\{\boldsymbol{z}\in [C_{\beta},C_{\gamma}]^{d_x}\}}  
\prod_{i=1}^{d_x} \frac{\mathcal{M}^k_i(\boldsymbol{z})}{\sigma_t(2\pi)^{1/2}}
d\boldsymbol{z} \\
&= \sum_{\boldsymbol{m},\boldsymbol{n}}~ \sum_{\alpha(\boldsymbol{u},\boldsymbol{v})< C_\alpha}C_{\boldsymbol{m},\boldsymbol{n},\boldsymbol{u},\boldsymbol{v}}
\Psi_{\boldsymbol{m},\boldsymbol{n},\boldsymbol{u},\boldsymbol{v}}(t,
\boldsymbol{x},\boldsymbol{w}),
\end{aligned}
\end{equation}
In this representation, the $\boldsymbol{z}$-cell indicator is incorporated into the integration domains $D_{i,m}$, whereas the factor $\eta_{\boldsymbol{n},N}(\boldsymbol{w})$ remains outside the integral. Here
\begin{gather*}
\Psi_{\boldsymbol{m},\boldsymbol{n},\boldsymbol{u},\boldsymbol{v}}(t,
\boldsymbol{x},\boldsymbol{w}):= \eta_{\boldsymbol{n},N}(\boldsymbol{w})\prod_{i=1}^{d_w} \Big(w_i-\frac{n_i}{N}\Big)^{v_i}\cdot \prod_{i=1}^{d_x}   \sum_{j=0}^{k-1}\psi(t,x_i,m_i,u_i,j), \\
\psi(t,r,m,u,j):=\frac{1}{\sigma_t(2\pi)^{1/2}}\int_{D_{i,m}}\mathbbm{1}_{\{z\in[C_\beta,C_\gamma]\}}\left(z-\frac{m}{N}\right)^u\frac{1}{j!}\left(-\frac{(r-\mu_tz)^2}{2\sigma_t^2}\right)^jdz\\
D_{i,m}:=a_{i,x}\cap\Big(\frac{m-1}{N},\frac{m}{N}\Big].
\end{gather*}
Similarly, we can define $\boldsymbol{h}_1(t,\boldsymbol{x},\boldsymbol{w})$ to approximate $\sigma_t\nabla p_t({\boldsymbol{x}}|\boldsymbol{w})$, where
$$
\boldsymbol{h}_1(t,{\boldsymbol{x}},\boldsymbol{w}) :=  \int_{\mathbb{R}^{d_x}}g(\boldsymbol{z}|\boldsymbol{w})\mathbbm{1}_{\{\boldsymbol{z}\in [0,1]^{d_x}\}}\cdot
\frac{\mu_t \boldsymbol{z}-\boldsymbol{x}}{\sigma_t^{d_x+1}(2\pi)^{\frac{d_x}{2}}} \exp\left(-\frac{\Vert{\boldsymbol{x}}-\mu_t \boldsymbol{z}\Vert ^2}{2\sigma_t^2}\right)  d\boldsymbol{z}.
$$
The difference between $\boldsymbol{h}_1(t,{\boldsymbol{x}},\boldsymbol{w})$ and $\sigma_t\nabla p_t({\boldsymbol{x}}|\boldsymbol{w})$ can be bounded as
$$
\norm{\boldsymbol{h}_1(t,{\boldsymbol{x}},\boldsymbol{w})-\sigma_t\nabla p_t({\boldsymbol{x}}|\boldsymbol{w})} \lesssim N^{-\alpha}.
$$
We can also define $\boldsymbol{h}_2(t,{\boldsymbol{x}},\boldsymbol{w})$ and $\boldsymbol{h}_3(t,{\boldsymbol{x}},\boldsymbol{w})$ as follows
\begin{align*}
\boldsymbol{h}_2(t,{\boldsymbol{x}},\boldsymbol{w}) :&=  \int_{A_{{\boldsymbol{x}}}}g(\boldsymbol{z}|\boldsymbol{w})\mathbbm{1}_{\{\boldsymbol{z}\in [0,1]^{d_x}\}}\cdot
\frac{\mu_t \boldsymbol{z}-\boldsymbol{x}}{\sigma_t^{d_x+1}(2\pi)^{\frac{d_x}{2}}} \exp\left(-\frac{\Vert{\boldsymbol{x}}-\mu_t \boldsymbol{z}\Vert ^2}{2\sigma_t^2}\right)  d\boldsymbol{z}, \\
\boldsymbol{h}_3(t,{\boldsymbol{x}},\boldsymbol{w}):&= \int_{A_{{\boldsymbol{x}}}}g(\boldsymbol{z}|\boldsymbol{w})\mathbbm{1}_{\{\boldsymbol{z}\in [0,1]^{d_x}\}}
\cdot 
\frac{\mu_t \boldsymbol{z}-\boldsymbol{x}}{\sigma_t^{d_x+1}(2\pi)^{\frac{d_x}{2}}} \prod_{i=1}^{d_x}
\mathcal{M}^k_i(\boldsymbol{z})
d\boldsymbol{z}.
\end{align*}
Then,  we have
$$
\norm{\boldsymbol{h}_1(t,{\boldsymbol{x}},\boldsymbol{w}) -\boldsymbol{h}_2(t,{\boldsymbol{x}},\boldsymbol{w}) } \lesssim \epsilon,
$$
and
$$
\norm{\boldsymbol{h}_2(t,{\boldsymbol{x}},\boldsymbol{w}) -\boldsymbol{h}_3(t,{\boldsymbol{x}},\boldsymbol{w}) }  \lesssim \epsilon \cdot \log^{\frac{d_x+1}{2}}\epsilon^{-1}.
$$
The $\ell$-th element of $\boldsymbol{h}_3(t,{\boldsymbol{x}},\boldsymbol{w})$ can be expressed as the combination of the following functional forms.
\begin{equation}\label{eq:diffusion-local-poly-int-h3} \begin{aligned} h_{3,\ell}(t,\boldsymbol{x},\boldsymbol{w})=&\sum_{\boldsymbol{m},\boldsymbol{n}}\sum_{\alpha(\boldsymbol{u},\boldsymbol{v})<C_\alpha}C_{\boldsymbol{m},\boldsymbol{n},\boldsymbol{u},\boldsymbol{v}}\eta_{\boldsymbol{n},N}(\boldsymbol{w})\bigg(\prod_{i=1}^{d_w}\left(w_i-\frac{n_i}{N}\right)^{v_i}\bigg)\\ &\times\bigg(\prod_{i\neq\ell}^{d_x}\sum_{j=0}^{k-1}\psi(t,x_i,m_i,u_i,j)\bigg)\bigg(\sum_{j=0}^{k-1}\psi^{\nabla}(t,x_\ell,m_\ell,u_\ell,j)\bigg). \end{aligned} \end{equation}
where
\begin{align*} \psi(t,x,m,u,j)&:=\frac{1}{\sigma_t(2\pi)^{\frac{1}{2}}}\int_{D_{i,m}}\mathbbm{1}_{\{z\in[C_\beta,C_\gamma]\}}\left(z-\frac{m}{N}\right)^u\frac{1}{j!}\left(-\frac{(x-\mu_tz)^2}{2\sigma_t^2}\right)^jdz,\\ \psi^{\nabla}(t,x,m,u,j)&:=\frac{1}{\sigma_t(2\pi)^{\frac{1}{2}}}\int_{D_{\ell,m}}\mathbbm{1}_{\{z\in[C_\beta,C_\gamma]\}}\left(z-\frac{m}{N}\right)^u\frac{\mu_tz-x}{\sigma_t}\frac{1}{j!}\left(-\frac{(x-\mu_tz)^2}{2\sigma_t^2}\right)^jdz. \end{align*}

\subsection{Approximating of Local Polynomial Integrals via ReLU Neural Networks}\label{subsec:c3}
In this section, we employ ReLU neural networks to approximate the local polynomial integrals we have derived. Specifically, we focus on the approximation of the quantity in \eqref{eq:diffusion-local-poly-g3} under the condition that $\Vert{\boldsymbol{x}}\Vert_{\infty} \leq C_0$, where $C_0 > 0$ is a fixed constant.

Before constructing a ReLU approximation of $\psi(t,x,m,u,j)$, we first approximate the coefficients $\mu_t=e^{-2t}$ and $\sigma_t^2=1-e^{-4t}$ appearing in its representation.

\begin{lemma}[Approximating the OU Coefficients]\label{lem:relu-ou-coefficients}
Let $T\geq1$, $a\in\{2,4\}$, and $0<\eta<1/2$. There exists a ReLU network $E_{a,T,\eta}:\mathbb{R}\rightarrow[0,1]$ such that
\[
\sup_{t\in[0,T]}\left|E_{a,T,\eta}(t)-e^{-at}\right|\leq\eta.
\]
Its parameters satisfy
\begin{gather*}
\cD=\mathcal{O}\left((T+\log\eta^{-1})^2\right),~~\cW=\mathcal{O}\left((T+\log\eta^{-1})^2\right),\\
\cS=\mathcal{O}\left((T+\log\eta^{-1})^3\right),~~\cB=\exp\left(\mathcal{O}(T+\log\eta^{-1})\right).
\end{gather*}
\end{lemma}

\begin{proof}
Let $M:=\left\lceil\max\left\{2eaT,\log_2 \left(\frac{2}{\eta}\right)\right\}\right\rceil,$ and $P_M(t):=\sum_{q=0}^{M}\frac{(-aT)^q}{q!}\left(\frac{t}{T}\right)^q.$
By Taylor's theorem and $(M+1)!\ge\left(\frac{M+1}{e}\right)^{M+1}$, we have
\[
\sup_{t\in[0,T]}|P_M(t)-e^{-at}|\le\frac{(aT)^{M+1}}{(M+1)!}\le\left(\frac{eaT}{M+1}\right)^{M+1}\le2^{-(M+1)}\le\frac{\eta}{2}.
\]
Set $\delta:=\frac{\eta e^{-aT}}{2}$. Following the convention above, put $\Pi_{0,\delta}(t):=1$ and $\Pi_{1,\delta}(t):=\frac{t}{T}$. For $2\le q\le M$, Lemma~\ref{lem:relu-product}, with $C=1$ and the shared input $\frac{t}{T}$, gives
\[
\Pi_{q,\delta}(t):=s_{{\rm prod},q} \Big(\underbrace{\frac{t}{T},\ldots,\frac{t}{T}}_{q \text{times}}\Big),~~ \sup_{t\in[0,T]}\left|\Pi_{q,\delta}(t)-\left(\frac{t}{T}\right)^q\right|\le\delta.
\]
By Lemma~\ref{lem:relu-parallelization}, 
\[
\widetilde P_M(t):=\sum_{q=0}^{M}\frac{(-aT)^q}{q!}\Pi_{q,\delta}(t),~~ \sup_{t\in[0,T]}|\widetilde P_M(t)-P_M(t)|\le\delta\sum_{q=0}^{M}\frac{(aT)^q}{q!}\le\frac{\eta}{2}.
\]
Define $E_{a,T,\eta}(t):=s_{\rm clip}(\widetilde P_M(t),0,1)$. Since $e^{-at}\in[0,1]$ and clipping is nonexpansive,
\[
\sup_{t\in[0,T]}|E_{a,T,\eta}(t)-e^{-at}|\le\eta.
\]
For $2\le q\le M$, Lemma~\ref{lem:relu-product} gives
\[
\mathcal D_q=\mathcal O \left(\log q [\log q+\log\delta^{-1}]\right),\quad \mathcal W_q=48q,\quad \mathcal S_q=\mathcal O \left(q[\log q+\log\delta^{-1}]\right),\quad \mathcal B_q=1.
\]
Because $M=\mathcal O(T+\log\eta^{-1})$ and $\log\delta^{-1}=\mathcal O(T+\log\eta^{-1})$, parallelization yields
\[
\mathcal D=\mathcal O \left((T+\log\eta^{-1})^2\right),~~ \mathcal W=\mathcal O \left((T+\log\eta^{-1})^2\right),~~ \mathcal S=\mathcal O \left((T+\log\eta^{-1})^3\right).
\]
Finally, $\max\limits_{0\le q\le M}\frac{(aT)^q}{q!}\le e^{aT}$, so the affine output layer and final clipping layer give $\mathcal B=\exp\left(\mathcal O\left(T+\log\eta^{-1}\right)\right)$.
\end{proof}

Using Lemma~\ref{lem:relu-ou-coefficients}, we now construct a ReLU neural network to approximate $\psi(t,x,m,u,j)$ over the domain $\|\boldsymbol{x}\|_{\infty}\leq C_0$.

\begin{lemma}\label{lem:relu-approximate-psi}
Given $N \gg 1$, $C_0 > 0$, for any $\epsilon_0 > 0$, $1\leq m \leq N$, $0\leq u < C_{\alpha}$, and $0\leq j \leq k - 1$, there exists a ReLU neural network $\psi_{\rm{R}} \in  \cG(\cD,\cW,\cS,\cB)$ with
$$
\begin{aligned}
&\cD = \mathcal{O}\left(\log^2\epsilon^{-1}_0 + \log^2C_0 + \log^2 N + \log^4 \epsilon^{-1} \right),\\ 
&\cW = \mathcal{O}\left(\log^3\epsilon^{-1}_0 + \log^3C_0 + \log^3 N + \log^6 \epsilon^{-1} \right), \\
&\cS = \mathcal{O}\left(\log^4\epsilon^{-1}_0 + \log^4C_0 + \log^4 N + \log^8 \epsilon^{-1} \right),\\ 
&\cB = \exp\left(\mathcal{O}\left(\log^2\epsilon^{-1}_0 + \log^2C_0 + \log^2 N + \log^4 \epsilon^{-1} \right)\right),
\end{aligned}
$$
such that 
$$
|\psi_{\rm{R}}(t,{x},m,u,j) - \psi (t,{x},m,u,j)| \lesssim \epsilon_0, ~ x\in[-C_0, C_0], ~ t\in[T_1, T-T_1].
$$
\end{lemma}

\begin{proof}
By the definition of $\psi (t,x,m,u,j)$, we take the transformation $y = \frac{x -\mu_t z}{\sigma_t}$, then
$$
\begin{aligned}
\psi (t,x,m,u,j) & =\frac{1}{\sigma_t(2\pi)^{1/2}} \int_{D_{i,m}} \mathbbm{1}_{\{z\in [C_\beta,C_\gamma]\}}
\left(z - \frac{m}{N}\right)^{u} 
\frac{1}{j!}\left(-\frac{(x-\mu_t z)^2}{2\sigma_t^2}\right)^j  dz\\
&= \frac{1}{ j!(2\pi )^{1/2}\mu_t}\int_{D}\left(\frac{x-\sigma_t  y }{\mu_t}- \frac{m}{N}\right)^{u} 
\left(-\frac{y^2}{2}\right)^j  dy \\
& = \frac{1}{j!(2\pi)^{1/2}\mu_t} \sum_{l=0}^{u} \binom{u}{l} \left(\frac{x}{\mu_t} - \frac{m}{N}\right)^{u-l} \int_D \Big(-\frac{\sigma_t y}{\mu_t}\Big)^l \left(-\frac{y^2}{2}\right)^j  dy  \\
&= \frac{1}{j!(2\pi)^{1/2}\mu_t} \sum_{l=0}^{u} \binom{u}{l}\left(\frac{x}{\mu_t} - \frac{m}{N}\right)^{u-l} (-1)^{l+j} \frac{\sigma_t^l}{2^{j}\mu_t^l}\int_D y^{l+2j} dy  \\
& = \frac{1}{j!(2\pi)^{1/2}}  \sum_{l=0}^{u} \frac{(-1)^{l+j}}{2^{j}} \binom{u}{l} \frac{(x-\frac{m\mu_t}{N})^{u-l}\sigma_t^l}{\mu_t^{u+1}}\cdot \frac{D_U^{l+ 2j + 1}(x) - D_L^{l + 2j + 1}(x)}{l + 2j + 1},
\end{aligned}
$$
where $D = [D_L(t,x), D_U(t,x)]$ with
$$
D_L(t,x) = \psi_{\mathrm{clip}}\left(\frac{\psi_{\mathrm{clip}}\left(x -\frac{m\mu_t}{N}, x-\mu_tC_{\gamma}, x - \mu_t C_{\beta} \right)}{\sigma_t}, -C\sqrt{\log\epsilon^{-1}}, C\sqrt{\log\epsilon^{-1}} \right),
$$
and 
$$
D_U(t,x) = \psi_{\mathrm{clip}}\left(\frac{\psi_{\mathrm{clip}}\left(x -\frac{(m-1)\mu_t}{N}, x-\mu_t C_{\gamma}, x - \mu_t C_{\beta} \right)}{\sigma_t}, -C\sqrt{\log\epsilon^{-1}}, C\sqrt{\log\epsilon^{-1}} \right).
$$
Thus, we only need to approximate the function of the form
$$
\psi_{m,l,j}(t, x) =\sigma_t^l \Big(x-\frac{m\mu_t}{N}\Big)^{u-l}\cdot \frac{D_U^{l+ 2j + 1}(x) - D_L^{l + 2j + 1}(x)}{\mu_t^{u+1}}.
$$
For a fixed accuracy $\eta>0$, Lemmas~\ref{lem:relu-ou-coefficients} and~\ref{lem:relu-clip} provide the ReLU networks
\[
\widehat{\mu}_t:=s_{\rm clip}\left(E_{2,T,\eta}(t),\mu_T,1\right),~~ \widehat{\sigma_t^2}:=s_{\rm clip}\left(1-E_{4,T,\eta}(t),\sigma_{T_1}^2,1\right),
\]
such that
\[
\sup_{t\in[T_1,T-T_1]}\left(\left|\widehat{\mu}_t-\mu_t\right|\vee\left|\widehat{\sigma_t^2}-\sigma_t^2\right|\right)\leq\eta.
\]
For $r\in\{m-1,m\}$, define $c_r:=\min\left\{C_\gamma,\max\left\{\frac{r}{N},C_\beta\right\}\right\}.$
Since $\mu_t>0$ and $\widehat{\mu}_t\geq\mu_T>0$, we have
\[
\psi_{\rm clip}\left(x-\frac{r\mu_t}{N},x-\mu_tC_\gamma,x-\mu_tC_\beta\right)=x-\mu_tc_r,
\]
and
\[
\psi_{\rm clip}\left(x-\frac{r\widehat{\mu}_t}{N},x-\widehat{\mu}_tC_\gamma,x-\widehat{\mu}_tC_\beta\right)=x-\widehat{\mu}_tc_r.
\]
Thus, the analytic definitions of $D_L$ and $D_U$ remain unchanged. In the network constructions below, their inner clipping terms are represented exactly by $x-\widehat{\mu}_tc_m$ and $x-\widehat{\mu}_tc_{m-1}$, respectively, while $\sigma_t^2$ is replaced by $\widehat{\sigma_t^2}$.
We first consider the approximation of $D_L^{l+2j+1}(t,x)$. We may choose a ReLU neural network
$$
\begin{aligned}
\psi_{m,j,l}^{(L)}&(t,x)  = \psi_{\mathrm{prod,D_1}}(\underbrace{\cdot, \cdots, \cdot}_{l+2j+1~ \text{times}}) \circ \psi_{\mathrm{clip}}(\cdot, -C\sqrt{\log\epsilon^{-1}}, C\sqrt{\log\epsilon^{-1}}) \\
\circ& \psi_{{\rm prod},D_2} \left(x-\widehat{\mu}_t c_m, \psi_{\rm rec} \big(\psi_{\rm root}(\widehat{\sigma_t^2})\big)\right).
\end{aligned}
$$
For \(D_U^{l+2j+1}(t,x)\), we use the following ReLU neural network for approximation
$$
\begin{aligned}
\psi_{m,j,l}^{(U)}& (t,x) = \psi_{\mathrm{prod,D_1}}(\underbrace{\cdot, \cdots, \cdot}_{l+2j+1~ \text{times}}) \circ \psi_{\mathrm{clip}}(\cdot, -C\sqrt{\log\epsilon^{-1}}, C\sqrt{\log\epsilon^{-1}}) \\
\circ& \psi_{{\rm prod},D_2} \left(x-\widehat{\mu}_t c_{m-1}, \psi_{\rm rec} \big(\psi_{\rm root}(\widehat{\sigma_t^2})\big)\right).
\end{aligned}
$$
Thus, the difference \(D_U^{l+2j+1}(t,x) - D_L^{l+2j+1}(t,x)\) can be approximated by
$$
\psi_{m,j,l}^{(D)}(t,x) = \psi_{m,j,l}^{(U)}(t,x) - \psi_{m,j,l}^{(L)}(t,x).
$$
Next, we turn to the term \(\Big(x-\frac{m\mu_t}{N}\Big)^{u-l}\). Given that \(\left|x-\frac{m\mu_t}{N}\right| \leq C_0 + 1\), by Lemma \ref{lem:relu-product},
there exists a ReLU neural network
$$
\psi_{m,j,l}^{(N)}(t,x) = \psi_{\mathrm{prod,N}}(\underbrace{\cdot,\cdots,\cdot}_{u - l ~ \text{times}}) \circ \Big(x-\frac{m\widehat{\mu}_t}{N}\Big)
$$
that approximates $\left(x-\frac{m\mu_t}{N}\right)^{u-l}$. 
For \(\sigma_t^l\), we select
$$
\psi_{m,j,l}^{(\sigma_t)}(t) = \psi_{\mathrm{prod},\sigma_t}(\underbrace{\cdot,\cdots,\cdot}_{l~ \text{times}}) \circ \psi_{\mathrm{root}}\left(\widehat{\sigma_t^2}\right).
$$
For \(\frac{1}{\mu_t^{u+1}}\), the chosen is
$$
\psi_{m,j,l}^{(\mu_t)}(t) = \psi_{\mathrm{prod},\mu_t}(\underbrace{\cdot,\cdots,\cdot}_{u+1~ \text{times}}) \circ \psi_{\mathrm{rec}}\left(\widehat{\mu}_t\right).
$$
Finally, we can construct a ReLU neural network
$$
\psi_{m,j,l}^{(F)}(t,x) = \psi_{\mathrm{prod,F}}\Big(\psi_{m,j,l}^{(\sigma_t)}(t), ~\psi_{m,j,l}^{(N)}(t,x), ~\psi_{m,j,l}^{(D)}(t,x),~ \psi_{m,j,l}^{(\mu_t)}(t)\Big)
$$
to approximate $\psi_{m,l,j}(t, x)$.

Next, we derive the error bound between $\psi_{m,j,l}^{(F)}(t,x)$ and $\psi_{m,l,j}(t,x)$. For convenience, we introduce notation for approximation errors: Let $\epsilon^{(U)}, $ $\epsilon^{(L)},$ $ \epsilon^{(D)},$ $ \epsilon^{(\mu_t)},$ $ \epsilon^{(\sigma_t)}, $ $\epsilon^{(N)}, $ $\epsilon^{(F)}$ correspond to the errors of $\psi_{m,j,l}^{(U)},$ $ \psi_{m,j,l}^{(L)},$ $ \psi_{m,j,l}^{(D)},$ $ \psi_{m,j,l}^{(\mu_t)},$ $ \psi_{m,j,l}^{(\sigma_t)},$ $ \psi_{m,j,l}^{(N)}, $ $\psi_{m,j,l}^{(F)}$, respectively. Separately, let $\epsilon_{\mathrm{prod,D_1}},$ $\epsilon_{\mathrm{prod,D_2}},$ $ \epsilon_{\mathrm{prod,D}},$ $ \epsilon_{\mathrm{prod,N}},$ $ \epsilon_{\mathrm{prod,\sigma_t}},$ $ \epsilon_{\mathrm{prod,\mu_t}},$ $ \epsilon_{\mathrm{prod,F}}$ denote the approximation errors of $\psi_{\mathrm{prod,D_1}},$ $ \psi_{\mathrm{prod,D_2}},$ $ \psi_{\mathrm{prod,D}},$ $ \psi_{\mathrm{prod,N}},$ $ \psi_{\mathrm{prod,\sigma_t}},$ $ \psi_{\mathrm{prod,\mu_t}},$ $ \psi_{\mathrm{prod,F}}$, respectively.

Since $\left|x-\frac{m\mu_t}{N}\right|^{u-l} \leq \left(C_0 + 1\right)^{u-l}$, $|D_L(t,x)| \leq C\sqrt{\log\epsilon^{-1}}$, $|D_U(t,x)| \leq C\sqrt{\log\epsilon^{-1}}$, $\sigma_t^l\leq 1$ and $\frac{1}{\mu_t^{u+1}} \leq \frac{1}{\mu_{T}^{u+1}}$, we set
$$
C_1 = \max\left\{ \left(C_0 + 1\right)^{u-l}, 1, 2(C\sqrt{\log\epsilon^{-1}})^{l + 2j + 1}, \frac{1}{\mu_{T}^{u+1}}\right\}.
$$
By Lemma \ref{lem:relu-product}, we have
$$
|\psi_{m,j,l}^{(F)}(t, x) - \psi_{m,l,j}(t, x)| \leq \epsilon_{\mathrm{prod,F}} + 4C_1^3\cdot \max\{\epsilon^{(D)}, \epsilon^{(N)}, \epsilon^{(\sigma_t)},\epsilon^{(\mu_t)}\}.
$$
By Lemmas~\ref{lem:relu-reciprocal} and \ref{lem:relu-square-root},
we can bound $\epsilon^{(U)}$ and $\epsilon^{(L)}$. Since $C_\beta,C_\gamma\in[0,1]$ and the clipping function is $1$-Lipschitz, replacing $\mu_t$ by $\widehat{\mu}_t$ changes each inner clipped factor by at most $\eta$. Let $C_2 = \max\left\{C_0 + 1, \frac{1}{\sqrt{1-e^{-4T_1}}}\right\}$. Then
$$
\epsilon^{(L)} \leq \epsilon_{\mathrm{prod,D_1}} + (l + 2j + 1)(C\sqrt{\log\epsilon^{-1}})^{l+ 2j}\left[\epsilon_{\mathrm{prod,D_2}} + 2C_2\left(\eta+\epsilon_{\mathrm{rec}} + \frac{\epsilon_{\mathrm{root}}+\frac{\eta}{\sqrt{\epsilon_{\mathrm{root}}}}}{\epsilon_{\mathrm{rec}}^2}\right)\right]
$$
and
$$
\epsilon^{(U)} \leq \epsilon_{\mathrm{prod,D_1}} + (l + 2j  + 1)(C\sqrt{\log\epsilon^{-1}})^{l + 2j}\left[\epsilon_{\mathrm{prod,D_2}} + 2C_2\left(\eta+\epsilon_{\mathrm{rec}} + \frac{\epsilon_{\mathrm{root}}+\frac{\eta}{\sqrt{\epsilon_{\mathrm{root}}}}}{\epsilon_{\mathrm{rec}}^2}\right)\right].
$$
Since
$$
\epsilon^{(D)} \leq \epsilon^{(U)} + \epsilon^{(L)}, 
$$ 
we obtain
$$
\epsilon^{(D)} \leq 2\epsilon_{\mathrm{prod,D_1}} +2 (l + 2j  + 1)(C\sqrt{\log\epsilon^{-1}})^{l + 2j}\left[\epsilon_{\mathrm{prod,D_2}} + 2C_2\left(\eta+\epsilon_{\mathrm{rec}} + \frac{\epsilon_{\mathrm{root}}+\frac{\eta}{\sqrt{\epsilon_{\mathrm{root}}}}}{\epsilon_{\mathrm{rec}}^2}\right)\right].
$$
We  can  also bound $\epsilon^{(N)}, \epsilon^{(\sigma_t)}$ and $\epsilon^{(\mu_t)}$  as
\begin{gather*}
\epsilon^{(N)} \leq \epsilon_{\mathrm{prod,N}}+(u-l)(C_0+1)^{(u-l-1)_+}\eta,~ \epsilon^{(\sigma_t)} \leq \epsilon_{\mathrm{prod},\sigma_t} + l\left(\epsilon_{\mathrm{root}}+\frac{\eta}{\sqrt{\epsilon_{\mathrm{root}}}}\right),\\
\epsilon^{(\mu_t)} \leq \epsilon_{\mathrm{prod},\mu_t} + \frac{u+1}{\mu_T^{u}}\left(\epsilon_{\mathrm{rec}}+\frac{\eta}{\epsilon_{\mathrm{rec}}^2}\right).
\end{gather*}
To ensure that $\epsilon^{(F)} \leq \epsilon_0$, we take $\epsilon_{\mathrm{prod,F}} = \frac{\epsilon_0}{2}$ and $\max\{\epsilon^{(D)}, \epsilon^{(N)}, \epsilon^{(\mu_t)} , \epsilon^{(\sigma_t)}\} \leq \frac{\epsilon_0}{8C_1^3}=:\epsilon^*$. In detail, we take
\begin{gather*}
\epsilon_{\mathrm{prod,D_1}} = \frac{\epsilon^*}{4},~  \epsilon_{\mathrm{prod,D_2}} = \frac{\epsilon^*}{8(l + 2j + 1)(C\sqrt{\log\epsilon^{-1}})^{l + 2j}}, \\
\epsilon_{\mathrm{prod,N}}=\frac{\epsilon^*}{2}, ~ \epsilon_{\mathrm{prod,\mu_t}} = \frac{\epsilon^*}{2},~ \epsilon_{\mathrm{prod,\sigma_t}} = \frac{\epsilon^*}{2}.
\end{gather*}
Moreover, we take
\begin{gather*}
\epsilon_{\mathrm{rec}} =\min \left\{\mu_T,\sigma_{T_1},\frac{\epsilon^* \mu_T^{u}}{4(u+1)},\frac{\epsilon^*}{64C_2(l + 2j + 1)(C\sqrt{\log\epsilon^{-1}})^{l + 2j}}\right\},\\ \epsilon_{\mathrm{root}}=\min\left\{\sigma_{T_1}^2,\epsilon_{\mathrm{rec}}^3,\frac{\epsilon^*}{4(l\vee 1)}\right\}.
\end{gather*}
Finally, we choose
\[
\eta:=\min\left\{\frac14,\epsilon_{\rm root}^{3/2},\epsilon_{\rm rec}^3,\frac{\epsilon^*}{2\left[1+(u-l)(C_0+1)^{(u-l-1)_+}\right]}\right\}.
\]
These choices give
$$
\epsilon_{\mathrm{root}}+\frac{\eta}{\sqrt{\epsilon_{\mathrm{root}}}}\leq2\epsilon_{\mathrm{root}},~~ \epsilon_{\mathrm{rec}}+\frac{\eta}{\epsilon_{\mathrm{rec}}^2}\leq2\epsilon_{\mathrm{rec}},~~ \eta+\epsilon_{\mathrm{rec}}+\frac{\epsilon_{\mathrm{root}}+\frac{\eta}{\sqrt{\epsilon_{\mathrm{root}}}}}{\epsilon_{\mathrm{rec}}^2}\leq4\epsilon_{\mathrm{rec}}.
$$
Therefore, $\max\{\epsilon^{(D)}, \epsilon^{(N)},\epsilon^{(\sigma_t)},  \epsilon^{(\mu_t)}\} \leq \epsilon^*$. Note that $l \leq u \leq C_{\alpha} = \mathcal{O}(1)$ and $j \leq k - 1 = \mathcal{O}(\log \epsilon^{-1})$. Moreover, the choices above imply
$$
\log\eta^{-1}=\mathcal{O}\left(\log\epsilon_0^{-1}+\log(C_0+2)+\log N+\log^2\epsilon^{-1}\right).
$$
Since $T=C_\lambda\log N$, Lemma~\ref{lem:relu-ou-coefficients} shows that the additional networks $\widehat{\mu}_t$ and $\widehat{\sigma_t^2}$ are absorbed into the parameter bounds below. Subsequently, we can obtain the network structures of $\psi_{m,j,l}^{(\cdot)} $. We now detail the network complexity for each component.
For $\psi_{m,j,l}^{(U)}, \psi_{m,j,l}^{(L)}$,  and $\psi_{m,j,l}^{(D)}$, the parameters satisfy
\begin{align*}
&\cD = \mathcal{O}\left(\log^2\epsilon^{-1}_0 + \log^2 N + \log^2C_0 + \log^4 \epsilon^{-1}\right),\\ 
&\cW = \mathcal{O}\left(\log^3\epsilon^{-1}_0 + \log^3 N + \log^3C_0 + \log^6 \epsilon^{-1}\right), \\
&\cS = \mathcal{O}\left(\log^4\epsilon^{-1}_0 + \log^4C_0 + \log^4 N + \log^8 \epsilon^{-1}\right), \\ 
&\cB =\exp\left(\mathcal{O}\left(\log^2\epsilon^{-1}_0 + \log^2 N + \log^2C_0 + \log^4 \epsilon^{-1} \right)\right).
\end{align*}
Regarding $\psi_{m,j,l}^{(N)}$, the additional composition with $\widehat{\mu}_t$ gives
\begin{align*}
&\cD = \mathcal{O}\left(\log^2\epsilon^{-1}_0 + \log^2C_0 + \log^2 N + \log^4 \epsilon^{-1}\right),\\
&\cW = \mathcal{O}\left(\log^3\epsilon^{-1}_0 + \log^3C_0 + \log^3 N + \log^6 \epsilon^{-1}\right), \\
&\cS = \mathcal{O}\left(\log^4\epsilon^{-1}_0 + \log^4C_0 + \log^4 N + \log^8 \epsilon^{-1}\right),\\
&\cB =\exp\left(\mathcal{O}\left(\log^2\epsilon^{-1}_0 + \log^2C_0 + \log^2 N + \log^4 \epsilon^{-1}\right)\right).
\end{align*}
Similarly, for $\psi_{m,j,l}^{(\sigma_t)}$, we obtain
\begin{align*}
&\cD = \mathcal{O}\left(\log^2\epsilon^{-1}_0 + \log^2C_0 + \log^2 N + \log^4 \epsilon^{-1} \right),\\
&\cW = \mathcal{O}\left(\log^3\epsilon^{-1}_0 + \log^3C_0 + \log^3 N + \log^6 \epsilon^{-1} \right), \\
&\cS = \mathcal{O}\left(\log^4\epsilon^{-1}_0 + \log^4C_0 + \log^4 N + \log^8 \epsilon^{-1} \right),\\ &\cB =\exp\left(\mathcal{O}\left(\log\epsilon^{-1}_0 + \log C_0 + \log N + \log^2 \epsilon^{-1} \right)\right).
\end{align*}
For the component $\psi_{m,j,l}^{(\mu_t)}$, the complexity is
\begin{align*}
&\cD = \mathcal{O}\left(\log^2\epsilon^{-1}_0 + \log^2 N + \log^2C_0 + \log^4 \epsilon^{-1}\right),\\ &\cW = \mathcal{O}\left(\log^3\epsilon^{-1}_0 + \log^3 N + \log^3C_0 + \log^6 \epsilon^{-1}\right), \\
&\cS = \mathcal{O}\left(\log^4\epsilon^{-1}_0 + \log^4C_0 + \log^4 N + \log^8 \epsilon^{-1}\right), \\ &\cB =\exp\left(\mathcal{O}\left(\log^2\epsilon^{-1}_0 + \log^2 N + \log^2C_0 + \log^4 \epsilon^{-1} \right)\right).
\end{align*}
Therefore, by Lemma \ref{lem:relu-concatenation} and Lemma \ref{lem:relu-parallelization}, we finally derive the network structure of $\psi_{m,j,l}^{(F)}$ as follows
$$
\begin{aligned}
&\cD = \mathcal{O}\left(\log^2\epsilon^{-1}_0 + \log^2C_0 + \log^2 N + \log^4 \epsilon^{-1} \right), \\ 
&\cW = \mathcal{O}\left(\log^3\epsilon^{-1}_0 + \log^3C_0 + \log^3 N + \log^6 \epsilon^{-1} \right), \\
&\cS = \mathcal{O}\left(\log^4\epsilon^{-1}_0 + \log^4C_0 + \log^4 N + \log^8 \epsilon^{-1} \right),\\ &\cB =\exp\left(\mathcal{O}\left(\log^2\epsilon^{-1}_0 + \log^2C_0 + \log^2 N + \log^4 \epsilon^{-1} \right)\right).
\end{aligned}
$$
Now, we consider the following ReLU neural  network
\[
\psi_{\rm{R}}(t,x,m,u,j)
=
\frac{1}{j!(2\pi)^{1/2}}
\sum_{l=0}^{u}
\frac{(-1)^{l+j}}{2^{j}}
\binom{u}{l}
\frac{1}{l+2j+1}
\psi_{m,j,l}^{(F)}(t,x).
\]
Then, we have
$$
\begin{aligned}
\big|\psi_{\rm{R}}(t,x,m,u,j) -\psi (t,x,m,u,j) \big| & \leq \frac{1}{j!(2\pi)^{1/2}}\left(\sum_{l=0}^{u} \binom{u}{l} \frac{1}{2^{j}(l + 2j + 1)}
\right) \cdot \epsilon_0 \\
& \leq \frac{\epsilon_0}{j!2^{j-u}(2\pi)^{1/2}} \lesssim \epsilon_0.
\end{aligned}
$$
By Lemma \ref{lem:relu-parallelization}, the parameters of the network $\psi_{\rm{R}}$ satisfy
$$
\begin{aligned}
&\cD = \mathcal{O}\left(\log^2\epsilon^{-1}_0 + \log^2C_0 + \log^2 N + \log^4 \epsilon^{-1} \right),\\ &\cW = \mathcal{O}\left(\log^3\epsilon^{-1}_0 + \log^3C_0 + \log^3 N + \log^6 \epsilon^{-1} \right), \\
&\cS = \mathcal{O}\left(\log^4\epsilon^{-1}_0 + \log^4C_0 + \log^4 N + \log^8 \epsilon^{-1} \right), \\ 
&\cB =\exp\left(\mathcal{O}\left(\log^2\epsilon^{-1}_0 + \log^2C_0 + \log^2 N + \log^4 \epsilon^{-1} \right)\right).
\end{aligned}
$$
The proof is complete.
\end{proof}

With Lemma \ref{lem:relu-approximate-psi}, we can construct a ReLU neural network to approximate $\Psi_{\boldsymbol{m},\boldsymbol{n},\boldsymbol{u},\boldsymbol{v}}(t,
\boldsymbol{x},\boldsymbol{w}) $.
\begin{lemma}\label{lem:relu-approximate-Psi}
Given $N \gg 1$, $C_0 > 0$, for any $\epsilon_0 > 0$, there exists a ReLU neural network $\Psi_{\rm{R}} \in  \cG(\cD,\cW,\cS,\cB)$ with
\begin{align*}
&\cD = \mathcal{O}\left(\log^2\epsilon^{-1}_0 + \log^2C_0 + \log^2 N + \log^4 \epsilon^{-1} \right), \\
&\cW = \mathcal{O}\left(\log \epsilon^{-1}(\log^3\epsilon^{-1}_0 + \log^3C_0 + \log^3 N + \log^6 \epsilon^{-1}) \right),\\
&\cS = \mathcal{O}\left(\log \epsilon^{-1} (\log^4\epsilon^{-1}_0+ \log^4C_0 + \log^4 N  + \log^8 \epsilon^{-1} \right)), \\
&\cB =\exp\left(\mathcal{O}\left(\log^2\epsilon^{-1}_0 + \log^2C_0 + \log^2 N + \log^4 \epsilon^{-1} \right)\right),
\end{align*}
such that
$$
\left|\Psi_{\rm{R}}(t,\boldsymbol{x},\boldsymbol{w})- \Psi_{\boldsymbol{m},\boldsymbol{n},\boldsymbol{u},\boldsymbol{v}}(t,
\boldsymbol{x},\boldsymbol{w}) \right| \lesssim \epsilon_0,$$
which is given by
$$
\Psi_{\boldsymbol{m},\boldsymbol{n},\boldsymbol{u},\boldsymbol{v}}(t,
\boldsymbol{x},\boldsymbol{w})= \eta_{\boldsymbol{n},N}(w) \prod_{i=1}^{d_w} \Big(w_i-\frac{n_i}{N}\Big)^{v_i}\cdot \prod_{i=1}^{d_x}   \sum_{j=0}^{k-1}\psi(t,x_i,m_i,u_i,j),
$$
for $ {\boldsymbol{x}}\in[-C_0,C_0]^{d_x}, ~ \boldsymbol{w}\in [0,1]^{d_w},~t\in[T_1, T-T_1].$

\end{lemma}

\begin{proof}
By Lemma \ref{lem:relu-parallelization}, we first construct a ReLU neural network $\psi_{\mathrm{sum}}\in \cG(\mathcal{D},\mathcal{W},\mathcal{S},\mathcal{B})$, which  satisfies
$$
\psi_{\mathrm{sum}}(t,x,m,u) = \sum_{j=0}^{k-1}\psi_{\rm{R}}(t,x,m,u,j).
$$
Here, the network parameters are
\begin{align*}
&\mathcal{D} = \mathcal{O}\left(\log^2\epsilon^{-1}_{\rm{R}} + \log^2C_0 + \log^2 N + \log^4 \epsilon^{-1} \right), \\
&\mathcal{W} = \mathcal{O}\left(\log \epsilon^{-1}(\log^3\epsilon^{-1}_{\rm{R}} + \log^3C_0 + \log^3 N + \log^6 \epsilon^{-1})\right), \\
&\mathcal{S}= \mathcal{O}\left(\log \epsilon^{-1}(\log^4\epsilon^{-1}_{\rm{R}} + \log^4C_0 + \log^4 N + \log^8 \epsilon^{-1} \right)), \\
&\mathcal{B} = \exp\left(\mathcal{O}\left(\log^2\epsilon^{-1}_{\rm{R}} + \log^2C_0 + \log^2 N + \log^4 \epsilon^{-1} \right)\right),
\end{align*}
where $\epsilon_{\rm{R}}$ is the approximation error of $\psi_{\rm{R}}$. 

Then, for any $\boldsymbol{m}\in [N]^{d_x}$ and $\boldsymbol{u}\in\mathbb{N}^{d_x}, \Vert\boldsymbol{u}\Vert_1 \leq \alpha(\boldsymbol{u},\boldsymbol{v})< C_\alpha$, we can construct the following ReLU network
$$
\psi_{\boldsymbol{m}, \boldsymbol{u}}(t,\boldsymbol{x}) = \psi_{\mathrm{prod}}(\psi_{\mathrm{sum}}(t,x_1,m_1,u_1), \cdots, \psi_{\mathrm{sum}}(t,x_{d_x}, m_{d_x}, u_{d_x}))
$$
to approximate the term $\prod\limits_{i=1}^{d_x}\sum\limits_{j=0}^{k-1}\psi_{\rm{R}}(t,x_i,m_i,u_i,j)$.  The approximation error can be written as
$$
\epsilon_{\psi_{\boldsymbol{m}, \boldsymbol{u}}} \leq  \epsilon_{\mathrm{prod}} + d_x C_3^{d_x-1} \cdot k\epsilon_{\rm{R}},
$$
where
$$
C_3 = \max_{|x_i| \leq C_0, 1\leq i \leq d_x} \sum_{j = 0}^{k-1} |\psi(t,x_i,m_i,u_i,j)|.
$$
The term $|\psi(t,x_i,m_i,u_i,j)|$ ($1 \leq i \leq d_x$) can be bounded as follows
\begin{align*}
\big|\psi(t,x_i,m_i,u_i,j)\big| &=\frac{1}{j!(2\pi)^{1/2}}  \sum_{l=0}^{u_i} \frac{(-1)^{l+j}}{2^{j}} \binom{u_i}{l} \frac{(x_i-\frac{m_i\mu_t}{N})^{u_i-l}\sigma_t^l}{\mu_t^{u_i+1}}\cdot \frac{D_U^{l+ 2j + 1}(x_i) - D_L^{l + 2j + 1}(x_i)}{l + 2j + 1}\\
& \leq \frac{1}{2^{j}j!(2\pi)^{1/2}\mu_{T}^{u_i+1}}\sum_{l=0}^{u_i} \binom{u_i}{l}(C_0 + 1)^{u_i-l}(C\sqrt{\log\epsilon^{-1}})^{l + 2j + 1} \\
& \leq \frac{(C\sqrt{\log\epsilon^{-1}})^{u_i + 2(k-1) + 1}}{2^{j}j!(2\pi)^{1/2}\mu_{T}^{u_i+1}}     \sum_{l=0}^{u_i} \binom{u_i}{l}(C_0 + 3)^{u_i-l} \\
& \leq \frac{(C\sqrt{\log\epsilon^{-1}})^{u_i + 2k - 1}}{j! (2\pi)^{1/2}\mu_{T}^{u_i+1}} \cdot (C_0 + 4)^{u_i}.
\end{align*}
It implies that
\begin{align*}
\sum_{j=0}^{k-1} \big|\psi(t,x_i,m_i,u_i,j)\big|  & \leq \frac{(C\sqrt{\log\epsilon^{-1}})^{u_i + 2k - 1}}{(2\pi)^{1/2}\mu_{T}^{u_i+1}} \cdot (C_0 + 4)^{u_i} \sum_{j=0}^{k-1}\frac{1}{j!}  \notag \\
& \leq \frac{e(C\sqrt{\log\epsilon^{-1}})^{u_i + 2k - 1}}{(2\pi)^{1/2}\mu_{T}^{u_i+1}} \cdot (C_0 + 4)^{u_i}. 
\end{align*}
Taking
$$
\epsilon_{\mathrm{prod}} = \frac{\epsilon_1}{2}, ~ \epsilon_{\rm{R}} = \frac{\epsilon_1}{2d_xC_3^{d_x-1}k},
$$
we have $\epsilon_{\psi_{\boldsymbol{m},\boldsymbol{u}}} \leq \epsilon_1$. 
Note that $k = \mathcal{O}(\log \epsilon^{-1})$, therefore, the parameters of $\psi_{\mathrm{prod}}$ satisfy
\begin{align*}
& \mathcal{D} = \mathcal{O}\left(\log\epsilon_1^{-1} + \log C_0 + \log^2 \epsilon^{-1}\right), ~ \mathcal{W}= \mathcal{O}(1),\\
& \mathcal{S} = \mathcal{O}\left(\log\epsilon_1^{-1} + \log C_0 + \log^2 \epsilon^{-1}\right), ~ \mathcal{B} = \exp\left(\mathcal{O}(\log C_0 + \log^2 \epsilon^{-1})\right).
\end{align*}
Thus we obtain the network parameters of $\psi_{\boldsymbol{m},\boldsymbol{u}}$, which satisfy
\begin{align*}
&\mathcal{D} = \mathcal{O}\left(\log^2\epsilon^{-1}_1 + \log^2C_0 + \log^2 N + \log^4 \epsilon^{-1} \right), \\
&\mathcal{W} = \mathcal{O}\left(\log \epsilon^{-1}(\log^3\epsilon^{-1}_1 + \log^3C_0 + \log^3 N + \log^6 \epsilon^{-1}) \right),\\
&\mathcal{S} = \mathcal{O}\left(\log \epsilon^{-1} (\log^4\epsilon^{-1}_1+ \log^4C_0 + \log^4 N  + \log^8 \epsilon^{-1} \right)), \\
&\mathcal{B} = \exp\left(\mathcal{O}\left(\log^2\epsilon^{-1}_1 + \log^2C_0 + \log^2 N + \log^4 \epsilon^{-1} \right)\right).
\end{align*}
We now approximate the complete $\boldsymbol{w}$-dependent factor
\begin{equation}\label{eq:relu-w} \eta_{\boldsymbol{n},N}(\boldsymbol{w})\prod_{i=1}^{d_w}\left(w_i-\frac{n_i}{N}\right)^{v_i}. \end{equation}
For $w\in[0,1]$, the local weights introduced above have the exact ReLU representations
\[ \eta_{1,N}(w)=1-\operatorname{ReLU}(Nw-1)+\operatorname{ReLU}(Nw-2), \]
\[ \eta_{n,N}(w)=\operatorname{ReLU}(Nw-n+1)-2\operatorname{ReLU}(Nw-n)+\operatorname{ReLU}(Nw-n-1), 2\leq n\leq N-1, \]
and
\[ \eta_{N,N}(w)=\operatorname{ReLU}(Nw-N+1). \]
Hence, each $\eta_{n_i,N}(w_i)$ is represented exactly by a fixed-depth, fixed-width ReLU subnetwork whose weights are bounded by $\mathcal{O}(N)$.
For each $1\leq i\leq d_w$, by Lemma \ref{lem:relu-product}, there exists a ReLU neural network
\[ \psi(w_i,n_i,v_i)=\psi_{\mathrm{prod,w}}(\underbrace{\cdot,\ldots,\cdot}_{v_i~\mathrm{times}})\circ\left(w_i-\frac{n_i}{N}\right) \]
approximating $\left(w_i-\frac{n_i}{N}\right)^{v_i}$, with approximation error $\epsilon_{\psi_{n_i,v_i,w_i}}$. 
Parallelizing the exact subnetworks for $\eta_{n_i,N}(w_i)$, the networks $\psi(w_i,n_i,v_i)$, and one product network, we construct
\[ \psi_{\boldsymbol{n},\boldsymbol{v}}(\boldsymbol{w})=\psi_{\mathrm{prod,W}}\left(\eta_{n_1,N}(w_1),\ldots,\eta_{n_{d_w},N}(w_{d_w}),\psi(w_1,n_1,v_1),\ldots,\psi(w_{d_w},n_{d_w},v_{d_w})\right). \]
This network approximates the factor in \eqref{eq:relu-w}, and its error satisfies
\[ \epsilon_{\psi_{\boldsymbol{n},\boldsymbol{v}}}\leq\epsilon_{\mathrm{prod,w}}+d_wC_w^{2d_w-1}\max_{1\leq i\leq d_w}\epsilon_{\psi_{n_i,v_i,w_i}}, \]
where $C_w:=\max\{1,2^{C_\alpha}\}$. Taking
\[ \epsilon_{\mathrm{prod,w}}=\frac{\epsilon_1}{2},~~\max_{1\leq i\leq d_w}\epsilon_{\psi_{n_i,v_i,w_i}}=\frac{\epsilon_1}{2d_wC_w^{2d_w-1}}\]
gives $\epsilon_{\psi_{\boldsymbol{n},\boldsymbol{v}}}\leq\epsilon_1$. The parameters of $\psi_{\boldsymbol{n},\boldsymbol{v}}(\boldsymbol{w})$ satisfy
\[ \mathcal{D}=\mathcal{O}\left(\log\epsilon_1^{-1}+\log N\right),~~\mathcal{W}=\mathcal{O}(1),~~\mathcal{S}=\mathcal{O}\left(\log\epsilon_1^{-1}+\log N\right),~~\mathcal{B}=\exp\left(\mathcal{O}(\log N)\right). \]
The exact local-weight subnetworks add only constant depth and width, and their $\mathcal{O}(N)$ weights are absorbed by the displayed bound.
Finally, we can construct a ReLU neural network for the term
$$
\Psi_{\rm{R}}(t,\boldsymbol{x},\boldsymbol{w})=\Psi_{\rm{prod}} \Big( \psi_{\boldsymbol{n},\boldsymbol{v}}(\boldsymbol{w}) ,~ \psi_{\boldsymbol{m}, \boldsymbol{u}}(t,\boldsymbol{x})  \Big).
$$
The approximation error of $\Psi_{\rm{R}}(t,\boldsymbol{x},\boldsymbol{w})$ satisfies 
$$
\epsilon_{\Psi}\leq  \epsilon_{\rm{prod} , \Psi} +2C\max \{\epsilon_{\psi_{\boldsymbol{n}, \boldsymbol{v}}} , ~\epsilon_{\psi_{\boldsymbol{m}, \boldsymbol{u}}} \} \leq \epsilon_{\rm{prod} , \Psi} +2C\epsilon_1,
$$
where 
$$
C=\max \left\{\prod_{i=1}^{d_w} 2^{v_i},\prod_{i=1}^{d_x} \frac{e(C\sqrt{\log\epsilon^{-1}})^{u_i + 2k - 1}}{(2\pi)^{1/2}\mu_{T}^{u_i+1}} \cdot (C_0 + 4)^{u_i}\right\}.
$$
Taking $\epsilon_{\mathrm{prod},\Psi}=\frac{\epsilon_0}{2}$ and $\epsilon_1=\frac{\epsilon_0}{4C}$ gives $\epsilon_\Psi\leq\epsilon_0$. Because the factors $\eta_{n_i,N}$ are represented exactly, they introduce no additional approximation term. The resulting network parameters are
\begin{align*}
&\mathcal{D} = \mathcal{O}\left(\log^2\epsilon^{-1}_0 + \log^2C_0 + \log^2 N + \log^4 \epsilon^{-1} \right), \\
&\mathcal{W}= \mathcal{O}\left(\log \epsilon^{-1}(\log^3\epsilon^{-1}_0 + \log^3C_0 + \log^3 N + \log^6 \epsilon^{-1}) \right),\\
&\mathcal{S}= \mathcal{O}\left(\log \epsilon^{-1} (\log^4\epsilon^{-1}_0+ \log^4C_0 + \log^4 N  + \log^8 \epsilon^{-1} \right)), \\
&\mathcal{B} = \exp\left(\mathcal{O}\left(\log^2\epsilon^{-1}_0 + \log^2C_0 + \log^2 N + \log^4 \epsilon^{-1} \right)\right).
\end{align*}
This completes the proof.
\end{proof}

With Lemma \ref{lem:relu-approximate-Psi}, we can construct a ReLU neural network to approximate \eqref{eq:diffusion-local-poly-g3}.

\begin{lemma}\label{lem:relu-approximate-g3}
Given $N \gg 1$ and $C_0>0$, for any $\epsilon_0>0$, there exists a ReLU neural network $g_{\rm R}\in\cG(\mathcal D,\mathcal W,\mathcal S,\mathcal B)$ with
\begin{align*}
&\mathcal D=\mathcal O\left(\log^2\epsilon_{\Psi}^{-1}+\log^2C_0+\log^2N+\log^4\epsilon^{-1}\right),\\
&\mathcal W=\mathcal O\left(N^{d_x+d_w}\log\epsilon^{-1}\left(\log^3\epsilon_{\Psi}^{-1}+\log^3C_0+\log^3N+\log^6\epsilon^{-1}\right)\right),\\
&\mathcal S=\mathcal O\left(N^{d_x+d_w}\log\epsilon^{-1}\left(\log^4\epsilon_{\Psi}^{-1}+\log^4C_0+\log^4N+\log^8\epsilon^{-1}\right)\right),\\
&\mathcal B=\exp\left(\mathcal O\left(\log^2\epsilon_{\Psi}^{-1}+\log^2C_0+\log^2N+\log^4\epsilon^{-1}\right)\right),
\end{align*}
such that
\[
\left|g_{\rm R}(t,\boldsymbol{x},\boldsymbol{w})-g_3(t,\boldsymbol{x},\boldsymbol{w})\right|\lesssim\epsilon_0
\]
for $\boldsymbol{x}\in[-C_0,C_0]^{d_x}$, $\boldsymbol{w}\in[0,1]^{d_w}$, and $t\in[T_1,T-T_1]$.
\end{lemma}

\begin{proof}
By Lemmas~\ref{lem:relu-approximate-psi} and \ref{lem:relu-approximate-Psi}, each network $\Psi_{\rm R}$ is constructed from the local networks containing $\widehat{\mu}_t$ and $\widehat{\sigma_t^2}$. Therefore, its approximation error $\epsilon_\Psi$ already includes the corresponding schedule approximation errors.

We can construct a ReLU neural network
\[
g_{\rm R}(t,\boldsymbol{x},\boldsymbol{w})=\sum_{\boldsymbol{m},\boldsymbol{n}}\sum_{\alpha(\boldsymbol{u},\boldsymbol{v})<C_\alpha}C_{\boldsymbol{m},\boldsymbol{n},\boldsymbol{u},\boldsymbol{v}}\Psi_{\rm R}(t,\boldsymbol{x},\boldsymbol{w})
\]
with network parameters
\begin{align*}
&\mathcal D=\mathcal O\left(\log^2\epsilon_{\Psi}^{-1}+\log^2C_0+\log^2N+\log^4\epsilon^{-1}\right),\\
&\mathcal W=\mathcal O\left(N^{d_x+d_w}\log\epsilon^{-1}\left(\log^3\epsilon_{\Psi}^{-1}+\log^3C_0+\log^3N+\log^6\epsilon^{-1}\right)\right),\\
&\mathcal S=\mathcal O\left(N^{d_x+d_w}\log\epsilon^{-1}\left(\log^4\epsilon_{\Psi}^{-1}+\log^4C_0+\log^4N+\log^8\epsilon^{-1}\right)\right),\\
&\mathcal B=\exp\left(\mathcal O\left(\log^2\epsilon_{\Psi}^{-1}+\log^2C_0+\log^2N+\log^4\epsilon^{-1}\right)\right).
\end{align*}
Taking
\[
\epsilon_\Psi=\frac{\epsilon_0}{(d_x+d_w)^{C_\alpha}N^{d_x+d_w}},
\]
the approximation error between $g_{\rm R}(t,\boldsymbol{x},\boldsymbol{w})$ and \eqref{eq:diffusion-local-poly-g3} satisfies
\begin{align*}
\left|g_{\rm R}(t,\boldsymbol{x},\boldsymbol{w})-g_3(t,\boldsymbol{x},\boldsymbol{w})\right|
&=\left|g_{\rm R}(t,\boldsymbol{x},\boldsymbol{w})-\sum_{\boldsymbol{m},\boldsymbol{n}}\sum_{\alpha(\boldsymbol{u},\boldsymbol{v})<C_\alpha}C_{\boldsymbol{m},\boldsymbol{n},\boldsymbol{u},\boldsymbol{v}}\Psi_{\boldsymbol{m},\boldsymbol{n},\boldsymbol{u},\boldsymbol{v}}(t,\boldsymbol{x},\boldsymbol{w})\right|\\
&\leq\sum_{\boldsymbol{m},\boldsymbol{n}}\sum_{\alpha(\boldsymbol{u},\boldsymbol{v})<C_\alpha}\left|C_{\boldsymbol{m},\boldsymbol{n},\boldsymbol{u},\boldsymbol{v}}\right|\epsilon_\Psi\\
&\lesssim\sum_{\boldsymbol{m},\boldsymbol{n}}\sum_{\alpha(\boldsymbol{u},\boldsymbol{v})<C_\alpha}\frac{\epsilon_\Psi}{\boldsymbol{u}!\boldsymbol{v}!}\\
&\lesssim(d_x+d_w)^{C_\alpha}N^{d_x+d_w}\frac{\epsilon_0}{(d_x+d_w)^{C_\alpha}N^{d_x+d_w}}\\
&\lesssim\epsilon_0.
\end{align*}
Each $\Psi_{\rm R}$ approximates the complete local summand containing $\eta_{\boldsymbol{n},N}(\boldsymbol{w})$, and $\epsilon_\Psi$ already accounts for the OU schedule approximation. Thus, the summation introduces no additional approximation term. This completes the proof.
\end{proof}

In the same way, we can choose a ReLU neural network $h_{\rm R}(t,\boldsymbol{x},\boldsymbol{w})\in\mathbb{R}^{d_x}$ to approximate \eqref{eq:diffusion-local-poly-int-h3}. We omit the proof and give the following lemma.

\begin{lemma}\label{lem:relu-approximate-h3}
Given $N\gg1$, $C_0>0$, and $1\leq\ell\leq d_x$. For any $\epsilon_0 > 0$, there exists a ReLU neural network $h_{{\rm R},\ell}\in\cG(\cD_\ell,\cW_\ell,\cS_\ell,\cB_\ell)$ with
\begin{align*}
&\mathcal{D}_\ell = \mathcal{O}\left(\log^2\epsilon^{-1}_{\Psi}+ \log^2C_0 + \log^2 N + \log^4 \epsilon^{-1} \right), \\
&\mathcal{W}_\ell  = \mathcal{O}\left(N^{d_x+d_w}\log \epsilon^{-1}(\log^3\epsilon^{-1}_{\Psi}+ \log^3C_0 + \log^3 N + \log^6 \epsilon^{-1})\right), \\
&\mathcal{S}_\ell = \mathcal{O}\left(N^{d_x+d_w}\log \epsilon^{-1}(\log^4\epsilon^{-1}_{\Psi} + \log^4C_0 + \log^4 N + \log^8 \epsilon^{-1} \right)), \\
&\mathcal{B}_\ell = \exp\left(\mathcal{O}\left(\log^2\epsilon^{-1}_{\Psi} + \log^2C_0 + \log^2 N + \log^4 \epsilon^{-1} \right)\right),
\end{align*}
such that
$$
\left|h_{{\rm{R}},\ell}(t,\boldsymbol{x},\boldsymbol{w})-h_{3,\ell}(t,\boldsymbol{x},\boldsymbol{w})\right| \lesssim \epsilon_0$$
for $ {\boldsymbol{x}}\in[-C_0,C_0]^{d_x}, ~ \boldsymbol{w}\in [0,1]^{d_w},~t\in[T_1, T-T_1].$
\end{lemma}

\begin{proof}
Repeat the construction of Lemma~\ref{lem:relu-approximate-Psi} for each local summand in \eqref{eq:diffusion-local-poly-int-h3}, replacing the last factor $\psi$ by $\psi^\nabla$. Under the same change of variables used in Lemma~\ref{lem:relu-approximate-psi}, the additional factor $\frac{\mu_tz-x}{\sigma_t}$ increases the exponent in the antiderivative from $l+2j+1$ to $l+2j+2$. Hence, $\psi^\nabla$ admits the same product-network construction as $\psi$.

The corresponding network $\psi_{\rm R}^\nabla$ uses the same schedule networks $\widehat{\mu}_t$ and $\widehat{\sigma_t^2}$ as in Lemma~\ref{lem:relu-approximate-psi}, while the target $\psi^\nabla$ retains the exact coefficients $\mu_t$ and $\sigma_t^2$. Using the same schedule-error allocation gives
\[
\left|\psi_{\rm R}^\nabla(t,x,m,u,j)-\psi^\nabla(t,x,m,u,j)\right|\lesssim\epsilon_\Psi,
\]
and the increase of the exponent by one does not change the stated network-complexity bounds. The factor $\eta_{\boldsymbol{n},N}(\boldsymbol{w})$ is represented by the same exact subnetworks.
Taking
\[
\epsilon_\Psi=\frac{\epsilon_0}{(d_x+d_w)^{C_\alpha}N^{d_x+d_w}},
\]
the same coefficient summation as in Lemma~\ref{lem:relu-approximate-g3} gives
\[
\left|h_{{\rm R},\ell}(t,\boldsymbol{x},\boldsymbol{w})-h_{3,\ell}(t,\boldsymbol{x},\boldsymbol{w})\right|\lesssim\epsilon_0.
\]
The exact representation of $\eta_{\boldsymbol{n},N}(\boldsymbol{w})$ introduces no additional approximation term, and the schedule approximation is already included in $\epsilon_\Psi$. Therefore, the network parameter bounds remain unchanged.
\end{proof}

Similarly, by applying Lemma \ref{lem:relu-approximate-h3} and using the parallelization property of ReLU neural networks, there exists a neural network
$$h_{\rm{R}}(t,{\boldsymbol{x}},\boldsymbol{w}) = [h_{{\rm{R}},1}(t,\boldsymbol{x},\boldsymbol{w}), \cdots, h_{{\rm{R}},d_x}(t,\boldsymbol{x},\boldsymbol{w})]^{\top} \in \mathbb{R}^{d_x}$$
that approximates \eqref{eq:diffusion-local-poly-int-h3}. 
The $\ell_2$-norm error between $h_{\rm{R}}(t,{\boldsymbol{x}},\boldsymbol{w}) $ and \eqref{eq:diffusion-local-poly-int-h3} is bounded by $\mathcal{O}(\epsilon_0)$.

\subsection{Error Bound for Approximating $\nabla\log p_t(\boldsymbol{x}|\boldsymbol{w})$ with ReLU Neural Networks}
In this section, we construct two ReLU neural networks to approximate the true score function $\nabla\log p_t(\boldsymbol x|\boldsymbol w)$ on different time intervals. 
The approximation error analysis is inspired by the diffusion approximation arguments of \citet{oko2023diffusion,fu2024unveil}. 

\begin{lemma} \label{lem:score-approximate-interval-1}
Let $N\gg 1$, there exists a ReLU network $ \boldsymbol{s}^{(1)}\in \cG(\mathcal{D},\mathcal{W},\mathcal{S},\mathcal{B})$ with
$$
\mathcal{D}  = \mathcal{O}(\log^4 N),~ \mathcal{W} = \mathcal{O}(N^{d_x+d_w}\log^7 N), \mathcal{S} = \mathcal{O}(N^{d_x+d_w}\log^9 N),~ \mathcal{B} = \exp\left(\mathcal{O}(\log^4 N)\right)
$$
that satisfies
$$
\int_{\mathbb{R}^{d_x}} p_t(\boldsymbol{x}|\boldsymbol{w})\Vert  \boldsymbol{s}^{(1)}(t,\boldsymbol{x}, \boldsymbol{w}) - \nabla\log p_t(\boldsymbol{x}|\boldsymbol{w})\Vert^2  d\boldsymbol{x} \lesssim \frac{N^{-2\alpha}\log N}{\sigma_t^2}, ~~ t\in [T_1,  3N^{-1}].
$$
Moreover, we can take $ \boldsymbol{s}^{(1)}$ satisfying $\Vert  \boldsymbol{s}^{(1)} (t, \boldsymbol{x},\boldsymbol{w})\Vert_{\infty} \lesssim \frac{\sqrt{\log N}}{\sigma_t}$. 
\end{lemma}
\begin{proof}
The approximation error admits a decomposition into three terms
\begin{align*}
& \int_{\mathbb{R}^{d_x}} p_t(\boldsymbol{x}|\boldsymbol{w})\Vert  \boldsymbol{s}^{(1)}(t,\boldsymbol{x}, \boldsymbol{w}) - \nabla\log p_t(\boldsymbol{x}|\boldsymbol{w})\Vert^2  d\boldsymbol{x} \\
& = \underbrace{ \int_{D_t(\boldsymbol{x}) > C\sigma_t\sqrt{\log\epsilon^{-1}_1}} p_t(\boldsymbol{x}|\boldsymbol{w}) \Vert  \boldsymbol{s}^{(1)}(t,\boldsymbol{x}, \boldsymbol{w}) - \nabla\log p_t(\boldsymbol{x}|\boldsymbol{w})\Vert^2  d\boldsymbol{x} }_{\text{I}} \\
&~~ +  \underbrace{\int_{D_t(\boldsymbol{x}) \leq  C\sigma_t\sqrt{\log\epsilon^{-1}_1}} p_t(\boldsymbol{x}|\boldsymbol{w}) \mathbbm{1}_{\{p_t(\boldsymbol{x}|\boldsymbol{w}) < \epsilon_p\}} \Vert  \boldsymbol{s}^{(1)}(t,\boldsymbol{x}, \boldsymbol{w}) - \nabla\log p_t(\boldsymbol{x}|\boldsymbol{w})\Vert^2  d\boldsymbol{x} }_{\text{II}} \\
&~~ +  \underbrace{\int_{D_t(\boldsymbol{x}) \leq  C\sigma_t\sqrt{\log\epsilon^{-1}_1}} p_t(\boldsymbol{x}|\boldsymbol{w}) \mathbbm{1}_{\{p_t(\boldsymbol{x}|\boldsymbol{w}) \geq \epsilon_p\}}\Vert  \boldsymbol{s}^{(1)}(t,\boldsymbol{x}, \boldsymbol{w}) - \nabla\log p_t(\boldsymbol{x}|\boldsymbol{w})\Vert^2  d\boldsymbol{x} }_{\text{III}}. 
\end{align*}

\paragraph{Bound Term I.}
By Lemma~\ref{lem:clipping-bound}, there exists a constant $C > 0$ such that for any $\epsilon_1 > 0$, 
\begin{align*}
\text{Term I}  &\leq 2 \int_{D_t(\boldsymbol{x}) >  C\sigma_t\sqrt{\log\epsilon^{-1}_1}}  p_t(\boldsymbol{x}|\boldsymbol{w}) \norm{ \boldsymbol{s}^{(1)}(t,\boldsymbol{x}, \boldsymbol{w})}^2  d\boldsymbol{x} \\
&~~ +2  \int_{D_t(\boldsymbol{x}) > C\sigma_t\sqrt{\log\epsilon^{-1}_1}} p_t(\boldsymbol{x}|\boldsymbol{w}) \norm{\nabla\log p_t(\boldsymbol{x}|\boldsymbol{w}) }^2  d\boldsymbol{x} \\
& \lesssim \frac{\epsilon_1}{\sigma_t} + \sigma_t\epsilon_1\Vert  \boldsymbol{s}^{(1)}(t,
\boldsymbol{x},\boldsymbol{w})\Vert_{\infty}^2.
\end{align*}
We then invoke Lemma \ref{lem:derivatives-bound}, which yields that \( \left\Vert \nabla\log p_t(\boldsymbol{x}|\boldsymbol{w}) \right\Vert \leq \frac{C\sqrt{\log\epsilon^{-1}_1}}{\sigma_t} \) whenever \( D_t(\boldsymbol{x}) \leq   C\sigma_t\sqrt{\log\epsilon^{-1}_1} \). Motivated by this bound, we choose the function \( \boldsymbol{s}^{(1)} \) such that its supremum norm satisfies \( \left\Vert \boldsymbol{s}^{(1)}(t, \boldsymbol{x},\boldsymbol{w}) \right\Vert_{\infty} \lesssim \frac{\sqrt{\log\epsilon^{-1}_1}}{\sigma_t} \). Taking \( \epsilon_1 = N^{-(2\alpha+1)} \), we have 
\begin{equation}\label{eq:int-psTlogp2-term-I-bound}
\text{Term I} \lesssim \frac{\epsilon_1 \log \epsilon^{-1}_1}{\sigma_t}\lesssim \frac{N^{-(2\alpha+1)}}{\sigma_t}\log N.
\end{equation}

\paragraph{Bound Term II.}
Choosing $\epsilon_p = N^{-(2\alpha+1)}$ and leveraging Lemma \ref{lem:clipping-bound} as before, the following bound holds:
\begin{equation}\label{eq:int-psTlogp2-term-II-bound}
\begin{aligned}
\text{Term II}  &\leq 2 \int_{D_t(\boldsymbol{x}) \leq  C\sigma_t\sqrt{\log\epsilon^{-1}_1}} p_t(\boldsymbol{x}|\boldsymbol{w})\mathbbm{1}_{\{p_t(\boldsymbol{x}|\boldsymbol{w}) < \epsilon_p\}} \Vert 
\boldsymbol{s}^{(1)}(t,\boldsymbol{x}, \boldsymbol{w})\Vert^2  d\boldsymbol{x} \\
&~~ +2  \int_{D_t(\boldsymbol{x}) \leq C\sigma_t\sqrt{\log\epsilon^{-1}_1}} p_t(\boldsymbol{x}|\boldsymbol{w})\mathbbm{1}_{\{p_t(\boldsymbol{x}|\boldsymbol{w}) < \epsilon_p\}} \Vert  \nabla\log p_t(\boldsymbol{x}|\boldsymbol{w})\Vert^2  d\boldsymbol{x} \\
& \lesssim \frac{\epsilon_p}{\sigma_t^2} (\log\epsilon^{-1}_1)^{\frac{d_x+2}{2}} +  \epsilon_p  (\log\epsilon^{-1}_1)^{\frac{d_x}{2}} \Vert \boldsymbol{s}^{(1)}(t,\boldsymbol{x},\boldsymbol{w})\Vert_{\infty}^2\\
&\lesssim \frac{\epsilon_p}{\sigma^2_t} (\log \epsilon^{-1}_1)^{\frac{d_x+2}{2}}\\
&\lesssim \frac{N^{-(2\alpha+1)}}{\sigma^2_t}\log^{\frac{d_x+2}{2}}N.
\end{aligned}
\end{equation}

\paragraph{Bound Term III.}
Since $\Vert \nabla\log p_t(\boldsymbol{x}|\boldsymbol{w}) \Vert \lesssim \frac{\sqrt{\log N}}{\sigma_t}$ on the region of Term III, there exists a constant ${C_1} > 0$ such that $\Vert\nabla\log p_t(\boldsymbol{x}|\boldsymbol{w})\Vert_{\infty} \leq \Vert \nabla\log p_t(\boldsymbol{x}|\boldsymbol{w}) \Vert \leq \frac{{C_1}\sqrt{\log N}}{\sigma_t}$. By \eqref{eq:diffusion-local-poly-g1}, we define $$
\boldsymbol{h}^{\prime}(t,\boldsymbol{x}, \boldsymbol{w}):= \max\left\{\min\left\{ \frac{\boldsymbol{h}_1(t,\boldsymbol{x}, \boldsymbol{w})}{ g_1(t,\boldsymbol{x}, \boldsymbol{w}) \vee N^{-(2\alpha+ 1)}}, {C_1}\sqrt{\log N} \right\}, -{C_1}\sqrt{\log N}\right\}.
$$
We decompose Term III into two subterms
\begin{align*}
\text{Term III}&\lesssim \underbrace{  \int_{D_t(\boldsymbol{x})\leq C\sigma_t\sqrt{\log\epsilon^{-1}_1}} p_t(\boldsymbol{x}|\boldsymbol{w}) \mathbbm{1}_{\{p_t(\boldsymbol{x}|\boldsymbol{w}) \geq \epsilon_p\}} \left\Vert  \boldsymbol{s}^{(1)}(t,\boldsymbol{x}, \boldsymbol{w}) - \frac{\boldsymbol{h}^{\prime}(t,\boldsymbol{x}, \boldsymbol{w})}{\sigma_t} \right\Vert^2  d\boldsymbol{x}  }_{\text{III.A}}\\
&~~ +  \underbrace{ \int_{D_t(\boldsymbol{x})\leq  C\sigma_t\sqrt{\log\epsilon^{-1}_1}} p_t(\boldsymbol{x}|\boldsymbol{w}) \mathbbm{1}_{\{p_t(\boldsymbol{x}|\boldsymbol{w}) \geq \epsilon_p\}} \left\Vert \frac{\boldsymbol{h}^{\prime}(t,\boldsymbol{x}, \boldsymbol{w})}{\sigma_t}  - 
\nabla\log p_t(\boldsymbol{x}|\boldsymbol{w})
\right\Vert^2  d\boldsymbol{x} }_{\text{III.B}}.
\end{align*}
To bound the approximation error in (III.A), we construct the ReLU neural network $ \boldsymbol{s}^{(1)}$. We take $C_0 = 1 + C \sqrt{(2\alpha+1)\log N} = \mathcal{O}(\sqrt{\log N})$. In every network constructed below, we replace the state input $\boldsymbol{x}$ by $s_{\mathrm{clip}}(\boldsymbol{x},-C_0,C_0)$ and retain the same notation. On the region of Term III, $\|\boldsymbol{x}\|_\infty\leq C_0$, so this composition leaves the following estimates unchanged and ensures the required input-clipping identity on the whole space.
By Lemmas \ref{lem:relu-approximate-g3} and \ref{lem:relu-approximate-h3}, for any $\epsilon_0 > 0$, we can construct two ReLU neural networks $s_3(t,\boldsymbol{x}, \boldsymbol{w})$ and $\boldsymbol{s}_4(t,\boldsymbol{x}, \boldsymbol{w})$ such that
\begin{align*}
&|s_3(t,\boldsymbol{x}, \boldsymbol{w}) \vee N^{-(2\alpha+ 1)} - g_1(t,\boldsymbol{x}, \boldsymbol{w}) \vee N^{-(2\alpha+ 1)}| \\
& \leq |s_3(t,\boldsymbol{x}, \boldsymbol{w}) - g_1(t,\boldsymbol{x}, \boldsymbol{w})| \\ 
& \leq |s_3(t,\boldsymbol{x}, \boldsymbol{w}) - g_3(t,\boldsymbol{x}, \boldsymbol{w})| + |g_1(t,\boldsymbol{x}, \boldsymbol{w}) -g_3(t,\boldsymbol{x}, \boldsymbol{w})| \\
& \lesssim \epsilon_0 + \epsilon\log^{\frac{d_x}{2}}\epsilon^{-1},
\end{align*}
and
$$
\begin{aligned}
\Vert \boldsymbol{s}_4(t,\boldsymbol{x}, \boldsymbol{w}) - \boldsymbol{h}_1(t,\boldsymbol{x}, \boldsymbol{w})\Vert & \leq \Vert \boldsymbol{s}_4(t,\boldsymbol{x}, \boldsymbol{w}) - \boldsymbol{h}_3(t,\boldsymbol{x}, \boldsymbol{w})\Vert + \Vert \boldsymbol{h}_1(t,\boldsymbol{x}, \boldsymbol{w}) - \boldsymbol{h}_3(t,\boldsymbol{x}, \boldsymbol{w})\Vert
\\
& \lesssim \epsilon_0 + \epsilon\log^{\frac{d_x+1}{2}}\epsilon^{-1}.
\end{aligned}
$$
Note that $g_1(t,\boldsymbol{x}, \boldsymbol{w}) \vee N^{-(2\alpha+ 1)}$ can be rewritten using the ReLU function as
$$
g_1(t,\boldsymbol{x}, \boldsymbol{w}) \vee N^{-(2\alpha+ 1)} = \mathrm{ReLU}(g_1(t,\boldsymbol{x}, \boldsymbol{w}) - N^{-(2\alpha+ 1)}) + N^{-(2\alpha+ 1)},
$$ 
which is approximated by
$$
s_5 (t,\boldsymbol{x}, \boldsymbol{w}):= \mathrm{ReLU}(s_3 (t,\boldsymbol{x}, \boldsymbol{w}) - N^{-(2\alpha+ 1)}) + N^{-(2\alpha+ 1)}.
$$
For $1 \leq i \leq d_x$, we define $\boldsymbol{s}_6(t,\boldsymbol{x}, \boldsymbol{w})$ to approximate $[\boldsymbol{h}^{\prime}(t,\boldsymbol{x}, \boldsymbol{w})]_i$ as
$$
[\boldsymbol{s}_6(t,\boldsymbol{x}, \boldsymbol{w})]_i:= s_{\mathrm{clip}}\left(s_{\mathrm{prod},1}\left([\boldsymbol{s}_4(t,\boldsymbol{x}, \boldsymbol{w})]_i, s_{\mathrm{rec},1}(s_5(t,\boldsymbol{x}, \boldsymbol{w}))\right), -{C_1}\sqrt{\log N}, {C_1}\sqrt{\log N} \right).
$$
For a fixed $\eta_{\mathrm{sch}}>0$, Lemmas~\ref{lem:relu-ou-coefficients} and \ref{lem:relu-clip} provide the ReLU network
\[
\widehat{\sigma_t^2}:=s_{\mathrm{clip}}\left(1-E_{4,T,\eta_{\mathrm{sch}}}(t),\sigma_{T_1}^2,1\right),
~~
\sup_{t\in[T_1,3N^{-1}]}\left|\widehat{\sigma_t^2}-\sigma_t^2\right|\leq\eta_{\mathrm{sch}}.
\]
Set
\[
q_t:=s_{\mathrm{rec},2}\left(s_{\mathrm{root}}\left(\widehat{\sigma_t^2}\right)\right),~~
\delta_\sigma:=\epsilon_{\mathrm{rec},2}+\frac{\epsilon_{\mathrm{root}}+\frac{\eta_{\mathrm{sch}}}{\sqrt{\epsilon_{\mathrm{root}}}}}{\epsilon_{\mathrm{rec},2}^2}.
\]
Then
\[
\sup_{t\in[T_1,3N^{-1}]}\left|q_t-\frac{1}{\sigma_t}\right|\leq\delta_\sigma.
\]
For $b\geq0$, define
\[
\operatorname{clip}_{\mathrm R}(z,b):=-b+\mathrm{ReLU}(z+b)-\mathrm{ReLU}(z-b).
\]
This map is an exact ReLU network and satisfies $\operatorname{clip}_{\mathrm R}(z,b)=\max\{\min\{z,b\},-b\}$ even when $b$ is the output of another network. Subsequently, $\boldsymbol{s}^{(1)}(t,\boldsymbol{x},\boldsymbol{w})$ is defined to approximate $\frac{[\boldsymbol{h}^{\prime}(t,\boldsymbol{x},\boldsymbol{w})]_i}{\sigma_t}$ as
\[
[\boldsymbol{s}^{(1)}(t,\boldsymbol{x},\boldsymbol{w})]_i:=\operatorname{clip}_{\mathrm R}\left(s_{\mathrm{prod},2}\left([\boldsymbol{s}_6(t,\boldsymbol{x},\boldsymbol{w})]_i,q_t\right),C_1\sqrt{\log N}\left(q_t+\delta_\sigma\right)\right).
\]
Since $q_t+\delta_\sigma\geq\frac{1}{\sigma_t}$ and $|[\boldsymbol{h}^{\prime}]_i|\leq C_1\sqrt{\log N}$, the clipping interval contains the target $[\boldsymbol{h}^{\prime}]_i/\sigma_t$. Therefore,
\[
\epsilon_{\boldsymbol{s}^{(1)},i}\leq\epsilon_{\mathrm{prod},2}+2C_2\max\left\{\epsilon_{\boldsymbol{s}_6,i},\delta_\sigma\right\},
\]
where
\[
C_2:=1+\max\left\{\sup_{t\in[T_1,3N^{-1}]}\frac{1}{\sigma_t},\sup_{t,\boldsymbol{x},\boldsymbol{w}}\left|[\boldsymbol{h}^{\prime}(t,\boldsymbol{x},\boldsymbol{w})]_i\right|\right\}=\mathcal{O}\left(N^{\frac{C_\mu}{2}}\right),
\]
and
\[
\epsilon_{\boldsymbol{s}_6,i}:=\sup_{\substack{t\in[T_1,3N^{-1}], \boldsymbol{x}\in[-C_0,C_0]^{d_x},\\ \boldsymbol{w}\in[0,1]^{d_w}}}\left|[\boldsymbol{s}_6(t,\boldsymbol{x},\boldsymbol{w})]_i-[\boldsymbol{h}^{\prime}(t,\boldsymbol{x},\boldsymbol{w})]_i\right|.
\]
The error $\epsilon_{\boldsymbol{s}_{6},i}$ is bounded as follows
$$
\epsilon_{\boldsymbol{s}_{6},i} \lesssim \epsilon_{\mathrm{prod},1} + 2C_3\max \left\{\epsilon_0 + \epsilon\log^{\frac{d_x+1}{2}}\epsilon^{-1}, \epsilon_{\mathrm{rec},1} + \frac{\epsilon_0 + \epsilon\log^{\frac{d_x}{2}}\epsilon^{-1}}{\epsilon_{\mathrm{rec},1}^2} \right\},
$$
with $C_3 =\max \left\{ \sup\limits_{t,\boldsymbol{x},\boldsymbol{w}}|[\boldsymbol{h}_1(t,\boldsymbol{x}, \boldsymbol{w})]_i|, \sup\limits_{t,\boldsymbol{x},\boldsymbol{w}}|g_1(t,\boldsymbol{x}, \boldsymbol{w}) \vee N^{-(2\alpha+ 1)}| \right\} = \mathcal{O}(1)$.
By taking
$$
\epsilon = N^{-\left(3\alpha+\frac{3C_\mu}{2}+2\right)}, ~ \epsilon_{\mathrm{prod},1} = \frac{N^{-\alpha}}{8C_2}, ~ \epsilon_{\mathrm{prod},2} = \frac{N^{-\alpha}}{2}, ~ \epsilon_{\mathrm{rec},1} = \frac{N^{-\alpha}}{32C_2 C_3}, ~ \epsilon_{\mathrm{rec},2} = \frac{N^{-\alpha}}{8C_2},
$$
and
$$
\epsilon_0 = \min\left\{
\frac{N^{-\alpha}}{16C_2 C_3}, \left(\frac{N^{-\alpha}}{32C_2C_3}\right)^3
\right\}, ~ \epsilon_{\mathrm{root}} = \frac{1}{2}\left(\frac{N^{-\alpha}}{8C_2}\right)^3, ~ \eta_{\mathrm{sch}}=\epsilon_{\mathrm{root}}^{\frac{3}{2}},
$$
we obtain
\[
\delta_\sigma\leq\frac{N^{-\alpha}}{4C_2},~~ \epsilon_{\boldsymbol{s}_{6},i}\lesssim\frac{N^{-\alpha}}{C_2},~~ \epsilon_{\boldsymbol{s}^{(1)},i}\lesssim N^{-\alpha}.
\]
Since $\sigma_{T_1}^{-1}=\mathcal{O}(N^{\frac{C_\mu}{2}})$, the above choice gives $\log\eta_{\mathrm{sch}}^{-1}=\mathcal{O}(\log N)$. Hence, Lemma~\ref{lem:relu-ou-coefficients} and the additional exact clipping layer do not change the stated network-complexity orders. Moreover,
\[
\left|[\boldsymbol{s}^{(1)}(t,\boldsymbol{x},\boldsymbol{w})]_i\right|\leq C_1\sqrt{\log N}\left(q_t+\delta_\sigma\right)\leq C_1\sqrt{\log N}\left(\frac{1}{\sigma_t}+2\delta_\sigma\right)\lesssim\frac{\sqrt{\log N}}{\sigma_t}.
\]
Subsequently, we can obtain the network parameters of $[ \boldsymbol{s}^{(1)}(t,\boldsymbol{x},\boldsymbol{w})]_i$, $1\leq i \leq d_x$:
$$
\mathcal{D}_i = \mathcal{O}(\log^4 N),~\mathcal{W}_i = \mathcal{O}(N^{d_x+d_w}\log^7 N),~ \mathcal{S}_i = \mathcal{O}(N^{d_x+d_w}\log^9 N),~\mathcal{B}_i = \exp\left(\mathcal{O}(\log^4 N)\right).
$$
Combining $[ \boldsymbol{s}^{(1)}(t,\boldsymbol{x},\boldsymbol{w})]_i$, $1\leq i \leq d_x$, and by Lemma \ref{lem:relu-parallelization}, we construct the ReLU neural network 
$$
\boldsymbol{s}^{(1)}(t,\boldsymbol{x},\boldsymbol{w}):=\left[ [ \boldsymbol{s}^{(1)}(t,\boldsymbol{x},\boldsymbol{w})]_1, [ \boldsymbol{s}^{(1)}(t,\boldsymbol{x},\boldsymbol{w})]_2, \cdots, [ \boldsymbol{s}^{(1)}(t,\boldsymbol{x},\boldsymbol{w})]_{d_x} \right]^{\top},
$$
with network parameters
$$
\mathcal{D} = \mathcal{O}(\log^4 N),~ \mathcal{W} = \mathcal{O}(N^{d_x+d_w}\log^7 N),~ \mathcal{S} = \mathcal{O}(N^{d_x+d_w}\log^9 N),~ \mathcal{B}  = \exp\left(\mathcal{O}(\log^4 N)\right),
$$
and it satisfies $\Vert \boldsymbol{s}^{(1)}(t,\boldsymbol{x},\boldsymbol{w})\Vert_{\infty} \lesssim \frac{\sqrt{\log N}}{\sigma_t}$. Thus, the approximation error $(\text{III.A})$ is bounded by
\begin{equation}\label{eq:int-psTlogp2-term-III.A-bound}
\text{Term III.A} ~\lesssim d_xN^{-2\alpha} \lesssim N^{-2\alpha}.
\end{equation}
Next, we bound the approximation error (III.B), we always have $p_t({\boldsymbol{x}}|\boldsymbol{w}) \geq \epsilon_p=N^{-(2\alpha + 1)}$, and in this region, $\Vert \nabla\log p_t({\boldsymbol{x}}|\boldsymbol{w})\Vert \lesssim \frac{\sqrt{\log N}}{\sigma_t}$ and $\|\boldsymbol{x}\|_\infty\leq \mu_t+C\sigma_t\sqrt{(2\alpha+1)\log N}\leq 1+C\sqrt{(2\alpha+1)\log N}\lesssim C_0$ where we let $C_0=\cO(\sqrt{\log N})$. Now, we split $\{\boldsymbol{x}:D_t(\boldsymbol{x})\le C\sigma_t\sqrt{(2\alpha+1)\log N}\}$
into two parts:
\[
D_1:=\{\boldsymbol{x}:D_t(\boldsymbol{x})=0\}=[0,\mu_t]^{d_x},
~~
D_2:=\{\boldsymbol{x}:0<D_t(\boldsymbol{x})\le C\sigma_t\sqrt{(2\alpha+1)\log N}\}.
\]
\subparagraph{Region $D_1$.}
For $\boldsymbol{x}\in D_1$, we have $\boldsymbol{x}\in[0,\mu_t]^{d_x}$ and Lemma~\ref{lem:p-bound} gives $p_t(\boldsymbol{x}|\boldsymbol{w})\gtrsim1$. For $1\le i\le d_x$, we obtain
\begin{align*}
& \left|
\frac{[\boldsymbol{h}^{\prime}(t,\boldsymbol{x}, \boldsymbol{w})]_i}{\sigma_t} - [\nabla\log p_t(\boldsymbol{x}|\boldsymbol{w})]_i
\right| =
\left| \frac{[\boldsymbol{h}^{\prime}(t,\boldsymbol{x}, \boldsymbol{w})]_i-\sigma_t [\nabla\log p_t(\boldsymbol{x}|\boldsymbol{w})]_i}{\sigma_t}
\right| \\
& \lesssim \frac{1}{\sigma_t} 
\left| \frac{[\boldsymbol{h}_1(t,\boldsymbol{x}, \boldsymbol{w})]_i}{g_1(t,\boldsymbol{x}, \boldsymbol{w}) \vee N^{-(2\alpha+ 1)}} - \frac{\sigma_t[\nabla p_t(\boldsymbol{x}|\boldsymbol{w})]_i}{p_t(\boldsymbol{x}|\boldsymbol{w})}\right|
 \\
& \lesssim \frac{1}{\sigma_t} \cdot
\left| 
\frac{
[\boldsymbol{h}_1(t,\boldsymbol{x}, \boldsymbol{w})]_i ~p_t(\boldsymbol{x}|\boldsymbol{w}) - \sigma_t[\nabla p_t(\boldsymbol{x}|\boldsymbol{w})]_i (g_1(t,\boldsymbol{x}, \boldsymbol{w}) \vee N^{-(2\alpha+ 1)})
}{(g_1(t,\boldsymbol{x}, \boldsymbol{w}) \vee N^{-(2\alpha+ 1)}) p_t(\boldsymbol{x}|\boldsymbol{w})}
\right| \\
& \lesssim \frac{\Big|[\boldsymbol{h}_1(t,\boldsymbol{x}, \boldsymbol{w})]_i \Big| \cdot \Big|p_t(\boldsymbol{x}|\boldsymbol{w}) - g_1(t,\boldsymbol{x}, \boldsymbol{w}) \vee N^{-(2\alpha+ 1)}\Big|}{\sigma_t\cdot (g_1(t,\boldsymbol{x}, \boldsymbol{w}) \vee N^{-(2\alpha+ 1)})} + 
\frac{
\Big|[\boldsymbol{h}_1(t,\boldsymbol{x}, \boldsymbol{w})]_i - \sigma_t[\nabla p_t(\boldsymbol{x}|\boldsymbol{w})]_i\Big|
}{\sigma_t \cdot p_t(\boldsymbol{x}|\boldsymbol{w})}.
\end{align*}
Using $p_t(\boldsymbol{x}|\boldsymbol{w})\gtrsim1$, 
$g_1(t,\boldsymbol{x}, \boldsymbol{w}) \vee N^{-(2\alpha+ 1)}\gtrsim1$, and $|[\boldsymbol{h}_1(t,\boldsymbol{x},\boldsymbol{w})]_i|\lesssim\sqrt{\log N}$, we obtain
\begin{equation}\label{eq:h'sigma-Tp-bound-r-in-mut}
\begin{aligned}
\left|
\frac{[\boldsymbol{h}^{\prime}(t,\boldsymbol{x}, \boldsymbol{w})]_i}{\sigma_t} - [\nabla\log p_t(\boldsymbol{x}|\boldsymbol{w})]_i
\right| &\lesssim \frac{\sqrt{\log N}}{\sigma_t} \big|p_t(\boldsymbol{x}|\boldsymbol{w}) - g_1(t,\boldsymbol{x}, \boldsymbol{w}) \vee N^{-(2\alpha+ 1)}\big| \\
&~~ + \frac{1}{\sigma_t}\cdot \big|[\boldsymbol{h}_1(t,\boldsymbol{x}, \boldsymbol{w})]_i - \sigma_t[\nabla p_t(\boldsymbol{x}|\boldsymbol{w})]_i\big|\\
&\lesssim \frac{\sqrt{\log N}}{\sigma_t} \Big(
\big| p_t(\boldsymbol{x}|\boldsymbol{w}) - g_1(t,\boldsymbol{x}, \boldsymbol{w})\big| \\
&~~ + \Vert \boldsymbol{h}_1(t,\boldsymbol{x}, \boldsymbol{w}) - \sigma_t \nabla p_t(\boldsymbol{x}|\boldsymbol{w})\Vert
\Big).
\end{aligned}
\end{equation}
Here,  we use $p_t(\boldsymbol{x}|\boldsymbol{w}) \geq N^{-(2\alpha+ 1)}$ and $|p_t(\boldsymbol{x}|\boldsymbol{w}) \vee N^{-(2\alpha+ 1)} - g_1(t,\boldsymbol{x}, \boldsymbol{w}) \vee N^{-(2\alpha+ 1)}| \leq |p_t(\boldsymbol{x}|\boldsymbol{w}) - g_1(t,\boldsymbol{x}, \boldsymbol{w})|$. Substituting \eqref{eq:h'sigma-Tp-bound-r-in-mut} into the integral gives
$$
\begin{aligned}
& \int_{D_1} p_t(\boldsymbol{x}|\boldsymbol{w}) \mathbbm{1}_{\{p_t(\boldsymbol{x}|\boldsymbol{w}) \geq N^{-(2\alpha+ 1)}\}} \left\Vert
\frac{\boldsymbol{h}^{\prime}(t,\boldsymbol{x}, \boldsymbol{w})}{\sigma_t} - \nabla\log p_t(\boldsymbol{x}|\boldsymbol{w})
\right\Vert^2  d\boldsymbol{x} \\
& \lesssim \int_{D_1} \frac{\log N}{\sigma_t^2} \left(
| p_t(\boldsymbol{x}|\boldsymbol{w}) - g_1(t,\boldsymbol{x}, \boldsymbol{w})|^2 + \Vert \boldsymbol{h}_1(t,\boldsymbol{x}, \boldsymbol{w}) - \sigma_t \nabla p_t(\boldsymbol{x}|\boldsymbol{w})\Vert^2
\right)  d\boldsymbol{x} \\
& \lesssim \frac{\log N}{\sigma_t^2} \int_{D_1 } \Bigg(\left|
\int_{\mathbb{R}^{d_x}}\frac{1}{\sigma_t^{d_x}(2\pi)^{\frac{d_x}{2}}} [p(\boldsymbol{z}|\boldsymbol{w})-g(\boldsymbol{z}|\boldsymbol{w})]\exp\left(-\frac{\Vert\boldsymbol{x}-\mu_t \boldsymbol{z}\Vert^2}{2\sigma_t^2}\right) d\boldsymbol{z} \right|^2 \\
& ~~~~~~~~~ + \left\Vert \int_{\mathbb{R}^{d_x}}\frac{\boldsymbol{x}-\mu_t \boldsymbol{z}}{\sigma_t^{d_x+1}(2\pi)^{\frac{d_x}{2}}} [p(\boldsymbol{z}|\boldsymbol{w})-g(\boldsymbol{z}|\boldsymbol{w})]\exp\left(-\frac{\Vert\boldsymbol{x}-\mu_t \boldsymbol{z}\Vert^2}{2\sigma_t^2}\right) d\boldsymbol{z} \right\Vert^2 \Bigg)  d\boldsymbol{x} \\
& \lesssim \frac{\log N}{\sigma_t^2} \int_{D_1} \int_{\mathbb{R}^{d_x}}\frac{1+\frac{\Vert\boldsymbol{x}-\mu_t\boldsymbol{z}\Vert^2}{\sigma_t^2}}{\sigma_t^{d_x}(2\pi)^{\frac{d_x}{2}}} | p(\boldsymbol{z}|\boldsymbol{w})-g(\boldsymbol{z}|\boldsymbol{w}) |^2\exp\left(-\frac{\Vert\boldsymbol{x}-\mu_t\boldsymbol{z}\Vert^2}{2\sigma_t^2}\right) d\boldsymbol{z} d\boldsymbol{x}\\
& \lesssim \frac{\log N}{\sigma_t^2 \mu_t^{d_x}} \int_{D_1}N^{-2\alpha}  d\boldsymbol{x} \\
&\lesssim \frac{N^{-2\alpha}\log N}{\sigma_t^2}.
\end{aligned}
$$

\subparagraph{Region $D_2$.}
For $1 \leq i \leq d_x$, we have
$$
\begin{aligned}
\left| 
\frac{[\boldsymbol{h}^{\prime}(t,\boldsymbol{x}, \boldsymbol{w})]_i}{\sigma_t} - [\nabla\log p_t(\boldsymbol{x}|\boldsymbol{w})]_i
\right| &\leq \left| 
\frac{[\boldsymbol{h}_1(t,\boldsymbol{x}, \boldsymbol{w})]_i}{\sigma_t\cdot (g_1(t,\boldsymbol{x}, \boldsymbol{w}) \vee N^{-(2\alpha+ 1)})} - [\nabla\log p_t(\boldsymbol{x}|\boldsymbol{w})]_i
\right| \\
& \lesssim \frac{\left|
[\boldsymbol{h}_1(t,\boldsymbol{x}, \boldsymbol{w})]_i - \sigma_t [\nabla p_t(\boldsymbol{x}|\boldsymbol{w})]_i
\right|}{\sigma_t\cdot (g_1(t,\boldsymbol{x},\boldsymbol{w}) \vee N^{-(2\alpha+ 1)})}  \\
&~~ + |[\nabla p_t(\boldsymbol{x}|\boldsymbol{w})]_i|\cdot \left|\frac{1}{g_1(t,\boldsymbol{x}, \boldsymbol{w}) \vee N^{-(2\alpha+ 1)}} - \frac{1}{p_t(\boldsymbol{x}|\boldsymbol{w})}\right|.
\end{aligned}
$$
For the first term on the right-hand side, we use $g_1(t,\boldsymbol{x}, \boldsymbol{w}) \vee N^{-(2\alpha+ 1)} \geq N^{-(2\alpha+ 1)}$ to get
$$
\frac{\left|
[\boldsymbol{h}_1(t,\boldsymbol{x}, \boldsymbol{w})]_i - \sigma_t [\nabla p_t(\boldsymbol{x}|\boldsymbol{w})]_i
\right|}{\sigma_t \cdot (g_1(t,\boldsymbol{x}, \boldsymbol{w}) \vee N^{-(2\alpha+ 1)})} \lesssim \frac{\left|
[\boldsymbol{h}_1(t,\boldsymbol{x}, \boldsymbol{w})]_i - \sigma_t [\nabla p_t(\boldsymbol{x}|\boldsymbol{w})]_i
\right|}{\sigma_t N^{-(2\alpha+ 1)}}.
$$
For the second term, using $\frac{|[\nabla p_t(\boldsymbol{x}|\boldsymbol{w})]_i|}{p_t(\boldsymbol{x}|\boldsymbol{w})}=|[\nabla\log p_t(\boldsymbol{x}|\boldsymbol{w})]_i|\lesssim\frac{\sqrt{\log N}}{\sigma_t}$ and $g_1(t,\boldsymbol{x}, \boldsymbol{w}) \vee N^{-(2\alpha+ 1)} \geq N^{-(2\alpha+ 1)}$ yields
\begin{align*}
&|[\nabla p_t(\boldsymbol{x}|\boldsymbol{w})]_i|\cdot \left|\frac{1}{g_1(t,\boldsymbol{x}, \boldsymbol{w}) \vee N^{-(2\alpha+ 1)}} - \frac{1}{p_t(\boldsymbol{x}|\boldsymbol{w})}\right| \\
&\lesssim \frac{\sqrt{\log N}}{\sigma_t} \cdot \frac{|p_t(\boldsymbol{x}|\boldsymbol{w}) - g_1(t,\boldsymbol{x}, \boldsymbol{w}) \vee N^{-(2\alpha+ 1)}|}{g_1(t,\boldsymbol{x}, \boldsymbol{w}) \vee N^{-(2\alpha+ 1)}} \\
& \lesssim \frac{N^{2\alpha+ 1}\sqrt{\log N}}{\sigma_t} |p_t(\boldsymbol{x}|\boldsymbol{w}) - g_1(t,\boldsymbol{x}, \boldsymbol{w})|.
\end{align*}
Combining these two bounds gives
\begin{equation} \label{eq:h'sigma-Tp-bound-r-notin-mut}
\begin{aligned}
\left\Vert
\frac{\boldsymbol{h}^{\prime}(t,\boldsymbol{x}, \boldsymbol{w})}{\sigma_t} - \nabla\log p_t(\boldsymbol{x}|\boldsymbol{w})
\right\Vert \lesssim  & \frac{N^{2\alpha+ 1}\sqrt{\log N}}{\sigma_t} \Big(
|p_t(\boldsymbol{x}|\boldsymbol{w}) - g_1(t,\boldsymbol{x}, \boldsymbol{w})|\\
& + \Vert \boldsymbol{h}_1(t,\boldsymbol{x}, \boldsymbol{w}) - \sigma_t \nabla p_t(\boldsymbol{x}|\boldsymbol{w})\Vert
\Big).
\end{aligned}
\end{equation}
Substituting \eqref{eq:h'sigma-Tp-bound-r-notin-mut} into the integral gives
$$
\begin{aligned}
&  \int_{D_2} p_t(\boldsymbol{x}|\boldsymbol{w}) \mathbbm{1}_{\{p_t(\boldsymbol{x}|\boldsymbol{w}) \geq N^{-(2\alpha+ 1)}\}} \left\Vert
\frac{\boldsymbol{h}^{\prime}(t,\boldsymbol{x}, \boldsymbol{w})}{\sigma_t} - \nabla\log p_t(\boldsymbol{x}|\boldsymbol{w})
\right\Vert^2  d\boldsymbol{x} \\
& \lesssim \int_{D_2} \frac{N^{4\alpha + 2}\log N}{\sigma_t^2} \Big(
| p_t(\boldsymbol{x}|\boldsymbol{w}) - g_1(t,\boldsymbol{x}, \boldsymbol{w})|^2+ \Vert \boldsymbol{h}_1(t,\boldsymbol{x}, \boldsymbol{w}) - \sigma_t \nabla p_t(\boldsymbol{x}|\boldsymbol{w})\Vert^2
\Big)  d\boldsymbol{x} \\
& \lesssim \frac{N^{4\alpha + 2}\log N}{\sigma_t^2} \int_{D_2} \Bigg(\Bigg|
\int_{\mathbb{R}^{d_x}}\frac{1}{\sigma_t^{d_x}(2\pi)^{\frac{d_x}{2}}} [p(\boldsymbol{z}|\boldsymbol{w})-g(\boldsymbol{z}|\boldsymbol{w})]\exp\Bigg(-\frac{\Vert\boldsymbol{x}-\mu_t\boldsymbol{z}\Vert^2}{2\sigma_t^2}\Bigg) d\boldsymbol{z} \Bigg|^2 \\
& ~~~ + \left\Vert \int_{\mathbb{R}^{d_x}}\frac{\boldsymbol{x}-\mu_t\boldsymbol{z}}{\sigma_t^{d_x+1}(2\pi)^{\frac{d_x}{2}}} [p(\boldsymbol{z}|\boldsymbol{w})-g(\boldsymbol{z}|\boldsymbol{w})]\exp\left(-\frac{\Vert\boldsymbol{x}-\mu_t\boldsymbol{z}\Vert^2}{2\sigma_t^2}\right) d\boldsymbol{z} \right\Vert^2 \Bigg)  d\boldsymbol{x} \\
& \lesssim \frac{N^{4\alpha + 2}\log N}{\sigma_t^2} \int_{D_2} \Bigg(
\int_{\mathbb{R}^{d_x}}\frac{1}{\sigma_t^{d_x}(2\pi)^{\frac{d_x}{2}}} |p(\boldsymbol{z}|\boldsymbol{w})-g(\boldsymbol{z}|\boldsymbol{w})|^2\exp\left(-\frac{\Vert\boldsymbol{x}-\mu_t\boldsymbol{z}\Vert^2}{2\sigma_t^2}\right) d\boldsymbol{z} \\
& ~~~+  \int_{\mathbb{R}^{d_x}}\frac{\Vert\boldsymbol{x}-\mu_t\boldsymbol{z}\Vert^2}{\sigma_t^{d_x+2}(2\pi)^{\frac{d_x}{2}}} |p(\boldsymbol{z}|\boldsymbol{w})-g(\boldsymbol{z}|\boldsymbol{w})|^2\exp\left(-\frac{\Vert\boldsymbol{x}-\mu_t\boldsymbol{z}\Vert^2}{2\sigma_t^2}\right) d\boldsymbol{z}  \Bigg)  d\boldsymbol{x}. 
\end{aligned}
$$
By setting  $\epsilon = N^{-(6\alpha + 2)}$ in Lemma \ref{lem:integral-clipping} and replacing $p(\boldsymbol{z}|\boldsymbol{w})$ with $|p(\boldsymbol{z}|\boldsymbol{w})-g(\boldsymbol{z}|\boldsymbol{w})|^2$, we obtain 
$$
\begin{aligned}
& \int_{D_2}
\int_{\mathbb{R}^{d_x}}\frac{|p(\boldsymbol{z}|\boldsymbol{w})-g(\boldsymbol{z}|\boldsymbol{w})|^2}{\sigma_t^{d_x}(2\pi)^{\frac{d_x}{2}}}\exp\left(-\frac{\Vert\boldsymbol{x}-\mu_t\boldsymbol{z}\Vert^2}{2\sigma_t^2}\right) d\boldsymbol{z}  d\boldsymbol{x} \\
&	\lesssim \int_{D_2}
\int_{\Vert\boldsymbol{x} - \mu_t\boldsymbol{z}\Vert_{\infty} \leq \mathcal{O}(1)\sigma_t\sqrt{\log N}}\frac{|p(\boldsymbol{z}|\boldsymbol{w})-g(\boldsymbol{z}|\boldsymbol{w})|^2}{\sigma_t^{d_x}(2\pi)^{\frac{d_x}{2}}} \exp\left(-\frac{\Vert\boldsymbol{x}-\mu_t\boldsymbol{z}\Vert^2}{2\sigma_t^2}\right) d\boldsymbol{z} d\boldsymbol{x} + N^{-(6\alpha + 2)},
\end{aligned}
$$
and
$$
\begin{aligned}
& \int_{D_2}
\int_{\mathbb{R}^{d_x}}\frac{\Vert\boldsymbol{x} - \mu_t \boldsymbol{z}\Vert^2 |p(\boldsymbol{z}|\boldsymbol{w})-g(\boldsymbol{z}|\boldsymbol{w})|^2}{\sigma_t^{d_x+2}(2\pi)^{\frac{d_x}{2}}}\exp\left(-\frac{\Vert\boldsymbol{x}-\mu_t\boldsymbol{z}\Vert^2}{2\sigma_t^2}\right) d\boldsymbol{z} d\boldsymbol{x} \\
&\lesssim \int_{D_2}
\int_{\Vert\boldsymbol{x} - \mu_t \boldsymbol{z}\Vert_{\infty} \leq \mathcal{O}(1)\sigma_t\sqrt{\log N}}\frac{\log N |p(\boldsymbol{z}|\boldsymbol{w})-g(\boldsymbol{z}|\boldsymbol{w})|^2}{\sigma_t^{d_x}(2\pi)^{\frac{d_x}{2}}} \exp\left(-\frac{\Vert\boldsymbol{x}-\mu_t\boldsymbol{z}\Vert^2}{2\sigma_t^2}\right) d\boldsymbol{z} d\boldsymbol{x}+ N^{-(6\alpha + 2)}.
\end{aligned}
$$
On the clipped region, we have
$\|\boldsymbol{x}-\mu_t\boldsymbol{z}\|_\infty\le C\sigma_t\sqrt{\log N}$, and hence
$\|\boldsymbol{x}-\mu_t\boldsymbol{z}\|_\infty^2\lesssim\sigma_t^2\log N\lesssim\log N$. 
Now take $\boldsymbol{x}\in D_2$. Since $D_t(\boldsymbol{x})>0$, there exists an index $j$ such that either $x_j<0$ or $x_j>\mu_t$. Since $\boldsymbol{z}\in[0,1]^{d_x}$, if $x_j<0$, then
\[
\mu_tz_j\le |x_j-\mu_tz_j|\le C\sigma_t\sqrt{\log N},
\]
and if $x_j>\mu_t$, then
\[
\mu_t(1-z_j)\le |x_j-\mu_tz_j|\le C\sigma_t\sqrt{\log N}.
\]
For sufficiently large $N$, $\mu_t=e^{-2t}\asymp1$ for $t\in[T_1,3N^{-1}]$. Hence $z_j\le C\sigma_t\sqrt{\log N}$ or  $1-z_j\le C\sigma_t\sqrt{\log N}.$
For sufficiently large $N$, $C\sigma_t\sqrt{\log N}\le a_0$. Hence, for this index $j$, either $z_j\le a_0$ or $z_j\ge 1-a_0$, which implies $\boldsymbol{z}\in[0,1]^{d_x}\backslash[a_0,1-a_0]^{d_x}.$
Thus, by Lemma~\ref{lem:density-smooth-local-polynomials}, it holds that
$$
\begin{aligned}
& \int_{D_2}
\int_{\Vert\boldsymbol{x} -\mu_t \boldsymbol{z}\Vert_{\infty} \leq \mathcal{O}(1)\sigma_t\sqrt{\log N}}\frac{|p(\boldsymbol{z}|\boldsymbol{w})-g(\boldsymbol{z}|\boldsymbol{w})|^2}{\sigma_t^{d_x}(2\pi)^{\frac{d_x}{2}}} \exp\left(-\frac{\Vert\boldsymbol{x}-\mu_t\boldsymbol{z}\Vert^2}{2\sigma_t^2}\right) d\boldsymbol{z} d\boldsymbol{x} \\
&\lesssim \int_{D_2}
\int_{\{\boldsymbol{z}\in [0,1]^{d_x}\backslash [a_0, 1-a_0]^{d_x}\}} \frac{|p(\boldsymbol{z}|\boldsymbol{w})-g(\boldsymbol{z}|\boldsymbol{w})|^2}{\sigma_t^{d_x}(2\pi)^{\frac{d_x}{2}}} \exp\left(-\frac{\Vert\boldsymbol{x}-\mu_t\boldsymbol{z}\Vert^2}{2\sigma_t^2}\right) d\boldsymbol{z} d\boldsymbol{x} \\
& \lesssim \int_{\{\boldsymbol{z}\in [0,1]^{d_x}\backslash [a_0, 1-a_0]^{d_x}\}}|p(\boldsymbol{z}|\boldsymbol{w})-g(\boldsymbol{z}|\boldsymbol{w})|^2  d\boldsymbol{z} \\
&\lesssim N^{-(6\alpha + 4)}.
\end{aligned}
$$
Thus, we have
\begin{equation}\label{eq:int-psTlogp2-term-III.B-bound}
\begin{aligned}
& \int_{D_2} p_t(\boldsymbol{x}|\boldsymbol{w}) \mathbbm{1}_{\{p_t(\boldsymbol{x}|\boldsymbol{w}) \geq N^{-(2\alpha+ 1)}\}} \left\Vert
\frac{\boldsymbol{h}^{\prime}(t,\boldsymbol{x}, \boldsymbol{w})}{\sigma_t} - \nabla\log p_t(\boldsymbol{x}|\boldsymbol{w})
\right\Vert^2  d\boldsymbol{x} \\
& \lesssim  \frac{N^{4\alpha + 2}\log ^2 N}{\sigma_t^2} \cdot N^{-(6\alpha + 4)} + \frac{N^{4\alpha + 2}\log N}{\sigma_t^2} \cdot N^{-(6\alpha + 2)} \\
& \lesssim \frac{N^{-2\alpha}\log N}{\sigma_t^2}.
\end{aligned}
\end{equation}
Combining \eqref{eq:int-psTlogp2-term-III.A-bound} and \eqref{eq:int-psTlogp2-term-III.B-bound}, we finally establish that Term III $\lesssim \frac{N^{-2\alpha}\log N}{\sigma_t^2}$. Together with the bounds for Term I and Term II (cf. \eqref{eq:int-psTlogp2-term-I-bound}, \eqref{eq:int-psTlogp2-term-II-bound}), this completes the proof of the lemma.
\end{proof}

Next, we employ ReLU neural networks to approximate $\nabla\log p_t(\boldsymbol{x}|\boldsymbol{w})$ for all $t\in [2N^{-1}, T]$. Set $t^{\prime}_0=N^{-1}$. By virtue of the Markov property, the probability density function admits the following integral representation:
$$
p_t(\boldsymbol{x}|\boldsymbol{w}) = \frac{1}{(2\pi)^{\frac{d_x}{2}}\sigma_{t-t^{\prime}_0}^{d_x}}\int_{\mathbb{R}^{d_x}} p_{t^{\prime}_0}(\boldsymbol{z}|\boldsymbol{w}) \exp\left(-\frac{\Vert\boldsymbol{x} -\mu_{t-t^{\prime}_0}  \boldsymbol{z}\Vert^2}{2\sigma_{t-t^{\prime}_0}^2}\right)  d\boldsymbol{z}.
$$
For $t\in[2N^{-1},T]$ and $t^{\prime}_0=N^{-1}$, we have
\[
1\ge \frac{\sigma^2_{t-t^{\prime}_0}}{\sigma_t^2}
=1-\frac{e^{4t^{\prime}_0}-1}{e^{4t}-1}
\ge 1-\frac{e^{4\cdot N^{-1}}-1}{e^{4\cdot 2N^{-1}}-1}.
\]
Since the last quantity is bounded away from zero for sufficiently large $N$. Hence
\begin{equation}\label{eq:sigma-mu-equivalance}
\mu_{t-t^{\prime}_0}\asymp \mu_t,\quad \sigma_{t-t^{\prime}_0}\asymp \sigma_t.
\end{equation}
We first approximate $g_{t^{\prime}_0}$ and present the following lemma.

\begin{lemma}\label{lem:density-pT0-local-polynomials}
Let $N\gg 1$, there exists a constant $C > 0$ and a function $g_{t^{\prime}_0}$ such that
$$
\left(\int_{\mathbb{R}^{d_x}} \left|p_{t^{\prime}_0}(\boldsymbol{x}|\boldsymbol{w}) - g_{t^{\prime}_0}(\boldsymbol{x}|\boldsymbol{w}) \right|^2  d\boldsymbol{x} \right)^{\frac{1}{2}} \lesssim N^{-(3\alpha + 1)},
$$
where $g_{t^{\prime}_0}$ has the following form
\begin{align*}
g_{{t^{\prime}_0}}(\boldsymbol{x}|\boldsymbol{w}) =& {t^{\prime}_0}^{-\frac{k}{2}}\sum_{\boldsymbol{m},\boldsymbol{n}}~ \sum_{\norm{\boldsymbol{u}}_1+\norm{\boldsymbol{v}}_1< k}C_{\boldsymbol{m},\boldsymbol{n},\boldsymbol{u},\boldsymbol{v}}\\
&\cdot g_{\boldsymbol{m},\boldsymbol{n},\boldsymbol{u},\boldsymbol{v}}\left(\frac{\boldsymbol{x}}{2(1+ C\sqrt{\log N})}+\frac{1}{2},\boldsymbol{w}\right) \mathbbm{1}_{\{\Vert \boldsymbol{x} \Vert_\infty \leq 1+ C\sqrt{\log N}\}},
\end{align*}
with $k = \lfloor 6\alpha+3 \rfloor + 1 $.
\end{lemma}

\begin{proof}
Let $k \geq 1,0\leq l\leq k$. According to Lemma \ref{lem:derivatives-bound}, for any $\boldsymbol{x}\in\mathbb{R}^{d_x},\boldsymbol{w}\in [0,1]^{d_w}$, we have
$$
\Vert \partial_{x_{i_1}x_{i_2}\cdots x_{i_{k-l}}}\partial_{w_{j_1}\cdots w_{j_l}} p_{{t^{\prime}_0}} (\boldsymbol{x}|\boldsymbol{w})\Vert \lesssim \frac{1}{\sigma_{{t^{\prime}_0}}^k} \lesssim {t^{\prime}_0}^{-\frac{k}{2}}.
$$
Thus $p_{{t^{\prime}_0}} {t^{\prime}_0}^{\frac{k}{2}}$ is $k$-H\"older continuous on $\mathbb{R}^{d_x}$ with some constant $C_k$. By replacing $p_{{t^{\prime}_0}}$ with $p_{{t^{\prime}_0}}^2$ in Lemma \ref{lem:clipping-bound}, we can claim that there exists a constant $C$ such that $$\left(\int_{\Vert\boldsymbol{x}\Vert_\infty > 1+ C\sqrt{\log N}}p_{{t^{\prime}_0}}^2(\boldsymbol{x}|\boldsymbol{w})  d\boldsymbol{x} \right)^{\frac{1}{2}}\leq \left(\int_{\Vert\boldsymbol{x}\Vert_\infty > \mu_{{t^{\prime}_0}}+ C\sigma_{{t^{\prime}_0}}\sqrt{\log N}}p_{{t^{\prime}_0}}^2(\boldsymbol{x}|\boldsymbol{w})  d\boldsymbol{x} \right)^{\frac{1}{2}} \lesssim N^{-(3\alpha + 1)}.$$
Following a similar approach to Lemma \ref{lem:density-holder-local-polynomials}, we let $$
f(\boldsymbol{x},\boldsymbol{w}) = {t^{\prime}_0}^{\frac{k}{2}}p_{{t^{\prime}_0}}\left((1+ C\sqrt{\log N})(2\boldsymbol{x}-1)|\boldsymbol{w} \right), ~ \boldsymbol{x} \in [0,1]^{d_x},~\boldsymbol{w}\in [0,1]^{d_w}.
$$ 
Here, \( \Vert f \Vert_{\mathcal{H}([0,1]^{d_x+d_w})} \lesssim 2^k(1+ C\sqrt{\log N})^k = \mathcal{O}(\log^{\frac{k}{2}} N) \). Thus, there exists a function \( g_{t,0} \) satisfying
$$
\left| g_{t,0}(\boldsymbol{x}|\boldsymbol{w}) - {t^{\prime}_0}^{\frac{k}{2}}p_{{t^{\prime}_0}}(\boldsymbol{x}|\boldsymbol{w}) \right| \lesssim N^{-k}\log^{\frac{k}{2}}N, ~ \Vert\boldsymbol{x}\Vert_\infty \leq 1 + C\sqrt{\log N},
$$
where 
\begin{align*}
g_{t,0}(\boldsymbol{x}|\boldsymbol{w}) =& \sum_{\boldsymbol{m},\boldsymbol{n}}~ \sum_{\norm{\boldsymbol{u}}_1+\norm{\boldsymbol{v}}_1< k}C^\prime_{\boldsymbol{m},\boldsymbol{n},\boldsymbol{u},\boldsymbol{v}}\\
&\cdot g_{\boldsymbol{m},\boldsymbol{n},\boldsymbol{u},\boldsymbol{v}}\left(\frac{\boldsymbol{x}}{2(1+ C\sqrt{\log N})}+\frac{1}{2},\boldsymbol{w}\right) \mathbbm{1}_{\{\Vert \boldsymbol{x} \Vert_\infty \leq 1+ C\sqrt{\log N}\}},
\end{align*}
and $C^\prime_{\boldsymbol{m},\boldsymbol{n},\boldsymbol{u},\boldsymbol{v}}=\cO(\log^{\frac{k}{2}}N)$.
We then set \( g_{{t^{\prime}_0}}(\boldsymbol{x}|\boldsymbol{w}) = {t^{\prime}_0}^{-\frac{k}{2}}g_{t,0}(\boldsymbol{x}|\boldsymbol{w}) \), yielding
$$
\left|p_{{t^{\prime}_0}}(\boldsymbol{x}|\boldsymbol{w}) - g_{{t^{\prime}_0}}(\boldsymbol{x}|\boldsymbol{w})\right| \lesssim N^{-\frac{k}{2}}\log^{\frac{k}{2}}N.
$$
Choosing \( k = \lfloor 6\alpha+3 \rfloor + 1 \geq 3\alpha + 2 \), we find
$$
\begin{aligned}
\left(\int_{\Vert\boldsymbol{x}\Vert_\infty \leq 1+ C\sqrt{\log N}}\left|p_{{t^{\prime}_0}}(\boldsymbol{x}|\boldsymbol{w}) - g_{{t^{\prime}_0}}(\boldsymbol{x}|\boldsymbol{w})\right|^2  d\boldsymbol{x} \right) ^{\frac{1}{2}} \lesssim N^{-(3\alpha + 1)}
\end{aligned}
$$
for sufficient large $N$. Finally, \( g_{{t^{\prime}_0}} \) satisfies
$$
\begin{aligned}
&\left(\int_{\mathbb{R}^{d_x}} \left| p_{{t^{\prime}_0}}(\boldsymbol{x}|\boldsymbol{w}) - g_{{t^{\prime}_0}}(\boldsymbol{x}|\boldsymbol{w}) \right|^2  d\boldsymbol{x} \right)^{\frac{1}{2}}\\
& \leq \left(\int_{\Vert\boldsymbol{x}\Vert_\infty \leq 1+C\sqrt{\log N}}\left|p_{{t^{\prime}_0}}(\boldsymbol{x}|\boldsymbol{w}) - g_{{t^{\prime}_0}}(\boldsymbol{x}|\boldsymbol{w})\right|^2  d\boldsymbol{x} + \int_{\Vert\boldsymbol{x}\Vert_\infty > 1 + C\sqrt{\log N}}p_{{t^{\prime}_0}}^2(\boldsymbol{x}|\boldsymbol{w})  d\boldsymbol{x} \right)^{\frac{1}{2}} \\
&\lesssim {N}^{-(3\alpha + 1)}.
\end{aligned}
$$
We have therefore established the desired result.
\end{proof}

Based on Lemma \ref{lem:density-pT0-local-polynomials}, we can construct a ReLU neural network to approximate $\nabla\log p_t(\boldsymbol{x}|\boldsymbol{w})$ on $[2N^{-1}, T]$.

\begin{lemma} \label{lem:score-approximate-interval-2}
Let $N \gg 1$, there exists a ReLU neural network $ \boldsymbol{s}^{(2)}\in \cG(\mathcal{D},\mathcal{W},\mathcal{S},\mathcal{B})$ with
$$
\mathcal{D} = \mathcal{O}(\log^4 N),~\mathcal{W} = \mathcal{O}(N^{d_x+d_w}\log^7 N),~\mathcal{S}= \mathcal{O}({N}^{d_x+d_w}\log^9 N),~ \mathcal{B} = \exp\left(\mathcal{O}(\log^4 N)\right)
$$
that satisfies
$$
\int_{\mathbb{R}^{d_x}} p_t(\boldsymbol{x}|\boldsymbol{w})\Vert  \boldsymbol{s}^{(2)}(t,\boldsymbol{x}, \boldsymbol{w}) - \nabla\log p_t(\boldsymbol{x}|\boldsymbol{w})\Vert^2  d\boldsymbol{x} \lesssim \frac{N^{-2\alpha}\log N}{\sigma_t^2}, ~~ t\in [2N^{-1}, T].
$$
Moreover, we can take $ \boldsymbol{s}^{(2)}$ satisfying $\Vert  \boldsymbol{s}^{(2)} (t, \boldsymbol{x},\boldsymbol{w})\Vert_{\infty} \lesssim \frac{\sqrt{\log N}}{\sigma_t}$. 
\end{lemma}

\begin{proof}
The proof is similar to the proof of Lemma \ref{lem:score-approximate-interval-1}.
We also decompose the approximation error into three terms:
\begin{align*}
& \int_{\mathbb{R}^{d_x}} p_t(\boldsymbol{x}|\boldsymbol{w})\Vert  \boldsymbol{s}^{(2)}(t,\boldsymbol{x}, \boldsymbol{w}) - \nabla\log p_t(\boldsymbol{x}|\boldsymbol{w})\Vert^2  d\boldsymbol{x} \\
& = \underbrace{ \int_{D_t(\boldsymbol{x})>C\sigma_t\sqrt{\log\epsilon_1^{-1}},} p_t(\boldsymbol{x}|\boldsymbol{w}) \Vert  \boldsymbol{s}^{(2)}(t,\boldsymbol{x}, \boldsymbol{w}) - \nabla\log p_t(\boldsymbol{x}|\boldsymbol{w})\Vert^2  d\boldsymbol{x} }_{\text{I}} \\
&~~ +  \underbrace{\int_{D_t(\boldsymbol{x})\le C\sigma_t\sqrt{\log\epsilon_1^{-1}}} p_t(\boldsymbol{x}|\boldsymbol{w}) \mathbbm{1}_{\{p_t(\boldsymbol{x}|\boldsymbol{w}) < \epsilon_p\}} \Vert  \boldsymbol{s}^{(2)}(t,\boldsymbol{x}, \boldsymbol{w}) - \nabla\log p_t(\boldsymbol{x}|\boldsymbol{w})\Vert^2  d\boldsymbol{x} }_{\text{II}} \\
&~~ +  \underbrace{\int_{D_t(\boldsymbol{x})\le C\sigma_t\sqrt{\log\epsilon_1^{-1}}} p_t(\boldsymbol{x}|\boldsymbol{w}) \mathbbm{1}_{\{p_t(\boldsymbol{x}|\boldsymbol{w}) \geq \epsilon_p\}}\Vert  \boldsymbol{s}^{(2)}(t,\boldsymbol{x}, \boldsymbol{w}) - \nabla\log p_t(\boldsymbol{x}|\boldsymbol{w})\Vert^2  d\boldsymbol{x} }_{\text{III}}. 
\end{align*}

\noindent	\textbf{Bound Term I.}  To begin with, we analyze Term I using Lemma 
\ref{lem:clipping-bound}, there exists a constant $C > 0$ such that for any $\epsilon_1 > 0$, 
\begin{align*}
\text{Term I}  &\leq 2 \int_{D_t(\boldsymbol{x})>C\sigma_t\sqrt{\log\epsilon_1^{-1}},}  p_t(\boldsymbol{x}|\boldsymbol{w}) \norm{ \boldsymbol{s}^{(2)}(t,\boldsymbol{x}, \boldsymbol{w})}^2  d\boldsymbol{x} \\
&~~ +2  \int_{D_t(\boldsymbol{x})>C\sigma_t\sqrt{\log\epsilon_1^{-1}},} p_t(\boldsymbol{x}|\boldsymbol{w}) \norm{\nabla\log p_t(\boldsymbol{x}|\boldsymbol{w}) }^2  d\boldsymbol{x} \\
& \lesssim \frac{\epsilon_1}{\sigma_t} + \sigma_t\epsilon_1\Vert  s^{(2)}(t,
\boldsymbol{x},\boldsymbol{w})\Vert_{\infty}^2.
\end{align*}
We then invoke Lemma \ref{lem:derivatives-bound}, which yields that \( \left\Vert \nabla\log p_t(\boldsymbol{x}|\boldsymbol{w}) \right\Vert \leq \frac{C\sqrt{\log\epsilon^{-1}_1}}{\sigma_t} \) whenever \( D_t(\boldsymbol{x})\leq  C\sigma_t\sqrt{\log\epsilon^{-1}_1} \). Motivated by this bound, we choose the function \( \boldsymbol{s}^{(2)} \) such that its supremum norm satisfies \( \left\Vert \boldsymbol{s}^{(2)}(t, \boldsymbol{x},\boldsymbol{w}) \right\Vert_{\infty} \lesssim \frac{\sqrt{\log\epsilon^{-1}_1}}{\sigma_t} \). Taking \( \epsilon_1 = N^{-(2\alpha+1)} \), we have 
\begin{equation*}
\text{Term I} \lesssim \frac{\epsilon_1 \log \epsilon^{-1}_1}{\sigma_t}\lesssim \frac{N^{-(2\alpha+1)}}{\sigma_t}\log N.
\end{equation*}

\noindent \textbf{Bound Term II.}  Choosing $\epsilon_p = N^{-(2\alpha+1)}$, Lemma \ref{lem:clipping-bound} gives
\begin{equation*}
\begin{aligned}
\text{Term II}  &\leq 2 \int_{D_t(\boldsymbol{x})\leq  C\sigma_t\sqrt{\log\epsilon^{-1}_1}} p_t(\boldsymbol{x}|\boldsymbol{w})\mathbbm{1}_{\{p_t(\boldsymbol{x}|\boldsymbol{w}) < \epsilon_p\}} \Vert 
\boldsymbol{s}^{(2)}(t,\boldsymbol{x}, \boldsymbol{w})\Vert^2  d\boldsymbol{x} \\
&~~ +2  \int_{D_t(\boldsymbol{x})\leq  C\sigma_t\sqrt{\log\epsilon^{-1}_1}} p_t(\boldsymbol{x}|\boldsymbol{w})\mathbbm{1}_{\{p_t(\boldsymbol{x}|\boldsymbol{w}) < \epsilon_p\}} \Vert  \nabla\log p_t(\boldsymbol{x}|\boldsymbol{w})\Vert^2  d\boldsymbol{x} \\
& \lesssim \frac{\epsilon_p}{\sigma_t^2} (\log\epsilon^{-1}_1)^{\frac{d_x+2}{2}} +  \epsilon_p  (\log\epsilon^{-1}_1)^{\frac{d_x}{2}} \Vert \boldsymbol{s}^{(2)}(t,\boldsymbol{x},\boldsymbol{w})\Vert_{\infty}^2\\
&\lesssim \frac{\epsilon_p}{\sigma^2_t} (\log \epsilon^{-1}_1)^{\frac{d_x+2}{2}}\\
&\lesssim \frac{N^{-(2\alpha+1)}}{\sigma^2_t}\log^{\frac{d_x+2}{2}}N.
\end{aligned}
\end{equation*}

\noindent \textbf{Bound Term III.}   We first take $C_1 \geq C\sqrt{2\alpha+ 1}$ (Replace $C$ in Lemma \ref{lem:density-pT0-local-polynomials} with $C_1$) and $C_0 =1 + C_1\sqrt{\log N} = \mathcal{O}(\sqrt{\log N})$. 
Then, we replace $\boldsymbol{z} - \frac{m}{N}$ with
$\frac{\boldsymbol{z}}{2(1+ C_1\sqrt{\log N})} +\frac{1}{2}- \frac{m}{N}$ in $\psi(t,x,m,u,j)$, 
and replace $\alpha(\boldsymbol{u},\boldsymbol{v})< C_{\alpha}$ and $C_\beta,C_\gamma$ with $\alpha(\boldsymbol{u},\boldsymbol{v})\leq k=\mathcal{O}(1)$  
and $1+ C_1\sqrt{\log N}$.  We keep the notations
\begin{align*}
&g_1(t,\boldsymbol{x}, \boldsymbol{w}) := \int_{\mathbb{R}^{d_x}}g_{{t^{\prime}_0}}(\boldsymbol{z}|\boldsymbol{w})\mathbbm{1}_{\{\Vert\boldsymbol{z}\Vert_{\infty}\leq 1+C_1\sqrt{\log N}\}}\cdot
\frac{1}{\sigma_{t-{t^{\prime}_0}}^{d_x}(2\pi)^{\frac{d_x}{2}}}\exp\left(-\frac{\Vert\boldsymbol{x} - \mu_{t-{t^{\prime}_0}} \boldsymbol{z}\Vert^2}{2\sigma_{t-{t^{\prime}_0}}^2}\right)  d\boldsymbol{z},\\
&\boldsymbol{h}_1(t,\boldsymbol{x}, \boldsymbol{w}):= \int_{\mathbb{R}^{d_x}}g_{{t^{\prime}_0}}(\boldsymbol{z}|\boldsymbol{w}) \mathbbm{1}_{\{\Vert\boldsymbol{z}\Vert_{\infty}\leq 1+C_1\sqrt{\log N}\}}\cdot
\frac{\mu_{t-{t^{\prime}_0}}\boldsymbol{z}-\boldsymbol{x}}{\sigma_{t-{t^{\prime}_0}}^{d_x+1}(2\pi)^{\frac{d_x}{2}}}\exp\left(-\frac{\Vert\boldsymbol{x} -\mu_{t-{t^{\prime}_0}} \boldsymbol{z}\Vert^2}{2\sigma_{t-{t^{\prime}_0}}^2}\right)  d\boldsymbol{z},\\
&\boldsymbol{h}^{\prime}(t,\boldsymbol{x}, \boldsymbol{w}):= \max\left\{\min\left\{ \frac{\boldsymbol{h}_1(t,\boldsymbol{x}, \boldsymbol{w})}{ g_1(t,\boldsymbol{x}, \boldsymbol{w}) \vee {N}^{-(2\alpha+ 1)}}, \mathcal{O}(\sqrt{\log N}) \right\}, -\mathcal{O}(\sqrt{\log N})\right\}.
\end{align*}

Applying the argument of Lemma \ref{lem:score-approximate-interval-1}, we can construct a ReLU neural network $ \boldsymbol{s}^{(2)}$ with network parameters
$$
\mathcal{D} = \mathcal{O}(\log^4 N),~ \mathcal{W} = \mathcal{O}({N}^{d_x+d_w}\log^7 N),~ \mathcal{S} = \mathcal{O}({N}^{d_x+d_w}\log^9 N),~ \mathcal{B} = \exp\left(\mathcal{O}(\log^4 N)\right)
$$
that satisfy 
$$  
\left\Vert  \boldsymbol{s}^{(2)}(t,\boldsymbol{x}, \boldsymbol{w}) - \frac{\boldsymbol{h}^{\prime}(t,\boldsymbol{x}, \boldsymbol{w})}{\sigma_{t-{t^{\prime}_0}}} \right\Vert_\infty \lesssim {N}^{-\alpha}, ~ \Vert\boldsymbol{x}\Vert_{\infty} \leq 1 + C_1\sqrt{\log N},
$$
and
$$
\Vert  \boldsymbol{s}^{(2)}(t, \boldsymbol{x},\boldsymbol{w}) \Vert_{\infty} \lesssim \frac{\sqrt{\log N}}{\sigma_{t-{t^{\prime}_0}}}\lesssim \frac{\sqrt{\log N}}{\sigma_t}~~\text{(By \eqref{eq:sigma-mu-equivalance})}.
$$
Therefore, we have
\begin{align*}
\int_{D_t(\boldsymbol{x})\le C\sigma_t\sqrt{\log\epsilon_1^{-1}}} p_t(\boldsymbol{x}|\boldsymbol{w})\mathbbm{1}_{\{p_t(\boldsymbol{x}|\boldsymbol{w})\geq {N}^{-(2\alpha+ 1)}\}} \left\Vert  \boldsymbol{s}^{(2)}(t,\boldsymbol{x}, \boldsymbol{w}) - \frac{\boldsymbol{h}^{\prime}(t,\boldsymbol{x}, \boldsymbol{w})}{\sigma_{t-{t^{\prime}_0}}} \right\Vert^2  d \boldsymbol{x} \lesssim d_x{N}^{-2\alpha} \lesssim {N}^{-2\alpha}.
\end{align*}
By \eqref{eq:h'sigma-Tp-bound-r-notin-mut} in Lemma \ref{lem:score-approximate-interval-1},
\begin{align*}
\left\Vert
\frac{\boldsymbol{h}^{\prime}(t,\boldsymbol{x}, \boldsymbol{w})}{\sigma_{t-{t^{\prime}_0}}} - \nabla\log p_t(\boldsymbol{x}|\boldsymbol{w})
\right\Vert \lesssim  & \frac{N^{2\alpha+ 1}\sqrt{\log N}}{\sigma_{t}} \Big(
|p_t(\boldsymbol{x}|\boldsymbol{w}) - g_1(t,\boldsymbol{x}, \boldsymbol{w})|\\
& + \Vert \boldsymbol{h}_1(t,\boldsymbol{x}, \boldsymbol{w}) - \sigma_{t-{t^{\prime}_0}} \nabla p_t(\boldsymbol{x}|\boldsymbol{w})\Vert
\Big).
\end{align*}
Then, we have
$$
\begin{aligned}
& \int_{D_t(\boldsymbol{x})\le C\sigma_t\sqrt{\log\epsilon_1^{-1}}} p_t(\boldsymbol{x}|\boldsymbol{w})\mathbbm{1}_{\{p_t(\boldsymbol{x}|\boldsymbol{w})\geq {N}^{-(2\alpha+ 1)}\}} \left\Vert \frac{\boldsymbol{h}^{\prime}(t,\boldsymbol{x}, \boldsymbol{w})}{\sigma_{t-{t^{\prime}_0}}} - \nabla\log p_t(\boldsymbol{x}|\boldsymbol{w})\right\Vert^2  d \boldsymbol{x} \\
& \lesssim \frac{N^{4\alpha+2}\log N}{\sigma_{t-{t^{\prime}_0}}^2}
\int_{\mathbb{R}^{d_x}}\left(
|p_t(\boldsymbol{x}|\boldsymbol{w})-g_1(t,\boldsymbol{x},\boldsymbol{w})|^2
+\|\boldsymbol{h}_1(t,\boldsymbol{x},\boldsymbol{w})-\sigma_{t-{t^{\prime}_0}}\nabla p_t(\boldsymbol{x}|\boldsymbol{w})\|^2
\right)d\boldsymbol{x}\\
&\lesssim \frac{N^{4\alpha+2}\log N}{\sigma_{t-{t^{\prime}_0}}^2}
\int_{\mathbb{R}^{d_x}}\int_{\mathbb{R}^{d_x}}
\left(1+\frac{\|\boldsymbol{x}-\mu_{t-{t^{\prime}_0}}\boldsymbol{z}\|^2}{\sigma_{t-{t^{\prime}_0}}^2}\right)
\frac{|p_{{t^{\prime}_0}}(\boldsymbol{z}|\boldsymbol{w})-g_{{t^{\prime}_0}}(\boldsymbol{z}|\boldsymbol{w})|^2}
{\sigma_{t-{t^{\prime}_0}}^{d_x}(2\pi)^{\frac{d_x}{2}}}\\
&\quad\quad\quad\quad\quad \times
\exp{\bigg(-\frac{\|\boldsymbol{x}-\mu_{t-{t^{\prime}_0}}\boldsymbol{z}\|^2}{2\sigma_{t-{t^{\prime}_0}}^2}\bigg)}
d\boldsymbol{z}d\boldsymbol{x}\\
& \lesssim \frac{{N}^{4\alpha + 2}\log N}{\sigma_{t-{t^{\prime}_0}}^2} \int_{\mathbb{R}^{d_x}} \left|p_{{t^{\prime}_0}}(\boldsymbol{z}|\boldsymbol{w})-g_{{t^{\prime}_0}}(\boldsymbol{z}|\boldsymbol{w}) \right|^2  d\boldsymbol{z} \\
&\lesssim \frac{{N}^{4\alpha + 2}\log N}{\sigma_t^2} \cdot {N}^{-(6\alpha + 2)} \\
&\lesssim \frac{{N}^{-2\alpha}\log N}{\sigma_t^2}.
\end{aligned}
$$
Therefore, we finally obtain
$$
\text{Term III} \lesssim \frac{{N}^{-2\alpha}\log N}{\sigma_t^2},
$$
which implies that
$$
\int_{\mathbb{R}^{d_x}} p_t(\boldsymbol{x}|\boldsymbol{w})\Vert  \boldsymbol{s}^{(2)}(t,\boldsymbol{x}, \boldsymbol{w}) - \nabla\log p_t(\boldsymbol{x}|\boldsymbol{w})\Vert^2  d\boldsymbol{x} \lesssim \frac{{N}^{-2\alpha}\log N}{\sigma_t^2}.
$$
The proof is complete.
\end{proof}

Combining Lemma \ref{lem:score-approximate-interval-1} and Lemma \ref{lem:score-approximate-interval-2}, we immediately obtain Lemma \ref{lem:score-approximation}.

\begin{proof}[Proof of Lemma \ref{lem:score-approximation}]
By 
Lemma \ref{lem:score-approximate-interval-1} and Lemma \ref{lem:score-approximate-interval-2}, there exist two ReLU neural networks $\boldsymbol{s}^{(1)}
(t,\boldsymbol{x}, \boldsymbol{w})$ and $\boldsymbol{s}^{(2)}(t,\boldsymbol{x}, \boldsymbol{w})$ approximating the score function $\nabla\log p_t(\boldsymbol{x}|\boldsymbol{w})$ over the time intervals $[T_1,3N^{-1}]$ and $[2N^{-1},T]$, respectively. 
Combined with Lemma \ref{lem:relu-product} and \ref{lem:relu-switch}, we are able to assemble a piecewise ReLU neural network via switching weighting:
\begin{align*}
\boldsymbol{s}(t,\boldsymbol{x}, \boldsymbol{w}):&=s_{\mathrm{prod}}\Big( s_{\mathrm{switch},1}(t,2N^{-1},3N^{-1}) , \boldsymbol{s}^{(1)}
(t,\boldsymbol{x}, \boldsymbol{w}) \Big)\\
&~~ + s_{\mathrm{prod}}\Big(s_{\mathrm{switch},2}(t,2N^{-1},3N^{-1}) , \boldsymbol{s}^{(2)}(t,\boldsymbol{x}, \boldsymbol{w})\Big).
\end{align*}
Lemma \ref{lem:relu-product} with $\epsilon=N^{-\alpha}$ and $\epsilon_0=0$, gives
\[
\left\|\boldsymbol{s}-s_{\mathrm{switch},1}\boldsymbol{s}^{(1)}-s_{\mathrm{switch},2}\boldsymbol{s}^{(2)}\right\|_\infty\leq2N^{-\alpha}.
\]
The product networks add only $\mathcal{O}(\log N)$ depth and size per coordinate, which are dominated by the original network complexities. Hence,
\[
\mathcal{D}=\mathcal{O}(\log^4N),\quad \mathcal{W}=\mathcal{O}(N^{d_x+d_w}\log^7N),\quad \mathcal{S}=\mathcal{O}(N^{d_x+d_w}\log^9N),\quad \mathcal{B}=\exp \left(\mathcal{O}(\log^4N)\right).
\]

We next establish the approximation error. Since $s_{\mathrm{switch},1}+s_{\mathrm{switch},2}=1$,
\[
\begin{aligned}
&\int_{\mathbb{R}^{d_x}}p_t(\boldsymbol{x}|\boldsymbol{w})\left\|\boldsymbol{s}-\nabla\log p_t(\boldsymbol{x}|\boldsymbol{w})\right\|^2d\boldsymbol{x}\\
&\lesssim s_{\mathrm{switch},1}^2\int_{\mathbb{R}^{d_x}}p_t(\boldsymbol{x}|\boldsymbol{w})\left\|\boldsymbol{s}^{(1)}-\nabla\log p_t(\boldsymbol{x}|\boldsymbol{w})\right\|^2d\boldsymbol{x}\\
&\quad+s_{\mathrm{switch},2}^2\int_{\mathbb{R}^{d_x}}p_t(\boldsymbol{x}|\boldsymbol{w})\left\|\boldsymbol{s}^{(2)}-\nabla\log p_t(\boldsymbol{x}|\boldsymbol{w})\right\|^2d\boldsymbol{x}+N^{-2\alpha}\\
&\lesssim\left(s_{\mathrm{switch},1}^2+s_{\mathrm{switch},2}^2\right)\frac{N^{-2\alpha}\log N}{\sigma_t^2}+N^{-2\alpha}\lesssim\frac{N^{-2\alpha}\log N}{\sigma_t^2}.
\end{aligned}
\]
Here, we use $0\leq s_{\mathrm{switch},1},s_{\mathrm{switch},2}\leq1$, $s_{\mathrm{switch},1}+s_{\mathrm{switch},2}=1$, and $\sigma_t^2\leq1$.
This completes the proof.
\end{proof}

\subsection{Statistical Error of the Score Function}\label{sec:score-decompose}
Recall that in Section~\ref{sec:er1}, the auxiliary losses are given by
\[
{\mathcal{L}}_{\bS_x}(\boldsymbol{s})
=
\frac{1}{n_x}\sum_{i=1}^{n_x}
\ell(\boldsymbol{s},\bar{\boldsymbol{x}}_{0,i},\boldsymbol{W}_i),
~~
\widehat{\mathcal{L}}_{\bS_x,\bT,\bZ}(\boldsymbol{s})
=
\frac{1}{n_x}\sum_{i=1}^{n_x}
\widehat{\ell}(\boldsymbol{s},\bar{\boldsymbol{x}}_{0,i},\boldsymbol{W}_i),
\]
where
\[
\ell(\boldsymbol{s},\bar{\boldsymbol{x}}_{0},\boldsymbol{W})
=
\frac{1}{T-2T_1}
\int_{T_1}^{T-T_1}
\mathbb{E}_{Z}
\left\|
\boldsymbol{s}\big(t,\bar{\boldsymbol{x}}_{0}\mu_t+Z\sigma_t,\boldsymbol{W}\big)
+\frac{Z}{\sigma_t}
\right\|^2dt,
\]
and
\[
\widehat{\ell}(\boldsymbol{s},\bar{\boldsymbol{x}}_{0},\boldsymbol{W})
=
\frac{1}{m_x}
\sum_{j=1}^{m_x}
\left\|
\boldsymbol{s}\big(t_j,\bar{\boldsymbol{x}}_{0}\mu_{t_j}+Z_j\sigma_{t_j},\boldsymbol{W}\big)
+\frac{Z_j}{\sigma_{t_j}}
\right\|^2.
\]
Then, for any
\[
\begin{aligned}
\mathcal{L}(\widehat{\boldsymbol{s}}) - \mathcal{L}(\boldsymbol{s}^*)
&= \mathcal{L}(\widehat{\boldsymbol{s}}) - 2{\mathcal{L}}_{\bS_x}(\widehat{\boldsymbol{s}}) + \mathcal{L}(\boldsymbol{s}^*) + 2\Big( {\mathcal{L}}_{\bS_x}(\widehat{\boldsymbol{s}}) - \mathcal{L}(\boldsymbol{s}^*) \Big) \\
&= \mathcal{L}(\widehat{\boldsymbol{s}}) - 2{\mathcal{L}}_{\bS_x}(\widehat{\boldsymbol{s}}) + \mathcal{L}(\boldsymbol{s}^*) + 2\left( {\mathcal{L}}_{\bS_x}(\widehat{\boldsymbol{s}}) - \widehat{\mathcal{L}}_{\bS_x,\bT,\bZ}(\widehat{\boldsymbol{s}}) \right) + 2\Big( \widehat{\mathcal{L}}_{\bS_x,\bT,\bZ}(\widehat{\boldsymbol{s}}) - \mathcal{L}(\boldsymbol{s}^*) \Big) \\
&\leq \mathcal{L}(\widehat{\boldsymbol{s}}) - 2{\mathcal{L}}_{\bS_x}(\widehat{\boldsymbol{s}}) + \mathcal{L}(\boldsymbol{s}^*) + 2\left( {\mathcal{L}}_{\bS_x}(\widehat{\boldsymbol{s}}) - \widehat{\mathcal{L}}_{\bS_x,\bT,\bZ}(\widehat{\boldsymbol{s}}) \right) + 2\inf_{\boldsymbol{s}\in \cG}\Big( \widehat{\mathcal{L}}_{\bS_x,\bT,\bZ}({\boldsymbol{s}}) - \mathcal{L}(\boldsymbol{s}^*) \Big).
\end{aligned}
\]
Taking expectations, followed by taking the infimum over $\boldsymbol{s} \in  \cG$ on both sides of the above inequality, it holds that
\[
\begin{aligned}
&\mathbb{E}_{\bS_x,\bT,\bZ} \bigg[ \frac{1}{T - 2T_1} \int_{T_1}^{T-T_1} \mathbb{E}_{\bar{\boldsymbol{x}}_t, \boldsymbol{W}} \| \widehat{\boldsymbol{s}}(t, \bar{\boldsymbol{x}}_t, \boldsymbol{W}) - \nabla_{\bar{\boldsymbol{x}}_t} \log p_t(\boldsymbol{x}_t|\boldsymbol{W}) \|^2 dt \bigg]\\
&= \mathbb{E}_{\bS_x,\bT,\bZ} \Big[ \mathcal{L}(\widehat{\boldsymbol{s}}) - \mathcal{L}(\boldsymbol{s}^*) \Big] \\
&\leq \mathbb{E}_{\bS_x,\bT,\bZ} \Big[ \mathcal{L}(\widehat{\boldsymbol{s}}) - 2{\mathcal{L}}_{\bS_x}(\widehat{\boldsymbol{s}}) + \mathcal{L}(\boldsymbol{s}^*) \Big] + 2\mathbb{E}_{\bS_x,\bT,\bZ} \Big[ {\mathcal{L}}_{\bS_x}(\widehat{\boldsymbol{s}}) - \widehat{\mathcal{L}}_{\bS_x,\bT,\bZ}(\widehat{\boldsymbol{s}}) \Big] + 2 \inf_{\boldsymbol{s} \in  \cG} \Big[ \mathcal{L}(\boldsymbol{s}) - \mathcal{L}(\boldsymbol{s}^*) \Big].
\end{aligned}
\]
In the above inequality, the terms
\[
\mathbb{E}_{\bS_x,\bT,\bZ} \Big[ \mathcal{L}(\widehat{\boldsymbol{s}}) - 2{\mathcal{L}}_{\bS_x}(\widehat{\boldsymbol{s}}) + \mathcal{L}(\boldsymbol{s}^*) \Big],~ 2\mathbb{E}_{\bS_x,\bT,\bZ} \Big[ {\mathcal{L}}_{\bS_x}(\widehat{\boldsymbol{s}}) - \widehat{\mathcal{L}}_{\bS_x,\bT,\bZ}(\widehat{\boldsymbol{s}}) \Big] 
\]
and
\[
\inf_{\boldsymbol{s} \in  \cG} \Big[ \mathcal{L}(\boldsymbol{s}) - \mathcal{L}(\boldsymbol{s}^*) \Big]
\]
denote the statistical error and approximation error, respectively. We begin by providing an upper bound for $\ell(\boldsymbol{s},\bar{\boldsymbol{x}}_0,\boldsymbol{W})$.

\begin{paragraph}{Upper bound for $\ell(\boldsymbol{s},\bar{\boldsymbol{x}}_0,\boldsymbol{W})$}
First, we have
\begin{align*}
\mathbb{E}_{{Z}}\left\Vert \boldsymbol{s}(t,\bar{\boldsymbol{x}}_0\mu_t + {Z}\sigma_t,\boldsymbol{W}) + \frac{{Z}}{\sigma_{t}}\right\Vert^2 
\lesssim \frac{\log N}{\sigma^2_t} + \frac{\mathbb{E}\|{Z}\|^2}{\sigma^2_t} 
\lesssim \frac{\log N}{1-e^{-4t}}.
\end{align*}
Then, substituting the above estimate yields
\begin{align*}
\ell(\boldsymbol{s},\bar{\boldsymbol{x}}_0,\boldsymbol{W})
&=\frac{1}{T-2T_1} \int_{T_1}^{T-T_1} \mathbb{E}_{{Z}} \left\| \boldsymbol{s}\big(t, \bar{\boldsymbol{x}}_0 \mu_t + {Z}\sigma_t, \boldsymbol{W}\big) +\frac{{Z}}{\sigma_t} \right\|^2dt\\
&\lesssim \frac{\log N}{T-2T_1} \int_{T_1}^{T-T_1} \frac{1}{1-e^{-4t}} dt\\
&\lesssim \log N \cdot \frac{\log(e^{4(T-T_1)}-1)-\log(e^{4T_1}-1)}{T-2T_1}\\
& \lesssim  \log N  \cdot \frac{C_\lambda \log N}{C_\lambda \log N-2N^{-C_\mu}}\\
& \leq \log N
\end{align*}
for all sufficiently large $N$.
\end{paragraph}

\begin{paragraph}{Lipschitz continuity for $\ell(\boldsymbol{s},\bar{\boldsymbol{x}}_0,\boldsymbol{W})$.}
Now we derive the Lipschitz continuity for $\ell(\boldsymbol{s},\bar{\boldsymbol{x}}_0,\boldsymbol{W})$. Note that we restrict ReLU neural networks into class $\mathcal{G}$. For any $\boldsymbol{s}_1$, $\boldsymbol{s}_2 \in \mathcal{G}$, by the construction structure of $\boldsymbol{s}_1$, $\boldsymbol{s}_2$,
we have
\begin{equation*}
\begin{aligned}
&\big|\ell({\boldsymbol{s}_1},\bar{\boldsymbol{x}}_0,\boldsymbol{W}) - \ell(\boldsymbol{s}_2,\bar{\boldsymbol{x}}_0,\boldsymbol{W})\big|\\
&\leq\frac{1}{T-2T_1}\int_{T_1}^{T-T_1}
\mathbb{E}_{{Z}} \left[\Vert \boldsymbol{s}_1 - \boldsymbol{s}_2\Vert \cdot \left\Vert \boldsymbol{s}_1 + \boldsymbol{s}_2 + \frac{2{Z}}{\sigma_t}\right\Vert\right]  dt\\
&\leq\frac{1}{T-2T_1}\int_{T_1}^{T-T_1}\left(\mathbb{E}_{{Z}}\Vert \boldsymbol{s}_1 - \boldsymbol{s}_2 \Vert^2\right)^{\frac{1}{2}}\left(\mathbb{E}_{{Z}}\left\Vert \boldsymbol{s}_1 + \boldsymbol{s}_2 + \frac{2{Z}}{\sigma_t}\right\Vert^2\right)^{\frac{1}{2}} dt\\
&\lesssim \frac{1}{T-2T_1}\int_{T_1}^{T-T_1}\left(\mathbb{E}_{{Z}}\Vert \boldsymbol{s}_1 - \boldsymbol{s}_2 \Vert^2\right)^{\frac{1}{2}}\left(\frac{\log N}{\sigma_t^2} + \frac{d_x}{\sigma_t^2}\right)^{\frac{1}{2}} dt\\
&\lesssim \left(\frac{1}{T-2T_1}\int_{T_1}^{T-T_1}\mathbb{E}_{{Z}}\left\Vert \boldsymbol{s}_1 - \boldsymbol{s}_2 \right\Vert^2  dt\right)^{\frac{1}{2}}\left( \frac{1}{T-2T_1}\int_{T_1}^{T-T_1} \frac{\log N}{\sigma^2_t}  dt\right)^{\frac{1}{2}}\\
&\lesssim  \log^{\frac{1}{2}} N \cdot \left(\frac{1}{T-2T_1}\int_{T_1}^{T-T_1}\mathbb{E}_{{Z}}\Vert \boldsymbol{s}_1 - \boldsymbol{s}_2 \Vert^2  dt\right)^{\frac{1}{2}}\\
&\lesssim  \log^{\frac{1}{2}} N \cdot
\Vert \boldsymbol{s}_1 - \boldsymbol{s}_2 \Vert_{L^{\infty}([T_1,T-T_1]\times\mathbb{R}^{d_x}\times [0,1]^{d_w})}\\
& \lesssim \log^{\frac{1}{2}} N \cdot \Vert \boldsymbol{s}_1 - \boldsymbol{s}_2 \Vert_{L^{\infty}([N^{-C_{\mu}},C_{\lambda}{\log N}]\times[-\mathcal{O}(1)\sqrt{\log N}, \mathcal{O}(1) \sqrt{\log N}]^{d_x}\times [0,1]^{d_w})}.
\end{aligned}
\end{equation*}
\end{paragraph}

\begin{paragraph}{Covering number evaluation.} 
The covering number of the neural network class is evaluated as follows:
\begin{equation*}
\begin{aligned}
&\log\mathcal{N}\left(\mathcal{G}, \delta,\Vert\cdot\Vert_{L^{\infty}([T_1,T-T_1]\times\mathbb{R}^{d_x}\times [0,1]^{d_w})}\right) \\
&\lesssim \log\mathcal{N}\left(\mathcal{G}, \delta,\Vert\cdot\Vert_{L^{\infty}([N^{-C_{\mu}},C_{\lambda}\log N]\times[-\mathcal{O}(1)\sqrt{\log N}, \mathcal{O}(1) \sqrt{\log N}]^{d_x}\times [0,1]^{d_w})}\right)\\
&\lesssim \cS \cD\log{\left(\frac{\cD \cW \cB (C_{\lambda}\log N \vee \mathcal{O}(1)\sqrt{\log N}\vee 1)}{\delta}\right)}\\
& \lesssim {N}^{d_x+d_w}\log^{13}N\left(\log^4 N + \log \frac{1}{\delta} \right),
\end{aligned}
\end{equation*} 
where we use 
\begin{gather*}
\cD = \mathcal{O}(\log^4 N), \cW = \mathcal{O}({N}^{d_x+d_w}\log^7 N), \cS = \mathcal{O}(N^{d_x+d_w}\log^9 N), \cB = \exp\left(\mathcal{O}(\log^4 N)\right).
\end{gather*}
The details of this derivation can be found in \citet[Lemma 5.3]{CJLZ2022nonparametric}.
\end{paragraph}

\begin{proof}[Proof of Lemma \ref{lem:score-statistical-error}]
Let $\tilde{\ell}(\boldsymbol{s},\bar{\boldsymbol{x}}_0,\boldsymbol{W}) = \ell(\boldsymbol{s},\bar{\boldsymbol{x}}_0,\boldsymbol{W}) - \ell ({\boldsymbol{s}^*},\bar{\boldsymbol{x}}_0,\boldsymbol{W})$ and $\bS_x^{\prime} = \{\boldsymbol{W}^{\prime}_i,\bar{\boldsymbol{x}}_{0,i}^{\prime}\}_{i=1}^{n_x}$ be an independent copy of $\bS_x$, 
then for any $\boldsymbol{s}_1$, $\boldsymbol{s}_2\in\mathcal{G}$, there exists a constant $C_1 > 0$ such that
\begin{align*}
\big|\tilde{\ell}(\boldsymbol{s}_1,\bar{\boldsymbol{x}}_0,\boldsymbol{W}) - \tilde{\ell}(\boldsymbol{s}_2,\bar{\boldsymbol{x}}_0,\boldsymbol{W})\big| &= \big|\ell(\boldsymbol{s}_1,\bar{\boldsymbol{x}}_0,\boldsymbol{W}) - \ell(\boldsymbol{s}_2,\bar{\boldsymbol{x}}_0,\boldsymbol{W})\big|\\
&\leq C_1\log^{\frac{1}{2}} N \Vert \boldsymbol{s}_1 - \boldsymbol{s}_2 \Vert_{L^{\infty}([T_1,T-T_1]\times\mathbb{R}^{d_x}\times[0,1]^{d_w})}. 
\end{align*}
We first estimate $\mathbb{E}_{\bS_x,\bT,\bZ} \Big[ \mathcal{L}(\widehat{\boldsymbol{s}}) - 2{\mathcal{L}}_{\bS_x}(\widehat{\boldsymbol{s}}) + \mathcal{L}(\boldsymbol{s}^*) \Big]$. 
It follows that
\begin{equation*}
\begin{aligned}
&\mathbb{E}_{\bS_x,\bT,\bZ} \Big[ \mathcal{L}(\widehat{\boldsymbol{s}}) - 2{\mathcal{L}}_{\bS_x}(\widehat{\boldsymbol{s}}) + \mathcal{L}(\boldsymbol{s}^*) \Big]\\ 
&= \mathbb{E}_{\bS_x,\bT,\bZ} \Bigg(\mathbb{E}_{\bS_x^{\prime}}\left[\frac{1}{n_x}\sum_{i=1}^{n_x}\Big(\ell({\widehat{\boldsymbol{s}}},\bar{\boldsymbol{x}}_{0,i}^{\prime},\boldsymbol{W}^{\prime}_i) - \ell(\boldsymbol{s}^*,\bar{\boldsymbol{x}}_{0,i}^{\prime},\boldsymbol{W}^{\prime}_i)\Big)\right]\\
&~~~~~~~~~~~~~~~~ -\frac{2}{n_x}\sum_{i=1}^{n_x}\Big(\ell({\widehat{\boldsymbol{s}}},\bar{\boldsymbol{x}}_{0,i},\boldsymbol{W}_i) - \ell({\boldsymbol{s}^*},\bar{\boldsymbol{x}}_{0,i},\boldsymbol{W}_i)\Big)
\Bigg)\\
&=\mathbb{E}_{\bS_x,\bT,\bZ}\left[\frac{1}{n_x}\sum_{i=1}^{n_x}G(\widehat{\boldsymbol{s}},\bar{\boldsymbol{x}}_{0,i},\boldsymbol{W}_i)\right],
\end{aligned}
\end{equation*}
where, for any $\boldsymbol{s}\in\mathcal G$, define
\[
G(\boldsymbol{s},\bar{\boldsymbol{x}}_{0,i},\boldsymbol{W}_i):=\mathbb{E}_{\bS_x^{\prime}} \left[\tilde{\ell}(\boldsymbol{s},\bar{\boldsymbol{x}}_{0,i}^{\prime},\boldsymbol{W}_i^{\prime})\right]-2\tilde{\ell}(\boldsymbol{s},\bar{\boldsymbol{x}}_{0,i},\boldsymbol{W}_i).
\]
Let $\mathcal{G}_{\delta}$ be the $\delta$-covering of $\mathcal{G}$ with minimum cardinality 
$$\mathcal{N}_\delta:=\mathcal{N}( \cG,\delta,\Vert\cdot\Vert_{L^{\infty}([T_1,T-T_1]\times\mathbb{R}^{d_x}\times [0,1]^{d_w})}),$$
then for any $\boldsymbol{s}\in\mathcal{G}$, there exists a $\boldsymbol{s}_{\delta}\in\mathcal{G}_{\delta}$ such that
\begin{align*}
|\tilde{\ell}(\boldsymbol{s},\bar{\boldsymbol{x}}_{0},\boldsymbol{W}) - \tilde{\ell}(\boldsymbol{s}_{\delta},\bar{\boldsymbol{x}}_{0},\boldsymbol{W})|&\leq C_1 \log^{\frac{1}{2}} N \Vert \boldsymbol{s} - \boldsymbol{s}_{\delta}\Vert_{L^{\infty}([T_1,T-T_1]\times\mathbb{R}^{d_x}\times[0,1]^{d_w})}\\
&\leq C_1 \delta \log^{\frac{1}{2}} N .
\end{align*}
Therefore, 
$$
G(\widehat{\boldsymbol{s}},\bar{\boldsymbol{x}}_{0,i},\boldsymbol{W}_i)\leq G(\boldsymbol{s}_{\delta},\bar{\boldsymbol{x}}_{0,i},\boldsymbol{W}_i) + 3C_1 \delta \log^{\frac{1}{2}} N.
$$
Since $|\tilde{\ell}(\boldsymbol{s}_{\delta},\bar{\boldsymbol{x}}_{0,i},\boldsymbol{W}_i)|\lesssim \log N$, there exists a constant $C_2 > 0$ such that
$|\tilde{\ell}(\boldsymbol{s}_{\delta},\bar{\boldsymbol{x}}_{0,i},\boldsymbol{W}_i)| \leq C_2 \log N$.
We have that $|\tilde{\ell}(\boldsymbol{s}_{\delta}, \bar{\boldsymbol{x}}_{0,i},\boldsymbol{W}_i) - \mathbb{E}\tilde{\ell}(\boldsymbol{s}_{\delta},\bar{\boldsymbol{x}}_{0,i},\boldsymbol{W}_i)|\leq 2C_2\log N$. Let $\sigma^2 = \mathrm{Var}[\tilde{\ell}(\boldsymbol{s}_{\delta},\bar{\boldsymbol{x}}_{0,i},\boldsymbol{W}_i)]$, then 
\begin{equation*}
\begin{aligned}
\sigma^2&\leq\mathbb{E}_{\bS_x}\big[\tilde{\ell}(\boldsymbol{s}_{\delta},\bar{\boldsymbol{x}}_{0,i},\boldsymbol{W}_i)\big]^2\\
&\leq C_1^2  \log N \cdot \bE_{\bS_x} \left[\frac{1}{T-2T_1}\int_{T_1}^{T-T_1}\mathbb{E}_{{Z}}\Vert \boldsymbol{s}_\delta - \boldsymbol{s}^*  \Vert^2  dt\right]\\
&= C_1^2   \log N \cdot  \big[\mathcal{L}(\boldsymbol{s}_\delta)-\cL (\boldsymbol{s}^*)\big]\\
&= C_1^2  \log N \mathbb{E}_{\bS_x}\big[\ell({\boldsymbol{s}_{\delta}},\bar{\boldsymbol{x}}_{0,i},\boldsymbol{W}_i)- \ell({\boldsymbol{s}^*},\bar{\boldsymbol{x}}_{0,i},\boldsymbol{W}_i)\big]\\
&= C_1^2  \log N \mathbb{E}_{\bS_x}[\tilde{\ell}(\boldsymbol{s}_{\delta}, \bar{\boldsymbol{x}}_{0,i},\boldsymbol{W}_i)].
\end{aligned}
\end{equation*}
We obtain
$$
\mathbb{E}_{\bS_x}[\tilde{\ell}(\boldsymbol{s}_{\delta},\bar{\boldsymbol{x}}_{0,i},\boldsymbol{W}_i)]\geq\frac{\sigma^2}{C_1^2 \log N}.
$$
For each fixed $\boldsymbol{s}_\delta\in\mathcal G_\delta$, define
$$
F(\bS_x,\bar{\boldsymbol{x}}_{0,i},\boldsymbol{W}):=\mathbb{E}_{\bS_x}\left[\frac{1}{n_x}\sum_{i=1}^{n_x}\tilde{\ell}(\boldsymbol{s}_{\delta}, \bar{\boldsymbol{x}}_{0,i},\boldsymbol{W})\right] - \frac{1}{n_x}\sum_{i=1}^{n_x}\tilde{\ell}(\boldsymbol{s}_{\delta},\bar{\boldsymbol{x}}_{0,i},\boldsymbol{W}_i).
$$
By Bernstein's inequality  (See Lemma \ref{lem:bernstein-hoeffding-inequality}), we have
\begin{equation*}
\begin{aligned}
&\mathbb{P}_{\bS_x,\bT,\bZ}\left[\frac{1}{n_x}\sum_{i=1}^{n_x}G(\boldsymbol{s}_{\delta},\bar{\boldsymbol{x}}_{0,i},\boldsymbol{W}_i) > t\right]\\ 
&=\mathbb{P}_{\bS_x,\bT,\bZ}\Bigg( \mathbb{E}_{\bS_x^{\prime}}\Big[\tilde{\ell}(\boldsymbol{s}_{\delta},\bar{\boldsymbol{x}}^{\prime}_{0,i},\boldsymbol{W}^{\prime}) - 2\tilde{\ell}(\boldsymbol{s}_{\delta},{\bar{\boldsymbol{x}}_{0,i}},\boldsymbol{W})\Big]>t\Bigg)\\
&=\mathbb{P}_{\bS_x,\bT,\bZ}\left(F(\bS_x^{\prime},\bar{\boldsymbol{x}}^{\prime}_{0,i},\boldsymbol{W}^{\prime}_i) > \frac{t}{2} + \mathbb{E}_{\bS_x^{\prime}}\left[\frac{1}{2n}\sum_{i=1}^{n_x}\tilde{\ell}(\boldsymbol{s}_{\delta}, \bar{\boldsymbol{x}}^{\prime}_{0,i},\boldsymbol{W}^{\prime}_i) \right]\right)\\
&=\mathbb{P}_{\bS_x,\bT,\bZ}\left(F(\bS_x,\bar{\boldsymbol{x}}^{\prime}_{0,i},\boldsymbol{W}_i)> \frac{t}{2} + \mathbb{E}_{\bS_x}\left[\frac{1}{2n}\sum_{i=1}^{n_x}\tilde{\ell}(\boldsymbol{s}_{\delta}, \bar{\boldsymbol{x}}_{0,i},\boldsymbol{W}_i) \right]\right)\\
&\leq\mathbb{P}_{\bS_x,\bT,\bZ}\left(F(\bS_x,\bar{\boldsymbol{x}}_{0,i},\boldsymbol{W}_i)> \frac{t}{2} + \frac{\sigma^2}{2C_1^2\log N}\right)\\
&\leq\exp\left(-\frac{nv^2}{2\sigma^2 + \frac{4vC_2 \log N}{3}}\right)\\
&\leq\exp\left(-\frac{nt}{8\left(C_1^2 + \frac{C_2}{3}\right)\log N}\right),
\end{aligned}
\end{equation*}
where $v = \frac{t}{2} + \frac{\sigma^2}{2C_1^2 \log N}$, and we use $v\geq\frac{t}{2}$ and $\sigma^2\leq 2vC_1^2 \log N $. 
Hence,  for any $t > 3C_1 \delta \log^{\frac{1}{2}} N $, we have
\begin{equation*}
\begin{aligned}
& \mathbb{P}_{\bS_x,\bT,\bZ}\left[\frac{1}{n_x}\sum_{i=1}^{n_x}G(\widehat{\boldsymbol{s}}, \bar{\boldsymbol{x}}_{0,i},\boldsymbol{W}_i) > t\right] \\
&\leq\mathbb{P}_{\bS_x,\bT,\bZ}\left[\mathop{\mathrm{sup}}_{\boldsymbol{s}\in\mathcal{G}}\frac{1}{n_x}\sum_{i=1}^{n_x}G(\boldsymbol{s}, \bar{\boldsymbol{x}}_{0,i},\boldsymbol{W}_i) > t\right]\\
&\leq\mathbb{P}_{\bS_x,\bT,\bZ}\left[\mathop{\mathrm{max}}_{\boldsymbol{s}_{\delta}\in \mathcal{G}_{\delta}}\frac{1}{n_x}\sum_{i=1}^{n_x}G(\boldsymbol{s}_{\delta},\bar{\boldsymbol{x}}_{0,i},\boldsymbol{W}_i) > t- 3C_1\delta \log^{\frac{1}{2}}  N\right]\\
&\leq\mathcal{N}_{\delta}\mathop{\mathrm{max}}_{\boldsymbol{s}_{\delta}\in \mathcal{G}_{\delta}}\mathbb{P}_{\bS_x,\bT,\bZ}\left[\frac{1}{n_x}\sum_{i=1}^{n_x}G(\boldsymbol{s}_{\delta}, \bar{\boldsymbol{x}}_{0,i},\boldsymbol{W}_i) > t- 3C_1\delta \log^{\frac{1}{2}}  N\right]\\
&\leq\mathcal{N}_{\delta}\exp\left(-\frac{n_x (t-3C_1  \delta \log^{\frac{1}{2}}  N )}{8\left(C_1^2 + \frac{C_2}{3}\right)\log N}\right).
\end{aligned}
\end{equation*}
By setting $\epsilon=3C_1 \delta \log^{\frac{1}{2}}  N  + \epsilon_0, \delta = \frac{1}{n_x}$ and $\epsilon_0=\frac{8(C_1^2+\frac{C_2}{3}) \log N\log \cN_\delta}{n_x}$, then we obtain
\begin{equation}\label{eq:score-statsitical-error-1}
\begin{aligned}
&\mathbb{E}_{\bS_x,\bT,\bZ}\left[\frac{1}{n_x}\sum_{i=1}^{n_x}G(\widehat{\boldsymbol{s}}, \bar{\boldsymbol{x}}_{0,i},\boldsymbol{W}_i) \right]\\
&\leq \epsilon + \mathcal{N}_{\delta}\int_{\epsilon}^{\infty}\exp\left(-\frac{n_x(t-3C_1\delta \log^{\frac{1}{2}}  N)}{8\left(C_1^2 + \frac{C_2}{3}\right)\log N}\right) dt\\
& \leq \frac{3C_1  \log^{\frac{1}{2}}  N + 8\left(C_1^2 + \frac{C_2}{3} \right)  \log N\log\mathcal{N}_{\delta}}{n_x}\\
&~~~~ + \frac{8\left(C_1^2 + \frac{C_2}{3} \right)\log N}{n_x} \\
& \lesssim {N}^{d_x+d_w}\log^{14} N\left(\log^4 N + \log\frac{1}{\delta}\right) \delta + \frac{\log N}{n_x} \\
& \lesssim \frac{{N}^{d_x+d_w}\log^{14} N\left(\log^4 N + \log n_x\right)}{n_x}.
\end{aligned}
\end{equation}
Next, we bound
$$
\mathbb{E}_{\bS_x,\bT,\bZ} \Big[ {\mathcal{L}}_{\bS_x}(\widehat{\boldsymbol{s}}) - \widehat{\mathcal{L}}_{\bS_x,\bT,\bZ}(\widehat{\boldsymbol{s}}) \Big]=\mathbb{E}_{\bS_x,\bT,\bZ} \Big[ \frac{1}{n_x}\sum_{i=1}^{n_x}\left(\ell({\widehat{\boldsymbol{s}}},\bar{\boldsymbol{x}}_{0,i},\boldsymbol{W}_i) - \widehat{\ell}({\widehat{\boldsymbol{s}}},\bar{\boldsymbol{x}}_{0,i},\boldsymbol{W}_i)\right) \Big].
$$
We decompose $\frac{1}{n_x}\sum_{i=1}^{n_x}\left(\ell({\widehat{\boldsymbol{s}}},\bar{\boldsymbol{x}}_{0,i},\boldsymbol{W}_i) - \widehat{\ell}({\widehat{\boldsymbol{s}}},\bar{\boldsymbol{x}}_{0,i},\boldsymbol{W}_i)\right)$ into 
the following three terms
\begin{align*}
\frac{1}{n_x}\sum_{i=1}^{n_x}\left(\ell({\widehat{\boldsymbol{s}}},\bar{\boldsymbol{x}}_{0,i},\boldsymbol{W}_i) - \widehat{\ell}({\widehat{\boldsymbol{s}}},\bar{\boldsymbol{x}}_{0,i},\boldsymbol{W}_i)\right)=&\underbrace{\frac{1}{n_x}\sum_{i=1}^{n_x}\Big(\ell({\widehat{\boldsymbol{s}}},\bar{\boldsymbol{x}}_{0,i},\boldsymbol{W}_i) - \ell^{\mathrm{trunc}}({\widehat{\boldsymbol{s}}},\bar{\boldsymbol{x}}_{0,i},\boldsymbol{W}_i)
\Big)}_{\text{I}}\\
& +\underbrace{\frac{1}{n_x}\sum_{i=1}^{n_x}\Big(\ell^{\mathrm{trunc}}({\widehat{\boldsymbol{s}}},\bar{\boldsymbol{x}}_{0,i},\boldsymbol{W}_i)- \widehat{\ell}^{\mathrm{trunc}}({\widehat{\boldsymbol{s}}},\bar{\boldsymbol{x}}_{0,i},\boldsymbol{W}_i)
\Big)}_{\text{II}}\\
& +\underbrace{\frac{1}{n_x}\sum_{i=1}^{n_x}
\Big(\widehat{\ell}^{\mathrm{trunc}}({\widehat{\boldsymbol{s}}},\bar{\boldsymbol{x}}_{0,i},\boldsymbol{W}_i) - \widehat{\ell}({\widehat{\boldsymbol{s}}},\bar{\boldsymbol{x}}_{0,i},\boldsymbol{W}_i)
\Big)}_{\text{III}},
\end{align*}
where 
$$
\ell^{\mathrm{trunc}}({\widehat{\boldsymbol{s}}},{\bar{\boldsymbol{x}}}_0,\boldsymbol{W}) = \mathbb{E}_{{Z}}\left(\frac{1}{T-2T_1}\int_{T_1}^{T-T_1}\left\Vert\widehat{\boldsymbol{s}}(t,{\bar{\boldsymbol{x}}}_0\mu_t + {Z}\sigma_t,\boldsymbol{W}) + \frac{{Z}}{\sigma_t}\right\Vert^2 dt\cdot\mathbbm{1}_{\{\Vert{Z}\Vert_{\infty} \leq M\}}\right),
$$
and
$$
\widehat{\ell}^{\mathrm{trunc}}({\widehat{\boldsymbol{s}}},{\bar{\boldsymbol{x}}}_0,\boldsymbol{W}) = \frac{1}{m_x}\sum_{j=1}^{m_x}\left\Vert\widehat{\boldsymbol{s}}(t_j,{\bar{\boldsymbol{x}}}_0\mu_{t_j}+ {Z}_j\sigma_{t_j},\boldsymbol{W}) + \frac{{Z}_j}{\sigma_{t_j}}\right\Vert^2\mathbbm{1}_{\{\Vert{Z}_j\Vert_{\infty} \leq M\}}.
$$
We analyze the three terms individually. To begin with, we have
\begin{equation*}
\begin{aligned}
\text{Term I}&=\frac{1}{n_x}\sum_{i=1}^{n_x}\Big(\ell({\widehat{\boldsymbol{s}}},\bar{\boldsymbol{x}}_{0,i},\boldsymbol{W}_i) - \ell^{\mathrm{trunc}}({\widehat{\boldsymbol{s}}},\bar{\boldsymbol{x}}_{0,i},\boldsymbol{W}_i)
\Big)\\
&= \frac{1}{n_x}\sum_{i=1}^{n_x}\mathbb{E}_{{Z}}\left(\frac{1}{T-2T_1}\int_{T_1}^{T-T_1}\left\Vert\widehat{\boldsymbol{s}}(t,\bar{\boldsymbol{x}}_{0,i}\mu_t + {Z}\sigma_t ,\boldsymbol{W}_i)+ \frac{{Z}}{\sigma_t}\right\Vert^2 dt\cdot\mathbbm{1}_{\{\Vert{Z}\Vert_{\infty} > M\}}\right)\\
&\lesssim \Big(\log N\mathbb{P}(\Vert{Z}\Vert_{\infty} > M) + \mathbb{E}\left(
\Vert{Z}\Vert^2\mathbbm{1}_{\{\Vert{Z}\Vert_{\infty}>M\}}\right)\Big) \cdot \frac{1}{T-2T_1}\int_{T_1}^{T-T_1}\frac{1}{\sigma^2_t} dt\\
&\lesssim \log N \mathbb{P}(\Vert{Z}\Vert_{\infty} > M) + [\mathbb{E}
(\Vert{Z}\Vert^4)]^{\frac{1}{2}}\mathbb{P}({\Vert Z \Vert_{\infty}>M})^{\frac{1}{2}} \\
&\lesssim \log N \mathbb{P}({\Vert{Z}\Vert_{\infty}>M})^{\frac{1}{2}}\\
&\lesssim \log N\exp\left(-\frac{M^2}{4}\right).
\end{aligned}
\end{equation*}
Therefore, there exists a constant $C_3 > 0$ such that
$$
\mathbb{E}_{\bS_x,\bT,\bZ} \left(\frac{1}{n_x}\sum_{i=1}^{n_x}\Big(\ell({\widehat{\boldsymbol{s}}},\bar{\boldsymbol{x}}_{0,i},\boldsymbol{W}_i) - \ell^{\mathrm{trunc}}({\widehat{\boldsymbol{s}}},\bar{\boldsymbol{x}}_{0,i},\boldsymbol{W}_i)\Big)\right) \leq C_{3} \log N\exp\left(-\frac{M^2}{4}\right).
$$
Let $h({\boldsymbol{s}},t,\bar{\boldsymbol{x}}_{0,i},{Z}):= \Vert \boldsymbol{s}(t,\bar{\boldsymbol{x}}_{0,i} \mu_t+ {Z}\sigma_t ,\boldsymbol{W}_i)+ \frac{{Z}}{\sigma_t}\Vert^2\mathbbm{1}_{\{\Vert{Z}\Vert_{\infty}\leq M\}}$, then there exists a constant $C_{4} > 0$ such that
$$
0 \leq h({\boldsymbol{s}},t,\bar{\boldsymbol{x}}_{0,i},{Z})\lesssim \frac{M^2 + \log N}{\sigma^2_t} \leq C_{4} {N}^{C_\mu+4C_\lambda}(M^2 + \log N),
$$
where we use $\sigma^2_t=1-e^{-4t}=\frac{e^{4t}-1}{e^{4t}}\geq \frac{e^{4T_1}-1}{e^{4T}}$.
For any $\delta_1 > 0$, $\boldsymbol{s}\in \cG$, there exists a $\boldsymbol{s}_{\delta_1}\in \mathcal{G}_{\delta_1}$ and a constant $C_{5} > 0$ such that
\begin{equation*}
\begin{aligned}
\big|h({\boldsymbol{s}},t,\bar{\boldsymbol{x}}_{0,i},{Z})- h({\boldsymbol{s}}_{\delta_1},t,\bar{\boldsymbol{x}}_{0,i},{Z})\big|
&\lesssim \delta_1 \left\Vert \boldsymbol{s} + \boldsymbol{s}_{\delta_1} + \frac{2{Z}}{\sigma_t}\right\Vert\mathbbm{1}_{\{\Vert{Z}\Vert_{\infty}\leq M\}}\\
& \lesssim \frac{M + \sqrt{\log N}}{\sigma_t} \cdot \delta_1 \\
&\leq C_{5} {N}^{\frac{C_\mu}{2}+2C_\lambda}(M + \sqrt{\log N})\delta_1.
\end{aligned}
\end{equation*}
Then, for fixed $\bar{\boldsymbol{x}}_{0,i}$ and $\boldsymbol{W}_i$, we have
\begin{equation*}
\begin{aligned}
&\frac{1}{T-2T_1}\int_{T_1}^{T-T_1}\mathbb{E}_{{Z}}h(\widehat{\boldsymbol{s}},t,\bar{\boldsymbol{x}}_{0,i},{Z}) dt - \frac{1}{m_x}\sum_{j=1}^{m_x}h(\widehat{\boldsymbol{s}},t_j,\bar{\boldsymbol{x}}_{0,i},{Z}_j)\\
&\leq\mathop{\mathrm{sup}}_{\boldsymbol{s}\in\mathcal{G}}\left(\frac{1}{T-2T_1}\int_{T_1}^{T-T_1}\mathbb{E}_{{Z}}h({\boldsymbol{s}},t,\bar{\boldsymbol{x}}_{0,i},{Z}) dt - \frac{1}{m_x}\sum_{j=1}^{m_x}h({\boldsymbol{s}},t_j,\bar{\boldsymbol{x}}_{0,i},{Z}_j)\right)\\
&\leq\mathop{\mathrm{max}}_{\boldsymbol{s}_{\delta}\in \mathcal{G}_{\delta_1}}\left(\frac{1}{T-2T_1}\int_{T_1}^{T-T_1}\mathbb{E}_{{Z}}h({\boldsymbol{s}}_{\delta_1},t,\bar{\boldsymbol{x}}_{0,i},{Z}) dt - \frac{1}{m_x}\sum_{j=1}^{m_x}h({\boldsymbol{s}}_{\delta_1},t_j,\bar{\boldsymbol{x}}_{0,i},{Z}_j)\right) \\
&~~~~ + 2C_{5}{N}^{\frac{C_\mu}{2}+2C_\lambda}(M + \sqrt{\log N})\delta_1.
\end{aligned}
\end{equation*}
Let $$f(M,N):=C_{4} {N}^{C_\mu+4C_\lambda}(M^2 + \log N)  ,~  g(M,N) := 2C_{5} {N}^{\frac{C_\mu}{2}+2C_\lambda}(M + \sqrt{\log N})\delta_1.$$ 
For $t > g(M,N)$,  using Hoeffding's inequality (See Lemma \ref{lem:bernstein-hoeffding-inequality}) implies
$$
\begin{aligned}
&\mathbb{P}_{\bT,\bZ}\left(\frac{1}{T-2T_1}\int_{T_1}^{T-T_1}\mathbb{E}_{{Z}}h(\widehat{\boldsymbol{s}},t,\bar{\boldsymbol{x}}_{0,i},{Z}) dt - \frac{1}{m_x}\sum_{j=1}^{m_x}h(\widehat{\boldsymbol{s}},t_j,\bar{\boldsymbol{x}}_{0,i},{Z}_j)> t
\right)\\
&\leq\mathbb{P}_{\bT,\bZ}\left(\mathop{\mathrm{max}}_{\boldsymbol{s}_{\delta_1}\in \mathcal{G}_{\delta_1}}\left( \frac{1}{T-2T_1}\int_{T_1}^{T-T_1}\mathbb{E}_{{Z}}h({\boldsymbol{s}}_{\delta_1},t,\bar{\boldsymbol{x}}_{0,i},{Z}) dt - \frac{1}{m_x}\sum_{j=1}^{m_x}h({\boldsymbol{s}}_{\delta_1},t_j,\bar{\boldsymbol{x}}_{0,i},{Z}_j) \right) > t - g(M,N)
\right)\\
&\leq\mathcal{N}_{\delta_1}\mathop{\mathrm{max}}_{\boldsymbol{s}_{\delta_1}\in \mathcal{G}_{\delta_1}}\mathbb{P}_{\bT,\bZ}\left(\frac{1}{T-2T_1}\int_{T_1}^{T-T_1}\mathbb{E}_{{Z}}h({\boldsymbol{s}}_{\delta_1},t,\bar{\boldsymbol{x}}_{0,i},{Z}) dt - \frac{1}{m_x}\sum_{j=1}^{m_x}h({\boldsymbol{s}}_{\delta_1},t_j,\bar{\boldsymbol{x}}_{0,i},{Z}_j) > t - g(M,N)
\right)\\
&\leq \mathcal{N}_{\delta_1}\exp{\left(-\frac{2{m_x} \Big(t-g(M,N)\Big)^2}{f^2(M,N)}\right)}.
\end{aligned}
$$
Thus, by taking the expectation with respect to $\bT$ and $\bZ$, we deduce that for any $\epsilon > 0$, Term II satisfies the following bound
$$
\begin{aligned}
&\mathbb{E}_{\bT,\bZ}\left(\ell^{\mathrm{trunc}}({\widehat{\boldsymbol{s}}},\bar{\boldsymbol{x}}_{0,i},\boldsymbol{W}_i)- \widehat{\ell}^{\mathrm{trunc}}({\widehat{\boldsymbol{s}}},\bar{\boldsymbol{x}}_{0,i},\boldsymbol{W}_i)\right) \\ &=\int_{0}^{+\infty}\mathbb{P}_{\bT,\bZ}\left(
\ell^{\mathrm{trunc}}({\widehat{\boldsymbol{s}}},\bar{\boldsymbol{x}}_{0,i},\boldsymbol{W}_i)- \widehat{\ell}^{\mathrm{trunc}}({\widehat{\boldsymbol{s}}},\bar{\boldsymbol{x}}_{0,i},\boldsymbol{W}_i)> t
\right) dt\\
&\leq g(M,N) + \epsilon+ \mathcal{N}_{\delta_1}\int_{\epsilon}^{+\infty}\exp{\left(-\frac{2{m_x}t^2}{f^2(M,N)}\right)} dt\\
&\leq g(M,N) + \epsilon+ \frac{\sqrt{\pi}}{2}\mathcal{N}_{\delta_1}\exp{\left(-\frac{2{m_x}\epsilon^2}{f^2(M,N)}\right)}\frac{f(M,N)}{\sqrt{2{m_x}}}.
\end{aligned}
$$
Thus, we have
$$
\begin{aligned}
& \mathbb{E}_{\bS_x,\bT,\bZ}\left(\frac{1}{n_x}\sum_{i=1}^{n_x}\left(\ell^{\mathrm{trunc}}({\widehat{\boldsymbol{s}}},\bar{\boldsymbol{x}}_{0,i},\boldsymbol{W}_i)- \widehat{\ell}^{\mathrm{trunc}}({\widehat{\boldsymbol{s}}},\bar{\boldsymbol{x}}_{0,i},\boldsymbol{W}_i)\right)\right) \\
&= \frac{1}{n_x}\sum_{i=1}^{n_x}\mathbb{E}_{\bS_x}\left[\mathbb{E}_{\bT,\bZ}\left(\ell^{\mathrm{trunc}}({\widehat{\boldsymbol{s}}},\bar{\boldsymbol{x}}_{0,i},\boldsymbol{W}_i)- \widehat{\ell}^{\mathrm{trunc}}({\widehat{\boldsymbol{s}}},\bar{\boldsymbol{x}}_{0,i},\boldsymbol{W}_i)\right)\right]\\
&\leq g(M,N) + \epsilon+ \frac{\sqrt{\pi}}{2}\mathcal{N}_{\delta_1}\exp{\left(-\frac{2{m_x}\epsilon^2}{f^2(M,N)}\right)}\frac{f(M,N)}{\sqrt{2{m_x}}}.
\end{aligned}
$$

The last term can be expressed as
$$
\text{Term III} = -\frac{1}{{m_x}n_x}\sum_{i=1}^{n_x}\sum_{j=1}^{m_x}\left\Vert\widehat{\boldsymbol{s}}(t_j,\bar{\boldsymbol{x}}_{0,i} \mu_{t_j} + {Z}_j\sigma_{t_j} ,\boldsymbol{W}_i)+ \frac{{Z}_j}{\sigma_{t_j}}\right\Vert^2\mathbbm{1}_{\{\Vert{Z}_j\Vert_{\infty} > M\}}\leq 0,
$$
which implies
$$
\mathbb{E}_{\bS_x,\bT,\bZ}\left(
\widehat{\ell}^{\mathrm{trunc}}({\widehat{\boldsymbol{s}}},\bar{\boldsymbol{x}}_{0,i},\boldsymbol{W}_i) - \widehat{\ell}({\widehat{\boldsymbol{s}}},\bar{\boldsymbol{x}}_{0,i},\boldsymbol{W}_i)
\right)\leq 0.
$$
Combining the above inequalities, by setting $\epsilon = f(M,N)\sqrt{\frac{\log{\mathcal{N}_{\delta_1}}}{2{m_x}}}$, $M = 2\sqrt{\log{m_x}}$, and $\delta_1 = \frac{1}{m_x}$, we have
\begin{equation}\label{eq:score-statsitical-error-2}
\begin{aligned}
& \mathbb{E}_{\bS_x,\bT,\bZ}\left[\frac{1}{n_x}\sum_{i=1}^{n_x}\left( \ell({\widehat{\boldsymbol{s}}},\bar{\boldsymbol{x}}_{0,i},\boldsymbol{W}_i) - \widehat{\ell}({\widehat{\boldsymbol{s}}},\bar{\boldsymbol{x}}_{0,i},\boldsymbol{W}_i) \right)\right]\\
& \leq  C_{3}N^{C_\mu+4C_\lambda}\log N\exp\left(-\frac{M^2}{4}\right) + 2C_{5}{N}^{\frac{C_\mu}{2}+2C_\lambda}(M + \sqrt{\log N})\delta_1\\
&~~~~ + \frac{\sqrt{\pi}}{2}\mathcal{N}_{\delta_1}\exp{\left(-\frac{2{m_x}\epsilon^2}{f^2(M,N)}\right)}\frac{f(M,N)}{\sqrt{2{m_x}}}+\epsilon\\
& \lesssim  \frac{N^{C_\mu+4C_\lambda}\log N + {N}^{\frac{C_\mu}{2}+2C_\lambda}(\sqrt{\log {m_x}} + \sqrt{\log N})}{m_x} \\
&~~~~ + {N}^{C_\mu+4C_\lambda}(\log {m_x} + \log N) \cdot \frac{\sqrt{\log\mathcal{N}_{1/{m_x}}} + 1}{\sqrt{2{m_x}}}\\
& \lesssim  {N}^{C_\mu+4C_\lambda}(\log {m_x} + \log N) \cdot \frac{{N}^{\frac{d_x+d_w}{2}}\log^{\frac{13}{2}}N(\log^2 N + \sqrt{\log {m_x}})}{\sqrt{m_x}}\\
& \lesssim \frac{ {N}^{C_\mu+4C_\lambda+\frac{d_x+d_w}{2}} \log^{\frac{13}{2}}N(\log^3 N + \log^2 N \log {m_x} +\log ^{\frac{3}{2}} {m_x} )}{\sqrt{m_x}}.
\end{aligned}
\end{equation}
The proof is complete.
\end{proof}

\subsection{Error Bound for the Score Function}
Using Lemma \ref{lem:score-approximation} and Lemma \ref{lem:score-statistical-error}, we    prove Theorem \ref{thm:score-error}.

\begin{proof}[Proof of Theorem \ref{thm:score-error}]
With the choice $N = \lfloor n_x^{\frac{1}{d_x+d_w + 2\alpha}} \rfloor + 1 \lesssim n_x^{\frac{1}{d_x+d_w + 2\alpha}}$ from Lemma \ref{lem:score-approximation}, one can bound the approximation error as
$$
\begin{aligned}
\inf_{\boldsymbol{s} \in \mathcal{G}} \Big[ \mathcal{L}(\boldsymbol{s}) - \mathcal{L}(\boldsymbol{s}^*) \Big]
& \lesssim \frac{1}{T-2T_1}\int_{T_1}^{T-T_1} \frac{{N}^{-2\alpha}\log N}{\sigma^2_t}  d t \\
&\lesssim {N}^{-2\alpha} \log N\\
& \lesssim n_x^{-\frac{2\alpha}{d_x+d_w + 2\alpha}} \log n_x.
\end{aligned}
$$
Substituting the choices of $N$, ${m_x}=n_x^{\frac{(d_x+d_w + 12\alpha)(\alpha \wedge 1)+4\alpha}{(d_x+d_w + 2\alpha)(\alpha \wedge 1)}}$, $C_{\mu}=\frac{2\alpha}{\alpha \wedge 1}$ and $C_{\lambda}=\alpha$ into \eqref{eq:score-statsitical-error-1}–\eqref{eq:score-statsitical-error-2},
we derive the final bound on the statistical error
\begin{align*}
&\mathbb{E}_{\bS_x,\bT,\bZ} \Big[ \mathcal{L}(\widehat{\boldsymbol{s}}) - 2{\mathcal{L}}_{\bS_x}(\widehat{\boldsymbol{s}}) + \mathcal{L}(\boldsymbol{s}^*) \Big] + 2\mathbb{E}_{\bS_x,\bT,\bZ} \Big[ {\mathcal{L}}_{\bS_x}(\widehat{\boldsymbol{s}}) - \widehat{\mathcal{L}}_{\bS_x,\bT,\bZ}(\widehat{\boldsymbol{s}}) \Big] \\
&\lesssim n_x^{-\frac{2\alpha}{d_x+d_w + 2\alpha}}\log^{18}n_x + n_x^{-\frac{2\alpha}{d_x+d_w + 2\alpha}}\log^{\frac{19}{2}}n_x.
\end{align*}
Thus, we finally have
\begin{align*}
&\mathbb{E}_{\bS_x,\bT,\bZ} \bigg[ \frac{1}{T - 2T_1} \int_{T_1}^{T-T_1} \mathbb{E}_{\boldsymbol{x}_t, \boldsymbol{W}} \| \widehat{\boldsymbol{s}}(T -t, \boldsymbol{x}_t, \boldsymbol{W}) - \nabla \log p_{T -t}(\boldsymbol{x}_t|\boldsymbol{W}) \|^2 dt \bigg]\\
&=\mathbb{E}_{\bS_x,\bT,\bZ} \bigg[ \frac{1}{T - 2T_1} \int_{T_1}^{T-T_1} \mathbb{E}_{{\bar{\boldsymbol{x}}}_t, \boldsymbol{W}} \| \widehat{\boldsymbol{s}}(t, {\bar{\boldsymbol{x}}}_t, \boldsymbol{W}) - \nabla \log p_{t}({\bar{\boldsymbol{x}}}_t|\boldsymbol{W}) \|^2 dt \bigg]\\
&\lesssim n_x^{-\frac{2\alpha}{d_x+d_w + 2\alpha}}\log^{18}n_x + n_x^{-\frac{2\alpha}{d_x+d_w + 2\alpha}}\log^{\frac{19}{2}}n_x+ n_x^{-\frac{2\alpha}{d_x+d_w + 2\alpha}} \log n_x\\
&\lesssim  n_x^{-\frac{2\alpha}{d_x+d_w + 2\alpha}}\log^{18}n_x. 
\end{align*}
The proof is complete.
\end{proof}

\section{Bounds on Optimal Bellman Operator Learning}
\label{sec:apd-bounds-T*}

\subsection{Proof of Lemma \ref{lem:TV-time-discrete-error}}
With Lemma \ref{lem:Girsanov}, we can prove Lemma  \ref{lem:TV-time-discrete-error}.

\begin{proof}[Proof of Lemma \ref{lem:TV-time-discrete-error}]
Let $\mathbb P_w$ and $\widetilde{\mathbb P}_w$ denote the path laws of the true reverse process and the continuous time EM interpolation, respectively. The estimates below imply
\[
\mathbb E_{\mathbb P_w}\sum_{k=0}^{K-1}\int_{t_k}^{t_{k+1}}\|\widehat b(T-t_k,X_{t_k},w)-b(T-t,X_t,w)\|^2 \mathrm dt<\infty.
\]
For $m\geq1$, define
\[
\begin{aligned}
\tau_m:=\inf\Bigg\{t\in[T_1,T-T_1]:\;
&\sum_{k=0}^{K-1}\int_{t_k}^{t_{k+1}}\mathbbm 1_\{u\leq t\}\\
& \times\|\widehat b(T-t_k,X_{t_k},w)-b(T-u,X_u,w)\|^2 \mathrm du\geq m
\Bigg\}\wedge(T-T_1).
\end{aligned}
\]
Since the common diffusion coefficient is $2I_{d_x}$, the drift difference stopped at $\tau_m$ satisfies the exponential moment condition in Lemma~\ref{lem:Girsanov}. Applying Lemma~\ref{lem:Girsanov} up to $\tau_m$ and then letting $m\to\infty$, the lower semicontinuity of relative entropy and monotone convergence give
\[
\mathrm{KL}(\mathbb P_w\|\widetilde{\mathbb P}_w)\leq\frac18\mathbb E_{\mathbb P_w}\sum_{k=0}^{K-1}\int_{t_k}^{t_{k+1}}\|\widehat b(T-t_k,X_{t_k},w)-b(T-t,X_t,w)\|^2 \mathrm dt.
\]
Consequently, the data processing inequality and Pinsker's inequality yield
\[
\mathrm{TV}^2(\widetilde p_{T_1},p_{T_1})\leq\frac1{16}\mathbb E_{\mathbb P_w}\sum_{k=0}^{K-1}\int_{t_k}^{t_{k+1}}\|\widehat b(T-t_k,X_{t_k},w)-b(T-t,X_t,w)\|^2 \mathrm dt.
\]
For $t\in[T_1,T-T_1]$, the Gaussian convolution score identity and Jensen's inequality give
\[
\nabla\log p_{T-t}(\boldsymbol{x}_t|\boldsymbol{w})
=-\frac{1}{\sigma_{T-t}}\mathbb{E}\left[\boldsymbol{Z}|\boldsymbol{x}_t,\boldsymbol{w}\right],
~~
\mathbb{E}_{\mathbb{P}}\left\|\nabla\log p_{T-t}(\boldsymbol{x}_t|\boldsymbol{w})\right\|^2
\leq\frac{d_x}{\sigma_{T-t}^2}.
\]
The network class also gives $\left\|\widehat{\boldsymbol{s}}(T-t_k,\cdot,\boldsymbol{w})\right\|_\infty^2
\lesssim\frac{\log N}{\sigma_{T-t_k}^2}.$ Since
\[
\widehat{\boldsymbol b}(T-t_k,\boldsymbol{x}_{t_k},\boldsymbol w)
-\boldsymbol b(T-t,\boldsymbol{x}_t,\boldsymbol w)
=2(\boldsymbol{x}_{t_k}-\boldsymbol{x}_t)
+4\Big(\widehat{\boldsymbol s}(T-t_k,\boldsymbol{x}_{t_k},\boldsymbol w)
-\nabla\log p_{T-t}(\boldsymbol{x}_t|\boldsymbol w)\Big),
\]
we have
\[
\mathbb{E}_{\mathbb{P}}\left\|\widehat{\boldsymbol b}(T-t_k,\boldsymbol{x}_{t_k},\boldsymbol w)-\boldsymbol b(T-t,\boldsymbol{x}_t,\boldsymbol w)\right\|^2
\lesssim\mathbb{E}_{\mathbb{P}}\|\boldsymbol{x}_{t_k}-\boldsymbol{x}_t\|^2
+\frac{\log N}{\sigma_{T-t_k}^2}
+\frac{d_x}{\sigma_{T-t}^2}.
\]
Here $T-t,T-t_k\geq T_1$, so $\sigma_{T-t},\sigma_{T-t_k}\geq\sigma_{T_1}>0$. The OU representation and the bounded support of the initial distribution give $\sup\limits_{t\in[T_1,T-T_1]}\mathbb{E}_{\mathbb{P}}\|\boldsymbol{x}_t\|^2<\infty.$
It follows that
\[
\mathbb{E}_{\mathbb{P}}\left[\sum_{k=0}^{K-1}\int_{t_k}^{t_{k+1}}
\left\|(2I_{d_x})^{-1}
\left(\widehat{\boldsymbol b}(T-t_k,\boldsymbol{x}_{t_k},\boldsymbol w)
-\boldsymbol b(T-t,\boldsymbol{x}_t,\boldsymbol w)\right)\right\|^2dt\right]<\infty,
\]
The preceding bounds imply
\[
\sup_{0\leq k\leq K-1}\sup_{t\in[t_k,t_{k+1}]}\mathbb E_{\mathbb P}\left\|(2I_{d_x})^{-1}\left(\widehat{\boldsymbol b}(T-t_k,\boldsymbol{x}_{t_k},\boldsymbol w)-\boldsymbol b(T-t,\boldsymbol{x}_t,\boldsymbol w)\right)\right\|^2<\infty.
\]
Moreover, the true reverse process and the continuous time EM interpolation are both initialized from $p_{T-T_1}(\cdot | \boldsymbol w)$ and have the same diffusion coefficient $2I_{d_x}$. Therefore, by Pinsker's inequality, the TV distance between the true terminal distribution $p_{T_1}$ and the discretized terminal distribution $\widetilde p_{T_1}$ satisfies
\[
\mathrm{TV}^2(\widetilde p_{T_1},p_{T_1})\lesssim \mathrm{KL}(p_{T_1}\|\widetilde p_{T_1})\lesssim \mathbb{E}_{\mathbb P}\left[\sum_{k=0}^{K-1}\int_{t_k}^{t_{k+1}}
\left\|\widehat{\boldsymbol b}(T-t_k,\boldsymbol{x}_{t_k},\boldsymbol w)
-\boldsymbol b(T-t,\boldsymbol{x}_t,\boldsymbol w)\right\|^2dt\right].
\]
Hence,
\[
\begin{aligned}
&\mathbb{E}_{\bS,\bT,\bZ,\boldsymbol{W}} \left[\mathrm{TV}^2(\widetilde p_{T_1},p_{T_1})\right]\\
&\lesssim \mathbb{E}_{\bS,\bT,\bZ,\boldsymbol{W}} \left[\sum_{k=0}^{K-1}\int_{t_k}^{t_{k+1}} \mathbb{E}_{\mathbb{P}} \bigl\|\widehat{\boldsymbol{s}}(T-t_k,{\boldsymbol{x}}_{t_k},\boldsymbol{W}) -\nabla\log p_{T-t}({\boldsymbol{x}}_t|\boldsymbol{W})\bigr\|^2dt\right]\\
&\quad+ \mathbb{E}_{\bS,\bT,\bZ,\boldsymbol{W}} \left[\sum_{k=0}^{K-1}\int_{t_k}^{t_{k+1}} \mathbb{E}_{\mathbb{P}} \bigl\|\boldsymbol{x}_{t_k}-\boldsymbol{x}_t\bigr\|^2dt\right].
\end{aligned}
\]
First, we have 
\begin{align*}
\mathbb{E}_{\mathbb{P}}\Vert\boldsymbol{x}_{t_k}-\boldsymbol{x}_t\Vert^2
&=\mathbb{E}_{\mathbb{P}}\Vert\bar{\boldsymbol{x}}_{T-t_k}-\bar{\boldsymbol{x}}_{T-t}\Vert^2\\
&=\mathbb{E}_{\mathbb{P}}\left\Vert\left(e^{-2(t-t_k)}-1\right)\bar{\boldsymbol{x}}_{T-t}
+2\int_{T-t}^{T-t_k}e^{-2(T-t_k-u)} d\boldsymbol{B}_u\right\Vert^2\\
&=\left(1-e^{-2(t-t_k)}\right)^2\mathbb{E}_{\mathbb{P}}\Vert\bar{\boldsymbol{x}}_{T-t}\Vert^2
+4d_x\int_{T-t}^{T-t_k}e^{-4(T-t_k-u)} du\\
&=\left(1-e^{-2(t-t_k)}\right)^2\mathbb{E}_{\mathbb{P}}\Vert\bar{\boldsymbol{x}}_{T-t}\Vert^2
+d_x\left(1-e^{-4(t-t_k)}\right)\\
&\lesssim t-t_k.
\end{align*}
The same representation and the Burkholder--Davis--Gundy inequality give, for
$q\in\{2,4,6\}$,
\begin{equation}\label{eq:score-term-2.1}
\mathbb{E}_{\mathbb{P}} \bigl\|\boldsymbol{x}_{t_k}-\boldsymbol{x}_t\bigr\|^q \lesssim (t-t_k)^{\frac{q}{2}}. 
\end{equation}
Thus, with
$\Delta=\max\limits_{0\leq k\leq K-1}(t_{k+1}-t_k)$,
\begin{equation}\label{eq:score-term-int.3}
\mathbb{E}_{\bS,\bT,\bZ,\boldsymbol{W}} \left[\sum_{k=0}^{K-1}\int_{t_k}^{t_{k+1}} \mathbb{E}_{\mathbb{P}} \bigl\|\boldsymbol{x}_{t_k}-\boldsymbol{x}_t\bigr\|^2dt\right]
\lesssim \sum_{k=0}^{K-1}(t_{k+1}-t_k)^2 \lesssim T\Delta \lesssim \Delta\log N.
\end{equation}
Moreover,
\[
\begin{aligned}
&\mathbb{E}_{\bS,\bT,\bZ} \bE_{\boldsymbol{W}} \left[\sum_{k=0}^{K-1}\int_{t_k}^{t_{k+1}} \mathbb{E}_{\mathbb{P}} \bigl\|\widehat{\boldsymbol{s}}(T-t_k,{\boldsymbol{x}}_{t_k},\boldsymbol{W}) -\nabla\log p_{T-t}({\boldsymbol{x}}_t|\boldsymbol{W})\bigr\|^2dt\right]\\
&\lesssim \underbrace{\mathbb{E}_{\bS,\bT,\bZ} \bE_{\boldsymbol{W}} \left[\sum_{k=0}^{K-1}\int_{t_k}^{t_{k+1}} \mathbb{E}_{\mathbb{P}} \bigl\|\widehat{\boldsymbol{s}}(T-t_k,{\boldsymbol{x}}_{t_k},\boldsymbol{W}) -\nabla\log p_{T-t_k}({\boldsymbol{x}}_{t_k}|\boldsymbol{W})\bigr\|^2dt\right]}_{\mathrm{I}}\\
&~~ + \underbrace{\mathbb{E}_{\bS,\bT,\bZ} \bE_{\boldsymbol{W}}  \left[\sum_{k=0}^{K-1}\int_{t_k}^{t_{k+1}} \mathbb{E}_{\mathbb{P}} \bigl\|\nabla\log p_{T-t_k}({\boldsymbol{x}}_{t_k}|\boldsymbol{W}) -\nabla\log p_{T-t_k}({\boldsymbol{x}}_t|\boldsymbol{W})\bigr\|^2dt\right]}_{\mathrm{II}}\\
&~~ + \underbrace{\mathbb{E}_{\bS,\bT,\bZ} \bE_{\boldsymbol{W}}  \left[\sum_{k=0}^{K-1}\int_{t_k}^{t_{k+1}} \mathbb{E}_{\mathbb{P}} \bigl\|\nabla\log p_{T-t_k}({\boldsymbol{x}}_t|\boldsymbol{W}) -\nabla\log p_{T-t}({\boldsymbol{x}}_t|\boldsymbol{W})\bigr\|^2dt\right]}_{\mathrm{III}}.
\end{aligned}
\]

\paragraph{Bound I.}
By \eqref{eq:discrete-generalization}, for any $\epsilon>0$, there exists
$\Delta_\epsilon>0$ such that, if
$\max\limits_{0\leq k\leq K-1}(t_{k+1}-t_k)\leq\Delta_\epsilon$, then
\begin{equation}\label{eq:score-term-int.0}
\text{Term I}\lesssim n_x^{-\frac{2\alpha}{d_x+d_w + 2\alpha}}\log^{19}n_x+\epsilon.
\end{equation}

Let $\mathcal E$ denote the score error integrated over time in \eqref{eq:discrete-generalization}, and let $\mathcal V_\Delta$ denote the integrated mean square variation of the same score residual between $t_k$ and $t$. All expectations are taken over the training randomness, $\boldsymbol W$, and the same OU path. By the Lipschitz bound in the preceding remark and \eqref{eq:score-term-2.1},
\[
\mathbb E_{\mathbb P}\left\|\widehat{\boldsymbol s}(T-t_k,\boldsymbol x_{t_k},\boldsymbol W)-\widehat{\boldsymbol s}(T-t,\boldsymbol x_t,\boldsymbol W)\right\|^2\lesssim\mathcal L_N^2(t-t_k).
\]
Combining this estimate with the bounds for the true score in \eqref{eq:score-term-int.1} and \eqref{eq:score-term-int.2} gives
\[
\mathcal V_\Delta\lesssim\log N\left[\left(\mathcal L_N^2+N^{8C_\lambda+2C_\mu}\right)\Delta+N^{12C_\lambda+3C_\mu}\Delta^2+N^{16C_\lambda+4C_\mu}\Delta^3+N^{24C_\lambda+6C_\mu}\Delta^8\right].
\]
For $\Delta\le\Delta_{n_x}$, the terms involving $\Delta^2$, $\Delta^3$, and $\Delta^8$ are bounded by the first term. Hence,
\[
\mathcal V_\Delta\lesssim N^{-2\alpha}(\log N)^{-20}.
\]
Moreover, Theorem~\ref{thm:score-error} gives
\[
\mathcal E\lesssim N^{-2\alpha}\log^{19}N.
\]
Expanding the squared norm and applying the Cauchy--Schwarz inequality yield
\[
\mathrm{Term~I}-\mathcal E\le2\sqrt{\mathcal E\mathcal V_\Delta}+\mathcal V_\Delta\lesssim N^{-2\alpha}(\log N)^{-\frac12}+N^{-2\alpha}(\log N)^{-20}.
\]
For all sufficiently large $N$, the last expression is bounded by
\[
N^{-2\alpha}\le n_x^{-\frac{2\alpha}{d_x+d_w+2\alpha}}=\epsilon.
\]
Therefore, the explicit value of $\Delta_{n_x}$ given in the preceding remark provides the deterministic threshold required in \eqref{eq:discrete-generalization}.

\paragraph{Bound II.}
Under $\mathbb P$, $\boldsymbol{x}_t\sim p_{T-t}(\cdot|\boldsymbol w)$. By the Gaussian convolution representation,
\[
\boldsymbol{x}_t\overset{d}{=}\mu_{T-t}\bar{\boldsymbol{x}}_0+\sigma_{T-t}\boldsymbol Z,\qquad \bar{\boldsymbol{x}}_0\sim p(\cdot|\boldsymbol w),\qquad \boldsymbol Z\sim N(0,I_{d_x}),
\]
where $\bar{\boldsymbol{x}}_0$ and $\boldsymbol Z$ are independent and $\bar{\boldsymbol{x}}_0\in[0,1]^{d_x}$. Hence, for each $1\leq i\leq d_x$,
\[
(-x_{t,i})_+\leq \sigma_{T-t}(-Z_i)_+,\qquad (x_{t,i}-\mu_{T-t})_+\leq \sigma_{T-t}(Z_i)_+.
\]
It follows that $D_{T-t}(\boldsymbol{x}_t)\leq \sigma_{T-t}\|\boldsymbol Z\|_\infty,$
and therefore, for every fixed positive integer $q$,
\begin{equation}\label{eq:score-term-2.2}
\mathbb E_{\mathbb P}\left[D_{T-t}(\boldsymbol{x}_t)^q\right]\leq \sigma_{T-t}^q\mathbb E\|\boldsymbol Z\|_\infty^q\lesssim \sigma_{T-t}^q.
\end{equation}
For $\theta\in[0,1]$, let $\boldsymbol{x}_\theta=(1-\theta)\boldsymbol{x}_t+\theta\boldsymbol{x}_{t_k}$. Since $D_{T-t_k}$ is $1$-Lipschitz,
\[
D_{T-t_k}(\boldsymbol{x}_\theta)\leq D_{T-t_k}(\boldsymbol{x}_{t_k})+\|\boldsymbol{x}_\theta-\boldsymbol{x}_{t_k}\|\leq D_{T-t_k}(\boldsymbol{x}_{t_k})+\|\boldsymbol{x}_{t_k}-\boldsymbol{x}_t\|.
\]
By the fundamental theorem of calculus, Jensen's inequality, and Lemma~\ref{lem:derivatives-bound},
\[
\begin{aligned}
&\left\|\nabla\log p_{T-t_k}(\boldsymbol{x}_{t_k}|\boldsymbol w)-\nabla\log p_{T-t_k}(\boldsymbol{x}_t|\boldsymbol w)\right\|^2\\
&=\left\|\int_0^1\partial_{\boldsymbol{x}}\nabla\log p_{T-t_k}(\boldsymbol{x}_\theta|\boldsymbol w)(\boldsymbol{x}_{t_k}-\boldsymbol{x}_t) d\theta\right\|^2\\
&\leq \int_0^1\left\|\partial_{\boldsymbol{x}}\nabla\log p_{T-t_k}(\boldsymbol{x}_\theta|\boldsymbol w)\right\|^2d\theta \|\boldsymbol{x}_{t_k}-\boldsymbol{x}_t\|^2\\
&\lesssim \frac{1}{\sigma_{T-t_k}^4}\int_0^1\left(\frac{D_{T-t_k}(\boldsymbol{x}_\theta)^2}{\sigma_{T-t_k}^2}\vee1\right)^2d\theta \|\boldsymbol{x}_{t_k}-\boldsymbol{x}_t\|^2\\
&\lesssim \frac{1}{\sigma_{T-t_k}^4}\left(1+\frac{D_{T-t_k}(\boldsymbol{x}_{t_k})^4}{\sigma_{T-t_k}^4}+\frac{\|\boldsymbol{x}_{t_k}-\boldsymbol{x}_t\|^4}{\sigma_{T-t_k}^4}\right)\|\boldsymbol{x}_{t_k}-\boldsymbol{x}_t\|^2.
\end{aligned}
\]
By \eqref{eq:score-term-2.1}, \eqref{eq:score-term-2.2} and the Cauchy--Schwarz inequality,
\[
\mathbb E_{\mathbb P}\Big[D_{T-t_k}(\boldsymbol{x}_{t_k})^4\|\boldsymbol{x}_{t_k}-\boldsymbol{x}_t\|^2\Big]\leq \left(\mathbb E_{\mathbb P}D_{T-t_k}(\boldsymbol{x}_{t_k})^8\right)^{\frac{1}{2}}\left(\mathbb E_{\mathbb P}\|\boldsymbol{x}_{t_k}-\boldsymbol{x}_t\|^4\right)^{\frac{1}{2}}\lesssim \sigma_{T-t_k}^4(t-t_k).
\]
Taking expectations, we obtain
\[
\begin{aligned}
&\mathbb E_{\mathbb P}\left\|\nabla\log p_{T-t_k}(\boldsymbol{x}_{t_k}|\boldsymbol w)-\nabla\log p_{T-t_k}(\boldsymbol{x}_t|\boldsymbol w)\right\|^2\\
&\lesssim \frac{1}{\sigma_{T-t_k}^4}\left((t-t_k)+\frac{\left(\mathbb E_{\mathbb P}D_{T-t_k}(\boldsymbol{x}_{t_k})^8\right)^{\frac{1}{2}}\left(\mathbb E_{\mathbb P}\|\boldsymbol{x}_{t_k}-\boldsymbol{x}_t\|^4\right)^{\frac{1}{2}}}{\sigma_{T-t_k}^4}+\frac{\mathbb E_{\mathbb P}\|\boldsymbol{x}_{t_k}-\boldsymbol{x}_t\|^6}{\sigma_{T-t_k}^4}\right)\\
&\lesssim \frac{1}{\sigma_{T-t_k}^4}\left((t-t_k)+\frac{\sigma_{T-t_k}^4(t-t_k)}{\sigma_{T-t_k}^4}+\frac{(t-t_k)^3}{\sigma_{T-t_k}^4}\right)\\
&\lesssim \frac{t-t_k}{\sigma_{T-t_k}^4}+\frac{(t-t_k)^3}{\sigma_{T-t_k}^8}.
\end{aligned}
\]
Consequently,
\[
\begin{aligned}
\text{Term II}
&\lesssim \sum_{k=0}^{K-1}\int_{t_k}^{t_{k+1}}\left(\frac{t-t_k}{\sigma_{T-t_k}^4}+\frac{(t-t_k)^3}{\sigma_{T-t_k}^8}\right)dt\\
&=\frac12\sum_{k=0}^{K-1}\frac{(t_{k+1}-t_k)^2}{\sigma_{T-t_k}^4}+\frac14\sum_{k=0}^{K-1}\frac{(t_{k+1}-t_k)^4}{\sigma_{T-t_k}^8}\\
&\leq \frac{\Delta}{2}\sum_{k=0}^{K-1}\frac{t_{k+1}-t_k}{\sigma_{T-t_k}^4}+\frac{\Delta^3}{4}\sum_{k=0}^{K-1}\frac{t_{k+1}-t_k}{\sigma_{T-t_k}^8}.
\end{aligned}
\]
Then
\begin{equation}\label{eq:score-term-int.1}
\text{Term II}\lesssim T\Delta\sigma_{T_1}^{-4}+T\Delta^3\sigma_{T_1}^{-8}\lesssim \Delta N^{8C_\lambda+2C_\mu}\log N
+\Delta^3N^{16C_\lambda+4C_\mu}\log N.
\end{equation}

\paragraph{Bound III.}
For every $u\in[T-t,T-t_k]$, $\sigma_u\geq\sigma_{T-t}$. Moreover, the bounded support of the initial distribution and $|\partial_u\mu_u|\lesssim1$ imply
\[
D_u(\boldsymbol{x}_t)\lesssim D_{T-t}(\boldsymbol{x}_t)+|\mu_u-\mu_{T-t}|\lesssim D_{T-t}(\boldsymbol{x}_t)+(t-t_k).
\]
Hence, by the preceding moment bound,
\[
\mathbb E_{\mathbb P}\left[\left(\frac{D_u(\boldsymbol{x}_t)^2}{\sigma_u^2}\vee1\right)^3\right]\lesssim 1+\frac{\mathbb E_{\mathbb P}\big[D_{T-t}(\boldsymbol{x}_t)^6\big]+(t-t_k)^6}{\sigma_{T-t}^6}\lesssim1+\frac{(t-t_k)^6}{\sigma_{T-t}^6}.
\]
Since $|\partial_u\mu_u|\lesssim1$ and $|\partial_u\sigma_u|\lesssim\sigma_u^{-1}$, by Lemma~\ref{lem:derivatives-bound}, we obtain
\[
\mathbb E_{\mathbb P}\bigl\|\partial_u\nabla\log p_u(\boldsymbol{x}_t | \boldsymbol w)\bigr\|^2\lesssim\frac{1}{\sigma_{T-t}^6}+\frac{(t-t_k)^6}{\sigma_{T-t}^{12}}.
\]
By the Cauchy--Schwarz inequality,
\[
\begin{aligned}
&\mathbb E_{\mathbb P}\bigl\|\nabla\log p_{T-t_k}(\boldsymbol{x}_t | \boldsymbol w)-\nabla\log p_{T-t}(\boldsymbol{x}_t | \boldsymbol w)\bigr\|^2\\
&=\mathbb E_{\mathbb P}\left\|\int_{T-t}^{T-t_k}\partial_u\nabla\log p_u(\boldsymbol{x}_t | \boldsymbol w) du\right\|^2\\
&\leq(t-t_k)\int_{T-t}^{T-t_k}\mathbb E_{\mathbb P}\bigl\|\partial_u\nabla\log p_u(\boldsymbol{x}_t | \boldsymbol w)\bigr\|^2du\\
&\lesssim\frac{(t-t_k)^2}{\sigma_{T-t}^6}+\frac{(t-t_k)^8}{\sigma_{T-t}^{12}}.
\end{aligned}
\]
Thus, we conclude that
\begin{equation}\label{eq:score-term-int.2}
\begin{aligned}
 \text{Term III}
&\lesssim\sum_{k=0}^{K-1}\int_{t_k}^{t_{k+1}}\left(\frac{(t-t_k)^2}{\sigma_{T-t}^6}+\frac{(t-t_k)^8}{\sigma_{T-t}^{12}}\right)dt \\
&\lesssim\Delta^2N^{12C_\lambda+3C_\mu}\log N+\Delta^8N^{24C_\lambda+6C_\mu}\log N.
\end{aligned}
\end{equation}
Combining \eqref{eq:score-term-int.3}, \eqref{eq:score-term-int.0}, \eqref{eq:score-term-int.1}, and \eqref{eq:score-term-int.2}, we obtain
\begin{align*}
\mathbb E_{\bS,\bT,\bZ} \bE_{\boldsymbol W}\Big[\mathrm{TV}^2\Big(\widetilde p_{T_1}(\cdot|\boldsymbol{W}), p_{T_1}(\cdot|\boldsymbol{W})\Big)\Big]&\lesssim n_x^{-\frac{2\alpha}{d_x+d_w+2\alpha}}\log^{19}n_x+\epsilon+\Delta\log N\\
&~~ +\Delta N^{8C_\lambda+2C_\mu}\log N+\Delta^3N^{16C_\lambda+4C_\mu}\log N\\
&~~ +\Delta^2N^{12C_\lambda+3C_\mu}\log N+\Delta^8N^{24C_\lambda+6C_\mu}\log N.
\end{align*}
Set $\epsilon=n_x^{-\frac{2\alpha}{d_x+d_w+2\alpha}},\Delta_{n_x}=\Delta_\epsilon$
and 
\[
\Delta=\mathcal O\left(\min\left\{\Delta_{n_x},N^{-(2\alpha+8C_\lambda+2C_\mu)}\right\}\right)=\mathcal O\Big(\min\Big\{\Delta_{n_x},n_x^{-\frac{10\alpha(\alpha\wedge1)+4\alpha}{(d_x+d_w+2\alpha)(\alpha\wedge1)}}\Big\}\Big).
\]
Consequently,
\[
\mathbb E_{\bS,\bT,\bZ} \bE_{\boldsymbol W}\left[\mathrm{TV}^2\bigl(\widetilde p_{T_1}(\cdot | \boldsymbol W),p_{T_1}(\cdot | \boldsymbol W)\bigr)\right]\lesssim n_x^{-\frac{2\alpha}{d_x+d_w+2\alpha}}\log^{19}n_x.
\]
\end{proof}

\subsection{Proof of Lemma \ref{lem:TV-initial-distribution-error}}
\begin{proof}
	By the convexity of the KL divergence and Lemma \ref{lem:kl-guass}, we have
	\begin{align*}
		\mathrm{KL} \Big( p_{T-T_1}(\cdot|\boldsymbol{w}) \big\| \mathcal{N}(\boldsymbol{0}, \boldsymbol{I}_{d_x}) \Big)
		&= \mathrm{KL} \left( \int_{\mathbb{R}^{d_x}} p_{T-T_1}(\cdot|{\bar{\boldsymbol{x}}}_0, \boldsymbol{w}) p({\bar{\boldsymbol{x}}}_0|\boldsymbol{w}) d{\bar{\boldsymbol{x}}}_0 \big\| \mathcal{N}(\boldsymbol{0}, \boldsymbol{I}_{d_x}) \right)\\
		&\lesssim \int_{\mathbb{R}^{d_x}} \mathrm{KL} \Big( p_{T-T_1}(\cdot|{\bar{\boldsymbol{x}}}_0, \boldsymbol{w}) \big\| \mathcal{N}(\boldsymbol{0}, \boldsymbol{I}_{d_x}) \Big) p({\bar{\boldsymbol{x}}}_0|\boldsymbol{w}) d{\bar{\boldsymbol{x}}}_0\\
		&=\frac{1}{2} \left[ e^{-4(T-T_1)} \mathbb{E} \|{\bar{\boldsymbol{x}}}_0\|^2 - d_x \log(1 - e^{-4(T-T_1)}) - d_x e^{-4(T-T_1)} \right]\\
		&\lesssim e^{-4T}.
	\end{align*}
	By the Pinsker's inequality, we have 	$$\mathrm{TV}^2\Big( p_{T-T_1}(\cdot|\boldsymbol{w}) , \mathcal{N}(\boldsymbol{0}, \boldsymbol{I}_{d_x}) \Big) \leq \frac{1}{2}\mathrm{KL}\Big( p_{T-T_1}(\cdot|\boldsymbol{w}) \big\| \mathcal{N}(\boldsymbol{0}, \boldsymbol{I}_{d_x}) \Big).$$
	By \eqref{eq:data-processing-inequality}, we obtain
	\begin{align*}
		\mathbb{E}_{\bS,\bT,\bZ, \boldsymbol{W}}\Big[\TV^2 \Big(\check{p}_{T_1}(\cdot|\boldsymbol{W}),\widetilde{p}_{T_1}(\cdot|\boldsymbol{W})\Big)\Big] \lesssim e^{-4T}
		\lesssim n_x^{-\frac{2\alpha}{d_x+d_w+2\alpha}}.
	\end{align*}
	The proof is complete.
\end{proof}

\subsection{Proof of Lemma \ref{lem:TV-time-truncation-error}}
\begin{proof}
Since $p(\cdot|\boldsymbol w)$ is supported on $D=[0,1]^{d_x}$, then
\[
\begin{aligned}
\int_{\mathbb R^{d_x}}
|p_{T_1}(\boldsymbol x|\boldsymbol w)-p(\boldsymbol x|\boldsymbol w)|d\boldsymbol x
=
\underbrace{ \int_D|p_{T_1}(\boldsymbol x|\boldsymbol w)-p(\boldsymbol x|\boldsymbol w)|d\boldsymbol x}_{\text{A}}
+
\underbrace{ \int_{D^c}p_{T_1}(\boldsymbol x|\boldsymbol w)d\boldsymbol x}_{\text{B}} .
\end{aligned}
\]

\paragraph{Bound term A.}
For $\boldsymbol x\in D$, since
\[
\int_{\mathbb R^{d_x}}
\frac{\mu_{T_1}^{d_x}}
{\sigma_{T_1}^{d_x}(2\pi)^{d_x/2}}
\exp\left(
-\frac{\|\boldsymbol x-\mu_{T_1}\boldsymbol z\|^2}{2\sigma_{T_1}^2}
\right)d\boldsymbol z=1,
\]
we have
\[
\begin{aligned}
 p(\boldsymbol x|\boldsymbol w)-p_{T_1}(\boldsymbol x|\boldsymbol w)
&=
 \underbrace{ \int_D
\frac{\mu_{T_1}^{d_x} \big( p(\boldsymbol x|\boldsymbol w)-p(\boldsymbol z|\boldsymbol w)\big)}
{\sigma_{T_1}^{d_x}(2\pi)^{d_x/2}}
\exp\left(
-\frac{\|\boldsymbol x-\mu_{T_1}\boldsymbol z\|^2}{2\sigma_{T_1}^2}
\right)d\boldsymbol z}_{\text{I}}\\
&\quad+
 \underbrace{ (\mu_{T_1}^{d_x}-1)p_{T_1}(\boldsymbol x|\boldsymbol w)}_{\text{II}}\\
&\quad+
 \underbrace{ p(\boldsymbol x|\boldsymbol w)
\int_{D^c}
\frac{\mu_{T_1}^{d_x}}
{\sigma_{T_1}^{d_x}(2\pi)^{d_x/2}}
\exp\left(
-\frac{\|\boldsymbol x-\mu_{T_1}\boldsymbol z\|^2}{2\sigma_{T_1}^2}
\right)d\boldsymbol z}_{\text{III}}.
\end{aligned}
\]
For the Term I, let $\boldsymbol z=\frac{\boldsymbol x+\sigma_{T_1}\boldsymbol u}{\mu_{T_1}}$, by Assumption \ref{apt:diffusion-density-holder}, we have
\[
\begin{aligned}
&
\int_D \left| \int_D
\frac{\mu_{T_1}^{d_x} \big( p(\boldsymbol x|\boldsymbol w)-p(\boldsymbol z|\boldsymbol w)\big)}
{\sigma_{T_1}^{d_x}(2\pi)^{d_x/2}}
\exp\left(
-\frac{\|\boldsymbol x-\mu_{T_1}\boldsymbol z\|^2}{2\sigma_{T_1}^2}
\right)d\boldsymbol z \right| d\boldsymbol x
\\
&\le\int_D\int_D
\frac{\mu_{T_1}^{d_x}
|p(\boldsymbol x|\boldsymbol w)-p(\boldsymbol z|\boldsymbol w)|}
{\sigma_{T_1}^{d_x}(2\pi)^{d_x/2}}
\exp\left(
-\frac{\|\boldsymbol x-\mu_{T_1}\boldsymbol z\|^2}{2\sigma_{T_1}^2}
\right)d\boldsymbol z d\boldsymbol x\\
&\lesssim\int_D\int_D
\frac{\mu_{T_1}^{d_x}\|\boldsymbol x-\boldsymbol z\|^{\alpha\wedge1}}
{\sigma_{T_1}^{d_x}(2\pi)^{d_x/2}}
\exp\left(
-\frac{\|\boldsymbol x-\mu_{T_1}\boldsymbol z\|^2}{2\sigma_{T_1}^2}
\right)d\boldsymbol z d\boldsymbol x\\
&=\int_D
\int_{\frac{\mu_{T_1}D-\boldsymbol x}{\sigma_{T_1}}}
\left\|
\boldsymbol x-\frac{\boldsymbol x+\sigma_{T_1}\boldsymbol u}{\mu_{T_1}}
\right\|^{\alpha\wedge1}
(2\pi)^{-d_x/2}e^{-\|\boldsymbol u\|^2/2}
d\boldsymbol u d\boldsymbol x\\
&\le
\int_D\int_{\mathbb R^{d_x}}
\left\|
\boldsymbol x-
\frac{\boldsymbol x+\sigma_{T_1}\boldsymbol u}{\mu_{T_1}}
\right\|^{\alpha\wedge1}
(2\pi)^{-d_x/2}e^{-\|\boldsymbol u\|^2/2}
d\boldsymbol u d\boldsymbol x\\
&\lesssim
\int_D\int_{\mathbb R^{d_x}}
\left[
\left(\frac{\sigma_{T_1}}{\mu_{T_1}}\right)^{\alpha\wedge1}
\|\boldsymbol u\|^{\alpha\wedge1}
+
\left(\frac{1-\mu_{T_1}}{\mu_{T_1}}\right)^{\alpha\wedge1}
\|\boldsymbol x\|^{\alpha\wedge1}
\right]
(2\pi)^{-d_x/2}e^{-\|\boldsymbol u\|^2/2}
d\boldsymbol u d\boldsymbol x\\
&\lesssim
\left(\frac{\sigma_{T_1}}{\mu_{T_1}}\right)^{\alpha\wedge1}+\left(\frac{1-\mu_{T_1}}{\mu_{T_1}}\right)^{\alpha\wedge1}.
\end{aligned}
\]
For the Term II, since \(p_{T_1}(\cdot|\boldsymbol w)\) is a probability density, we have
\[
\begin{aligned}
\int_D\left|(\mu_{T_1}^{d_x}-1)p_{T_1}(\boldsymbol x|\boldsymbol w)
\right|d\boldsymbol x
&=|\mu_{T_1}^{d_x}-1|
\int_D p_{T_1}(\boldsymbol x|\boldsymbol w)d\boldsymbol x \\
&\le|\mu_{T_1}^{d_x}-1|
\int_{\mathbb R^{d_x}} p_{T_1}(\boldsymbol x|\boldsymbol w)d\boldsymbol x \\
&=|\mu_{T_1}^{d_x}-1| \\
&\lesssim 1-\mu_{T_1}.
\end{aligned}
\]
The last inequality follows from \(\mu_{T_1}\in(0,1]\) and
\[1-\mu_{T_1}^{d_x}=
(1-\mu_{T_1})(1+\mu_{T_1}+\cdots+\mu_{T_1}^{d_x-1})\le d_x(1-\mu_{T_1}).
\]
Moreover, by Assumption \ref{apt:diffusion-density-bound}, the boundary Term III satisfies
\[
\begin{aligned}
&\int_D p(\boldsymbol x|\boldsymbol w)
\int_{D^c} \frac{\mu_{T_1}^{d_x}}{\sigma_{T_1}^{d_x}(2\pi)^{d_x/2}}
\exp \left(-\frac{\|\boldsymbol x-\mu_{T_1}\boldsymbol z\|^2}{2\sigma_{T_1}^2}\right)
d\boldsymbol z d\boldsymbol x\\
&\le B_u\int_D
\int_{D^c} \frac{\mu_{T_1}^{d_x}}{\sigma_{T_1}^{d_x}(2\pi)^{d_x/2}}
\exp \left(-\frac{\|\boldsymbol x-\mu_{T_1}\boldsymbol z\|^2}{2\sigma_{T_1}^2}\right)
d\boldsymbol z d\boldsymbol x \\
&= B_u\int_D \mathbb P \left(
\frac{\boldsymbol x}{\mu_{T_1}}
+\frac{\sigma_{T_1}}{\mu_{T_1}}\boldsymbol\xi
\notin D
\right)d\boldsymbol x,\quad \boldsymbol\xi\sim N(0,I_{d_x}) \\
&\leq B_u\sum_{i=1}^{d_x}\int_D \left[ \mathbb \mathbb{P}\left(\boldsymbol\xi_i < -\frac{\boldsymbol{x}_i}{\sigma_{T_1}}
\right)+P\left(\boldsymbol\xi_i > \frac{\mu_{T_1}-\boldsymbol{x}_i}{\sigma_{T_1}}
\right) \right]d\boldsymbol x \\
&\le B_u\sum_{i=1}^{d_x} \int_0^1
\left[
\Phi \left(-\frac{s}{\sigma_{T_1}}\right)
+ \Phi \left(\frac{s-\mu_{T_1}}{\sigma_{T_1}}\right)
\right]ds\\
&\le B_u d_x \left[
\int_0^\infty \Phi \left(-\frac{s}{\sigma_{T_1}}\right)ds
+ \int_0^{\mu_{T_1}} \Phi \left(-\frac{\mu_{T_1}-s}{\sigma_{T_1}}\right)ds
+ \int_{\mu_{T_1}}^1 1 ds
\right] \\
&\lesssim \sigma_{T_1}+(1-\mu_{T_1}).
\end{aligned}
\]

\paragraph{Bound Term B.}
By Assumption \ref{apt:diffusion-density-bound}, we have
\[
\begin{aligned}
\int_{D^c}p_{T_1}(\boldsymbol x|\boldsymbol w)d\boldsymbol x
&= \int_D p(\boldsymbol z|\boldsymbol w)
\int_{D^c} \frac{1}{\sigma_{T_1}^{d_x}(2\pi)^{d_x/2}}
\exp \left(-\frac{\|\boldsymbol x-\mu_{T_1}\boldsymbol z\|^2}{2\sigma_{T_1}^2}\right)
d\boldsymbol x d\boldsymbol z \\
&\le B_u\int_D \mathbb P \left(
\mu_{T_1}\boldsymbol z+\sigma_{T_1}\boldsymbol\xi \notin D
\right)d\boldsymbol z,\quad \boldsymbol\xi\sim N(0,I_{d_x})\\
&\leq B_u\sum_{i=1}^{d_x}\int_D \left[ \mathbb \mathbb{P}\left(\boldsymbol\xi_i < -\frac{\mu_{T_1}\boldsymbol{z}_i}{\sigma_{T_1}}
\right)+P\left(\boldsymbol\xi_i > \frac{1-\mu_{T_1}\boldsymbol{z}_i}{\sigma_{T_1}}
\right) \right]d\boldsymbol z \\
&\le B_u\sum_{i=1}^{d_x} \int_0^1
\left[
\Phi \left(-\frac{\mu_{T_1}s}{\sigma_{T_1}}\right)
+ \Phi \left(\frac{\mu_{T_1}s-1}{\sigma_{T_1}}\right)
\right]ds \\
&\le B_u d_x \left[
\int_0^\infty \Phi \left(-\frac{\mu_{T_1}s}{\sigma_{T_1}}\right)ds
+ \frac{1}{\mu_{T_1}}
\int_{1-\mu_{T_1}}^1 \Phi \left(-\frac{u}{\sigma_{T_1}}\right)du
\right] \\
&\lesssim \frac{\sigma_{T_1}}{\mu_{T_1}}.
\end{aligned}
\]
Combining the preceding inequalities, for every $\boldsymbol w$, we have
\[
\TV\Bigl(p_{T_1}(\cdot|\boldsymbol w),p(\cdot|\boldsymbol w)\Bigr)
\lesssim\left(\frac{\sigma_{T_1}}{\mu_{T_1}}\right)^{\alpha\wedge1}+\left(\frac{1-\mu_{T_1}}{\mu_{T_1}}\right)^{\alpha\wedge1}+\frac{\sigma_{T_1}}{\mu_{T_1}}+(1-\mu_{T_1}).
\]
Therefore,
\[
\begin{aligned}
\bE_{\boldsymbol W}\left[\TV^2\Bigl(p_{T_1}(\cdot|\boldsymbol W),p(\cdot|\boldsymbol W)\Bigr)\right]
&\leq\bE_{\boldsymbol W}\left[\left(\int_{\mathbb R^{d_x}}|p_{T_1}(\boldsymbol x|\boldsymbol W)-p(\boldsymbol x|\boldsymbol W)| d\boldsymbol x\right)^2\right]\\
&\lesssim\left[\left(\frac{\sigma_{T_1}}{\mu_{T_1}}\right)^{\alpha\wedge1}+\left(\frac{1-\mu_{T_1}}{\mu_{T_1}}\right)^{\alpha\wedge1}+\frac{\sigma_{T_1}}{\mu_{T_1}}+(1-\mu_{T_1})\right]^2\\
&\lesssim T_1^{\alpha\wedge1}\\
&\lesssim N^{-C_\mu(\alpha\wedge1)}\\
&\lesssim n_x^{-\frac{2\alpha}{d_x+d_w+2\alpha}}.
\end{aligned}
\]
Hence, the proof is complete.
\end{proof}

\subsection{Proofs of Theorem \ref{thm:TV-diffusion-density-error} and Corollary~\ref{cor:rl-diffusion-bound}}
\begin{proof}[Proof of Theorem \ref{thm:TV-diffusion-density-error}]

By applying Lemmas~\ref{lem:TV-time-discrete-error}--\ref{lem:TV-time-truncation-error} to the three terms, respectively, and replacing 
$\check p_{T_1}(\cdot|X,A)$ with 
$
{\mathcal P}(\cdot|X,A)$, we obtain the desired bound.
\end{proof}

Similarly, we obtain the following error bound for reward law estimation.
\begin{proof}[Proof of Corollary~\ref{cor:rl-diffusion-bound}]
The proof follows the same argument as that of Theorem~\ref{thm:TV-diffusion-density-error}, with the  next-state variable  replaced by the reward variable.
\end{proof}

\subsection{Proof of Theorem \ref{lem:rl-diffusion-bound}}\label{sec:B.5}

\begin{proof}
By \eqref{eq:T-hat}, we obtain
\begin{equation}\label{eq:berman-1}
\begin{aligned}
&\Big(\mathcal{T}^* Q (X,A)-\widehat{\mathcal T}^* Q (X,A)\Big)^2\\
&\lesssim \Big|\Ebb [R(X, A)]-\frac{1}{m}\sum_{j=1}^m\widehat{R}_j(X,A)\Big|^2\\
&~~ +\gamma^2 \Big|\Big(\mathbb{E}_{X^{\prime}\sim \cP(\cdot|X,A)} \max _{a^{\prime} \in \cA}  Q \left(X^{\prime}, a^{\prime}\right)-\frac{1}{m}
	\sum_{j=1}^m\max_{a^{\prime}\in\mathcal A}  Q (\widehat{X }^{\prime}_j,a^{\prime})\Big)\Big|^2\\
&\lesssim \Big|\Ebb_{R\sim\mathcal R}[R(X,A)]-\Ebb_{R\sim\widehat{\mathcal R}}[R(X,A)]\Big|^2+\Big|\Ebb_{R\sim\widehat{\mathcal R}}[R(X,A)]-\frac{1}{m}\sum_{j=1}^m\widehat R_j(X,A)\Big|^2\\
&~~ +\Big|\mathbb{E}_{X^\prime\sim\cP(\cdot|X,A)}\max_{a^\prime\in\cA} Q (X^\prime,a^\prime)-\mathbb{E}_{X^\prime\sim\widehat{\cP}(\cdot|X,A)}\max_{a^\prime\in\cA} Q (X^\prime,a^\prime)\Big|^2\\
&~~ +\Big|\mathbb{E}_{X^\prime\sim\widehat{\cP}(\cdot|X,A)}\max_{a^\prime\in\cA} Q (X^\prime,a^\prime)-\frac{1}{m}\sum_{j=1}^m\max_{a^\prime\in\cA} Q (\widehat X^\prime_j,a^\prime)\Big|^2.
\end{aligned}
\end{equation}
Since the reward is bounded by $R_{\max}$, we have 
\begin{equation}\label{eq:berman-2}
\begin{aligned}
\Ebb_{\bS_r}\left[\left|\Ebb_{R\sim\widehat{\mathcal R}}[R(X,A)]-\frac{1}{m}\sum_{j=1}^m\widehat R_j(X,A)\right|^2\right]
=\frac{1}{m}\operatorname{Var}\left(\widehat R_1(X,A)\right)
\leq\frac{R_{\max}^2}{4m}.
\end{aligned}
\end{equation}
and
\begin{equation}\label{eq:berman-3}
\begin{aligned}
&\left|\Ebb_{R\sim\mathcal R(\cdot|X,A)}[R]-\Ebb_{R\sim\widehat{\mathcal R}(\cdot|X,A)}[R]\right|^2\\
&=\left|\int_{[0,R_{\max}]}R \mathcal R(dR|X,A)-\int_{[0,R_{\max}]}R \widehat{\mathcal R}(dR|X,A)\right|^2\\
&=\left|\int_{[0,R_{\max}]}R \left[\mathcal R(dR|X,A)-\widehat{\mathcal R}(dR|X,A)\right]\right|^2\\
&\lesssim R_{\max}^2\TV^2\left(\mathcal R(\cdot|X,A),\widehat{\mathcal R}(\cdot|X,A)\right).
\end{aligned}
\end{equation}
For $Q=\widehat{Q}$, since $\widehat Q\in\cQ$, we have
\begin{align*}
&\Ebb_{\bS_x}\left[\left|\mathbb E_{X^\prime\sim\widehat{\cP}(\cdot|X,A)}\max_{a^\prime\in\cA}\widehat Q(X^\prime,a^\prime)-\frac{1}{m}\sum_{j=1}^m\max_{a^\prime\in\cA}\widehat Q(\widehat X_j^\prime,a^\prime)\right|^2\right]\\
&\leq\Ebb_{\bS_x}\left[\sup_{\widehat Q\in\cQ}\left|\mathbb E_{X^\prime\sim\widehat{\cP}(\cdot|X,A)}\max_{a^\prime\in\cA}\widehat Q(X^\prime,a^\prime)-\frac{1}{m}\sum_{j=1}^m\max_{a^\prime\in\cA}\widehat Q(\widehat X_j^\prime,a^\prime)\right|^2\right].
\end{align*}
Let $\cQ_\varepsilon$ be an $\varepsilon$-cover of $\cQ$ under $\|\cdot\|_\infty$. For every $\widehat Q\in\cQ$, there exists a $\widehat Q_\varepsilon\in\cQ_\varepsilon$ such that $\|\widehat Q-\widehat Q_\varepsilon\|_\infty\leq\varepsilon$. Since
\[
\left|\max_{a^\prime\in\cA}\widehat Q(x^\prime,a^\prime)-\max_{a^\prime\in\cA}\widehat Q_\varepsilon(x^\prime,a^\prime)\right|\leq\max_{a^\prime\in\cA}|\widehat Q(x^\prime,a^\prime)-\widehat Q_\varepsilon(x^\prime,a^\prime)|\leq\varepsilon,
\]
we obtain
\begin{align*}
&\sup_{\widehat Q\in\cQ}\left|\mathbb E_{X^\prime\sim\widehat{\cP}(\cdot|X,A)}\max_{a^\prime\in\cA}\widehat Q(X^\prime,a^\prime)-\frac{1}{m}\sum_{j=1}^m\max_{a^\prime\in\cA}\widehat Q(\widehat X_j^\prime,a^\prime)\right|^2\\
&\lesssim\max_{\widehat Q_\varepsilon\in\cQ_\varepsilon}\left|\mathbb E_{X^\prime\sim\widehat{\cP}(\cdot|X,A)}\max_{a^\prime\in\cA}\widehat Q_\varepsilon(X^\prime,a^\prime)-\frac{1}{m}\sum_{j=1}^m\max_{a^\prime\in\cA}\widehat Q_\varepsilon(\widehat X_j^\prime,a^\prime)\right|^2+\varepsilon^2.
\end{align*}
The Hoeffding's lemma implies that the centered Monte Carlo fluctuation associated with each $\widehat Q_\varepsilon\in\cQ_\varepsilon$ is sub-Gaussian with variance factor $v=\frac{R_{\max}^2}{4m(1-\gamma)^2}$. Therefore,
\begin{align*}
&\mathbb P_{\bS_x}\left(\max_{\widehat Q_\varepsilon\in\cQ_\varepsilon}\left|\mathbb E_{X^\prime\sim\widehat{\cP}(\cdot|X,A)}\max_{a^\prime\in\cA}\widehat Q_\varepsilon(X^\prime,a^\prime)-\frac{1}{m}\sum_{j=1}^m\max_{a^\prime\in\cA}\widehat Q_\varepsilon(\widehat X_j^\prime,a^\prime)\right|^2>t\right)\\
&\leq 2\cN(\cQ,\varepsilon,\|\cdot\|_\infty)\exp\left(-\frac{t}{2v}\right).
\end{align*}
Integrating the preceding tail bound yields
\begin{align*}
&\Ebb_{\bS_x}\left[\max_{\widehat Q_\varepsilon\in\cQ_\varepsilon}\left|\mathbb E_{X^\prime\sim\widehat{\cP}(\cdot|X,A)}\max_{a^\prime\in\cA}\widehat Q_\varepsilon(X^\prime,a^\prime)-\frac{1}{m}\sum_{j=1}^m\max_{a^\prime\in\cA}\widehat Q_\varepsilon(\widehat X_j^\prime,a^\prime)\right|^2\right]\\
&\lesssim\frac{R_{\max}^2}{m(1-\gamma)^2}\log\Big(2\cN(\cQ,\varepsilon,\|\cdot\|_\infty)\Big).
\end{align*}
Combining the preceding inequalities and setting $\varepsilon=\frac{1}{m}$, we obtain
\begin{align*}
&\Ebb_{\bS_x}\left[\Big|\mathbb E_{X^\prime\sim\widehat{\cP}(\cdot|X,A)}\max_{a^\prime\in\cA}\widehat Q(X^\prime,a^\prime)-\frac{1}{m}\sum_{j=1}^m\max_{a^\prime\in\cA}\widehat Q(\widehat X_j^\prime,a^\prime)\Big|^2\right]\\
&\lesssim \frac{1}{m^2}+\frac{R_{\max}^2}{m(1-\gamma)^2}\log\Big(\cN\big(\cQ,\tfrac{1}{m},\|\cdot\|_\infty\big)\Big).
\end{align*}
With the size, depth, width, and weight bound of the ReLU DNN class $\cQ$ as
\[
\cS=O\left(n^{\frac{d_x+d_a}{d_x+d_a+2\alpha}}\log^5 n\right),~~\cD=O(\log n),~~\cW=O\left(n^{\frac{d_x+d_a}{d_x+d_a+2\alpha}}\log^3 n\right),~~\cB=O\left(n^{\frac{d_x+d_a}{d_x+d_a+2\alpha}}\right).
\]
By Lemma \ref{lem:NBound}, we have
\[
\log\Big(\cN\big(\cQ,\tfrac{1}{m},\|\cdot\|_\infty\big)\Big)\lesssim \cS\cD\log\left(m\cB\cW\cD\right)\lesssim n^{\frac{d_x+d_a}{d_x+d_a+2\alpha}}\log^6 n(\log m+\log n).
\]
Substituting this bound into the preceding estimate yields
\begin{equation}\label{eq:berman-6}
\begin{aligned}
&\Ebb_{\bS_x}\left[\left|\mathbb E_{X^\prime\sim\widehat{\cP}(\cdot|X,A)}\max_{a^\prime\in\cA}\widehat Q(X^\prime,a^\prime)-\frac{1}{m}\sum_{j=1}^m\max_{a^\prime\in\cA}\widehat Q(\widehat X_j^\prime,a^\prime)\right|^2\right]\\
&\lesssim\frac{1}{m^2}+\frac{R_{\max}^2}{(1-\gamma)^2}n^{\frac{d_x+d_a}{d_x+d_a+2\alpha}}\frac{\log^6 n(\log m+\log n)}{m}.
\end{aligned}
\end{equation}
Since $0\leq\max\limits_{a^\prime\in\mathcal A}\widehat Q(X^\prime,a^\prime)\leq \frac{R_{\max}}{1-\gamma}$, we have
\begin{equation}\label{eq:berman-7}
\begin{aligned}
&\left|\Ebb_{X^\prime\sim\mathcal P(\cdot|X,A)}\max_{a^\prime\in\mathcal A}\widehat Q(X^\prime,a^\prime)-\Ebb_{X^\prime\sim\widehat{\mathcal P}(\cdot|X,A)}\max_{a^\prime\in\mathcal A}\widehat Q(X^\prime,a^\prime)\right|^2\\
&=\left|\int_{[0,1]^{d_x}}\max_{a^\prime\in\mathcal A}\widehat Q(X^\prime,a^\prime) \mathcal P(dX^\prime|X,A)-\int_{[0,1]^{d_x}}\max_{a^\prime\in\mathcal A}\widehat Q(X^\prime,a^\prime) \widehat{\mathcal P}(dX^\prime|X,A)\right|^2\\
&=\left|\int_{[0,1]^{d_x}}\max_{a^\prime\in\mathcal A}\widehat Q(X^\prime,a^\prime)\left[\mathcal P(dX^\prime|X,A)-\widehat{\mathcal P}(dX^\prime|X,A)\right]\right|^2\\
&\lesssim \frac{R_{\max}^2}{(1-\gamma)^2}\TV^2\left(\mathcal P(\cdot|X,A),\widehat{\mathcal P}(\cdot|X,A)\right),
\end{aligned}
\end{equation}
For $Q=Q^*$, since $0\leq\max\limits_{a^\prime\in\mathcal A}Q^*(X^\prime,a^\prime)\leq \frac{R_{\max}}{1-\gamma}$, we have
\begin{equation}\label{eq:berman-4}
\begin{aligned}
&\Ebb_{\bS_x}\left[\left|\mathbb E_{X^\prime\sim\widehat{\cP}(\cdot|X,A)}\max_{a^\prime\in\cA}Q^*(X^\prime,a^\prime)-\frac{1}{m}\sum_{j=1}^m\max_{a^\prime\in\cA}Q^*(\widehat X^\prime_j,a^\prime)\right|^2\right]\\
&=\frac{1}{m}\Ebb_{\bS_x}\left[\operatorname{Var}_{X^\prime\sim\widehat{\cP}(\cdot|X,A)}\left(\max_{a^\prime\in\cA}Q^*(X^\prime,a^\prime)\right)\right]\\
&\leq\frac{R_{\max}^2}{4(1-\gamma)^2m}.
\end{aligned}
\end{equation}
and
\begin{equation}\label{eq:berman-5}
\begin{aligned}
&\left|\Ebb_{X^\prime\sim\mathcal P(\cdot|X,A)}\max_{a^\prime\in\mathcal A}Q^*(X^\prime,a^\prime)
-\Ebb_{X^\prime\sim\widehat{\mathcal P}(\cdot|X,A)}\max_{a^\prime\in\mathcal A}Q^*(X^\prime,a^\prime)\right|^2\\
&=\left|\int_{[0,1]^{d_x}}\max_{a^\prime\in\mathcal A}Q^*(X^\prime,a^\prime) \mathcal P(dX^\prime|X,A)
-\int_{[0,1]^{d_x}}\max_{a^\prime\in\mathcal A}Q^*(X^\prime,a^\prime) \widehat{\mathcal P}(dX^\prime|X,A)\right|^2\\
&=\left|\int_{[0,1]^{d_x}}\max_{a^\prime\in\mathcal A}Q^*(X^\prime,a^\prime)
\left[\mathcal P(dX^\prime|X,A)-\widehat{\mathcal P}(dX^\prime|X,A)\right]\right|^2\\
&\lesssim \frac{R_{\max}^2}{(1-\gamma)^2}\TV^2\left(\mathcal P(\cdot|X,A),\widehat{\mathcal P}(\cdot|X,A)\right),
\end{aligned}
\end{equation}
Therefore, by \eqref{eq:berman-1}--\eqref{eq:berman-7}, we have \eqref{eq:operator-hat-Q} and \eqref{eq:operator-Q*}.
\begin{align*}
\bE_{\bS} \|\mathcal T^*\widehat{Q}-\widehat{\mathcal T}^*\widehat{Q}\|_{L^2(\mu)}^2 &\lesssim R_{\max}^2\Ebb_{\bS}\Ebb_{(X,A)}\left[\TV^2\Big(\cR(\cdot|X,A),\widehat{\cR}(\cdot|X,A)\Big)\right]\\
&~~+\frac{R_{\max}^2}{(1-\gamma)^2}n^{\frac{d_x+d_a}{d_x+d_a+2\alpha}}\frac{\log^6 n(\log m+\log n)}{m}\\
&~~ +\frac{R_{\max}^2}{(1-\gamma)^2}\Ebb_{\bS}\Ebb_{(X,A)}\left[\TV^2\big(\cP(\cdot|X,A),\widehat{\cP}(\cdot|X,A)\big)\right].
\end{align*}
and  by \eqref{eq:berman-1}--\eqref{eq:berman-3}, \eqref{eq:berman-4}--\eqref{eq:berman-5}
\begin{align*}
\bE_{\bS} \|\widehat{\mathcal T}^*Q^*-\mathcal T^*Q^*\|_{L^2(\mu)}^2&\lesssim  R_{\max}^2 \bE_{\bS}\bE_{(X,A)} \Big[\TV^2\big(\mathcal R(\cdot|X,A),\widehat{\mathcal R}(\cdot|X,A)\big)\Big]+ \frac{R_{\max}^2}{(1-\gamma)^2m}\\
&~~+ \frac{R_{\max}^2}{(1-\gamma)^2} \bE_{\bS} \bE_{(X,A)}\left[\TV^2\bigl(\mathcal P(\cdot|X,A),\widehat{\mathcal P}(\cdot|X,A)\bigr)\right]
\end{align*}
With \eqref{eq:operator-hat-Q}, \eqref{eq:operator-Q*}, Theorem \ref{thm:TV-diffusion-density-error} and Corollary \ref{cor:rl-diffusion-bound}, we complete the proof.
\end{proof}

\section{Bounds on Value Function Learning}\label{sec:apd-bounds-on-rl}
\subsection{Proof of Lemma \ref{lem:LQ-sta}} \label{sec:proof-LQ-sta}
\begin{proof}
	By Lemma \ref{lem:T-hat-fixed-point}, for any $(x,a)\in [0,1]^{d_w}$, we have
	\begin{align*}
		\lvert {\widehat{\cT}^*}  Q_1(x,a) - \widehat{\cT}^*  Q_2(x,a) \rvert 
        \leq \gamma\norm{ Q_1-  Q_2}_{\infty}.
	\end{align*}
	By \eqref{eq:lwhat-Q-Z}, we establish
	\begin{align*}
		\left|\ell_{ Q_1}(x,a )-\ell_{ Q_2}(x,a )\right|&=\bigg| \left|\widehat{\mathcal T}^* Q_1(x,a)- Q_1(x,a)\right|^2-\left|\widehat{\mathcal T}^* Q_2(x,a)- Q_2(x,a)\right|^2 \bigg| \\
		&\leq \Big|  \Big(\widehat{\mathcal T}^* Q_1(x,a)- Q_1(x,a)-\widehat{\mathcal T}^* Q_2(x,a)+ Q_2(x,a)\Big)\\
		&~~~~\cdot\Big( \widehat{\mathcal T}^* Q_1(x,a)- Q_1(x,a)+\widehat{\mathcal T}^* Q_2(x,a)- Q_2(x,a) \Big)\Big|\\
		&\leq \frac{4R_{\max}}{1-\gamma}\Big| \widehat{\mathcal T}^* Q_1(x,a)-\widehat{\mathcal T}^* Q_2(x,a) \Big|+ \frac{4R_{\max}}{1-\gamma}\Big|  Q_1(x,a)- Q_2(x,a) \Big|\\
		&\leq \frac{4R_{\max}(1+\gamma)}{1-\gamma} \norm{ Q_1-  Q_2}_{\infty}.
	\end{align*}
	To bound the statistical error, we employ the technique of offset Rademacher complexity. Starting with the expectation over the distribution $\bD$, we obtain that
	\begin{align*}
&\mathbb{E}_{\bD}\sup_{ Q\in  \cQ}\Big(\cL_{\widehat{\mathcal{T}}^*}( Q)-2\widehat{\cL}_{\widehat{\mathcal{T}}^*}( Q)\Big)\\		&=\mathbb{E}_{\bD}\sup_{ Q \in  \cQ}\left(\mathbb{E}_{(X,A)  }[\ell_{ Q} (X, A )]-\frac{2}{n}\sum_{i = 1}^n
\ell_{ Q} (X_i, A_i )
\right)\\
&=\mathbb{E}_{\bD}\sup_{ Q \in  \cQ}\Bigg(\frac{3}{2}\mathbb{E}_{(X,A)}[\ell_{ Q} (X, A )]-\frac{1}{2}\mathbb{E}_{(X,A)}[\ell_{ Q} (X, A )]-\frac{3}{2n}\sum_{i = 1}^n\ell_{ Q} (X_i, A_i )-\frac{1}{2n}\sum_{i = 1}^n\ell_{ Q} (X_i, A_i )\Bigg)\\
		&\leq \mathbb{E}_{\bD}\sup_{ Q \in  \cQ}\Bigg(\frac{3}{2}\mathbb{E}_{(X,A)}[\ell_{ Q} (X, A )]-\frac{1}{2C}\mathbb{E}_{(X,A)}[\ell^2_{ Q} (X, A )]-\frac{3}{2n}\sum_{i = 1}^n\ell_{ Q} (X_i, A_i )-\frac{1}{2Cn}\sum_{i = 1}^n\ell^2_{ Q} (X_i, A_i )\Bigg)
	\end{align*}
	where we use the fact that $|Q(X,A)| \leq \frac{R_{\max}}{1-\gamma}$ and $\ell_{ Q}(X, A ) \leq \frac{4R^2_{\max}}{(1-\gamma)^2} =: {{C}}.$
	Now, we introduce a ghost sample
	$\bD^{*} =\{(X_i^{*}, A_i^{*} )\}_{i=1}^n$ 
    independent of $\bD$. Further, let $\tau=\{\tau_i\}_{i=1}^n$ denote a set of Rademacher variables. Substituting the expectation with the empirical mean derived from the ghost sample $\bD^{*}$, we obtain
\begin{align*}
&\mathbb{E}_{\bD}\sup_{ Q \in  \cQ}\Bigg(\frac{3}{2}\mathbb{E}_{(X,A)}[\ell_{ Q} (X, A )]-\frac{1}{2C}\mathbb{E}_{(X,A)}[\ell^2_{ Q} (X, A )]-\frac{3}{2n}\sum_{i = 1}^n\ell_{ Q} (X_i, A_i )-\frac{1}{2Cn}\sum_{i = 1}^n\ell^2_{ Q} (X_i, A_i )\Bigg)\\
&=\mathbb{E}_{\bD}\sup_{ Q \in  \cQ}\Bigg(\frac{3}{2}\mathbb{E}_{\mathbb{D}^*}[\frac{1}{n}\sum_{i=1}^n\ell_{ Q} (X^*_i, A^*_i )]-\frac{1}{2C}\mathbb{E}_{\mathbb{D}^*}[\frac{1}{n}\sum_{i=1}^n\ell^2_{ Q} (X^*_i, A^*_i )]\\
		&~~~~~~~~~~~~~~~~ -\frac{3}{2n}\sum_{i = 1}^n\ell_{ Q} (X_i, A_i )-\frac{1}{2Cn}\sum_{i = 1}^n\ell^2_{ Q} (X_i, A_i )\Bigg)\\
&\leq\mathbb{E}_{\bD,\bD^*}\sup_{ Q \in  \cQ}\Bigg(\frac{3}{2n}\sum_{i=1}^n\ell_{ Q} (X^*_i, A^*_i )-\frac{1}{2Cn}\sum_{i=1}^n\ell^2_{ Q} (X^*_i, A^*_i )\\
		&~~~~~~~~~~~~~~~~ -\frac{3}{2n}\sum_{i = 1}^n\ell_{ Q} (X_i, A_i )-\frac{1}{2Cn}\sum_{i = 1}^n\ell^2_{ Q} (X_i, A_i )\Bigg)\\
		&\leq\mathbb{E}_{\bD,\bD^*,\tau}\sup_{ Q \in  \cQ}\left(\frac{3}{2n}\sum_{i = 1}^n\tau_i\ell_{ Q} (X^*_i, A^*_i )-\frac{1}{2{{C}}n}\sum_{i = 1}^n\ell^2_{ Q} (X^*_i, A^*_i )\right)\\
		&~~~~+\mathbb{E}_{\bD,\bD^*,\tau}\sup_{ Q \in  \cQ}\left(-\frac{3}{2n}\sum_{i = 1}^n\tau_i\ell_{ Q} (X_i, A_i )-\frac{1}{2{{C}}n}\sum_{i = 1}^n\ell^2_{ Q} (X_i, A_i )\right)\\
		&\leq \mathbb{E}_{\bD,\tau}\sup_{ Q \in  \cQ}\left(\frac{3}{n}\sum_{i = 1}^n\tau_i\ell_{ Q} (X_i, A_i )-\frac{1}{{{C}} n}\sum_{i = 1}^n\ell^2_{ Q} (X_i, A_i )\right)\\
		&=3\mathfrak{R}_n^{\text{off}}(\mathcal{Q},\frac{1}{3C}).
	\end{align*}
	Let $\mathcal{N}(\mathcal{Q},~\delta,~\|\cdot\|_{\infty})$ be the covering number of $\mathcal{Q}$ with cover $\mathcal{Q}_{\delta}$, there exists a $ Q_{\delta}\in\mathcal{Q}_\delta$ such that $\norm{ Q- Q_{\delta}}_{\infty}\leq \delta$, that is, 
	\begin{align*}
		|\ell_{ Q} (X_i, A_i )-\ell_{ Q_{\delta}} (X_i, A_i )| &\leq \frac{4R_{\max}(1+\gamma)}{1-\gamma} \norm{ Q- Q_{\delta}}_{\infty}\leq  \frac{4R_{\max}(1+\gamma)}{1-\gamma} \delta.
	\end{align*}
	Thus,  
	\begin{align*}
		\frac{1}{n}\sum_{i = 1}^n\tau_i\ell_{ Q} (X_i, A_i )&\leq\frac{1}{n}\sum_{i = 1}^n\tau_i\ell_{ Q_{\delta}} (X_i, A_i )+\frac{1}{n}\sum_{i = 1}^n|\tau_i||\ell_{ Q} (X_i, A_i )-\ell_{ Q_{\delta}} (X_i, A_i )|\\
		&\leq\frac{1}{n}\sum_{i = 1}^n\tau_i\ell_{ Q_{\delta}} (X_i, A_i )+\frac{4R_{\max}(1+\gamma)}{1-\gamma}\frac{1}{n}\sum_{i = 1}^n|\tau_i| \norm{Q-Q_\delta}_\infty \\
		&\leq\frac{1}{n}\sum_{i = 1}^n\tau_i\ell_{ Q_\delta}(X_i,A_i )+\frac{4R_{\max}(1+\gamma)}{1-\gamma}\delta.
	\end{align*}
	Using $0\leq\ell_Q(X,A)\leq C$ and $0\leq\ell_{Q_\delta}(X,A)\leq C$, we also have
	\begin{align*}
		-\frac{1}{n}\sum_{i = 1}^n\ell^2_{ Q}(X_i,A_i )&=-\frac{1}{n}\sum_{i = 1}^n\ell^2_{ Q_{\delta}}(X_i,A_i )+\frac{1}{n}\sum_{i = 1}^n \Big(\ell^2_{ Q_{\delta}}(X_i,A_i )-\ell^2_{ Q}(X_i,A_i )\Big)\\
		&\leq-\frac{1}{n}\sum_{i = 1}^n\ell^2_{ Q_{\delta}}(X_i,A_i )+2{{C}}\frac{1}{n}\sum_{i = 1}^n|\ell_{ Q_{\delta}} (X_i, A_i )-\ell_{ Q} (X_i, A_i )|\\
		&\leq-\frac{1}{n}\sum_{i = 1}^n\ell^2_{ Q_{\delta}}(X_i,A_i )+\frac{8R_{\max}(1+\gamma)C}{1-\gamma}\frac{1}{n}\sum_{i = 1}^n \norm{Q-Q_\delta}_\infty\\
		&\leq-\frac{1}{n}\sum_{i = 1}^n\ell^2_{ Q_{\delta}}(X_i,A_i )+\frac{8R_{\max}(1+\gamma)C}{1-\gamma}\delta.
	\end{align*}
	Since $\tau=\{\tau_i\}_{i = 1}^n$ is a sequence of i.i.d. Rademacher variables independent of 
	$\bD$, 
	then conditionally on $\bD$ we have
	\begin{align*}
		\mathfrak{R}_n^{\text{off}}(\mathcal{Q},\eta  |  \mathbb{D})&\leq \bE_\tau\bigg[\max_{Q_\delta\in\cQ_\delta}\Big(\frac{1}{n}\sum_{i=1}^n\tau_i\ell_{Q_\delta}(X_i,A_i)-\frac{\eta}{n}\sum_{i=1}^n\ell_{Q_\delta}^2(X_i,A_i)\Big)\bigg]\\
		&~~~~ +\frac{4R_{\max}(1+\gamma)}{1-\gamma}(1 + 2\eta {C})\delta,
	\end{align*}
	for any $\eta>0$. 
	Recall that $\{\tau_i\ell_{ Q_{\delta}} (X_i, A_i )\}_{i = 1}^n$ are independent random variables conditioning on $\bD$,
	$$
	\mathbb{E}_{\tau}[\tau_i\ell_{ Q_{\delta}} (X_i, A_i )] = 0,
	$$
	and
	$$-\ell_{ Q_{\delta}} (X_i, A_i )\leq\tau_i\ell_{ Q_{\delta}} (X_i, A_i )\leq\ell_{ Q_{\delta}} (X_i, A_i ),\quad i = 1,\ldots,n.
	$$
	By Hoeffding's inequality (See Lemma \ref{lem:bernstein-hoeffding-inequality}), it yields that for any $ Q_{\delta}\in\mathcal{Q}_{\delta}$ and $\xi>0$,
	\begin{align*}
		&\mathbb{P}_{\tau}\left\{\frac{1}{n}\sum_{i = 1}^n\tau_i\ell_{ Q_{\delta}} (X_i, A_i )>\xi+\frac{\eta}{n}\sum_{i = 1}^n\ell^2_{ Q_{\delta}}(X_i,A_i )\right\}\\
		&\leq\exp\left(-\frac{\Big(n\xi+\eta\sum\limits_{i = 1}^n\ell^2_{ Q_{\delta}}(X_i,A_i )\Big)^2}{2\sum\limits_{i = 1}^n\ell^2_{ Q_{\delta}}(X_i,A_i )}\right)\\
		&\leq\exp(-2\eta n\xi).
	\end{align*}
	For simplicity, we denote $\mathcal{N}(\mathcal{Q},\delta,\|\cdot\|_{\infty})$ as $ \cN $. Therefore, setting $A = \frac{\log \cN}{2\eta n}$, we have
	\begin{align*}
		&\mathbb{E}_{\tau}\max_{ Q_{\delta}\in \mathcal{Q}_{\delta}}\left(\frac{1}{n}\sum_{i = 1}^n\tau_i\ell_{ Q_{\delta}} (X_i, A_i )-\frac{\eta}{n}\sum_{i = 1}^n\ell^2_{ Q_{\delta}}(X_i,A_i )\right)\\
		&\leq\int_{0}^{\infty} \mathbb{P}_{\tau}\left\{\max_{ Q_{\delta}\in\mathcal{Q}_{\delta}}\left(\frac{1}{n}\sum_{i = 1}^n\tau_i\ell_{ Q_{\delta}} (X_i, A_i )-\frac{\eta}{n}\sum_{i = 1}^n\ell^2_{ Q_{\delta}}(X_i,A_i )>\xi\right)\right\}d\xi\\
		&\leq\int_{0}^{\infty} \cN\max_{ Q_{\delta}\in\mathcal{Q}_{\delta}} \mathbb{P}_{\tau}\left\{\frac{1}{n}\sum_{i = 1}^n\tau_i\ell_{ Q_{\delta}} (X_i, A_i )>\xi+\frac{\eta}{n}\sum_{i = 1}^n\ell^2_{ Q_{\delta}}(X_i,A_i )\right\} d\xi\\
		&\leq A+\int_{A}^{\infty}\cN\exp(-2\eta n\xi)d\xi\\
		&\leq \frac{\log \cN}{2\eta n}+\frac{\cN}{2\eta n}\exp(-2\eta n\frac{\log \cN}{2\eta n})\\
		&=\frac{1+\log \cN}{2\eta n}.
	\end{align*}
	By setting $\delta = \frac{1}{n}$, applying Lemma \ref{lem:NBound}, and choosing $\eta = \frac{1}{3{{C}}}$, we obtain the following bound
	\begin{align*}
		\mathbb{E}_{\tau  |  \bD}\sup_{ Q \in  \cQ}\left(\frac{3}{n}\sum_{i = 1}^n\tau_i\ell_{ Q} (X_i, A_i )-\frac{1}{{{C}}n}\sum_{i = 1}^n\ell^2_{ Q}(X_i,A_i )\right)&\leq\frac{9{{C}}\Big(1+\log\cN\Big)}{2n}+\frac{20R_{\max}(1+\gamma)}{1-\gamma}\delta\\
&=\frac{18R^2_{\max}\Big(1+\log\cN\Big)}{(1-\gamma) ^2n}+\frac{20R_{\max}(1+\gamma)}{1-\gamma}\delta\\
		&\leq \cO \Big( \frac{R^2_{\max} \cS\cD \log (n\cB\cW\cD)}{(1-\gamma)^2n}\Big).
	\end{align*}
With the network size $\mathcal S=\mathcal O\big(n^{\frac{d_x+d_a}{d_x+d_a+2\alpha}}\log^5n\big)$, depth $\mathcal D=\mathcal O(\log n)$, width $\mathcal W=\mathcal O\big(n^{\frac{d_x+d_a}{d_x+d_a+2\alpha}}\log^3n\big)$, and weight bound $\mathcal B=\mathcal O\big(n^{\frac{d_x+d_a}{d_x+d_a+2\alpha}}\big)$, we obtain the statistical error bound \eqref{eq:RL-statistical-error}.
\end{proof}

\subsection{Proof of Lemma \ref{lem:Q-smoothness}}
\begin{proof}
For each $j\in\{R,P\}$, choose $L_j<\infty$ to bound the relevant derivatives and H\"older seminorms of a H\"older extension of $p_j$. In particular, on the original domains,
$$
\max_{|\nu|\leq p}\sup_{z,w}|D_w^\nu p_j(z\mid w)|\leq L_j,~\max_{|\nu|=p}\sup_{z,w\neq\widetilde w}\frac{|D_w^\nu p_j(z\mid w)-D_w^\nu p_j(z\mid\widetilde w)|}{\|w-\widetilde w\|^q}\leq L_j.
$$
Here, $\nu\in\mathbb N_0^{d_x+d_a}$ is a multi-index, and $D_w^\nu$ differentiates only the conditioning variable $w$. By the Bellman equation, we have
$$
Q^*(w)=\int_0^{R_{\max}}rp_R(r\mid w)\mathrm dr+\gamma\int_{\mathcal X}V^*(y)p_P(y\mid w)\mathrm dy.
$$
Applying the triangle inequality to the Bellman equation, we obtain
\begin{align*}
|Q^*(\widetilde w)-Q^*(w)|&\leq\int_0^{R_{\max}}r|p_R(r\mid\widetilde w)-p_R(r\mid w)|\mathrm dr\\
&~~+\gamma\|V^*\|_\infty\int_{\mathcal X}|p_P(y\mid\widetilde w)-p_P(y\mid w)|\mathrm dy.
\end{align*}
Using continuity of the densities, we obtain pointwise convergence of the integrands to zero as $\widetilde w\to w$. We bound them by the integrable functions $2rL_R$ and $2L_P$, respectively. Since $\|V^*\|_\infty<\infty$, we apply dominated convergence to conclude that the right-hand side tends to zero. Hence $Q^*$ is continuous.

We next establish the derivative identity
$$
D_w^\nu Q^*(w)=\int_0^{R_{\max}}rD_w^\nu p_R(r\mid w)\mathrm dr+\gamma\int_{\mathcal X}V^*(y)D_w^\nu p_P(y\mid w)\mathrm dy,~~ |\nu|\leq p.
$$
For $|\nu|=0$, this is the Bellman equation, which also covers the case $p=0$. For $p\geq1$, suppose that the identity holds for all $|\nu|=m$, where $0\leq m<p$. Fix such a $\nu$, an interior point $w\in\mathcal W^\circ$, and a coordinate unit vector $\ell\in\mathbb N_0^{d_x+d_a}$ with $|\ell|=1$. Choose $h_0>0$ small enough that $w+uh\ell\in\mathcal W^\circ$ for all $u\in[0,1]$ and $|h|<h_0$. For $0<|h|<h_0$, we obtain from the induction hypothesis that
\begin{align*}
\frac{D_w^\nu Q^*(w+h\ell)-D_w^\nu Q^*(w)}{h}&=\int_0^{R_{\max}}r\frac{D_w^\nu p_R(r\mid w+h\ell)-D_w^\nu p_R(r\mid w)}{h}\mathrm dr\\
&~~ +\gamma\int_{\mathcal X}V^*(y)\frac{D_w^\nu p_P(y\mid w+h\ell)-D_w^\nu p_P(y\mid w)}{h}\mathrm dy.
\end{align*}
Since $|\nu|+1\leq p$, we apply the fundamental theorem of calculus and the uniform derivative bound to obtain
$$
\frac{D_w^\nu p_j(z\mid w+h\ell)-D_w^\nu p_j(z\mid w)}{h}=\int_0^1D_w^{\nu+\ell}p_j(z\mid w+uh\ell)\mathrm du\leq L_j,~~ j\in\{R,P\}.
$$
Using the same derivative bound, we also bound the absolute value of the difference quotient by $L_j$. By continuity of $D_w^{\nu+\ell}p_j$, we have
$$
\lim_{h\to0}\frac{D_w^\nu p_j(z\mid w+h\ell)-D_w^\nu p_j(z\mid w)}{h}=D_w^{\nu+\ell}p_j(z\mid w).
$$
The corresponding integrands satisfy
\begin{align*}
&r\left|\frac{D_w^\nu p_R(r\mid w+h\ell)-D_w^\nu p_R(r\mid w)}{h}\right|\leq rL_R,\\
&\gamma|V^*(y)|\left|\frac{D_w^\nu p_P(y\mid w+h\ell)-D_w^\nu p_P(y\mid w)}{h}\right|\leq\frac{\gamma R_{\max}L_P}{1-\gamma}.
\end{align*}
These dominating functions are independent of $h$ and integrable
$$
\int_0^{R_{\max}}rL_R\mathrm dr=\frac{R_{\max}^2L_R}{2}<\infty,~~ \int_{\mathcal X}\frac{\gamma R_{\max}L_P}{1-\gamma}\mathrm dy=\frac{\gamma R_{\max}L_P}{1-\gamma}<\infty.
$$
We therefore apply dominated convergence to obtain
$$
\begin{aligned}
D_w^\ell D_w^\nu Q^*(w)&=\lim_{h\to0}\frac{D_w^\nu Q^*(w+h\ell)-D_w^\nu Q^*(w)}{h}\\
&=\int_0^{R_{\max}}rD_w^{\nu+\ell}p_R(r\mid w)\mathrm dr+\gamma\int_{\mathcal X}V^*(y)D_w^{\nu+\ell}p_P(y\mid w)\mathrm dy.
\end{aligned}
$$
Using the same domination argument, we conclude that these derivatives are continuous in the interior. Every multi-index of order $m+1$ can be written as $\nu+\ell$ with $|\nu|=m$ and $|\ell|=1$. This completes the induction.

For a boundary point $w\in\partial\mathcal W$, take a sequence $w_j\in\mathcal W^\circ$ converging to $w$. Continuity of the density derivatives on the closed domains and dominated convergence give
$$
\lim_{j\to\infty}D_w^\nu Q^*(w_j)=\int_0^{R_{\max}}rD_w^\nu p_R(r\mid w)\mathrm dr+\gamma\int_{\mathcal X}V^*(y)D_w^\nu p_P(y\mid w)\mathrm dy,~~ |\nu|\leq p.
$$
The limit is independent of the approaching sequence and defines the continuous extension of each derivative to the boundary. With this convention, the derivative identity gives, for every $|\nu|\leq p$,
$$
\begin{aligned}
\|D_w^\nu Q^*\|_{\infty,\mathcal W}&\leq L_R\int_0^{R_{\max}}r\mathrm dr+\gamma L_P\|V^*\|_\infty\int_{\mathcal X}\mathrm dy\\
&\leq\frac{R_{\max}^2L_R}{2}+\frac{\gamma R_{\max}L_P}{1-\gamma}=:C_Q.
\end{aligned}
$$
For $|\nu|=p$ and $w,\widetilde w\in\mathcal W$, the H\"older bounds on the densities imply
$$
\begin{aligned}
|D_w^\nu Q^*(w)-D_w^\nu Q^*(\widetilde w)|&\leq\int_0^{R_{\max}}r|D_w^\nu p_R(r\mid w)-D_w^\nu p_R(r\mid\widetilde w)|\mathrm dr\\
&\quad+\gamma\int_{\mathcal X}|V^*(y)||D_w^\nu p_P(y\mid w)-D_w^\nu p_P(y\mid\widetilde w)|\mathrm dy\\
&\leq\left(L_R\int_0^{R_{\max}}r\mathrm dr+\gamma L_P\|V^*\|_\infty\int_{\mathcal X}\mathrm dy\right)\|w-\widetilde w\|^q\\
&\leq C_Q\|w-\widetilde w\|^q.
\end{aligned}
$$
Consequently,
$$
\max_{|\nu|\leq p}\|D_w^\nu Q^*\|_{\infty,\mathcal W}\leq C_Q,~~ \max_{|\nu|=p}\sup_{\substack{w,\widetilde w\in\mathcal W\\w\neq\widetilde w}}\frac{|D_w^\nu Q^*(w)-D_w^\nu Q^*(\widetilde w)|}{\|w-\widetilde w\|^q}\leq C_Q.
$$
Using the full-space H\"older extensions of the densities in the Bellman integral representation, we extend $Q^*$ to $\mathbb R^{d_x+d_a}$ with $V^*$ and the integration domains unchanged. The preceding estimates apply to this extension, so $Q^*\in\mathcal H^\alpha(\mathcal X\times\mathcal A)$.
\end{proof}

\subsection{Proof of Theorem \ref{lem:RL-error}}\label{sec:rl-approximation}

We first show that the empirical optimal Bellman operator $\widehat{\mathcal{T}}^*$  is a $\gamma$-contraction and then give the proof of Theorem \ref{lem:RL-error}.
\begin{lemma}\label{lem:T-hat-fixed-point}
The operator $\widehat{\mathcal T}^{*}$ defined in \eqref{eq:T-hat} is a $\gamma$-contraction with respect to the supremum norm. 
\end{lemma}
\begin{proof}
Let $Q_1,Q_2$ be arbitrary measurable functions satisfying
$0\leq Q_1,Q_2\leq \frac{R_{\max}}{1-\gamma}$.
For any $(x,a)\in[0,1]^{d_x}\times[0,1]^{d_a}$, $\widehat{\mathcal T}^{*}$ uses the same samples at $(x,a)$ for $Q_1$ and $Q_2$, so the reward terms cancel and
\[
\begin{aligned}
\left|\widehat{\mathcal T}^{*}Q_1(x,a)-\widehat{\mathcal T}^{*}Q_2(x,a)\right|
&=\left|\frac{\gamma}{m}\sum_{j=1}^m\left[\max_{a^\prime\in\cA}Q_1\bigl(\widehat X^\prime_j(x,a),a^\prime\bigr)-\max_{a^\prime\in\cA}Q_2\bigl(\widehat X^\prime_j(x,a),a^\prime\bigr)\right]\right|\\
&\leq\frac{\gamma}{m}\sum_{j=1}^m\left|\max_{a^\prime\in\cA}Q_1\bigl(\widehat X^\prime_j(x,a),a^\prime\bigr)-\max_{a^\prime\in\cA}Q_2\bigl(\widehat X^\prime_j(x,a),a^\prime\bigr)\right|\\
&\leq\frac{\gamma}{m}\sum_{j=1}^m\max_{a^\prime\in\cA}\left|Q_1\bigl(\widehat X^\prime_j(x,a),a^\prime\bigr)-Q_2\bigl(\widehat X^\prime_j(x,a),a^\prime\bigr)\right|\\
&\leq\gamma\|Q_1-Q_2\|_\infty.
\end{aligned}
\]
Taking the supremum over $(x,a)\in[0,1]^{d_x}\times[0,1]^{d_a}$ gives
\[
\left\|\widehat{\mathcal T}^{*}Q_1-\widehat{\mathcal T}^{*}Q_2\right\|_\infty\leq\gamma\|Q_1-Q_2\|_\infty.
\]
Since $\gamma<1$, $\widehat{\mathcal T}^{*}$ is a $\gamma$-contraction. 
\end{proof}

\begin{proof}[Proof of Theorem \ref{lem:RL-error}]
By Lemma \ref{lem:Q-smoothness}, the normalized function $\frac{1-\gamma}{R_{\max}}Q^*$ belongs to an $\alpha$-H\"older class with a radius bounded independently of $\gamma$. We apply Lemma \ref{lem:ApproximationError} to this function and rescale the network output by $\frac{R_{\max}}{1-\gamma}$. For $\mathcal S=\mathcal O\big(n^{\frac{d_x+d_a}{d_x+d_a+2\alpha}}\log^5n\big)$, $\mathcal D=\mathcal O(\log n)$, $\mathcal W=\mathcal O\big(n^{\frac{d_x+d_a}{d_x+d_a+2\alpha}}\log^3n\big)$, and $\mathcal B=\mathcal O\big(n^{\frac{d_x+d_a}{d_x+d_a+2\alpha}}\big)$, we obtain
\begin{equation}\label{eq:RL-approximation-error}
\inf_{Q\in\cQ}\norm{Q-Q^*}_\infty^2\lesssim\frac{R_{\max}^2}{(1-\gamma)^2}n^{-\frac{2\alpha}{d_x+d_a+2\alpha}}.
\end{equation}
Combining \eqref{eq:RL-statistical-error} and \eqref{eq:RL-approximation-error}, we obtain the desired bound.
\end{proof}

\section{Main Results}\label{sec:main-results}
In this section, we first establish a lemma on Bellman residual loss transfer and then prove the error decomposition and our main results.

\begin{lemma}\label{lem:residual-loss-transfer}
If $\mu^{\pi^*}\ll\mu^{\pi_b}$ and the density ratio $\frac{\mathrm d\mu^{\pi^*}}{\mathrm d\mu^{\pi_b}}$ has a finite second moment under $\mu^{\pi_b}$, then we have
\[
\mathcal L_{\mathcal T^{\pi^*}}^{2}(Q)\lesssim\frac{R_{\max}^2}{(1-\gamma)^2}\mathcal L_{\mathcal T^*}(Q).
\]
Moreover, under the H\"older continuity condition in Assumption~\ref{apt:diffusion-density-holder}, the two losses have the same unique global minimizer
\[
\arg\min_Q\mathcal L_{\mathcal T^*}(Q)=\arg\min_Q\mathcal L_{\mathcal T^{\pi^*}}(Q)=\{Q^*\}.
\]
\end{lemma}

\begin{proof}
By a change of measure and the Cauchy--Schwarz inequality, we obtain
\begin{align*}
\mathcal L_{\mathcal T^{\pi^*}}^{2}(Q)
&=\left(\int_{\mathcal X\times\mathcal A}\frac{\mathrm d\mu^{\pi^*}}{\mathrm d\mu^{\pi_b}}(x,a)|\mathcal T^*Q(x,a)-Q(x,a)|^2\,\mu^{\pi_b}(\mathrm dx,\mathrm da)\right)^2\\
&\leq\left\|\frac{\mathrm d\mu^{\pi^*}}{\mathrm d\mu^{\pi_b}}\right\|_{L^2(\mu^{\pi_b})}^{2}\int_{\mathcal X\times\mathcal A}|\mathcal T^*Q(x,a)-Q(x,a)|^4\,\mu^{\pi_b}(\mathrm dx,\mathrm da)\\
&\lesssim\|\mathcal T^*Q-Q\|_\infty^{2}\mathcal L_{\mathcal T^*}(Q).
\end{align*}
Using the boundedness of rewards and $\|Q\|_\infty\lesssim\frac{R_{\max}}{1-\gamma}$, we have
\[
\|\mathcal T^*Q-Q\|_\infty\leq R_{\max}+(1+\gamma)\|Q\|_\infty\lesssim\frac{R_{\max}}{1-\gamma}.
\]
Substituting this bound into the preceding inequality gives the stated loss comparison.
By Lemma~\ref{lem:Q-smoothness}, $Q^*$ is continuous. Using $\mathcal T^*Q^*=Q^*$, we conclude that $Q^*$ is a global minimizer. Let $Q^\dagger$ be any global minimizer of $\mathcal L_{\mathcal T^*}$ over continuous functions on $\mathcal X\times\mathcal A$. Then $\mathcal L_{\mathcal T^*}(Q^\dagger)=0$, and $Q^\dagger$ is bounded by compactness of the domain. Using the continuity and boundedness of the conditional densities, we apply the dominated convergence theorem to the Bellman integral representation and obtain the continuity of $\mathcal T^*Q^\dagger$. Thus, the residual $\mathcal T^*Q^\dagger-Q^\dagger$ is continuous. 

If this residual were nonzero at some point, we could find a nonempty relatively open neighborhood $U$ of that point and $\varepsilon>0$ such that $|\mathcal T^*Q^\dagger-Q^\dagger|\geq\varepsilon$ on $U$. Since $\operatorname{supp}(\mu^{\pi_b})=\mathcal X\times\mathcal A$, we have $\mu^{\pi_b}(U)>0$, and hence
\[
\mathcal L_{\mathcal T^*}(Q^\dagger)\geq\varepsilon^2\mu^{\pi_b}(U)>0,
\]
contradicting $\mathcal L_{\mathcal T^*}(Q^\dagger)=0$. Therefore, $Q^\dagger=\mathcal T^*Q^\dagger$ throughout $\mathcal X\times\mathcal A$. By the contraction property of $\mathcal T^*$, we obtain
\[
\|Q^\dagger-Q^*\|_\infty
=\|\mathcal T^*Q^\dagger-\mathcal T^*Q^*\|_\infty
\leq\gamma\|Q^\dagger-Q^*\|_\infty.
\]
Since $\gamma<1$, we conclude that $Q^\dagger=Q^*$. Repeating this argument with $\mu^{\pi^*}$ in place of $\mu^{\pi_b}$ yields the same uniqueness conclusion for $\mathcal L_{\mathcal T^*}^{*}$.
\end{proof}

\begin{proof}[Proof of Lemma~\ref{lem:error-dec}]
Since $Q^*=\mathcal T^*Q^*$, we have $\mathcal L_{\mathcal T^*}(Q^*)=\|\mathcal T^*Q^*-Q^*\|_{L^2(\mu)}^2=0.$
Therefore,
\[
\begin{aligned}
\mathcal L_{\mathcal T^*}(\widehat Q)-\mathcal L_{\mathcal T^*}(Q^*)
&=\|\mathcal T^*\widehat Q-\widehat{\mathcal T}^*\widehat Q+\widehat{\mathcal T}^*\widehat Q-\widehat Q\|_{L^2(\mu)}^2\\
&\leq2\|\mathcal T^*\widehat Q-\widehat{\mathcal T}^*\widehat Q\|_{L^2(\mu)}^2+2\mathcal L_{\widehat{\mathcal T}^*}(\widehat Q)\\
&=2\|\mathcal T^*\widehat Q-\widehat{\mathcal T}^*\widehat Q\|_{L^2(\mu)}^2+2\left(\mathcal L_{\widehat{\mathcal T}^*}(\widehat Q)-2\widehat{\mathcal L}_{\widehat{\mathcal T}^*}(\widehat Q)\right)+4\widehat{\mathcal L}_{\widehat{\mathcal T}^*}(\widehat Q)\\
&\leq2\|\mathcal T^*\widehat{Q}-\widehat{\mathcal T}^*\widehat{Q}\|_{L^2(\mu)}^2+2\sup_{Q\in\mathcal Q}\left(\mathcal L_{\widehat{\mathcal T}^*}(Q)-2\widehat{\mathcal L}_{\widehat{\mathcal T}^*}(Q)\right)\\
&~~ +4\inf_{Q\in\mathcal Q}\widehat{\mathcal L}_{\widehat{\mathcal T}^*}(Q),
\end{aligned}
\]
where the last inequality follows from $\widehat{\mathcal L}_{\widehat{\mathcal T}^*}(\widehat Q)=\inf_{Q\in\mathcal Q}\widehat{\mathcal L}_{\widehat{\mathcal T}^*}(Q).$
Taking expectation over $\bD$ gives
\[
\begin{aligned}
\bE_{\bD}\left[\mathcal L_{\mathcal T^*}(\widehat Q)-\mathcal L_{\mathcal T^*}(Q^*)\right]
&\leq 2
\bE_{\bD}\left[\|\mathcal T^*\widehat{Q}-\widehat{\mathcal T}^*\widehat{Q}\|_{L^2(\mu)}^2\right]\\
&\quad+2\bE_{\bD}\sup_{Q\in\mathcal Q}\left(\mathcal L_{\widehat{\mathcal T}^*}(Q)-2\widehat{\mathcal L}_{\widehat{\mathcal T}^*}(Q)\right)+4\bE_{\bD}\inf_{Q\in\mathcal Q}\widehat{\mathcal L}_{\widehat{\mathcal T}^*}(Q).
\end{aligned}
\]
Conditional on the learned operator $\widehat{\mathcal T}^*$, the value sample $\bD$ remains independent and identically distributed according to $\mu$. Therefore,
\[
\bE_{\bD}\inf_{Q\in\mathcal Q}\widehat{\mathcal L}_{\widehat{\mathcal T}^*}(Q)\leq\inf_{Q\in\mathcal Q}\bE_{\bD}\widehat{\mathcal L}_{\widehat{\mathcal T}^*}(Q)=\inf_{Q\in\mathcal Q}\mathcal L_{\widehat{\mathcal T}^*}(Q).
\]
For every $Q\in\mathcal Q$, the identity $\mathcal T^*Q^*=Q^*$ gives
\[
\widehat{\mathcal T}^*Q-Q=\widehat{\mathcal T}^*Q-\widehat{\mathcal T}^*Q^*+\widehat{\mathcal T}^*Q^*-\mathcal T^*Q^*+Q^*-Q.
\]
Hence,
\[
\begin{aligned}
\mathcal L_{\widehat{\mathcal T}^*}(Q)
&=\left\|\widehat{\mathcal T}^*Q-\widehat{\mathcal T}^*Q^*+\widehat{\mathcal T}^*Q^*-\mathcal T^*Q^*+Q^*-Q\right\|_{L^2(\mu)}^2\\
&\leq2\|\widehat{\mathcal T}^*Q^*-\mathcal T^*Q^*\|_{L^2(\mu)}^2+2\left(\|\widehat{\mathcal T}^*Q-\widehat{\mathcal T}^*Q^*\|_{L^2(\mu)}+\|Q-Q^*\|_{L^2(\mu)}\right)^2\\
&\leq2\|\widehat{\mathcal T}^*Q^*-\mathcal T^*Q^*\|_{L^2(\mu)}^2+2(1+\gamma)^2\|Q-Q^*\|_\infty^2.
\end{aligned}
\]
The last inequality uses
\[
\|\widehat{\mathcal T}^*Q-\widehat{\mathcal T}^*Q^*\|_{L^2(\mu)}\leq\|\widehat{\mathcal T}^*Q-\widehat{\mathcal T}^*Q^*\|_\infty\leq\gamma\|Q-Q^*\|_\infty
\]
and $\|Q-Q^*\|_{L^2(\mu)}\leq\|Q-Q^*\|_\infty,$
since $\mu$ is a probability measure. Taking the infimum over $Q\in\mathcal Q$ gives
\[
4\inf_{Q\in\mathcal Q}\mathcal L_{\widehat{\mathcal T}^*}(Q)\leq8\|\widehat{\mathcal T}^*Q^*-\mathcal T^*Q^*\|_{L^2(\mu)}^2+8(1+\gamma)^2\inf_{Q\in\mathcal Q}\|Q-Q^*\|_\infty^2.
\]
Combining the preceding bounds, we obtain
\[
\begin{aligned}
\bE_{\bD}\left[\mathcal L_{\mathcal T^*}(\widehat Q)-\mathcal L_{\mathcal T^*}(Q^*)\right]
&\leq 2
\bE_{\bD}\left[\|\mathcal T^*\widehat{Q}-\widehat{\mathcal T}^*\widehat{Q}\|_{L^2(\mu)}^2\right]+8\|\widehat{\mathcal T}^*Q^*-\mathcal T^*Q^*\|_{L^2(\mu)}^2\\
&~~ +2\bE_{\bD}\sup_{Q\in\mathcal Q}\left(\mathcal L_{\widehat{\mathcal T}^*}(Q)-2\widehat{\mathcal L}_{\widehat{\mathcal T}^*}(Q)\right)+8(1+\gamma)^2\inf_{Q\in\mathcal Q}\|Q-Q^*\|_\infty^2.
\end{aligned}
\]
This completes the proof.
\end{proof}

\begin{proof}[Proof of Theorem \ref{thm:main-result}]
By the error decomposition of Lemma \ref{lem:error-dec}, together with Theorems \ref{lem:rl-diffusion-bound}--\ref{lem:RL-error}, the desired theorem follows directly.
\end{proof}

Using the similar argument of \citet{chen2019information}, we can prove Theorem \ref{thm:value-est} below.

\begin{proof}[Proof of Theorem \ref{thm:value-est}]
By Lemma \ref{lem:Q-smoothness}, we have $Q^*\in\mathcal H^\alpha(\mathcal X\times\mathcal A)$. Since $\mathcal A=[0,1]^{d_a}$ is compact and every $Q\in\mathcal Q$ is continuous,
and $Q^*\in\mathcal H^\alpha$,
the positive-measure condition in Assumption~4.6 ensures that
$\pi_{Q,\delta}$ is a well defined measurable stochastic kernel.
For every $a\in G_{Q,\delta}(x)$, we first show that
\begin{equation}\label{eq:comparison-action}\left|\max_{\widetilde a\in\cA}Q(x,\widetilde a)-\max_{\widetilde a\in\cA}Q^*(x,\widetilde a)\right|\leq|Q(x,a)-Q^*(x,a)|+\delta.\end{equation}
Suppose first that
\[
\max_{\widetilde a\in\cA}Q(x,\widetilde a)\geq\max_{\widetilde a\in\cA}Q^*(x,\widetilde a).
\]
For $a\in G_{Q,\delta}(x)$, the definition of $G_{Q,\delta}(x)$ gives
\[
\max\{Q(x,a),Q^*(x,a)\}\geq\max_{\widetilde a\in\cA}Q(x,\widetilde a)-\delta.
\]
If $Q(x,a)\geq Q^*(x,a)$, then
\[
\begin{aligned}\max_{\widetilde a\in\cA}Q(x,\widetilde a)-\max_{\widetilde a\in\cA}Q^*(x,\widetilde a)&\leq\max_{\widetilde a\in\cA}Q(x,\widetilde a)-Q^*(x,a)\\&=\max_{\widetilde a\in\cA}Q(x,\widetilde a)-Q(x,a)+Q(x,a)-Q^*(x,a)\\&\leq\delta+|Q(x,a)-Q^*(x,a)|.\end{aligned}
\]
If $Q^*(x,a)>Q(x,a)$, then
\[
\max_{\widetilde a\in\cA}Q(x,\widetilde a)-\max_{\widetilde a\in\cA}Q^*(x,\widetilde a)\leq\max_{\widetilde a\in\cA}Q(x,\widetilde a)-Q^*(x,a)\leq\delta.
\]
The opposite ordering follows by interchanging $Q$ and $Q^*$. This proves \eqref{eq:comparison-action}. For every $\nu\in\mathfrak M_{\delta}$, Jensen's inequality and \eqref{eq:comparison-action} give
\[
\begin{aligned}\|\cT^*Q-\cT^*Q^*\|_{L^2(\nu)}^2&=\gamma^2\int_{\cX\times\cA}\left(\int_{\cX}\left[\max_{\widetilde a\in\cA}Q(x',\widetilde a)-\max_{\widetilde a\in\cA}Q^*(x',\widetilde a)\right]\cP(\mathrm dx' |  x,a)\right)^2\nu(\mathrm dx,\mathrm da)\\
&\leq\gamma^2\int_{\cX\times\cA}\int_{\cX}\left|\max_{\widetilde a\in\cA}Q(x',\widetilde a)-\max_{\widetilde a\in\cA}Q^*(x',\widetilde a)\right|^2\cP(\mathrm dx' |  x,a)\nu(\mathrm dx,\mathrm da)\\
&\leq\gamma^2\int_{\cX\times\cA}\int_{\cX}\int_{\cA}\left(|Q(x',a')-Q^*(x',a')|+\delta\right)^2\\
&~~~~~~~~~~ \times \pi_{Q,\delta}(\mathrm da' |  x')\cP(\mathrm dx' |  x,a)\nu(\mathrm dx,\mathrm da)\\
&=\gamma^2\left\| |Q-Q^*|+\delta \right\|_{L^2(\cP(\nu)\otimes\pi_{Q,\delta})}^2.
\end{aligned}
\]
Taking square roots and applying Minkowski's inequality give
\begin{equation}\label{eq:bellman-comparison}\|\cT^*Q-\cT^*Q^*\|_{L^2(\nu)}\leq\gamma\|Q-Q^*\|_{L^2(\cP(\nu)\otimes\pi_{Q,\delta})}+\gamma\delta.\end{equation}
The density ratio condition in Assumption~\ref{apt:concentrability} gives
\[
\begin{aligned}\|Q-\cT^*Q\|_{L^2(\nu)}^2&=\int_{\cX\times\cA}|Q(x,a)-\cT^*Q(x,a)|^2\frac{\mathrm d\nu}{\mathrm d\mu}(x,a)\mu(\mathrm dx,\mathrm da)\\
&\leq C_\delta\int_{\cX\times\cA}|Q(x,a)-\cT^*Q(x,a)|^2\mu(\mathrm dx,\mathrm da)\\
&= C_\delta\|Q-\cT^*Q\|_{L^2(\mu)}^2.\end{aligned}
\]
Hence,
\begin{equation}\label{eq:residual-transfer}\|Q-\cT^*Q\|_{L^2(\nu)}\leq\sqrt{ C_\delta}\|Q-\cT^*Q\|_{L^2(\mu)}.\end{equation}
Using $Q^*=\cT^*Q^*$ together with \eqref{eq:bellman-comparison} and \eqref{eq:residual-transfer}, we obtain
\[
\begin{aligned}\|Q-Q^*\|_{L^2(\nu)}&\leq\|Q-\cT^*Q\|_{L^2(\nu)}+\|\cT^*Q-\cT^*Q^*\|_{L^2(\nu)}\\&\leq\sqrt{ C_\delta}\|Q-\cT^*Q\|_{L^2(\mu)}+\gamma\|Q-Q^*\|_{L^2(\cP(\nu)\otimes\pi_{Q,\delta})}+\gamma\delta.\end{aligned}
\]
By the definition of $\mathfrak M_{\delta}$, $\cP(\nu)\otimes\pi_{Q,\delta}\in\mathfrak M_{\delta}$. Taking the supremum over $\nu\in\mathfrak M_{\delta}$ gives
\[
\sup_{\nu\in\mathfrak M_{\delta}}\|Q-Q^*\|_{L^2(\nu)}\leq\sqrt{ C_\delta}\|Q-\cT^*Q\|_{L^2(\mu)}+\gamma\sup_{\nu\in\mathfrak M_{\delta}}\|Q-Q^*\|_{L^2(\nu)}+\gamma\delta.
\]
The supremum is finite because $Q$ and $Q^*$ are bounded. Rearranging gives
\[
\sup_{\nu\in\mathfrak M_{\delta}}\|Q-Q^*\|_{L^2(\nu)}\leq\frac{\sqrt{ C_\delta}}{1-\gamma}\|Q-\cT^*Q\|_{L^2(\mu)}+\frac{\gamma\delta}{1-\gamma}.
\]
Since $\mathfrak M\subseteq\mathfrak M_{\delta}$, for every $\rho\in\mathfrak M$,
\[
\|Q-Q^*\|_{L^2(\rho)}\leq\frac{\sqrt{ C_\delta}}{1-\gamma}\|Q-\cT^*Q\|_{L^2(\mu)}+\frac{\gamma\delta}{1-\gamma}.
\]
Applying Jensen's inequality, we obtain
\[
\begin{aligned}
\mathbb E\|\widehat Q-Q^*\|_{L^2(\rho)}&\leq\frac{\sqrt{C_\delta}}{1-\gamma}\mathbb E\|\widehat Q-\cT^*\widehat Q\|_{L^2(\mu)}+\frac{\gamma\delta}{1-\gamma}\\
&\leq\frac{\sqrt{C_\delta}}{1-\gamma}\left(\mathbb E\|\widehat Q-\cT^*\widehat Q\|_{L^2(\mu)}^2\right)^{\frac12}+\frac{\gamma\delta}{1-\gamma}\\
&=\frac{\sqrt{C_\delta}}{1-\gamma}\left(\mathbb E\left[\cL_{\cT^*}(\widehat Q)-\cL_{\cT^*}(Q^*)\right]\right)^{\frac12}+\frac{\gamma\delta}{1-\gamma},
\end{aligned}
\]
where the last equality follows from $Q^*=\cT^*Q^*$. Here, the expectation is taken over all the randomness in $\bD,\bS,\bT,\bZ$. Under the conditions of Theorem~\ref{thm:main-result}(\ref{item:main-rate-n}), we have
\[
\mathbb E\|\widehat Q-Q^*\|_{L^2(\rho)}\lesssim\frac{R_{\max}\sqrt{C_\delta}}{(1-\gamma)^2}n^{-\frac{\alpha}{d_x+d_a+2\alpha}}\log^{\frac{19}{2}}n+\frac{\gamma\delta}{1-\gamma}.
\]
Taking $\delta=n^{-\frac{\alpha}{d_x+d_a+2\alpha}}$ and using $C_\delta\leq\widetilde C$, we obtain, for sufficiently large $n$,
\[
\begin{aligned}
\mathbb E\|\widehat Q-Q^*\|_{L^2(\rho)}&\lesssim\frac{R_{\max}\sqrt{\widetilde C}}{(1-\gamma)^2}n^{-\frac{\alpha}{d_x+d_a+2\alpha}}\log^{\frac{19}{2}}n+\frac{\gamma}{1-\gamma}n^{-\frac{\alpha}{d_x+d_a+2\alpha}}\\
&\lesssim\frac{R_{\max}\sqrt{\widetilde C}}{(1-\gamma)^2}n^{-\frac{\alpha}{d_x+d_a+2\alpha}}\log^{\frac{19}{2}}n.
\end{aligned}
\]
This proves Theorem \ref{thm:value-est} \eqref{item:value-rate-n}. Under the conditions of Theorem~\ref{thm:main-result}(\ref{item:main-rate-total}), we have
\[
\mathbb E\|\widehat Q-Q^*\|_{L^2(\rho)}\lesssim\frac{R_{\max}\sqrt{C_\delta}}{(1-\gamma)^2}\widetilde n^{-\frac{\alpha}{2d_x+d_a+2\alpha}}\log^{\frac{19}{2}}\widetilde n+\frac{\gamma\delta}{1-\gamma}.
\]
Taking $\delta=\widetilde n^{-\frac{\alpha}{2d_x+d_a+2\alpha}}$ and using $C_\delta\leq\widetilde C$, we obtain, for sufficiently large $\widetilde n$,
\[
\begin{aligned}
\mathbb E\|\widehat Q-Q^*\|_{L^2(\rho)}&\lesssim\frac{R_{\max}\sqrt{\widetilde C}}{(1-\gamma)^2}\widetilde n^{-\frac{\alpha}{2d_x+d_a+2\alpha}}\log^{\frac{19}{2}}\widetilde n+\frac{\gamma}{1-\gamma}\widetilde n^{-\frac{\alpha}{2d_x+d_a+2\alpha}}\\
&\lesssim\frac{R_{\max}\sqrt{\widetilde C}}{(1-\gamma)^2}\widetilde n^{-\frac{\alpha}{2d_x+d_a+2\alpha}}\log^{\frac{19}{2}}\widetilde n.
\end{aligned}
\]
This proves Theorem \ref{thm:value-est} \eqref{item:value-rate-total} and completes the proof.
\end{proof}

\begin{proof}[Proof of Theorem \ref{thm:Q-minimax-lower}]
Set $\widetilde n=n+n_r+n_x$. We construct a finite family of MDPs in $\mathfrak E_\alpha$, the model class defined by the density bounds and $\alpha$-H\"older assumptions.
Let every MDP in this family have the common uniform transition density
\[
p_X(x'|x,a)=1,~~x'\in\cX.
\]
Choose a nonzero nonnegative infinitely differentiable function $\psi:\mathbb R^{d_x+d_a}\to[0,\infty)$ whose support lies strictly inside $\left(-\frac12,\frac12\right)^{d_x+d_a}$. For sufficiently small $h>0$, choose $J_h\asymp h^{-(d_x+d_a)}$ centers $z_1,\ldots,z_{J_h}\in[h,1-h]^{d_x+d_a}$ such that $\|z_j-z_k\|_\infty\geq2h$ whenever $j\neq k$. Then each translated and rescaled function $z\mapsto\psi\left(\frac{z-z_j}{h}\right)$ vanishes outside $[0,1]^{d_x+d_a}$, and no two of these functions are nonzero at the same point. For $v\in\{0,1\}^{J_h}$ and $z=(x,a)$, define
\[
f_v(z)=c_0h^\alpha\sum_{j=1}^{J_h}v_j\psi\left(\frac{z-z_j}{h}\right),~~b_v=\int_{\cX}\max_{a\in\cA}f_v(x,a) \mathrm dx,
\]
where $c_0>0$ is a fixed constant chosen below. Since $\max\limits_{a\in\cA}f_v(x,a)\leq\|f_v\|_\infty$, we have
\[
0\leq b_v\leq\|f_v\|_\infty\leq c_0h^\alpha\|\psi\|_\infty.
\]
Write $\alpha=p+q$, where $p\in\mathbb N_0$ and $q\in(0,1]$. By the chain rule and the separation of the bump supports, we obtain
\begin{align*}
&\max_{\|\xi\|_1\leq p}\|\partial^\xi f_v\|_\infty\leq c_0\max_{\|\xi\|_1\leq p}h^{\alpha-\|\xi\|_1}\|\partial^\xi\psi\|_\infty\lesssim c_0,\\
&\max_{\|\xi\|_1=p}\sup_{z\neq z'}\frac{|\partial^\xi f_v(z)-\partial^\xi f_v(z')|}{\|z-z'\|^q}\lesssim c_0h^{\alpha-p-q}=c_0.
\end{align*}
Thus, the functions $f_v$ have uniformly bounded $\alpha$-H\"older norms of order $c_0$.
Next, choose an infinitely differentiable function $\varphi:\mathbb R\to\mathbb R$ with compact support such that, for $r\in[0,R_{\max}]$,
\[
\varphi(r)=\frac{12\left(r-\frac{R_{\max}}{2}\right)}{R_{\max}^3}.
\]
Direct integration gives
\[
\int_0^{R_{\max}}\varphi(r) \mathrm dr=0,~~\int_0^{R_{\max}}r\varphi(r) \mathrm dr=1,~~\int_0^{R_{\max}}\varphi(r)^2 \mathrm dr=\frac{12}{R_{\max}^3},~~\sup_{r\in[0,R_{\max}]}|\varphi(r)|=\frac6{R_{\max}^2}.
\]
Define
\[
Q_v(x,a)=\frac{R_{\max}}{2(1-\gamma)}+f_v(x,a)
\]
and
\[
p_{R,v}(r|x,a)=\frac{1}{R_{\max}}+\bigl(f_v(x,a)-\gamma b_v\bigr)\varphi(r),~~r\in[0,R_{\max}].
\]
Let $M_v$ denote the resulting MDP, where $R$ and $X'$ are conditionally independent given $(X,A)$. The identities for $\varphi$ imply
\[
\int_0^{R_{\max}}p_{R,v}(r|x,a) \mathrm dr=1,~~\int_0^{R_{\max}}r p_{R,v}(r|x,a) \mathrm dr=\frac{R_{\max}}{2}+f_v(x,a)-\gamma b_v.
\]
Moreover,
\[
\left|\bigl(f_v(x,a)-\gamma b_v\bigr)\varphi(r)\right|\leq\frac{6(1+\gamma)c_0h^\alpha\|\psi\|_\infty}{R_{\max}^2}.
\]
Hence, for sufficiently small fixed $c_0$,
\[
\frac{1}{2R_{\max}}\leq p_{R,v}(r|x,a)\leq\frac{3}{2R_{\max}}.
\]
Therefore, $p_{R,v}$ is a valid conditional density supported on $[0,R_{\max}]$ and satisfies fixed positive density bounds. The product rule and the preceding H\"older bounds show that $p_{R,v}$ belongs uniformly to a fixed $\alpha$-H\"older ball.
Reducing $c_0$ if necessary, we also have
\[
0\leq Q_v(x,a)\leq\frac{R_{\max}}{1-\gamma}.
\]
The conditional mean reward under $M_v$ is
\[
\mathbb E[R|X=x,A=a]=\frac{R_{\max}}{2}+f_v(x,a)-\gamma b_v.
\]
Since the transition density is uniform,
\[
\mathbb E\left[\max_{a'\in\cA}Q_v(X',a')\mid X=x,A=a\right]=\frac{R_{\max}}{2(1-\gamma)}+b_v.
\]
Consequently,
\[
\begin{aligned}
(\mathcal T_{M_v}^*Q_v)(x,a)&=\frac{R_{\max}}{2}+f_v(x,a)-\gamma b_v+\gamma\left(\frac{R_{\max}}{2(1-\gamma)}+b_v\right)\\
&=\frac{R_{\max}}{2(1-\gamma)}+f_v(x,a)=Q_v(x,a).
\end{aligned}
\]
The optimal Bellman operator is a $\gamma$-contraction under the supremum norm. Its fixed point is unique, and hence
\[
Q_{M_v}^*=Q_v.
\]
Thus, $M_v\in\mathfrak E_\alpha$. By Lemma~\ref{lem:Q-smoothness}, its optimal action-value function also belongs to $\mathcal H^\alpha(\mathcal X\times\mathcal A)$.
By Lemma~\ref{lem:information-tools}, there exists a set $\mathcal V_h\subset\{0,1\}^{J_h}$ containing the zero vector such that
\[
\log|\mathcal V_h|\geq\frac{J_h}{8},~~d_{\mathrm H}(v,v')\geq\frac{J_h}{4},~~ v\neq v'.
\]
Since $Q_{M_v}^*(z)=\frac{R_{\max}}{2(1-\gamma)}+f_v(z)$, the common constant term cancels when $Q_{M_v}^*$ and $Q_{M_{v'}}^*$ are subtracted. Moreover, $\rho=\operatorname{Unif}(\cX\times\cA)$ and $\cX\times\cA$ has unit volume. The disjointness of the bump regions eliminates all cross terms. Using the change of variables $u=\frac{z-z_j}{h}$ and $\sum_{j=1}^{J_h}(v_j-v_j')^2=d_{\mathrm H}(v,v')$, we obtain, for distinct $v,v'\in\mathcal V_h$,
\begin{align*}
\|Q_{M_v}^*-Q_{M_{v'}}^*\|_{L^2(\rho)}^2&=\int_{\cX\times\cA}|f_v(z)-f_{v'}(z)|^2 \mathrm dz\\
&=c_0^2h^{2\alpha}\sum_{j=1}^{J_h}(v_j-v_j')^2\int_{\cX\times\cA}\psi^2\left(\frac{z-z_j}{h}\right) \mathrm dz\\
&=c_0^2h^{2\alpha+d_x+d_a}\|\psi\|_{L^2(\mathbb R^{d_x+d_a})}^2d_{\mathrm H}(v,v')\\
&\gtrsim h^{2\alpha+d_x+d_a}J_h\asymp h^{2\alpha},
\end{align*}
where the last line uses $d_{\mathrm H}(v,v')\gtrsim J_h$ and $J_h\asymp h^{-(d_x+d_a)}$. Consequently, there exists a constant $c_1>0$, independent of $h$ and $v$, such that
\[
\|Q_{M_v}^*-Q_{M_{v'}}^*\|_{L^2(\rho)}\geq2c_1h^\alpha,~~ v\neq v'.
\]
Let $P_v$ denote the distribution of one observation $(X,A,R,X')$ under $M_v$. Since $f_0=b_0=0$,
\[
p_{R,0}(r|x,a)=\frac{1}{R_{\max}}.
\]
By Lemma~\ref{lem:information-tools},
\[
\begin{aligned}
\mathrm{KL}(P_v\|P_0)&=\int_{\cX\times\cA}\mathrm{KL}\left(p_{R,v}(\cdot|x,a)\middle\|p_{R,0}(\cdot|x,a)\right)\mu(\mathrm dx,\mathrm da)\\
&\leq R_{\max}\int_{\cX\times\cA}\int_0^{R_{\max}}|f_v(x,a)-\gamma b_v|^2\varphi(r)^2 \mathrm dr \mathrm dx \mathrm da\\
&=\frac{12}{R_{\max}^2}\int_{\cX\times\cA}|f_v(x,a)-\gamma b_v|^2 \mathrm dx \mathrm da.
\end{aligned}
\]
Disjointness of the bump supports gives
\[
\int_{\cX\times\cA}f_v(x,a)^2 \mathrm dx \mathrm da=c_0^2h^{2\alpha+d_x+d_a}\|\psi\|_2^2\sum_{j=1}^{J_h}v_j\lesssim h^{2\alpha},
\]
and $b_v^2\lesssim h^{2\alpha}$. Hence,
\[
\mathrm{KL}(P_v\|P_0)\lesssim h^{2\alpha}.
\]
Let $\mathbb P_v$ denote the joint distribution of $\mathbb S_r,\mathbb S_x,\bD$ under $M_v$.
Each observation in these batches is a full observation with law $P_v$ or a marginal of it. By the data processing inequality, its KL divergence from the corresponding law under $M_0$ is at most $\mathrm{KL}(P_v\|P_0)$. Since the batches contain a total of $\widetilde n$ independent observations, we obtain
\[
\mathrm{KL}(\mathbb P_v\|\mathbb P_0)\leq(n_r+n_x+n)\mathrm{KL}(P_v\|P_0)\lesssim\widetilde n h^{2\alpha}.
\]
Choose $h=c_h\widetilde n^{-\frac{1}{d_x+d_a+2\alpha}}$, where $c_h>0$ is sufficiently small. Since $\log|\mathcal V_h|\gtrsim h^{-(d_x+d_a)}$, we have
\[
\frac{\widetilde n h^{2\alpha}}{\log|\mathcal V_h|}\lesssim\widetilde n h^{d_x+d_a+2\alpha}=c_h^{d_x+d_a+2\alpha}.
\]
Choosing $c_h$ sufficiently small gives
\[
\begin{aligned}
\frac1{|\mathcal V_h|}\sum_{v\in\mathcal V_h}\mathrm{KL}(\mathbb P_v\|\mathbb P_0)&\leq\frac1{|\mathcal V_h|}\sum_{v\in\mathcal V_h}\frac1{16}\log|\mathcal V_h|\\
&=\frac1{|\mathcal V_h|} |\mathcal V_h| \frac1{16}\log|\mathcal V_h|\\
&=\frac1{16}\log|\mathcal V_h|.
\end{aligned}
\]
Let $V$ be uniformly distributed on $\mathcal V_h$. Conditionally on $V=v$, let $Y=(\mathbb S_r,\mathbb S_x,\bD)$ have distribution $\mathbb P_v$. The marginal distribution of $Y$ is therefore
\[
\overline{\mathbb P}:=\frac1{|\mathcal V_h|}\sum_{v\in\mathcal V_h}\mathbb P_v.
\]
The mutual information between $V$ and $Y$ is
\[
I(V;Y):=\frac1{|\mathcal V_h|}\sum_{v\in\mathcal V_h}\mathrm{KL}(\mathbb P_v\|\overline{\mathbb P}).
\]
Then, we obtain
\[
\begin{aligned}
\frac{1}{|\mathcal V_h|}\sum_{v\in\mathcal V_h}\mathrm{KL}(\mathbb P_v\|\mathbb P_0)&=\frac{1}{|\mathcal V_h|}\sum_{v\in\mathcal V_h}\mathrm{KL}(\mathbb P_v\|\overline{\mathbb P})+\mathrm{KL}(\overline{\mathbb P}\|\mathbb P_0)\\
&=I(V;Y)+\mathrm{KL}(\overline{\mathbb P}\|\mathbb P_0)\geq I(V;Y).
\end{aligned}
\]
Consequently,
\[
I(V;Y)\leq\frac1{16}\log|\mathcal V_h|.
\]
Since $J_h\asymp h^{-(d_x+d_a)}$ and $\log|\mathcal V_h|\geq\frac{J_h}{8}$, there exist constants $c_{\mathcal V}>0$ and $h_0>0$ such that
\begin{equation}\label{eq:packing-cardinality}
\log|\mathcal V_h|\geq c_{\mathcal V}h^{-(d_x+d_a)},~~ 0<h\leq h_0.
\end{equation}
Applying Lemma~\ref{lem:information-tools} with $\theta_v=Q_{M_v}^*$ and $\delta=c_1h^\alpha$, we obtain
\[
\inf_{\widehat Q}\sup_{v\in\mathcal V_h}\mathbb P_{M_v}\left(\|\widehat Q-Q_{M_v}^*\|_{L^2(\rho)}\geq c_1h^\alpha\right)\geq\frac{15}{16}-\frac{\log2}{\log|\mathcal V_h|}.
\]
This also holds for randomized estimators: adjoining an independent random seed to $Y$ leaves $I(V;Y)$ unchanged.
Fix $c_2\in\left(0,\frac{15}{16}\right)$ and define
\[
N_0:=\left\lceil\max\left\{\left(\frac{c_h}{h_0}\right)^{d_x+d_a+2\alpha}, c_h^{d_x+d_a+2\alpha}\left(\frac{\log2}{c_{\mathcal V}(\frac{15}{16}-c_2)}\right)^{\frac{d_x+d_a+2\alpha}{d_x+d_a}}\right\}\right\rceil.
\]
By \eqref{eq:packing-cardinality} and $h=c_h\widetilde n^{-\frac{1}{d_x+d_a+2\alpha}}$, for every $\widetilde n\geq N_0$, we have
\[
\inf_{\widehat Q}\sup_{v\in\mathcal V_h}\mathbb P_{M_v}\left(\|\widehat Q-Q_{M_v}^*\|_{L^2(\rho)}\geq c_1h^\alpha\right)\geq c_2.
\]
Finally, using $\mathbb E[Z]\geq t\mathbb P(Z\geq t)$ for $Z\geq0$ and $\{M_v:v\in\mathcal V_h\}\subset\mathfrak E_\alpha$, we conclude that
\[
\begin{aligned}
\inf_{\widehat Q}\sup_{M\in\mathfrak E_\alpha}\mathbb E_M\|\widehat Q-Q_M^*\|_{L^2(\rho)}\geq c_1c_2h^\alpha\asymp\widetilde n^{-\frac{\alpha}{d_x+d_a+2\alpha}}.
\end{aligned}
\]
This completes the proof.
\end{proof}

\section{Auxiliary Results}\label{sec:apd-auxiliary-results}

\begin{lemma}[Theorem 2.5 of \cite{boucheron2013concentration}]
\label{lem:maximal-inequality}
Let $Z_1,\ldots,Z_N$ be real-valued random variables such that for every $\lambda\in(0,b)$ and $i=1,\ldots,N$, the logarithm of the moment generating function of $Z_i$ satisfies $\psi_{Z_i}(\lambda)\leq\psi(\lambda)$ where $\psi$ is a convex and continuously differentiable function on $[0,b)$ with $0<b\leq\infty$ such that $\psi(0)=\psi^\prime(0)=0$. Then
\[
\mathbb{E}\max_{i=1,\ldots,N}Z_i\leq\psi^{*-1}(\log N).
\]
In particular, if the $Z_i$ are sub-Gaussian with variance factor $v$, that is, $\psi(\lambda)=\lambda^2v/2$ for every $\lambda\in(0,\infty)$, then
\[
\mathbb{E}\max_{i=1,\ldots,N}Z_i\leq\sqrt{2v\log N}.
\]
\end{lemma}

\begin{lemma}[Lemma 3 of \cite{jiao2025model}]\label{lem:kl-guass}
	Let $p(\boldsymbol{y}) = \mathcal{N}(\boldsymbol{\mu}_1, \boldsymbol{\Sigma}_1)$, $q(\boldsymbol{y}) = \mathcal{N}(\boldsymbol{\mu}_2, \boldsymbol{\Sigma}_2)$, then the
	KL divergence between $p$ and $q$ satisfies
	\[
	\mathrm{KL}(p\|q) = \frac{1}{2}\left[ (\boldsymbol{\mu}_1 - \boldsymbol{\mu}_2)^\top \boldsymbol{\Sigma}_2^{-1} (\boldsymbol{\mu}_1 - \boldsymbol{\mu}_2) - \log\left| \boldsymbol{\Sigma}_2^{-1}\boldsymbol{\Sigma}_1 \right| + \mathrm{Tr}\left( \boldsymbol{\Sigma}_2^{-1}\boldsymbol{\Sigma}_1 \right) - d_y \right].
	\]
	In particular, when $q(\boldsymbol{y}) = \mathcal{N}(\boldsymbol{0}, \boldsymbol{I}_{d_y})$, the result simplifies to
	\[
	\mathrm{KL}(p\|q) = \frac{1}{2}\left[ \|\boldsymbol{\mu}_1\|^2 - \log|\boldsymbol{\Sigma}_1| + \mathrm{Tr}(\boldsymbol{\Sigma}_1) - d_y \right].
	\]
\end{lemma}

\begin{lemma}[Theorem 6 of \cite{chen2023sampling}]
\label{lem:Girsanov}
Let $\mathbb P$ and $\mathbb P'$ be two probability measures on $C([0,T];\mathbb R^d)$. Suppose that under $\mathbb P$ and $\mathbb P'$, respectively, the canonical process $\boldsymbol z=(\boldsymbol z_t)_{t\in[0,T]}$ satisfies
\[
\mathrm{d}\boldsymbol z_t=\boldsymbol s_t \mathrm{d}t+\sigma(t) \mathrm{d}\boldsymbol B_t,\qquad \mathrm{d}\boldsymbol z_t=\boldsymbol s'_t \mathrm{d}t+\sigma(t) \mathrm{d}\boldsymbol B'_t,\qquad \boldsymbol z_0\sim p_0,
\]
where $\boldsymbol s_t$ and $\boldsymbol s'_t$ are progressively measurable processes and $\sigma(t)>0$ is deterministic. Suppose that
\[
\mathbb E_{\mathbb P}\left[\exp\left(\frac12\int_0^T\frac{\|\boldsymbol s_t-\boldsymbol s'_t\|^2}{\sigma^2(t)} \mathrm{d}t\right)\right]<\infty.
\]
Then
\[
\frac{\mathrm{d}\mathbb P'}{\mathrm{d}\mathbb P}=\exp\left(-\int_0^T\frac{(\boldsymbol s_t-\boldsymbol s'_t)^\top}{\sigma(t)} \mathrm{d}\boldsymbol B_t-\frac12\int_0^T\frac{\|\boldsymbol s_t-\boldsymbol s'_t\|^2}{\sigma^2(t)} \mathrm{d}t\right).
\]
Consequently, if $p_T$ and $p'_T$ are the time-$T$ marginal distributions under $\mathbb P$ and $\mathbb P'$, respectively, then
\[
\mathrm{KL}(p_T\|p'_T)\leq\mathrm{KL}(\mathbb P\|\mathbb P')=\frac12\mathbb E_{\mathbb P}\int_0^T\frac{\|\boldsymbol s_t-\boldsymbol s'_t\|^2}{\sigma^2(t)} \mathrm{d}t.
\]
\end{lemma}

\begin{lemma}[Lemma 6.1 of \cite{jiao2025deep}]\label{lem:ApproximationError}
	Assume that $ f \in \mathcal{H}^\varsigma $ with $ \varsigma = s + r $, $ s \in \mathbb{N}_0 $ and $ r \in (0,1] $. For any $ \varepsilon \in (0,1) $, there exists a ReLU DNN function $ \psi $ with depth $ \cD \leq \mathcal{O}\left( \log(1/\varepsilon) \right) $, size $ \mathcal{S} \leq \mathcal{O}\left( \varepsilon^{-d/\varsigma} \log(1/\varepsilon) \right) $, and weight bound $ \mathcal{B} \leq \mathcal{O}\left( \varepsilon^{-d/\varsigma} \right) $ such that
	$$
	\| f - \psi \|_{\infty} \leq \varepsilon.
	$$
\end{lemma}

\begin{lemma}[Lemma 20 of \cite{feng2024deep}]\label{lem:NBound}
	Let $\mathcal{F}$ be the ReLU DNNs with width $\mathcal{W}$, depth $\mathcal{D}$, and size $\mathcal{S}$. Assume that the parameters of $\mathcal{F}$ are bounded by a constant $\mathcal{B}>0$, then for each $\delta > 0$,
	$$
	\log\mathcal{N}(\mathcal{F},~\delta,~\|\cdot\|_{\infty})\leq\mathcal{O}(\mathcal{S}\mathcal{D}\log(\mathcal{B}\mathcal{W}\mathcal{D}/\delta)).
	$$
\end{lemma}

\begin{lemma}[Theorem 2.8 and Equation 2.10  of \cite{boucheron2013concentration}]\label{lem:bernstein-hoeffding-inequality}
Let $X_1,\ldots,X_m$ be independent real-valued random variables, $\overline{X}:=\frac{\sum_{i=1}^m X_i}{m}$, and $\sigma^2:=\frac{\sum_{i=1}^m\operatorname{Var}(X_i)}{m}$.
If $a_i\le X_i\le b_i$ almost surely, then, for every $\epsilon>0$,
\begin{equation}
\Pr \left(\left|\overline{X}-\mathbb{E}[\overline{X}]\right|\ge\epsilon\right)\le 2\exp \left(-\frac{2m^2\epsilon^2}{\sum_{i=1}^m(b_i-a_i)^2}\right). \tag{Hoeffding's inequality}
\end{equation}
If $|X_i-\mathbb{E}[X_i]|\le M$ almost surely, then, for every $\epsilon>0$,
\begin{equation}
\Pr \left(\left|\overline{X}-\mathbb{E}[\overline{X}]\right|\ge\epsilon\right)\le 2\exp \left(-\frac{m\epsilon^2}{2\sigma^2+\frac{2}{3}M\epsilon}\right). \tag{Bernstein's inequality}
\end{equation}
\end{lemma}

\begin{lemma}[Section 2 and Section 9 of \cite{duchi2026statistics}]
\label{lem:information-tools}
The following results hold.
\begin{enumerate}
\item[(i)] For every sufficiently large integer $J$, there exists a set $\mathcal V\subset\{0,1\}^J$ containing the zero vector such that
\[
\log|\mathcal V|\geq\frac{J}{8},\qquad d_{\mathrm H}(v,v')\geq\frac{J}{4},\quad v\neq v'.
\]
\item[(ii)] Let $P$ and $Q$ be probability measures satisfying $P\ll Q$. Then
\[
\mathrm{KL}(P\|Q)\leq D_{\chi^2}(P\|Q).
\]
If $P_{X,Y}$ and $Q_{X,Y}$ have the same marginal distribution for $X$, then
\[
\mathrm{KL}(P_{X,Y}\|Q_{X,Y})=\mathbb E_X\left[\mathrm{KL}(P_{Y |  X}\|Q_{Y |  X})\right].
\]
Moreover, if $Y_1,\ldots,Y_N$ are independent under both $\mathbb P$ and $\mathbb Q$, with respective marginal distributions $P_i$ and $Q_i$, then
\[
\mathrm{KL}(\mathbb P\|\mathbb Q)=\sum_{i=1}^N\mathrm{KL}(P_i\|Q_i).
\]
\item[(iii)] Let $\{\theta_v:v\in\mathcal V\}$ satisfy $d(\theta_v,\theta_{v'})\geq2\delta$ for every $v\neq v'$. Let $V$ be uniformly distributed on $\mathcal V$, and conditionally on $V=v$, let $Y$ have distribution $\mathbb P_v$. Then
\[
\inf_{\widehat\theta}\sup_{v\in\mathcal V}\mathbb P_v\left(d(\widehat\theta,\theta_v)\geq\delta\right)\geq1-\frac{I(V;Y)+\log2}{\log|\mathcal V|}.
\]
The same conclusion holds for randomized estimators.
\end{enumerate}
\end{lemma}

\subsection{Construction of a Large ReLU Network}
In this and the subsequent sections, we summarize existing results and fundamental tools for function approximation using neural networks. See \cite{nakada2020adaptive,oko2023diffusion,fu2024unveil} for more details.

\begin{lemma}[Neural Network Concatenation]\label{lem:relu-concatenation}
	For a series of ReLU networks $\boldsymbol{s}_1: \mathbb{R}^{d_1}\rightarrow\mathbb{R}^{d_2}$, $\boldsymbol{s}_2:\mathbb{R}^{d_2}\rightarrow\mathbb{R}^{d_3}$, $\cdots$, $\boldsymbol{s}_k:\mathbb{R}^{d_k}\rightarrow\mathbb{R}^{d_{k+1}}$ with $\boldsymbol{s}_i\in \cS(\cD_i,\cW_i,\cS_i,\cB_i)~(i=1,2,\cdots,k)$, there exists a neural network $\boldsymbol{s}\in \cG(\cD,\cW,\cS,\cB)$ satisfying $\boldsymbol{s}(\boldsymbol{x}) = \boldsymbol{s}_k\circ\boldsymbol{s}_{k-1}\circ\cdots\circ\boldsymbol{s}_1(\boldsymbol{x})$ for all $\boldsymbol{x}\in\mathbb{R}^{d_1}$, with
	$$
	\cD = \sum_{i=1}^{k}\cD_i, ~~ \cW\leq 2\sum_{i=1}^{k}\cW_i, ~~ \cS\leq 2\sum_{i=1}^{k}\cS_i, ~~\text{and} ~ \cB\leq\max_{1\leq i \leq k}\cB_i.
	$$
	
\end{lemma}

\begin{lemma}[Identity Function]\label{lem:relu-identity}
	Given $d\in\mathbb{N}_{+}$ and $\cD\geq 2$, there exists a neural network $\boldsymbol{s}_{\mathrm{Id}, \cD} \in  \cG(\cD,\cW,\cS,\cB)$ that realizes $d$-dimensional identity function $\boldsymbol{s}_{\mathrm{Id},\cD}(\boldsymbol{x}) = \boldsymbol{x}$, $\boldsymbol{x}\in\mathbb{R}^{d}$. Here 
	$$
	\cW = 2d, ~~ \cS = 2d\cD, ~~ \cB = 1.
	$$
\end{lemma}

\begin{lemma}[Neural Network Parallelization]\label{lem:relu-parallelization}
	For a series of ReLU networks $\boldsymbol{s}_i: \mathbb{R}^{d_i}\rightarrow\mathbb{R}^{d_i^{\prime}}$ with $\boldsymbol{s}_i\in \cG(\cD_i,\cW_i,\cS_i,\cB_i)~(i=1,2,\cdots,k)$, there exists a neural network $\phi\in \cG(\cD,\cW,\cS,\cB)$ satisfying $\boldsymbol{s}(\boldsymbol{x}) = [\boldsymbol{s}_1^{\top}(\boldsymbol{x}_1), \boldsymbol{s}_2^{\top}(\boldsymbol{x}_2),\cdots,\boldsymbol{s}_{k}^{\top}(\boldsymbol{x}_k)]:\rightarrow \mathbb{R}^{d_1+d_2+\cdots+d_k}\rightarrow\mathbb{R}^{d_1^{\prime} + d_2^{\prime} + \cdots + d_k^{\prime}}$ for all $\boldsymbol{x} = (\boldsymbol{x}_1^{\top},\boldsymbol{x}_2^{\top},\cdots,\boldsymbol{x}_k^{\top})^{\top}\in\mathbb{R}^{d_1+d_2+\cdots+d_k}$($\boldsymbol{x}_i$ can be shared), with
	$$
	\begin{aligned}
		\cD = \cD, ~~ \cW\leq 2\sum_{i=1}^{k}\cW_i, ~~ \cS\leq 2\sum_{i=1}^{k}\cS_i, ~~\text{and} ~ \cB\leq\max_{1\leq i \leq k}\cB_i ~~ (\text{when} ~ \cD = \cD_i ~ \text{holds  for all} ~ i), \\
		\cD = \max_{1\leq i\leq k} \cD_i, ~~ \cW \leq 2\sum_{i=1}^{k} \cW_i, ~~ \cS \leq 2\sum_{i=1}^{k} (\cS_i + \cD d_i^{\prime}),~~ \text{and} ~ \cB\leq \max{\{\max_{1\leq i\leq k}\cB_i, 1\}} ~ (\text{otherwise}). 
	\end{aligned}
	$$    
	Moreover, for $\boldsymbol{x}_1 = \boldsymbol{x}_2 = \cdots = \boldsymbol{x}_k = \boldsymbol{x}\in\mathbb{R}^d$ and $d_1^{\prime} = d_2^{\prime} = \cdots = d_k^{\prime} = d^{\prime}$, there exists a neural network $\boldsymbol{s}_{\mathrm{sum}}\in \cG(\cD,\cW,\cS,\cB)$ that realizes $\boldsymbol{s}_{\mathrm{sum}}(\boldsymbol{x}) = \sum_{i=1}^{k}\boldsymbol{s}_i(\boldsymbol{x})$ with
	$$
	\cD = \max_{1\leq i\leq k} \cD_i + 1, ~~ \cW \leq 4\sum_{i=1}^{k} \cW_i, ~~ \cS \leq 4\sum_{i=1}^{k} (\cS_i + \cD d_i^{\prime}) + 2\cW,~~ \text{and} ~ \cB\leq \max{\{\max_{1\leq i\leq k}\cB_i, 1\}}.
	$$
\end{lemma}

\subsection{Approximation of Basic Functions with ReLU Networks}

\begin{lemma}[Approximating the Products]\label{lem:relu-product}
	Let $d\geq 2$, $C\geq 1$. For any $\epsilon > 0$, there exists a neural network $\mathrm{s}_{\mathrm{prod}}\in \cG(\cD,\cW,\cS,\cB)$ with $\cD = \mathcal{O}(\log d(\log d + \log\epsilon^{-1} + d\log C))$, $\cW = 48d$, $\cS = \mathcal{O}(d(\log d + \log\epsilon^{-1} + d\log C)$, $\cB = C^d$ such that
	$$
	\left|\mathrm{s}_{\mathrm{prod}}(x_1^{\prime},x_2^{\prime},\cdots,x_d^{\prime}) - \prod_{i=1}^d x_i\right| \leq \epsilon + dC^{d-1}\epsilon_0,
	$$
	for all $\boldsymbol{x}\in[-C,C]^d$ and $\boldsymbol{x}^{\prime}\in\mathbb{R}^d$ with $\Vert\boldsymbol{x} - \boldsymbol{x}^{\prime}\Vert_{\infty} \leq \epsilon_0$. Moreover, $|\mathrm{s}_{\mathrm{prod}}(\boldsymbol{x}^{\prime})|\leq C^d$ for all $\boldsymbol{x}^{\prime}\in\mathbb{R}^d$, and $\mathrm{s}_{\mathrm{prod}}(x_1^{\prime},x_2^{\prime},\cdots,x_d^{\prime}) = 0$ if at least one of $x_i^{\prime}$ is 0.
\end{lemma}

\begin{remark}
	We note that some of $x_i$, $x_j$ $(i\neq j)$ can be shared. For $\prod_{i=1}^{I}x_i^{u_i}$ with $u_i\in\mathbb{N}_{+} (i=1,2,\cdots,I)$ and $\sum_{i=1}^{I}u_i = d$, there exists a neural network satisfying the same bounds as above.    
\end{remark}

\begin{lemma}[Approximating the Reciprocal Function]\label{lem:relu-reciprocal}
	For any $0 < \epsilon < 1$, there exists a neural network $\mathrm{s}_{\mathrm{rec}}\in \cG(\cD,\cW,\cS,\cB)$ with $\cD = \mathcal{O}(\log^2\epsilon^{-1})$, $\cW = \mathcal{O}(\log^3\epsilon^{-1})$, $\cS = \mathcal{O}(\log^4\epsilon^{-1})$ and $\cB = \mathcal{O}(\epsilon^{-2})$ such that
	$$
	\left|\mathrm{s}_{\mathrm{rec}}(x^{\prime}) - \frac{1}{x}\right| \leq \epsilon + \frac{|x^{\prime} - x|}{\epsilon^2},
	$$
	for all $x \in [\epsilon, \epsilon^{-1}]$ and $x^{\prime}\in\mathbb{R}$.
\end{lemma}

\begin{lemma}[Approximating the Square Root Function]\label{lem:relu-square-root}
	For any $0 < \epsilon < 1$, there exists a neural network $\mathrm{s}_{\mathrm{root}}\in \cG(\cD,\cW,\cS,\cB)$ with $\cD = \mathcal{O}(\log^2\epsilon^{-1})$, $\cW = \mathcal{O}(\log^3\epsilon^{-1})$, $\cS = \mathcal{O}(\log^4\epsilon^{-1})$ and $\cB = \mathcal{O}(\epsilon^{-1})$ such that
	$$
	|\mathrm{s}_{\mathrm{root}}(x^{\prime}) - \sqrt{x}| \leq \epsilon + \frac{|x^{\prime} - x|}{\sqrt{\epsilon}},
	$$
	for all $x\in[\epsilon, \epsilon^{-1}]$ and $x^{\prime}\in\mathbb{R}$.
\end{lemma}

\subsection{Clipping and Switching Functions}
\begin{lemma}[Clipping Function]
	\label{lem:relu-clip}
	For any $\mathbf{a}$, $\mathbf{b}\in\mathbb{R}^d$ with $a_i \leq b_i$ 
	$(i=1,2,\cdots,d)$, there exists a clipping function $\mathrm{s}_{\mathrm{clip}}(\boldsymbol{x}, \mathbf{a},\mathbf{b})\in \cG(\cD,\cW,\cS,\cB)$ with
	$$
	\cD = 2, ~ \cW = 2d, ~ \cS = 7d, ~ \cB = \left(\max_{1\leq i \leq d}\max\{|a_i|, |b_i|\}\right) \vee 1, 
	$$
	such that
	$$
	\mathrm{s}_{\mathrm{clip}}(\boldsymbol{x}, \mathbf{a}, \mathbf{b})_i = \min\{b_i, \max\{x_i, a_i\}\} ~~ (i=1,2,\cdots,d).
	$$
	When $a_i=c_{min}$ and $b_i=c_{max}$ for all $i$, 
	we sometimes denote $\mathrm{s}_{\mathrm{clip}}(\boldsymbol{x},\mathbf{a},\mathbf{b})$ as $\mathrm{s}_{\mathrm{clip}}(\boldsymbol{x},c_{min},c_{max})$ using scalar values $c_{min}$ and $c_{max}$.
\end{lemma}

\begin{lemma}[Switching Function]\label{lem:relu-switch}
	Let $t_1 < t_2 < s_1 < s_2$, and $f(t,\boldsymbol{x})$ be a scalar-valued function (for a vector-valued function, we just apply this coordinate-wise). Assume that $|\phi_1(t,\boldsymbol{x}) - f(t,\boldsymbol{x})| \leq \epsilon$ on $[t_1, s_1]$ and $|\phi_2(t,\boldsymbol{x}) - f(t,\boldsymbol{x})| \leq \epsilon$ on $[t_2, s_2]$. Then, there exist two neural networks $\mathrm{s}_{\mathrm{switch},1}(t, t_2, s_1)$ and $\mathrm{s}_{\mathrm{switch},2}(t, t_2, s_1)\in \cG(\cD,\cW,\cS,\cB)$ with 
	$$
	\cD = 3, ~~ \cW = 2, ~~ \mathcal{S} = 8,~ \text{and}~ \cB = \max\{s_1, (s_1 - t_2)^{-1}\}
	$$
	such that
	$$
	|\mathrm{s}_{\mathrm{switch},1}(t,t_2,s_1)\phi_1(t,\boldsymbol{x}) + \mathrm{s}_{\mathrm{switch},2}(t,t_2,s_1)\phi_2(t,\boldsymbol{x}) - f(t,\boldsymbol{x})| \leq \epsilon
	$$
	holds for any $t\in[t_1, s_2]$, where
	$$
	\mathrm{s}_{\mathrm{switch},1}(t,t_2,s_1) = \frac{1}{s_1 - t_2}\mathrm{ReLU}\left(s_1 - \mathrm{s}_{\mathrm{clip}}(t,t_2,s_1)\right),
	$$
	$$
	\mathrm{s}_{\mathrm{switch},2}(t,t_2,s_1) = \frac{1}{s_1 - t_2}\mathrm{ReLU}\left( \mathrm{s}_{\mathrm{clip}}(t,t_2,s_1)-t_2\right).
	$$
\end{lemma}


\end{appendix}

\bibliographystyle{spbasic} 
\bibliography{ref.bib}   
\end{document}